%% file: arxiv.tex
\documentclass[english,12pt]{article}
\usepackage[arxiv]{optional}
\usepackage[T1]{fontenc}
\usepackage{amsmath}
\usepackage{amsthm}
\usepackage[pagebackref=true,colorlinks = true,urlcolor = Blue,linkcolor = Blue,citecolor = Blue,pdfborder={0 0 0}]{hyperref} 

\input{packages}

\usepackage{url} 
\usepackage{amssymb}
\usepackage{graphicx}
\usepackage{wasysym}
\usepackage[authoryear]{natbib}

\makeatletter

\theoremstyle{plain}
\newtheorem{thm}{\protect\theoremname}[section]
\newtheorem{thm*}{\protect\theoremname}
\newtheorem{lem}{\protect\lemmaname}[section]
\newtheorem{prop}{\protect\propositionname}[section]
\newtheorem{cor}{\protect\corollaryname}[section]
\theoremstyle{definition}
\newtheorem{defn}{\protect\definitionname}
\theoremstyle{remark}
\newtheorem{rem}{\protect\remarkname}
\newtheorem*{rem*}{\protect\remarkname}
\newtheorem{assumption}{\protect\assumptionname}
\newtheorem*{proofsketch}{Proof sketch}

\usepackage{authblk}
\usepackage{fullpage}
\usepackage{etoolbox}\apptocmd\appendix{\pretocmd\section{\clearpage}{}{}}{}{}

\author[1]{Heishiro Kanagawa\thanks{Correspondence: heishiro.kanagawa@gmail.com}}
\author[2]{Alessandro Barp}
\author[3]{\authorcr Arthur Gretton}
\author[4]{Lester Mackey}
\affil[1]{Newcastle University, UK}
\affil[2]{University College London}
\affil[3]{Gatsby Computational Neuroscience Unit, UCL}
\affil[4]{Microsoft Research New England}
\date{}

\usepackage{babel}
\usepackage{titling}
\predate{}
\postdate{}

\input{user_commands}
\makeatother
\begin{document}

\title{Controlling Moments with Kernel Stein Discrepancies }
\maketitle

\begin{abstract}
\input{abstract}
\end{abstract}
\input{./mathmacros.tex}

\input{./main.tex}

\paragraph*{Acknowledgements}
HK was supported by the EPSRC grant [EP/W019590/1] and in part by the Gatsby Charitable Foundation. 
AG acknowledges financial support by the Gatsby Charitable Foundation. 
The authors thank Chris Oates for his helpful feedback, Carl-Johann Simon-Gabriel for the constructive discussions, and 
the three anonymous referees for their thoughtful reviews, which helped improve the manuscript.

\bibliographystyle{plainnat}
\bibliography{ref}

\newpage
\begin{appendix}
    \input{./supp.tex}
\end{appendix}
\end{document}

%% file: packages.tex
\usepackage{physics}
\usepackage[usenames,dvipsnames]{xcolor}
\usepackage{bm}
\usepackage{float}
\usepackage{thmtools, thm-restate}
\usepackage{mathtools}
\usepackage{url}
\usepackage{xspace}
\makeatletter
\@ifundefined{showcaptionsetup}{}{%
 \PassOptionsToPackage{caption=false}{subfig}}
\usepackage{subfig}
\makeatother

\usepackage[capitalise, noabbrev]{cleveref}
\crefname{assumption}{Assumption}{Assumptions}
\crefname{lem}{Lemma}{Lemmas}
\crefname{thm}{Theorem}{Theorems}
\crefname{cor}{Corollary}{Corollaries}

%% file: user_commands.tex
\providecommand{\assumptionname}{Assumption}
\providecommand{\corollaryname}{Corollary}
\providecommand{\definitionname}{Definition}
\providecommand{\lemmaname}{Lemma}
\providecommand{\propositionname}{Proposition}
\providecommand{\remarkname}{Remark}
\providecommand{\theoremname}{Theorem}

\newcommand{\ncref}[1]{\cref{#1}: \nameref*{#1}} %
\newcommand{\pcref}[1]{Proof of \ncref{#1}} %

\newcommand*{\nolink}[1]{%
  {\protect\NoHyper#1\protect\endNoHyper\xspace}%
}%

\global\long\def\datdim{d}%
\global\long\def\dimidx{i}%
\global\long\def\seqidx{n}%
\global\long\def\weight{w}%
\global\long\def\polyorder{q}%
\global\long\def\idmat{\mathrm{Id}}%
\global\long\def\plipspace#1{\mathrm{PL}_{#1}}%
\global\long\def\plipconst#1#2{\left\Vert #1\right\Vert _{\mathrm{PL},#2}}%
\global\long\def\degnpolycoef#1#2#3{\tilde{\pi}(#1)_{#2,#3}}%
\global\long\def\finitemomentspace#1{{\cal P}_{#1}}%
\global\long\def\nmodemul#1{\bar{\times}_{#1}}%
\global\long\def\ksd#1#2{\mathrm{KSD}_{#1,#2}}%
\global\long\def\dksd#1#2{\mathrm{DKSD}_{#1,#2}}%
\global\long\def\mmd#1{\mathrm{MMD}_{#1}}%
\global\long\def\calC{\mathcal{C}}%
\global\long\def\rkhs#1{\mathcal{H}_{#1}}%
\newcommand{\toL}[1]{\stackrel{L^{#1}}{\to}}

%% file: abstract.tex
    Kernel Stein discrepancies (KSDs) 
    measure the quality of a distributional approximation and can be computed 
    even when the target density has an intractable normalizing constant.
    Notable applications include the diagnosis of approximate MCMC samplers and goodness-of-fit tests for unnormalized statistical models.
    The present work analyzes the convergence control properties of KSDs. 
    We first show that standard KSDs used for weak convergence control fail to control moment convergence. To address this limitation, we next provide sufficient conditions under which alternative diffusion KSDs control both moment and weak convergence. 
    As an immediate consequence we develop, for each $q > 0$, the first KSDs known to exactly characterize $q$-Wasserstein convergence. 

%% file: mathmacros.tex
\global\long\def\score#1{\mathbf{s}_{#1}}%

\global\long\def\hatscore#1{\hat{\mathbf{s}}_{#1}}%

\global\long\def\barscore#1{\bar{\mathbf{s}}_{#1}}%

\global\long\def\tildescore#1{\tilde{{\bf s}}_{#1}}%

\global\long\def\calX{\mathcal{X}}%

\global\long\def\calD{\mathcal{D}}%

\global\long\def\EE{\mathbb{E}}%

\global\long\def\var{\mathrm{Var}}%
\global\long\def\cov{\mathrm{Cov}}%

\global\long\def\calY{\mathcal{Y}}%
\global\long\def\calF{\mathcal{H}}%

\global\long\def\calZ{\mathcal{Z}}%
\global\long\def\calF{\mathcal{F}}%

\global\long\def\inner#1#2{\left\langle #1,#2\right\rangle }%

\global\long\def\la{\langle}%

\global\long\def\ra{\rangle}%
\global\long\def\norm#1{\left\Vert #1\right\Vert }%

\global\long\def\verts#1{\lvert#1\rvert}%

\global\long\def\Verts#1{\lVert#1\rVert}%

\global\long\def\bignorm#1{\big\Vert#1\big\Vert}%

\global\long\def\Normal#1#2{\mathcal{N}(#1,#2)}%

\global\long\def\dto{\overset{d}{\to}}%
\global\long\def\pto{\overset{p}{\to}}%

\global\long\def\diag{\mathrm{diag}}%

\global\long\def\Pr{\mathrm{Pr}}%

\global\long\def\bfone{\mathbf{1}}%

\global\long\def\bfzero{\mathbf{0}}%

\global\long\def\iidsim{\overset{\mathrm{i.i.d.}}{\sim}}%

\global\long\def\indsim{\overset{\mathrm{ind.}}{\sim}}%

\global\long\def\calX{\mathcal{X}}%

\global\long\def\calD{\mathcal{D}}%

\global\long\def\Ex{\mathbb{E}}%

\global\long\def\inner#1#2{\left\langle #1,#2\right\rangle }%

\global\long\def\la{\langle}%

\global\long\def\ra{\rangle}%

\global\long\def\Verts#1{\lVert#1\rVert}%

\global\long\def\bignorm#1{\big\Vert#1\big\Vert}%

\global\long\def\do#1{\mathrm{do}(#1)}%

\global\long\def\given{\vert}%

\global\long\def\biggiven{\big\vert}%

\global\long\def\Biggiven{\Big\vert}%

\newcommand\independent{\protect\mathpalette{\protect\independenT}{\perp}}
\def\independenT#1#2{\mathrel{\rlap{$#1#2$}\mkern2mu{#1#2}}} 

\global\long\def\bbR{\mathbb{R}}%

\global\long\def\bbN{\mathbb{N}}%

\global\long\def\tr#1{\mathrm{tr}\left[#1\right]}%

\global\long\def\lflo{\lfloor}%

\global\long\def\rflo{\rfloor}%

\global\long\def\floor#1#2{\left\lfloor #1,#2\right\rfloor }%

%% file: main.tex
\section{Introduction}\label{sec:intro}
\begin{sloppypar}
Consider the problem of computing the expectation $\EE_{P}[f]=\int f(x)\dd P(x)$ 
of a function $f$ on $\bbR^{\datdim}$ with respect to a probability
distribution $P.$ 
We will specifically focus on the scenario where $P$ is defined by a density function with an unknown normalization constant. 
In such cases, exact computation of the expectation is typically intractable, necessitating the use of approximation methods.
Numerical integration methods produce estimates of the form $\sum_{i=1}^\seqidx \weight_i f(x_i)$ by combining evaluations of the integrand. 
The associated weights $\weight_i$ and points $x_i$ can be considered as an approximating probability distribution $Q=\sum_{i=1}^\seqidx\weight_i \delta_{x_i}$ if $\sum_{i=1}^\seqidx \weight_i = 1$ and $\weight_i \geq 0$. 
Various techniques are available, and foremost among them is Markov chain Monte Carlo~\citep[MCMC,][]{Liu_2004} ensuring accurate estimates as $\seqidx$ tends to infinity.
\end{sloppypar}

The focus of this paper is on assessing the quality of such sample approximations. 
Our pursuit is motivated by the following two considerations. 
First, in practice, it is challenging to vouch for approximation quality at \emph{finite} sample sizes, even for established methods like MCMC. 
Second, 
practitioners are 
increasingly turning to asymptotically inexact methods that offer greater computational convenience at the expense of persistent inferential bias. 
One illustrative example can be found in modern Bayesian statistics, where 
exact MCMC becomes computationally unattractive due to a large number of datapoint terms or factors in the likelihood. %
In this setting, cheaper but asymptotically biased alternatives, such as approximate MCMC~\citep{Welling2011, Ma2015, Dal17, Durmus_2017} or variational Bayesian methods~\citep{Blei2017}, are often employed. 
Regardless of the type of method used, quantifying the quality
of approximation is desirable since we can then diagnose and compare candidate methods to obtain more accurate estimates. 
This objective demands a quality measure that is computable for a broad range of approximation methods.

This problem may be approached using \emph{Stein discrepancies}, a
class of measures of discrepancy between probability distributions
\citep{GorMac2015}. The key idea behind a Stein discrepancy is that
we can construct functions that are integrated to zero under the target distribution;
such functions may be used to measure the discrepancy of another distribution,
since a non-zero expectation under the distribution indicates its
deviation from the target. Formally, a Stein discrepancy of a distribution
$Q,$ with respect to target $P,$ is defined as the worst-case quantity
\[
\sup_{g\in{\cal {\cal G}}}\left|\EE_{Q}[{\cal T}_{P}g]\right|,
\]
where ${\cal T}_{P}$ and ${\cal G}$ are a \emph{Stein operator}
and a \emph{Stein set} (a function class in the domain of ${\cal T}_{P}$),
inducing functions ${\cal T}_{P}g$ whose expectations vanish under
$P.$ Different choices of the operator and the function class yield
distinct Stein discrepancies. A prototypical example is the Langevin
Stein operator \citep{GorMac2015,OatGirCho2017}, which may be defined
without the normalization constant; and two major classes based on
this operator are the graph Stein discrepancies \citep{GorMac2015,gorham2016measuring}
and the kernel Stein discrepancy (KSD) \citep{ChwStrGre2016,LiuLeeJor2016,OatGirCho2017,GorMac2017}.
Notably, it is possible to compute these discrepancies: the KSD circumvents
the supremum in the definition and admits a closed-form expression
involving kernel evaluations on samples, whereas the graph Stein discrepancy
involves solving a linear program, 
This computational advantage has sparked numerous statistical applications involving unnormalzed densities, 
including  
goodness-of-fit testing~\citep{ChwStrGre2016, LiuLeeJor2016, JitXuSzaFukGre2017,Kanagawa_2023}, 
parameter inference~\citep{BarpBriolDuncanEtAl2019Minimum,GrathwohlEtAl2020,Matsubara_2022},
and sampling~\citep{LiuLee2017,Chen2018, Chen2019, Hodgkinson2020, Riabiz_2022}. 
See the survey by \citet{Anastasiou_2023} for an overview. 

While successfully avoiding the intractable expectation, Stein discrepancies
may not be immediately interpreted in terms of the expectation $\EE_{P}[f]$
of interest. Prior work therefore investigated how Stein discrepancies
characterize closeness in expectation under suitable regularity conditions
on the integrand. 
With the exception of \citet{Barp2024}, the foregoing studies used Stein's method \citep{Stein1972bound}
to relate Stein discrepancies to integral probability metrics (IPMs)
\citep{Mueller1997Integral}, which measure the worst-case difference
in expectations with respect to given classes of functions.  \citet{GorMac2015}
showed that the Langevin graph Stein discrepancy controls the 1-Wasserstein
distance (the IPM defined by $1$-Lipschitz functions) for distantly
dissipative target distributions; \citet{gorham2016measuring} later
generalized this result to heavy-tailed targets with diffusion graph
Stein discrepancies. \citet{GorMac2017} proved that the Langevin
KSD with the inverse multi-quadratic kernel (IMQ) controls the bounded-Lipschitz
metric; \citet{Chen2018} offer other kernel choices. 
Feature Stein discrepancies by \citet{HugMac2018} can control the bounded-Lipschitz
metric, with estimators computable in near-linear time with respect to the sample size.
Using a distinct approach from the aforementioned works, \citet{Barp2024}
showed that the separating property of the KSD is equivalent to having
(tight) weak convergence control (i.e., convergence in expectations
of bounded continuous functions). 

Statistical analysis often requires handling unbounded functions; e.g., basic
statistical measures such as mean and variance are the expectation
of linear and quadratic functions. 
That said, it remains an open question 
to which class of unbounded integrands, Stein discrepancies are guaranteed to indicate the quality of expectation approximations.
This question has been in part tackled by works
on graph Stein discrepancies \citep{GorMac2015,gorham2016measuring}.
These analyses, however, are limited to linearly growing functions
-- it is unclear how these results extend to functions of faster
growth. Moreover, despite its computational appeal over the graph Stein discrepancies,
the KSD has only been shown to control weak convergence. 
These observations highlight a gap in our understanding of the broader applicability of Stein discrepancies.

This article aims to extend the reach of the KSD to functions of arbitrary
polynomial growth. In particular, we investigate conditions under
which the KSD controls the convergence of expectations of polynomially
growing continuous functions (\emph{$\polyorder$-Wasserstein convergence}, defined formally in Section \ref{subsec:KSD-convergence-and}). 
Our specific contributions are threefold. 
First, our analysis considers the diffusion kernel Stein discrepancy \citep{BarpBriolDuncanEtAl2019Minimum},
a generalization of the Langevin KSD \citep{ChwStrGre2016,LiuLeeJor2016,OatGirCho2017,GorMac2017}. This extension allows us to consider a broader target class, including heavy-tailed distributions, for which the Langevin KSD has proved inadequate to control weak convergence~\citep{GorMac2017, Barp2024}. 
Second, we show that standard KSDs used for controlling weak convergence fail to control moment convergence. %
As our third contribution, we address this limitation by identifying, for the first time, specific practical reproducing kernels that provide polynomial convergence control.
Consequently, we establish the first Stein discrepancies known to exactly control $\polyorder$-Wasserstein convergence for each $\polyorder>0$. 

The rest of the article is organized as follows. 
Section \ref{sec:Diffusion-kernel-Stein} introduces the KSD and its associated Stein operator. 
\cref{sec:Main-results} presents the motivating negative result of standard KSDs, followed by two main results addressing this shortcoming. 
In \cref{subsec:KSD-convergence-and}, 
our first result, building on the work of \citet{Barp2024}, establishes sufficient conditions on the kernel to facillitate the control of $\polyorder$-Wasserstein convergence by the KSD. 
In Section \ref{subsec:Rate-of--Wasserstein},
building on the finite Stein factor results of \citet{Erdogdu2018}, our
second result provides an explicit KSD upper bound on the IPM defined
by a class of pseudo-Lipschitz functions, a polynomial generalization
of Lipschitz functions. This bound enables us to describe a rate of Wasserstein convergence in terms of the KSD 
and to obtain an explicit bound on the $\polyorder$-Wasserstein distance.
Our experiments in \cref{sec:Numerical-illustration} confirm the claimed convergence
control property. We conclude in Section \ref{sec:Conclusion}. 

\subsection*{Notation and definitions}

We define the inner product $\la A, B\ra$ between multi-dimensional arrays $A,B \in \bbR^{\datdim_1\times \dots \times \datdim_n}$ 
by the sum of the elementwise product $\sum A_{\dimidx_1,\dots, \dimidx_n} B_{\dimidx_1, \dots, \dimidx_n}$, where $A_{\dimidx_1,\dots, \dimidx_n}$ and $B_{\dimidx_1,\dots, \dimidx_n}$ are the $(\dimidx_1, \dots, \dimidx_n)$-entry of the respective arrays. 
For $A\in \bbR^{\datdim_1\times \dots \times \datdim_n}$, we define its Frobenius norm by $\Verts{A}_{\mathrm{F}}=\sqrt{\la A, A\ra}$
and its operator norm by 
\[ \Verts{A}_{\mathrm{op}} \coloneqq \sup_{ \Verts{u^{(1)}}_2=\cdots=\Verts{u^{(n)}}_2=1 } \verts{ \la A, u^{(1)}\otimes \cdots \otimes u^{(n)} \ra }, \]
where $u^{(1)}\otimes \cdots \otimes u^{(n)} \in \bbR^{\datdim_1 \times \cdots \times \datdim_n}$ 
denotes the outer product of $(u^{(i)}\in\bbR^{\datdim_i})_{i=1}^{n}$; i.e., 
\[ 
    (u^{(1)}\otimes \cdots \otimes u^{(n)})_{\dimidx_1, \dots, \dimidx_n} = u^{(1)}_{\dimidx_1} \cdots u^{(n)}_{\dimidx_n} .
\]
The Euclidean norm of $a = (a_1, \dots, a_{\datdim}) \in \bbR^{\datdim}$ is denoted by $\Verts{a}_2 (=\Verts{a}_{\mathrm{op}}=\Verts{a}_{\mathrm{F}})$. 
With $\nabla = (\partial_1, \dots, \partial_{\datdim})^{\top}$, 
we define $\nabla^i$ to act on $g: \bbR^{\datdim} \to \bbR^{\datdim'}$ such that $\nabla^i g: \bbR^{\datdim} \to \bbR^{(\times_{k=1}^i \datdim) \times \datdim'}$ and $(\nabla^i g(x))_{j_1,\dots, j_{i}, j}=\partial_{j_1}\cdots \partial_{j_{i}} g_j(x)$, 
where $\partial_{\dimidx}$ stands for the partial derivative with respect to the $\dimidx$th coordinate. 
The divergence of a vector-valued function ${v}(x) = ({v}_1(x),\dots,{v}_{\datdim}(x))$ is defined as $\la \nabla, {v}(x) \ra \coloneqq \sum_{\dimidx=1}^\datdim \partial_{\dimidx} {v}_i(x)$. 
For $v\in \bbR^{\datdim_1}$ and $A\in\bbR^{\datdim_1\times \datdim_2}$, we define a vector-matrix product $\la v, A\ra \in \bbR^{\datdim_2}$ by $(\la v, A\ra)_i \coloneqq \sum_{j=1}^{\datdim_1} v_j A_{ji}$. 
Following this notation, for a matrix-valued function $A:\bbR^{\datdim} \to \bbR^{\datdim\times \datdim}$, we use $\la\nabla, A(x)\ra$ to denote its column-wise divergence, i.e., $\la\nabla,A(x)\ra_{i}=\sum_{j=1}^{\datdim}\partial_{j}A_{ji}(x)$. 
For a bivariate function $K:\bbR^{\datdim}\times\bbR^{\datdim}\to \bbR^{\datdim\times \datdim}$ and $\ell \in \{1,2\}$, we define $\partial_{\ell,i}K(a,b)$  
to be the elementwise partial derivative of $K$ at $(a,b)$ taken with respect to the $\ell$th argument and its $i$th coordinate. 

We denote by ${\cal P}$ the set of Borel probability measures on $\bbR^{\datdim}.$ Both $\int f\dd\mu$ and $\EE_{\mu}[f]$ stand for the expectation of a function $f$ with
respect to $\mu\in{\cal P}.$ 
For $k \geq 0$, the symbol $\calC^{k}(\bbR^{n})$ stands for the set of $\bbR^{n}$-valued
$k$-times continuously differentiable functions on $\bbR^{\datdim},$
and we simply write ${\cal C}^{k}$ for ${\cal \calC}^{k}(\bbR).$ 
We also define the set $\calC_{0}^{k}(\bbR^{n})\subset\calC^{k}(\bbR^{n})$ of
functions whose $j$th partial derivatives vanish at infinity for any $j\in\{0,\dots,k\}$, 
which is equipped with the norm 
$f\mapsto \max_{\verts{\bm{i}}\leq k} \sup_{x\in\bbR^{\datdim}}\Verts{\partial^{\bm{i}} f(x)}_2$, 
where $\partial^{\bm{i}} = \partial_{1}^{i_1}\cdots \partial_{d}^{i_d}$ with $\bm{i}=(i_1, \dots, i_d)$ and $\verts{\bm{i}}=i_1+\cdots+i_d$. 
We define $L^{\polyorder}(\bbR^{n},\mu)=\{f:\bbR^{\datdim}\to\bbR^{n}:\int\Verts f_{2}^{\polyorder}\dd\mu<\infty\};$
when $n=1,$ we similarly use the shorthand $L^{\polyorder}(\mu).$
The maximum (resp. minimum) of two real numbers $a,b$ is denoted
by $a\lor b$ (resp. $a\land b$); in particular, we use $(x)_{+}$ 
to denote $(0\lor x)$. 
Similarly, for a real-valued function $f$, we define $f^+(x)= (f(x)\lor0)$ and $f^{-}(x) = -(f(x)\land 
 0)$. 
For a set $S$ and $A\subset S$, we define the associated indicator function $1_A:S\to \bbR$ to yield $1$ on $A$ and $0$ otherwise. If $A\subset S$ is given in the form $\{s:S: f(s) \in B\}$, we also use the shorthand $1\{f(s)\in B\}$. 
The (order-$0$) \emph{generalized Fourier transform}~\citep[Definition 8.9]{Wendland2004} of (polynomially-increasing) function $\Phi\in \calC:\bbR^{\datdim}\to\bbR$ is defined as a measurable function $\hat{\Phi}:\bbR^{\datdim}\to \bbR$ satisfying the following: (a) $\hat{\Phi}$ is square integrable on every compact subset of $\bbR^{\datdim}\setminus\{0\}$ with respect to the Lebesgue measure, and (b) $\int \Phi(x) \hat{\gamma}(x) \dd x = \int \hat{\Phi}(x) \gamma(x) \dd x$, where $\gamma$ is an arbitrary Schwartz function and $\hat{\gamma}$ denotes its Fourier transform. 

\section{Diffusion kernel Stein discrepancies} \label{sec:Diffusion-kernel-Stein}

In this section, we recall the definition of the kernel Stein discrepancy
(KSD), a discrepancy measure derived by combining a diffusion Stein
operator and a reproducing kernel Hilbert space. 
In the process, we introduce our notation and assumptions
made throughout the paper. 

The \emph{diffusion Stein operator} \citep{gorham2016measuring,BarpBriolDuncanEtAl2019Minimum}
is a class of Stein operators that characterizes probability distributions
on $\bbR^{\datdim}$ where $\datdim\ge1.$ Suppose $P$ is a probability
distribution on $\bbR^{\datdim},$ defined by a density function $p$
with respect to the Lebesgue measure; we assume that $p$ is differentiable
and positive everywhere. Following \citet{BarpBriolDuncanEtAl2019Minimum},
we define a diffusion Stein operator ${\cal T}_{P}^{m}$ as an operator
that takes as input a vector-valued differentiable function $v:\bbR^{\datdim}\to\bbR^{\datdim}$
and outputs a real-valued function as follows: 
\begin{align}
{\cal T}_{P}^{m}v(x) & =\frac{1}{p(x)}\la\nabla,p(x)m(x)v(x)\ra,\label{eq:diff-Stein-op}
\end{align}
where $m:\bbR^{\datdim}\to\bbR^{\datdim\times\datdim}$ is a differentiable
matrix-valued function called a \emph{diffusion matrix}. 
A specific choice
of the diffusion matrix defines a particular Stein operator. The \emph{Langevin
Stein operator} \citep{GorMac2015} is a special case of this class 
and is given by taking the identity matrix $m\equiv\idmat.$ With
an appropriate choice, we may relate the operator to the generator
of an It\^{o} diffusion process, which will be critical in our analysis
in Section \ref{subsec:Rate-of--Wasserstein}. 
By the divergence theorem \citep[Theorem 2.36]{Pigola2014}, we have
$\EE_{P}\bigl[{\cal T}_{P}^{m}v\bigr]=0$ provided $mv\in L^{1}(\bbR^{\datdim},P)$, $\mathcal{T}_P^mv \in L^1(P)$, and $pmv\in\calC^{1}(\bbR^{\datdim})$.
Thus, the operator indeed generates zero-mean functions. 
Motivated by the final condition, $pmv\in\calC^{1}(\bbR^{\datdim})$, we will assume that $p\in\calC^{1}$ and $m\in{\cal C}^{1}(\bbR^{\datdim\times\datdim})$ hereafter. 

A Stein discrepancy is obtained once 
we specify a function class to be paired with a diffusion Stein operator. 
The KSD is a Stein discrepancy formed with functions from a \emph{reproducing kernel Hilbert space}~\citep[RKHS,][]{aronszajnTheoryReproducingKernels1950, Schwartz_1964}. 
Specifically, we make use of an RKHS $\rkhs{K}$ of $\bbR^{\datdim}$-valued functions. 
For a set $\calX$, a matrix-valued function $K: \calX \times \calX \to \bbR^{\datdim \times \datdim}$ is called a $\bbR^{\datdim}$-reproducing kernel if it satisfies $\sum_{i,j=1}^m\la a_i, K(x_i, x_j) a_j\ra \geq 0$ for all $m\geq1$, $x_1, \dots, x_m\in \calX$, and $a_1,\dots,a_m \in \mathbb{R}^{\datdim}$. 
A kernel $K$ uniquely determines an RKHS $\rkhs{K}$, a Hilbert space of $\bbR^{\datdim}$-valued functions on $\calX$ where point evaluation is a continuous linear functional.
We denote the inner product and canonical norm of the RKHS by $\la \cdot, \cdot\ra_{\rkhs{K}}$ and $\Verts{\cdot}_{\rkhs{K}} = \sqrt{\la \cdot, \cdot \ra_{\rkhs{K}}}$ respectively.  
The RKHS is also characterized by the following convenient 
\emph{reproducing property}: 
for each element $v \in \rkhs{K}$, 
$v_\dimidx(x) = \la v, \varphi_{\dimidx}(x)\ra_{\rkhs{K}}$, 
where $v_\dimidx$ is the $\dimidx$th component of $v$, 
$\varphi_{\dimidx}: x \mapsto K(\cdot, x)e_\dimidx \in \rkhs{K}$ 
with $e_\dimidx$ the $\dimidx$th standard basis in $\bbR^{\datdim}$, 
and $1\leq \dimidx \leq \datdim$. 
For a comprehensive treatment of vector-valued RKHSs, we refer the reader to \citet[Section 2.2]{Carmeli2006}. 

Equipped with the necessary ingredients, 
 we define the KSD as follows, following \citet[Definition 2]{Barp2024}: 
\begin{defn}[Kernel Stein discrepancy]
\label{defn:ksd}
Consider a target $P\in \mathcal{P}$ with positive density $p\in \calC^1$, 
and a matrix-valued kernel $K:\bbR^{\datdim}\times\bbR^{\datdim}\to\bbR^{\datdim\times\datdim}$
 with RKHS $\rkhs{K}$ satisfying the following conditions: 
    $\rkhs{K} \subset \calC^{1}(\bbR^{\datdim})$, 
    $\mathcal{T}_{P}^{\idmat}({\cal H}_{K})=\{\mathcal{T}_{P}^{\idmat}v:v\in{\cal H}_{K}\}\subset L^1(P)$, and 
    $\rkhs{K} \subset L^1(\bbR^{\datdim}, P)$. 
We define the \emph{kernel Stein discrepancy} $\ksd{K}{P}:\mathcal{P} \to [0,\infty]$ by 
\begin{equation}
\ksd KP(Q)\coloneqq \sup_{\Verts v_{\rkhs K}\leq1,\ (\mathcal{T}_P^{\idmat}v)^{-}\in L^1(Q)}\left|\EE_{X\sim Q}\left[{\cal T}_{P}^{\idmat}v(X)\right]\right|. \label{eq:ksd-def}
\end{equation}
\end{defn}
\noindent
This KSD is also referred to as the \emph{Langevin KSD}, as it is derived
from the Langevin Stein operator~${\cal T}_{P}^{\idmat}.$ 
The requirements on $\rkhs{K}$ in our definition guarantee the zero-mean property $\EE_{P}[\mathcal{T}_P^{\idmat}v]=0$ for each $v\in \rkhs{K}$. 
For measures with finite moments, treated in the coming sections, we can enforce the integrability conditions in the definition by choosing a kernel of suitable growth (see \cref{cor:Wass-imply-KSD-conv-kernel-cond}). 
Note that $\rkhs{K} \subset \calC^1(\bbR^{\datdim})$ iff $(x, y)\mapsto \partial_{1, \dimidx}\partial_{2, \dimidx}K(x,y)$ is separately continuous and locally bounded for all $\dimidx \in \{1,\dots, \datdim\}$~\citep[Theorem 2.11]{Micheli_2014}. 
Moreover, the restriction $(\mathcal{T}_P^{\idmat}v)^{-} \in L^1(Q)$ ensures that the KSD is well-defined on $\mathcal{P}$ and takes values in $[0,\infty]$.

Computational tractability is an appeal of the KSD. 
Indeed, the use of an RKHS yields a closed-form expression for the KSD.
To see this, let us define a scalar-valued kernel, \emph{a Stein kernel}
\emph{$k_{p,K}$ of $P$ with base kernel $K,$} by 
\begin{equation}
k_{p,K}(x,y)=\frac{1}{p(x)p(y)}\left\langle \nabla_{y},\la\nabla_{x},p(x)K(x,y)p(y)\ra\right\rangle ,\ x,y\in\bbR^{\datdim},\label{eq:stein-kernel}
\end{equation}
where $\nabla_x$ denotes applying the operator with respect to $x$. 
Note that since $p$ appears in both the numerator and the denominator, 
any unknown normalizing constant cancels out, allowing $k_{p, K}$ to be evaluated without knowing the constant. 
The RKHS~~$\rkhs{k_{p,K}}$ of the Stein kernel $k_{p,K}$ coincides
with $\mathcal{T}_{P}^{\idmat}({\cal H}_{K})$, 
and the KSD may be seen as \emph{the maximum mean discrepancy} (MMD)~\citep{GreBorRasSchSmo2012} defined by $\rkhs{k_{p,K}}$~\citep[Theorem 1]{Barp2024};
that is, the KSD is an IPM~\citep{Mueller1997Integral} (with one argument fixed to $P$)
\begin{align}
d_{{\cal F}}(P,Q)\coloneqq\sup_{f\in{\cal F}: f^{+} \in L^1(Q)\ \text{or}\ f^{-}\in L^1(Q)}\left|\int f\dd P-\int f\dd Q\right|, \label{eq:ipm-definition}
\end{align}
where $\mathcal{F}\subset L^1(P)$ is specified to be the unit ball $\{h\in \rkhs{k_{p,K}}: \Verts{h}_{\rkhs{k_{p,K}}}\leq 1\}$ of $\rkhs{k_{p,K}}$. 
Then, the squared KSD is expressed by a double integral \citep[Corollary 1]{Barp2024}: 
\begin{align}
\ksd KP(Q)^2 & =\int\int k_{p,K}(x,x')\dd Q(x)\dd Q(x'),\label{eq:dksd-sq-def}
\end{align}
provided ${\cal H}_{k_{p,K}}\subset L^{1}(Q)$. 
This expression is convenient particularly when $Q$ is given as
a finitely-supported distribution $Q_{\seqidx}=\sum_{i\leq \seqidx}w_{i}\delta_{x_{i}}$
with $\{x_{1},\dots,x_{\seqidx}\}\subset\bbR^{\datdim},$ $\delta_{x_{i}}$
denoting the Dirac measure having unit mass at $x_{i},$ and nonnegative
weights $w_{i}$ such that $\sum_{i\leq \seqidx}w_{i}=1.$ In this case,
we can compute the KSD exactly: the expression (\ref{eq:dksd-sq-def}) with $Q=Q_\seqidx$ 
results in 
\begin{equation}
    \ksd KP(Q_\seqidx)^2=\sum_{i,j=1}^{\seqidx}w_{i}w_{j}k_{p,K}(x_{i},x_{j}),\label{eq:ksdexact}
\end{equation}
which requires $O(n^2)$ fully parallelizable evaluations of $k_{p, K}$.\footnote{
In the worst case, evaluating the Stein kernel $k_{p,K}$ requires evaluating $O(\datdim^2)$ entries of the base kernel $K$ and its partial derivatives or only $O(\datdim)$ entries if $K$ is diagonal. In either case, the evaluations are again fully parallelizable. 
See Eq.~\eqref{eq:Stein-kernel-expanded} in 
\opt{aapnolink}{\nolink{\cref{subsec:Wass-conv-imply-KSD}} }%
\opt{aap,arxiv}{\cref{subsec:Wass-conv-imply-KSD} }%
with $m\equiv\idmat$ for an alternative expression of $k_{p,K}$. 
See also \cref{sec:Conclusion} for our discussion on the computational complexity of the KSD.}
The measure $Q_\seqidx$ may represent a weighted sample generated from a Monte Carlo algorithm, such as MCMC or sequential Monte Carlo~\citep{Moral2006, Chopin_2020}, or even a deterministic quadrature rule. 
When $Q_\seqidx$ is the empirical measure formed by independent samples from $Q$, 
the quantity $\eqref{eq:ksdexact}$ may alternatively be seen as the \textit{V}-statistic~\citep{Mises1947asymptotic, Ser2009} estimate of \eqref{eq:dksd-sq-def}. 
The KSD is thus a computable discrepancy measure; its computation only involves evaluating a target-dependent kernel on samples. 

In this article, we deal with the \emph{diffusion KSD} of \citet{BarpBriolDuncanEtAl2019Minimum}
\[
\dksd KP(Q)=\sup_{\Verts v_{\rkhs K}\leq1,\ (\mathcal{T}_P^m v)^{-}\in L^1(Q)}\left|\EE_{X\sim Q}\left[{\cal T}_{P}^{m}v(X)\right]\right|,
\]
which is defined by replacing the Langevin Stein operator in (\ref{eq:ksd-def})
with the general diffusion Stein operator (\ref{eq:diff-Stein-op}).
A benefit of a non-constant diffusion matrix $m$ is that it allows us to provide convergence control guarantees for a broader class of targets, including heavy-tailed ones (see \cref{subsec:heavy-tailed}). 
Nonetheless, the diffusion KSD may be treated as a specific instance of the
Langevin KSD: \citet{BarpBriolDuncanEtAl2019Minimum} shows that it
corresponds to the Langevin KSD defined by the kernel $K_{m}(x,y)=m(x)K(x,y)m(y)^{\top}$
tilted by the diffusion matrix $m,$ that is, $\dksd KP(Q)=\ksd{K_{m}}P(Q).$
This fact enables us to apply the theory for the Langevin KSD developed
by \citet{Barp2024}, on which we build our main result (\cref{thm:ksd-equiv-to-Wass-conv}). 

\section{KSD and $\polyorder$-Wasserstein convergence}\label{sec:Main-results}
\subsection*{High-level overview of the results}
We investigate the relationship between the KSD and
the quality of expectation approximation. 
In particular, we are interested in how a decrease in KSD relates to an improvement in the quality of the approximation. 
In this article, we focus on continuous integrands $f$ of polynomial growth. 
A natural starting point is to determine whether the KSD controls convergence in expectations of such functions. 
This question may be formally stated as follows: 
for a given sequence $(Q_{\seqidx})_{\seqidx\geq1}$ of approximating probability measures,
\[
\text{does convergence\ \ensuremath{\ksd KP}(\ensuremath{Q_{\seqidx}}) \ensuremath{\to} 0\ }\text{imply}\ \int f\dd Q_{\seqidx}\to\int f\dd P\ \text{?}
\]
The following sections present three results concerning this question. 
    In Section \ref{subsec:KSD-convergence-and}, we formally define our target mode of convergence ($\polyorder$-Wasserstein convergence) and present a negative result:  
    standard KSDs, when applied with kernels previously recommended for weak convergence control, fail to enforce $q$-Wasserstein convergence~(\cref{prop:ksd-conv-control-failure}).
    This observation highlights the importance of kernel choice. 
    To overcome this limitation, we establish sufficient conditions on the approximation and growth properties of the kernel, demonstrating the equivalence between KSD- and $\polyorder$-Wasserstein convergence under these new conditions~(\cref{thm:ksd-equiv-to-Wass-conv,prop:Wass-conv-implies-KSD}). 
    See also \cref{cor:default-kernel} and the surrounding discussion for a practical recommendation. 
In Section \ref{subsec:Rate-of--Wasserstein},
with additional assumptions, we obtain a rate of $\polyorder$-Wasserstein
convergence in terms of the KSD (\cref{thm:Matern-bound-linear,thm:Matern-bound-quad,thm:Wass-ksd-bound}). 
Together, these results establish
a theoretical foundation for the KSD as a quality measure for distribution approximation. 

\subsection{KSD convergence and its relation to Wasserstein convergence} \label{subsec:KSD-convergence-and}

We begin this section by formalizing our task. For fixed $\polyorder\in[0,\infty),$
define the class of Borel probability measures on $\bbR^{\datdim}$
with finite $\polyorder$th moments by 
\[
\finitemomentspace{\polyorder}\coloneqq\left\{ \mu\in{\cal P}:\int\Verts x_{2}^{\polyorder}\dd\mu(x)<\infty\right\} .
\]
We assume that our target $P$ is an element of $\finitemomentspace{\polyorder}$. 
Our goal is to relate KSD convergence to
the following type of convergence: 
\begin{defn}[$\polyorder$-Wasserstein convergence and $\polyorder$-growth functions]
Let $\polyorder\in[0,\infty)$. Suppose $P\in \finitemomentspace{\polyorder}$. 
For a sequence $(Q_{\seqidx})_{\seqidx\geq1}\subset\mathcal{P}$, 
the $\polyorder$-Wasserstein convergence $Q_{\seqidx} \toL{\polyorder} P$ is defined as the convergence 
\[
\int \verts{f}\dd Q_{\seqidx}\to\int \verts{f}\dd P 
\]
holding true for any continuous function $f:\bbR^{\datdim}\to\bbR$ of
\emph{$\polyorder$-growth}, i.e., such function satisfying the growth condition
\[
\verts{f(x)}\leq C\bigl(1+\Verts x_{2}^{\polyorder}\bigr)
\]
everywhere for some $C>0.$ 
\end{defn}

Equivalently, the above definition states that $Q_{\seqidx}$ eventually enters $\finitemomentspace{\polyorder}$, and then 
$\int f \dd Q_{\seqidx}$ is well-defined and converges to $\int f \dd P$, for each continuous $\polyorder$-growth function $f$. 
Thus, it extends the standard notion of \emph{weak convergence in $\finitemomentspace{\polyorder}$} to $\mathcal{P}$-valued sequences. 
To avoid possible
confusion, we will reserve the
term \emph{weak convergence} for the case $\polyorder=0,$ i.e., convergence 
with respect to all bounded continuous functions, and use the term 
$\polyorder$\emph{-Wasserstein convergence} for $\polyorder>0$. 
This choice is made because the aforementioned type of convergence is also metrized by the $\polyorder$-Wasserstein distance~\citep[Theorems 7.3 and 7.12]{Villani_2003}:
\begin{equation}
    W_{\polyorder}(P, Q) \coloneqq 
    \begin{cases}
    \left( \inf_{\pi \in \Pi(P, Q)} \EE_{X, Y \sim \pi}\bigl[\Verts{X-Y}_2^{\polyorder}\bigr] \right)^{1 \land (1/{\polyorder})} &\ \text{if}\ P, Q \in \finitemomentspace{\polyorder},\\
    \infty &\ \text{otherwise},
    \end{cases}\label{eq:Wass-dist-def}
\end{equation}
where $\Pi(P,Q)$ denotes the set of all possible couplings of the probability measures $P \text{ and } Q$.

To present our result, we make the following assumptions on the target density and the diffusion matrix:
\begin{assumption}[Dissipativity]
\label{assu:dissipativity} Let $b(x)=\la\nabla,p(x)m(x)\ra/\{2p(x)\}.$
For any $x\in\bbR^{\datdim},$ the target density $p$ and the diffusion
matrix $m$ satisfy 
\[
2\la b(x),x\ra+\tr{m(x)}\leq-\alpha\Verts x_{2}^{2}+\beta_{1}\Verts x_{2}+\beta_{0}
\]
where $\tr{m(x)}=\sum_{\dimidx=1}^{\datdim}m_{\dimidx\dimidx}(x)$
is the trace of $m(x),$  $\alpha>0,$ and $\beta_{1},\beta_{0}\geq0.$ 
\end{assumption}

\begin{assumption}[Growth]
\label{assu:poly-grow-coef}Let $b(x)=\la\nabla,p(x)m(x)\ra/\{2p(x)\}.$
The target density $p$ and the diffusion matrix $m$ satisfy 
for any $x\in\bbR^{\datdim},$  
\[
\Verts{b(x)}_{2}\leq\lambda_{b}\bigl(1+\Verts x_{2}\bigr),\ \Verts{m(x)}_{\mathrm{op}}\leq\lambda_{m}\bigl(c_{m}+\Verts x_{2}^{\polyorder_{m}+1}\bigr),
\]
where $c_{m}\geq0,$ $\lambda_{b},\lambda_{m}>0,$ and $\polyorder_{m}\in\{0,1\}.$ 
\end{assumption}

In essence, these two assumptions determine the tail behavior of the target distribution. 
To illustrate, suppose $m\equiv \idmat$; we have $2b(x)=\nabla \log p(x)$ in this case. 
Then, \cref{assu:dissipativity} expresses that the target density has %
strongly log concave tails, akin to those of a Gaussian  
(note that $2b(x)=-x$ for the standard multivariate Gaussian density). 
For example, a Gaussian mixture with shared covariance satisfies this condition~\citep[Example 3]{gorham2016measuring}, 
as well as the linear growth condition in~\cref{assu:poly-grow-coef}. 
For distributions with tails heavier than Gaussian, choosing an appropriate coercive diffusion matrix $m$ allows us to enforce 
the dissipativity \cref{assu:dissipativity}.
We illustrate this point using a Student's {\textit t}-distribution in \cref{subsec:heavy-tailed} and 
\opt{aapnolink}{\nolink{\cref{lem:dissipative-student}}. }%
\opt{aap,arxiv}{\cref{lem:dissipative-student}. }%
\cref{assu:dissipativity,assu:poly-grow-coef} can also be verified with the right choice of $m$ for a
multivariate Student's {\textit t}-regression posterior with a pseudo-Huber prior, as described by \citet[Example 4]{gorham2016measuring}. 

We will show that, for target distributions satisfying these assumptions, the KSD with a suitable kernel characterizes \emph{uniform integrability}, defined as follows: 
\begin{defn}[Eventual uniform integrability of $\polyorder$th moments]
\label{def:ui}
Let $\polyorder\in(0,\infty).$ A sequence $(Q_\seqidx)_{\seqidx \geq 1}\subset\mathcal{P}$
is said to have \emph{uniformly integrable $\polyorder$th moments} iff it
\emph{eventually} uniformly integrates $x\mapsto\Verts x_{2}^{\polyorder},$ i.e., 
\[
\lim_{r\to\infty}\limsup_{\seqidx\to\infty}\int_{\{\Verts x_{2}>r\}}\Verts x_{2}^{\polyorder}\dd Q_{\seqidx}(x)=0.
\]
\end{defn}
It is known that for weakly converging sequences, the $\polyorder$-Wasserstein
convergence is equivalent to the uniform integrability of $\polyorder$th
moments \citep[Theorem 7.12]{Villani_2003}. 
Intuitively, uniform integrability
is a condition enforcing that the probability mass does not diverge
too fast along the sequence. If $\polyorder=0,$ the above definition
is reduced to uniform tightness \citep[see, e.g.,][Chapter 9]{Dudley2002}.
Uniform integrability is a stricter condition in that it requires
the decay rate of the tail probability to stay faster than a degree-$\polyorder$
polynomial. 
Note that in contrast to the standard definition in the literature~\citep[see, e.g.,][Eq. 5.1.19]{Ambrosio2005}, 
our definition allows finitely many initial elements in the sequence to have infinite $\polyorder$th moments.

Under dissipativity assumptions as in \cref{assu:dissipativity}, 
prior work \citep{GorMac2017,Barp2024}
identified conditions for the KSD to enforce tightness and weak convergence,
whereas uniform integrability was not established. 
In fact, standard KSDs used for weak convergence control fall short of our purpose, as they can fail to enforce  uniform integrability and hence to control $\polyorder$-Wasserstein convergence: 

\newcommand{\ksdconvcontrolfailurename}{Standard KSDs fail to control moments}
\begin{prop}[\ksdconvcontrolfailurename]\label{prop:ksd-conv-control-failure}
    Let $\polyorder \in (0,\infty)$ and $P\in \finitemomentspace\polyorder$. 
    Let $K$ be a matrix-valued kernel satisfying the conditions in \cref{defn:ksd}. 
    Suppose 
    $%
    \sqrt{k_{p,K}(x,x)} = O(1+\Verts{x}^{\polyorder'}_2) 
    $ 
    for some $\polyorder'$ with $0 < \polyorder' < \polyorder$. 
    Then, $\ksd{K}{P}(Q_{\seqidx})\to 0$ does not imply $Q_{\seqidx}\toL{\polyorder} P$. %
    Consequently, under \cref{assu:poly-grow-coef},  
    if both $K(x,x)$  and $\{\partial_{1,\dimidx}\partial_{2, \dimidx}K(x,x)\}_{1\leq i \leq d}$  are bounded as functions of $x$, 
    the Langevin $\ksd{K}{P}(\cdot)$ cannot control $q$-Wasserstein convergence for $\polyorder > 1$. 
\end{prop}
The proof can be found in  
\opt{aapnolink}{\nolink{\cref{sec:q-wass-conv-necessary}}.}%
\opt{aap,arxiv}{\cref{sec:q-wass-conv-necessary}.}
Intuitively, this limitation arises from the lack of $\polyorder$-growth functions in the Stein RKHS. 
Indeed, under \cref{assu:poly-grow-coef}, the Stein kernel cannot have super-linear growth with a bounded base kernel, and hence the Langevin KSD is unable to control $\polyorder$th moment convergence for $\polyorder>1$. 
This negative result applies to the most common base kernels used with KSDs, 
including the IMQ and log-inverse kernels  recommended by \citet{GorMac2017} and \citet{Chen2018}, respectively; it also encompasses the B-spline, Gaussian, Mat\'{e}rn, sech, sinc, and Wendland's compactly supported kernels. 
Moreover, our experiment in \cref{subsec:non-conv-variance} demonstrates the inability of the IMQ-KSD to detect non-convergence in the second moment in practice. 
This observation motivates us to seek conditions under which the KSD controls both $\polyorder$th
moment uniform integrability and weak convergence. 
In the following, we first provide sufficient conditions for enforcing uniform integrability.
Then, building on the weak convergence result of \citet{Barp2024},
we prove that the resulting KSD controls $\polyorder$-Wasserstein convergence. 

A sufficient condition for enforcing $\polyorder$th moment uniform
integrability is given as follows: 
\begin{defn}[$\polyorder$-growth approximation]
\label{def:q-growth-approx} Let $\polyorder\in[0,\infty).$ A set
${\cal F}$ of real-valued functions is said to \emph{approximate $\polyorder$-growth}
if, for each $\varepsilon>0,$ there exists $r_{\varepsilon}>0$ and
a function $f_{\varepsilon}\in{\cal F}$ such that for any $x\in\bbR^{\datdim},$
\begin{equation}
f_{\varepsilon}(x) \geq\Verts x_{2}^{\polyorder}1\{\Verts x_{2}>r_{\varepsilon}\}-\varepsilon. \label{eq:def-dominating-indicator}
\end{equation}
We also say $\calF$ approximates $\polyorder$-growth with respect to $P$ 
if $\calF \subset L^1(P)$ and its $P$-centered version $\calF^P = \{f-\EE_P[f]: f\in \calF\}$ approximates $\polyorder$-growth. 
\end{defn}
Definition \ref{def:q-growth-approx} states that using an element of ${\cal F}$, we can either dominate or accurately approximate from below a function that is zero on a ball and grows like $\Verts x_{2}^{\polyorder}$ outside of that ball. 
When $\calF$ is $P$-centered, the case $\polyorder=0$ corresponds to the
indicator dominating property of \citet[Definition 4]{Barp2024},
which is used to enforce tightness with an IPM. 

\cref{lem:ui-IPM} below shows that an IPM can enforce uniform integrability if the defining function class approximates $\polyorder$-growth, which readily yields the conclusion for the KSD (\cref{cor:q-approx-uniform-integrability}). 
\begin{lem}[Controlling UI with growth approximation]
\label{lem:ui-IPM}
Fix $\polyorder\in[0,\infty)$. 
Let $\mathcal{F}$ be a real vector space of functions on $\bbR^\datdim$ and 
$\lambda: \mathcal{F} \to [0,\infty)$ be a function satisfying $\lambda(cf) \leq c \lambda(f)$ for any $c\in (0,\infty)$. 
If $\calF$ approximates $\polyorder$-growth with respect to $P\in\finitemomentspace\polyorder$,  
then, a sequence $(Q_{\seqidx})_{\seqidx\geq1}\subset{\cal P}$ has uniformly integrable $\polyorder$th moments
if $d_{\mathcal{B}}(P, Q_{\seqidx}) \to 0$, where $d_{\mathcal{B}}$ is the IPM \eqref{eq:ipm-definition} defined by $\mathcal{B} = \{f \in \mathcal{F}:  \lambda(f)\leq 1\}$. 
\end{lem}

\begin{proof}
For any $\varepsilon>0,$ by the $\polyorder$-growth approximation
property (\ref{eq:def-dominating-indicator}) of $\calF^P$, there exists $f_{\varepsilon}\in \calF$ 
and $r_{\varepsilon}>0$ such that 
\[
\int_{\{\Verts x_{2}>r_{\varepsilon}\}}\Verts x_{2}^{\polyorder}\dd Q_{\seqidx}(x)
\leq\int f_{\varepsilon} \dd Q_{\seqidx} - \int f_{\varepsilon}\dd P +\varepsilon  \leq \{1\lor \lambda(f_\varepsilon)\} d_{\mathcal{B}}(P, Q_{\seqidx})+\varepsilon.
\]
Taking the limit $\seqidx\to \infty$ concludes the proof. 
\end{proof}
\begin{cor}[Enforcing UI with KSD]
    \label{cor:q-approx-uniform-integrability} 
    Suppose $P$ and $K$ satisfy the conditions in \cref{defn:ksd}. 
    Assume that the Stein RKHS ${\cal T}_{P}^{m}({\cal H}_{K})$ approximates $\polyorder$-growth.
    Then, $(Q_{\seqidx})_{\seqidx\geq 1} \subset \mathcal{P}$ has uniformly integrable $\polyorder$th moments 
    if $\ksd{K_{m}}P(Q_{\seqidx})\to0$ as $n\to\infty$.
\end{cor}

We thus explore sufficient conditions on the reproducing kernel to guarantee $\polyorder$-growth approximation and weak convergence control. 
In what follows, we show that it suffices to use a kernel of the following form: 
\begin{defn}[Basic kernel form]
\label{def:kernel-form}
Let $\polyorder \in [0,\infty)$. 
Let $K:\bbR^{\datdim}\times\bbR^{\datdim}\to\bbR^{\datdim\times\datdim}$
be a matrix-valued kernel defined by 
\begin{equation}
K(x,y)=\weight_{\polyorder-1}(x)\left(L(x,y)+\bar{k}_{\mathrm{lin}}(x,y)\idmat\right)\weight_{\polyorder-1}(y)
,\label{eq:q-growth-approx-kernel}
\end{equation}
where, for $s \in \bbR$, $\weight_s(x)=\bigl(\tau^{2}+\Verts x_{2}^{2}\bigr)^{s/2}$ with $\tau>0$, 
$L$ is a matrix-valued kernel with~$\rkhs{L} \subset \calC^1(\bbR^{\datdim})$, 
and 
\[
\bar{k}_{\mathrm{lin}}(x,y)=\frac{k_{\mathrm{lin}}(x,y)}{\sqrt{k_{\mathrm{lin}}(x,x)} \sqrt{k_{\mathrm{lin}}(y,y)}}
\]
is the normalized version of a linear kernel $k_{\mathrm{lin}}(x,y)=\la x,y \ra+\tau^2$ with bias $\tau^2>0$. 
\end{defn}

Our next lemma identifies concrete kernels that realize the $\polyorder$-growth approximation
property. 
To establish this result, we will rely on the concept of $\calC_0^1(\bbR^{\datdim})$-universality~\cite[Definition 5]{SimonGabriel2018}, 
which states that each element of $\calC_0^1(\bbR^{\datdim})$, up to its first derivatives, can be approximated uniformly over $\bbR^{\datdim}$ by an element of the RKHS; that is, the RKHS is dense in $\calC_0^1(\bbR^{\datdim})$. 
\begin{lem}[Kernel choice for $\polyorder$-growth approximation]
\label{lem:existence-qgrowth-approx-universality}
Let $\polyorder\in[0,\infty).$
Define $K$ as in \cref{def:kernel-form}. 
Assume the dissipativity condition in Assumption \ref{assu:dissipativity}
with $\alpha>0$. 
Assume the growth conditions in Assumption \ref{assu:poly-grow-coef}. 
If $\polyorder_{m}=0$ in \cref{assu:poly-grow-coef}, 
assume that $\rkhs{L}$ is universal to $\calC_0^1(\bbR^{\datdim})$. 
If $\polyorder_{m}=1$, 
assume the following conditions: 
(a) $\alpha>\lambda_{m}(\polyorder+2),$
(b) $L=\ell\idmat$ where $\ell$ is a scalar-valued 
kernel such that $\ell(x,y) = \Phi(x-y) / \{\weight_1(x)\weight_1(y)\}$ for some $\Phi \in \calC^2$; 
and (c) $\Phi$ has a continuous non-vanishing generalized Fourier transform (see our definition in \cref{sec:intro}). 
Then, ${\cal T}_{P}^{m}(\rkhs K)$ approximates $\polyorder$-growth. 
\end{lem}
\noindent
\opt{aapnolink}{Our proof in \nolink{\cref{subsec:Proof-of-Lemma-existence-qgrowth-approx}} }%
\opt{aap,arxiv}{Our proof in \cref{subsec:Proof-of-Lemma-existence-qgrowth-approx} }%
relies on \cref{assu:dissipativity,assu:poly-grow-coef} to construct a function to approximate $\polyorder$-growth. As this may not be in the RKHS, we further approximate it using an element of the RKHS, which is possible due to our kernel choice, particularly the universal kernel $L$. 

Having derived a kernel choice for the KSD to enforce uniform integrability, 
we now present a result that addresses the question posed at the outset 
(see also \cref{cor:default-kernel} for a simplified version): 

\begin{thm}[Vanishing KSD implies $\polyorder$-Wasserstein convergence]
\label{thm:ksd-equiv-to-Wass-conv} 
Let $\polyorder\in (0,\infty)$ and $P\in \finitemomentspace\polyorder$. 
Define a base kernel $K$ as in \cref{def:kernel-form}. %
Suppose that diffusion matrix $m$
satisfies the following conditions: 
(a) $m(x)$ is invertible for each $x\in\bbR^{\datdim}$ with 
$
    \Verts{m^{-1}(x)}_{\mathrm{op}} = O(1)
$,
and 
(b) $\Verts{\la\nabla, m(x)\ra}_2 = O(\Verts{x}_2)$. 
Assume the dissipativity condition in \cref{assu:dissipativity} with $\alpha>0$. 
Assume the growth conditions in \cref{assu:poly-grow-coef} with $\lambda_m >0$ and $\polyorder_m \in \{0, 1\}$. 
Define 
\begin{align*}
 & L^{(1)}(x,y)=\frac{L_{0}(x,y)}{{\weight}_{\polyorder+1}(x){\weight}_{\polyorder+1}(y)},\ 
 L^{(2)}(x,y)=\frac{\ell(x,y)}{\weight_{\polyorder_m}(x)\weight_{\polyorder_m}(y)} \cdot \idmat,
\end{align*}
where  
$L_{0}$ is universal to $\calC_0^1(\bbR^{\datdim})$; 
and $\ell$ is a scalar-valued translation-invariant kernel %
universal to $\calC_0^1$ such that  
$\ell(x,y) = \Phi(x-y)$ for some $\Phi\in \calC^2$. 
If $\polyorder_{m}=1,$  
additionally, assume the following: 
(a) $\alpha>\lambda_{m}(\polyorder+2),$
and (b) $\Phi$ has continuous non-vanishing generalized Fourier transform. 
Specify kernel $L$ in $K$ to be either of $L^{(1)}$ or $L^{(2)}$ if $\polyorder_m=0$, 
and to be $L^{(2)}$ if $\polyorder_m=1$. 
Then, for any sequence $(Q_{\seqidx})_{\seqidx\geq1}\subset\mathcal{P}$, 
we have $Q_{\seqidx} \toL{\polyorder} P$ if $\ksd{K_{m}}P(Q_{n})\to0$ as $n\to\infty$. 
\end{thm}
\begin{proofsketch}
We outline our proof, which can be found in 
\opt{aapnolink}{\nolink{\cref{subsec:proof-of-KSD-implies-Wass}}. }%
\opt{aap,arxiv}{\cref{subsec:proof-of-KSD-implies-Wass}. }%
Our choices of $K$ and $m$ ensure that
the Stein RKHS ${\cal T}_{P}^{m}(\rkhs K)$ approximates $\polyorder$-growth (by \cref{lem:existence-qgrowth-approx-universality}) and hence that the KSD enforces $q$-uniform integrability (by \cref{cor:q-approx-uniform-integrability}). 
Moreover, each choice of $\mathcal{C}_0^1(\bbR^{\datdim})$-universal kernel $L$ ensures that the KSD controls weak convergence for tight sequences. 
Specifically, the tilting by $\weight_{\polyorder+1}$ in $L^{(1)}$ 
and the use of translation-invariant kernel in $L^{(2)}$ guarantee that 
the Stein RKHS admits a subset separating $P$ from arbitrary probability measures, which implies the required control of weak convergence~\citep[Theorems 6-9]{Barp2024}. 
\end{proofsketch}

\begin{rem}
If $\polyorder_m=1$, without the tilting in $L^{(2)}$ in \cref{thm:ksd-equiv-to-Wass-conv}, 
the Stein RKHS $\mathcal{T}_P^m (\rkhs{K})$ may contain functions of $(\polyorder+1)$-growth due to the quadratic growth of $m$. 
Nonetheless, \cref{thm:ksd-equiv-to-Wass-conv} holds with $L^{(3)}(x,y) = \ell(x,y)\idmat$ substituted for $L^{(2)}$ provided that $P\in \finitemomentspace{\polyorder+1}$. 
\end{rem}

\begin{rem}
\label{rem:allowed-growth-testfunction}In Theorem \ref{thm:ksd-equiv-to-Wass-conv}
(see also Lemma \ref{lem:existence-qgrowth-approx-universality}),
the quadratic growth ($\polyorder_m=1$) of $\Verts{m(x)}_{\mathrm{op}}$ requires 
$\alpha>\lambda_{m}(\polyorder+2),$ which limits the
allowed growth $\polyorder$ of test functions. This condition may
be relaxed to $\alpha>\lambda_{m}\{\polyorder+\min(2, \theta-2)\}$ 
if $P \in \finitemomentspace{\polyorder+\theta}$ for some $\theta>0$ by choosing the tilting function  $\weight_{\polyorder+\theta-1}(x)=\bigl(\tau^{2}+\Verts x_{2}^{2}\bigr)^{(\polyorder+\theta-1)/2}$ in the definition of $K$. 
Our argument relies on the existence of a \emph{coercive function},
with growth rate $\Verts x_{2}^{\polyorder+\theta}$ as $\norm{x}_2\to\infty$. %
Prior work \citep{GorMac2017,Chen2018,HugMac2018,Hodgkinson2020} used a
coercive function to ensure that the KSD enforces uniform tightness
\citep[see also][Lemma 1]{Barp2024}.
We defer the proof of this remark to 
\opt{aapnolink}{\nolink{\cref{sec:Coercive-functions-and}}. }%
\opt{aap,arxiv}{\cref{sec:Coercive-functions-and}. }%
\end{rem}

Our next results, proved in 
\opt{aapnolink}{\nolink{\cref{subsec:Wass-conv-imply-KSD}}, }%
\opt{aap,arxiv}{\cref{subsec:Wass-conv-imply-KSD}, } %
show that, under mild conditions, $q$-Wasserstein convergence implies KSD convergence to zero.
\begin{prop} [KSD detects $\polyorder$-Wasserstein convergence]
\label{prop:Wass-conv-implies-KSD}
For a target $P$ and a base kernel $K$, suppose the Stein kernel $k_{p,K}:\bbR^{\datdim}\times\bbR^{\datdim}\to\bbR$ is continuous with  %
    $
    \sqrt{k_{p, K}(x,x)} = O(1+\Verts{x}_2^{\polyorder}) 
    $
for some $\polyorder\in (0,\infty)$. 
Then, for any sequence $(Q_{\seqidx})_{\seqidx\geq1}\subset \mathcal{P}$, 
we have $\ksd{K}{P}(Q_{\seqidx}) \to 0$ if 
$Q_{\seqidx} \toL{\polyorder} P$. 
\end{prop}

\begin{cor}[User-friendly conditions for detecting $q$-Wasserstein convergence]
\label{cor:Wass-imply-KSD-conv-kernel-cond}
For a target $P$ and diffusion matrix $m$, suppose  \cref{assu:poly-grow-coef} holds with growth exponent $\polyorder_m\in \{0,1\}$. 
If, for a base kernel $K$, some $q\in(0,\infty)$, and each $\dimidx \in \{1,\dots,\datdim\}$, $\partial_{1,\dimidx }\partial_{2, \dimidx}K(x,y)$ is continuous and 
\[
   \sqrt{\Verts{K(x,x)}_{\mathrm{op}}} = O\left(\bigl(1+\Verts{x}_2\bigr)^{\polyorder-1}\right),\ 
  \sqrt{\Verts{\partial_{1, \dimidx}\partial_{2,\dimidx}K(x,x)}_{\mathrm{op}}}  = O\left(\bigl(1+\Verts{x}_2\bigr)^{\polyorder-\polyorder_{m}-1}\right),
\]
then the preconditions of \cref{prop:Wass-conv-implies-KSD} hold for the Stein kernel $k_{p,K_m}$.%
\end{cor}
Taken together, \cref{thm:ksd-equiv-to-Wass-conv,cor:Wass-imply-KSD-conv-kernel-cond} provide broad sufficient conditions under which KSD convergence is exactly equivalent to $q$-Wasserstein convergence.
The proof of this corollary can be found in 
\opt{aapnolink}{\nolink{\cref{subsec:proof-KSD-Wass-equiv}}. }%
\opt{aap,arxiv}{\cref{subsec:proof-KSD-Wass-equiv}. }%

\begin{cor}[KSD-Wasserstein equivalence]
\label{cor:ksd-Wass-equiv}
Instantiate the assumptions and notation of \cref{thm:ksd-equiv-to-Wass-conv}, and assume $\partial_{1,i}\partial_{2,i}L_0$ is continuous for each $i \in \{1,\dots,\datdim\}$.
Then, for any sequence $(Q_{\seqidx})_{\seqidx\geq1}\subset \mathcal{P}$, 
we have $\ksd{K_m}{P}(Q_n) \to 0$ if and only if 
$Q_{\seqidx} \toL{\polyorder} P$.
\end{cor}

Note that \cref{thm:ksd-equiv-to-Wass-conv} and \cref{prop:Wass-conv-implies-KSD} build on different sets of assumptions. 
Effectively, \cref{prop:Wass-conv-implies-KSD} only requires the continuous differentiablility and a suitable growth of the reproducing kernel, 
and unlike in \cref{thm:ksd-equiv-to-Wass-conv}, the corresponding Stein RKHS need not separate all distributions from the target. 
The form of the recommended kernel $K$ in \cref{thm:ksd-equiv-to-Wass-conv} is determined by the assumptions on the target distribution (\cref{assu:dissipativity,assu:poly-grow-coef}). 
In particular, the only freedom left to the user is the choice of the universal kernel $L$. 
The kernel $L^{(1)}$ in \cref{thm:ksd-equiv-to-Wass-conv} offers great generality but might be too abstract for applications. 
The choice $L^{(2)}$ is substantially simpler and includes a broad class of kernels used in practice including Mat{\'e}rn (with smoothness greater than 1, see Eq.~\ref{eq:Matern}), Gaussian, and IMQ kernels. 
For ease of application, we therefore provide a simplified form of \cref{cor:ksd-Wass-equiv} based on an IMQ kernel (see 
\opt{aapnolink}{\nolink{\cref{subsec:proof-user-friendly-KSD-Wass-equiv}} }%
\opt{aap,arxiv}{\cref{subsec:proof-user-friendly-KSD-Wass-equiv} }%
for a proof). 

\newcommand{\imqp}{\textup{IMQ$^+$}\xspace}
\newcommand{\matp}{\textup{Mat}$^+$\xspace}
\newcommand{\imq}{\textup{IMQ}}

\begin{cor}[User-friendly KSD-Wasserstein equivalence] %
\label{cor:default-kernel}
Fix $\polyorder \in (0, \infty)$. 
For a target $P$ and a diffusion matrix $m$, suppose that 
\cref{assu:dissipativity,assu:poly-grow-coef} hold with 
$\alpha > \lambda_m(\polyorder+2)$ if $\polyorder_m=1$, 
$m(x)$ invertible for each $x\in\bbR^{\datdim}$ with 
$
    \Verts{m^{-1}(x)}_{\mathrm{op}} = O(1)
$, 
and 
$\Verts{\la\nabla, m(x)\ra}_2 = O(\Verts{x}_2)$.
Let $K=k_{\imqp}\idmat$ for a scalar-valued kernel $k_{\imqp}:\bbR^{\datdim}\times\bbR^{\datdim}\to\bbR$ defined by 
\[
    k_{\imqp}(x,y) = \weight_{\polyorder-1}(x)\left( \frac{k_{\imq}(x,y)}{\weight_{\polyorder_m}(x)\weight_{\polyorder_m}(y)}+ \frac{1+\la x, y\ra}{\weight_1(x)\weight_1(y)}\right) \weight_{\polyorder-1}(y)
\]
where $k_{\imq}(x,y)=\bigl(1+\Verts{x-y}_2^2\bigr)^{-1/2}$, 
and 
$\tau = 1$ set for the tilting $\weight_{\bullet}$. 
Then, for any sequence $(Q_{\seqidx})_{\seqidx\geq1}\subset \mathcal{P}$, 
we have $\ksd{K_m}{P}(Q_n) \to 0$ if and only if 
$Q_{\seqidx} \toL{\polyorder} P$. 
\end{cor}

The informativeness of the KSD will depend on the choice of the kernel, particularly in high dimensions. 
We recommend the above as a default choice but not as a universal solution. 
A more refined approach would be to modify the kernel according to application, e.g., by introducing a scaling parameter to the input of $k_{\imq}$, 
using a different bias parameter $\tau^2$ when tilting with $\weight_{\bullet}$, or introducing a location parameter in evaluating the norm $\Verts{x}_2^2$ and the linear inner product. 
Kernel selection is an ongoing research topic and beyond the scope of this article. 
Nonetheless, as an empirical study, we discuss a heuristic kernel selection method from goodness-of-fit testing in \cref{subsec:Failure-mode:-distribution}. 

We close this section with a brief discussion of related work. 
\citet{Modeste2023} derived sufficient conditions for an MMD to metrize $\polyorder$-Wasserstein convergence. 
Recall that the KSD is a specific MMD defined by a Stein kernel (see \cref{sec:Diffusion-kernel-Stein} and 
\opt{aapnolink}{\nolink{\cref{sec:q-wass-conv-necessary}}}%
\opt{aap,arxiv}{{\cref{sec:q-wass-conv-necessary}}}%
). 
However, the results of \citeauthor{Modeste2023} do not  apply to KSDs, as they consider specific classes of kernels that exclude Stein kernels, namely energy kernels (the covariance functions of fractional Brownian motions) and sums of translation-invariant kernels and moment-characterizing kernels. 
To the best of our knowledge, our work establishes the first KSD-Wasserstein equivalence results.

\subsection{Rate of $\protect\polyorder$-Wasserstein convergence via a KSD bound}\label{subsec:Rate-of--Wasserstein}
The topological equivalence established in \cref{subsec:KSD-convergence-and} does not express the quantitative relation between the two modes of convergence. 
In this section, we clarify this connection by relating the KSD to
an IPM that metrizes $\polyorder$-Wasserstein convergence. 
Specifically, we obtain an explicit KSD upper bound on the IPM, thereby demonstrating the speed of $q$-Wasserstein convergence relative to KSD convergence. 

Before presenting our results, it is instructive to outline our strategy.
Our approach here is an adaptation of Stein's method \citep{Stein1972bound, ross2011}.
For a given $\polyorder$-growth function $f$ in a (sufficiently
regular) function class $\mathcal{F},$ consider the \emph{Stein equation}
\begin{equation}
{\cal T}_{P}^{m}v=f-\EE_{P}[f],\label{eq:SteinEq}
\end{equation}
where the equality is pointwise. If we assume the existence of a solution
$v_{f}$ to (\ref{eq:SteinEq}) for each $f\in\mathcal{F}$, we can then obtain a bound on the
worst-case expectation error with respect to ${\cal F}:$
\begin{align*}
d_{{\cal F}}(P,Q)%
\leq\sup_{f\in{\cal F}}\left|\int{\cal T}_{P}^{m}v_{f}\dd Q\right|,
\end{align*}
where $Q\in \finitemomentspace\polyorder$. 
Our goal is to express the upper bound in terms of the KSD. To this
end, we approximate each $v_{f}$ by an RKHS function $\tilde{v}_{f}\in{\cal H}_{K}$,
which yields an estimate 
\begin{align}
d_{{\cal F}}(P,Q) & \leq\sup_{f\in{\cal F}}\left|\int\bigl({\cal T}_{P}^{m}v_{f}-{\cal T}_{P}^{m}\tilde{v}_{f}\bigr)\dd Q\right|+\sup_{f\in{\cal F}}\Verts{\tilde{v}_{f}}_{\rkhs K}\cdot\ksd{K_{m}}P(Q).\label{eq:stein-bound-df}
\end{align}
Evaluating this upper bound requires assessing the approximation error
and the norm of the approximation; this task calls for regularity
properties of the approximand $v_{f}$ and a characterization of the complexity
of the RKHS. 

In the following, we first introduce a function class ${\cal F}$
that both provides a differentiable solution to the Stein equation
and induces an IPM $d_{{\cal F}}$ metrizing $\polyorder$-Wasserstein
convergence. We then specify a kernel function to obtain a concrete
bound. 

\subsubsection{Pseudo-Lipschitz functions}
\label{subsec:plipfuncs}

We first specify the function class ${\cal F}$. In our analysis,
we consider a set ${\cal F}_{\polyorder}$ of \emph{pseudo-Lipschitz}
functions of order $\polyorder-1$ defined by 
\begin{align*}
 {\cal F}_{\polyorder}\coloneqq\big\{  f:\bbR^{\datdim}\to\bbR:\ &f\in{\cal C}^{3}\ \text{with}\ \plipconst f{\polyorder-1}\leq1\\
 &\text{and}\ \sup_{x\in \bbR^{\datdim}}\Verts{\nabla^{i}f(x)}_{\mathrm{op}}/\bigl(1+\Verts x_{2}^{\polyorder-1}\bigr)\leq1\ \text{for }i\in\{2,3\} \big\},
\end{align*}
where 
\[
\plipconst f{\polyorder-1}=\sup_{x\ne y}\frac{\lvert f(x)-f(y) \rvert }{\bigl\{(1+ 1\{\polyorder>1\}(\Verts x_{2}^{\polyorder-1}+\Verts y_{2}^{\polyorder-1})\bigr\}\Verts{x-y}_{2}},\ \polyorder\geq1.
\]
Pseudo-Lipschitz functions of order $\polyorder-1$ are defined as
those $f$ with $\plipconst f{\polyorder-1}<\infty.$ The seminorm
$\plipconst{\cdot}{\polyorder-1}$ generalizes the Lipschitz seminorm
$\plipconst f0,$ and pseudo-Lipschitz functions are allowed to have
$\polyorder$-growth. 
Indeed, 
\opt{aapnolink}{\nolink{\cref{sec:Polynomial-functions-are}} }%
\opt{aap,arxiv}{\cref{sec:Polynomial-functions-are} }%
shows that ${\cal F}_{\polyorder}$ contains degree-$\polyorder$
monomial functions of $(x_{1},\dots,x_{\datdim})\in\bbR^{\datdim}$; 
in particular, if $f(x)=\prod_{\dimidx=1}^{\datdim}x_{\dimidx}^{\polyorder_{\dimidx}}$ with $\sum_{\dimidx=1}^{\datdim}\polyorder_{\dimidx}=\polyorder$, 
we have 
$\verts{\EE_P[f]-\EE_{Q}[f]}\leq c_{\polyorder, \datdim} d_{\calF_{\polyorder}}(P, Q)$,
where $c_{\polyorder, \datdim}$ is a constant given in 
\opt{aapnolink}{\nolink{\cref{lem:monomial-plip-const}}. }%
\opt{aap,arxiv}{\cref{lem:monomial-plip-const}. }%
Thus, any quantitative bound on $d_{\calF_{\polyorder}}$ translates directly into the $\polyorder$th moment approximation error, justifying our choice of the function class $\calF_{\polyorder}$. 

The corresponding IPM $d_{\mathcal{F}_{\polyorder}}$ is indeed a metric on $\finitemomentspace{\polyorder}$ and characterizes $\polyorder$-Wasserstein
convergence. 

\begin{restatable}[Pseudo-Lipschitz metric]{prop}{plipdistchar}
\label{prop:plip-IPM-convergence}
For fixed $\polyorder\in[1,\infty)$, 
assume $P\in \finitemomentspace{\polyorder}$.
For any sequence $(P_{\seqidx})_{\seqidx\geq1}\subset\finitemomentspace{},$
the convergence $d_{\mathcal{F}_{\polyorder}}(P, P_{\seqidx})\to0$
is equivalent to $P_n \toL{\polyorder} P$. 
\end{restatable}

\begin{proofsketch}
We consider the set $\plipspace{\polyorder-1} (\supset \mathcal{F}_{\polyorder})$ of $1$-pseudo-Lipschitz functions and prove that the corresponding IPM $d_{\plipspace{\polyorder-1}}$ metrizes $\polyorder$-Wasserstein convergence. 
For one direction, 
we show that 
it is sufficiently large to enforce uniform integrability and to control weak convergence via the bounded-Lipschitz metric \citep[see, e.g.,][Section 11.2]{Dudley2002}. 
For the other, we show that $q$-Wasserstein convergence is sufficient to control the pseudo-Lipschitz class $\plipspace{\polyorder-1}$. %
The claim about $d_{\mathcal{F}_\polyorder}$ then follows from the equivalence between $d_{\mathcal{F}_{\polyorder}}$ and $d_{\plipspace{\polyorder-1}}$ in characterizing convergence to $P$. 
See 
\opt{aapnolink}{\nolink{\cref{subsec:proof-plip-IPM-metrization}} }%
\opt{aap,arxiv}{\cref{subsec:proof-plip-IPM-metrization} }%
for a complete proof. 
\end{proofsketch}

\subsubsection{Characterization of a solution to the Stein equation}\label{subsec:Character-Stein-eq-solution}

\citet{Erdogdu2018} studied a solution to the Stein equation defined
by the generator of an It\^{o} diffusion process and a pseudo-Lipschitz
function. Building on their result, we can obtain a solution with
known regularity properties as in \citet{gorham2016measuring}. 

The diffusion Stein operator was derived by \citet{gorham2016measuring}
from the infinitesimal generators of It\^{o} diffusion processes. Here,
let us assume that the diffusion matrix $m$ is given by $m(x)=a(x)+c(x),$
where $a(x)$ is called a covariance coefficient and expressed as
$a(x)=\sigma(x)\sigma(x)^{\top}$ by a diffusion coefficient $\sigma(x):\bbR^{\datdim}\to\bbR^{\datdim\times\datdim},$
and $c(x)$ a stream coefficient satisfying $c(x)=-c(x)^{\top}.$
Consider the diffusion defined by the stochastic differential equation
\begin{equation}
\dd Z_{t}^{x}=b(Z_{t}^{x})\dd t+\sigma(Z_{t}^{x})\dd B_{t}\ \text{with }Z_{0}^{x}=x,\label{eq:ito_diffusion}
\end{equation}
where $(B_{t})_{t\geq0}$ is a $\datdim$-dimensional Wiener process.
 In the following, we additionally assume $b$ and $\sigma$ are locally Lipschitz to guarantee a unique solution to \eqref{eq:ito_diffusion}~\citep[Theorem 3.11, p.300]{Ethier_1986}. 
Then, the diffusion Stein operator (\ref{eq:diff-Stein-op}) is related
to the generator $\mathcal{A}_{P}$ of the diffusion by 
\[
{\cal A}_{P}u(x)=\frac{1}{2}{\cal T}_{P}^{m}\nabla u(x),
\]
where $u:\bbR^{\datdim}\to\bbR.$ Consequently, if we obtain a solution
$u_{f}$ to the \emph{Poisson equation} 
\begin{equation}
{\cal A}_{P}u=f-\EE_{P}[f],\ f\in{\cal F}_{\polyorder},\label{eq:SteinEq-generator}
\end{equation}
we then have a solution to the Stein equation (\ref{eq:SteinEq})
by taking $v_{f}=2^{-1}\nabla u_{f}.$ 

\citet{Erdogdu2018} proved that the Poisson equation \eqref{eq:SteinEq-generator} 
admits a suitably regular solution $u_{f}$ under the following additional assumptions on the diffusion: 
\begin{assumption}[Dissipativity]
 \label{assu:dissipativity-diffusion}For $\alpha,\beta>0,$ the
diffusion (\ref{eq:ito_diffusion}) satisfies the dissipativity condition
\[
\mathcal{A}_{P}\Verts x_{2}^{2}\leq-\alpha\Verts x_{2}^{2}+\beta,
\]
where $\mathcal{A}_{P}\Verts x_{2}^{2}=2\inner{b(x)}x+\Verts{\sigma(x)}_{F}^{2}.$ 
\end{assumption}

\begin{assumption}[Wasserstein rate]
\label{assu:diff-wasserstein-rate}
For $s\in \{1, 2\}$, 
the diffusion $Z_t^x$ and $s$-Wasserstein distance (\ref{eq:Wass-dist-def}) admit an \emph{$L^s$-Wasserstein rate} function 
$\rho_{s}:[0,\infty)\to[0,\infty)$
satisfying %
\[
W_{s}\bigl(\mathrm{Law}(Z_{t}^{x}),\mathrm{Law}(Z_{t}^{y}))\bigr)\leq\rho_{s}(t)\Verts{x-y}_{2},\ \text{for all}\ x,y\in\bbR^{\datdim}\ \text{and}\ t\geq0. 
\]
Define the following relative $L^{1}$- and $L^{2}$-Wasserstein
rates: 
\[
\tilde{\rho}_{1}(t)=\log\frac{\rho_{2}(t)}{\rho_{1}(t)}\ \text{and}\ \tilde{\rho}_{2}(t)=\frac{\log\{\rho_{1}(t)/\rho_{1}(0)\}-\log\rho_{2}(t)}{\log\{\rho_{1}(t)/\rho_{1}(0)\}}.
\]
 Let $\polyorder\in[1,\infty).$ Assume further that $L^{1}$- and $L^{2}$-Wasserstein rates satisfy 
\[
\int_{0}^{\infty}\rho_{1}(t)\left\{ 1+\rho_{1}(t)^{1-1/(\polyorder_{m}+1)}\left(\frac{\tilde{\rho}_{\polyorder_{m}+1}(t)}{\tilde{\alpha}_{\polyorder_{m}+1}}\right)^{\polyorder-1}\right\} \dd t<\infty,
\]
 where 
\[
\tilde{\alpha}_{1}=\alpha,\ \text{and}\ \tilde{\alpha}_{2}=\inf_{t\geq0}\left[\alpha-4\lambda_{m}(\polyorder-1)\bigl(1\lor\tilde{\rho}_{2}(t)\bigr)\right]_{+},
\]
with $\alpha$ and $\lambda_m$ from \cref{assu:dissipativity-diffusion}. 
\end{assumption}

Under these assumptions, \citet[Theorem 3.2]{Erdogdu2018} show that
the solution $u_{f}$ is pseudo-Lipschitz of order $\polyorder-1$
and has $(\polyorder-1)$-growth continuous derivatives up to the third
order 
\opt{aapnolink}{\nolink{(see \cref{thm:finite-stein-factors-reference} in \cref{sec:Known-results} for a complete statement). }}%
\opt{aap,arxiv}{(see \cref{thm:finite-stein-factors-reference} in \cref{sec:Known-results} for a complete statement). }%
Based on this result, we obtain the following
characterization of the solution $v_{f}=2^{-1}\nabla u_{f}$. 
\begin{lem}[{Adaptation of \citealt[Theorem 3.2]{Erdogdu2018}}]
\label{lem:finite-stein-factors} Assume $f\in{\cal F}_{\polyorder}$.
Under Assumptions \ref{assu:poly-grow-coef}, \ref{assu:dissipativity-diffusion},
and \ref{assu:diff-wasserstein-rate} with $s=1,$ there exists a
solution $v_{f}$ to the Stein equation (\ref{eq:SteinEq}). Moreover,
there exist positive constants $\zeta_{1}(P),\zeta_{2}(P),\zeta_{3}(P)$
(called Stein factors) independent of $f$ such that for any $x\in\bbR^{\datdim},$
\begin{align*}
 & \Verts{v_{f}(x)}_{2}\leq\sqrt{\datdim}\zeta_{1}(P)(1+\Verts x_{2}^{\polyorder-1}),\\
 & \Verts{\nabla^{i}v_{f}(x)}_{\mathrm{op}}\leq \frac{\zeta_i(P)}{2}(1+\Verts x_{2}^{\polyorder-1})\ \text{for }i\in\{2, 3\}. 
\end{align*}
\end{lem}

Assumption \ref{assu:dissipativity-diffusion} is a special
case of Assumption \ref{assu:dissipativity}. 
Although \cref{assu:dissipativity} also implies \cref{assu:dissipativity-diffusion}, the constants $\alpha$ appearing in these assumptions
do not agree. We use the formulation of Assumption
\ref{assu:dissipativity-diffusion} to more closely mimic the presentation of \citet{Erdogdu2018}. Assumption \ref{assu:diff-wasserstein-rate} essentially requires
the diffusion to have an integrable $L^{1}$-Wasserstein rate,
while allowing the $L^{2}$-Wasserstein rate to grow (e.g., we may
have $\rho_{1}(t)=e^{-c_{1}t}$ and $\rho_{2}(t)=e^{c_{2}t}$ with $c_1,c_2>0$). If
$\polyorder_{m}=1,$ the assumption also requires $\alpha>4\lambda_{m}(\polyorder-1)$
to ensure $\tilde{\alpha}_{2}>0.$ This condition trivially holds
if $\polyorder=1;$ otherwise, it may be stronger than the condition
$\alpha>\lambda_{m}(\polyorder+2)$ assumed in Lemma \ref{lem:existence-qgrowth-approx-universality}.
There are two known sufficient conditions for establishing exponential
Wasserstein decay. The first is uniform dissipativity, which is a
simple (but more restrictive) condition leading to exponential $L^{1}$-
or $L^{2}$- exponential decay rates. 
\begin{prop}[{Wasserstein decay from uniform dissipativity, \citealt[Theorem 2.5]{Wang_2020}}]
\label{prop:wasserstein-decay-uniform-diss}
Fix $s \in [1, \infty)$. 
A diffusion with drift and diffusion coefficients $b$ and $\sigma$ has $L^{s}$-Wasserstein
rate $\rho_{s}(t)=e^{-rt/2},$ if for all $x,y\in\bbR^{\datdim},$ and some $r \in \bbR$, 
the following uniform dissipativity holds: 
\[
2\la b(x)-b(y),x-y\ra+\Verts{\sigma(x)-\sigma(y)}_{\mathrm{F}}^{2}+(s-2)\Verts{\sigma(x)-\sigma(y)}_{\mathrm{op}}^{2}\leq-r\Verts{x-y}_{2}^{2}.
\]
\end{prop}

The second and more general condition is distant dissipativity. Explicit
$L^{1}$-Wasserstein decay rates from distant dissipativity are obtained
by the following result of \citealt{gorham2016measuring}, which builds
upon the analyses of \citet{Eberle2015} and \citet{Wang_2020}. 
\begin{prop}[{Wasserstein decay from distant dissipativity, \citealt[Corollary 12]{gorham2016measuring}}]
\label{prop:dist-dissipativity}A diffusion with drift $b$ and diffusion
coefficient $\sigma$ is called distantly dissipative if
it satisfies, 
for some $U>0$, $R,L\geq0,$ and 
some $s\in(0, s_{\sigma})$ with $s_{\sigma} =(\sup_{x\in\bbR^{\datdim}}\Verts{\sigma^{-1}(x)}_{\mathrm{op}})^{-1}$, 
\begin{align*}
 & \frac{2\la b(x)-b(y),x-y\ra}{s^2\Verts{x-y}_{2}^{2}}+\frac{\Verts{\tilde{\sigma}(x)-\tilde{\sigma}(y)}_{\mathrm{F}}^{2}}{s^{2}\Verts{x-y}_{2}^{2}}-\frac{\Verts{\bigl(\tilde{\sigma}(x)-\tilde{\sigma}(y)\bigr)^{\top}(x-y)}_{2}^{2}}{s^{2}\Verts{x-y}_{2}^{4}}\\
 & \leq\begin{cases}
-U & \Verts{x-y}_{2}>R\\
L & \Verts{x-y}_{2}\leq R
\end{cases},
\end{align*}
where 
$\tilde{\sigma}$ denotes the truncated diffusion coefficient 
$\tilde{\sigma}(x)\coloneqq(\sigma(x)\sigma(x)^{\top}-s^{2}\idmat)^{1/2}.$
If the distant dissipativity holds,
then the diffusion has $L^1$-Wasserstein rate $\rho_{1}(t)=2e^{LR^{2}/8}e^{-rt/2}$
for some $r>0$ satisfying
\[
s^{2}r^{-1}\leq\begin{cases}
\frac{e-1}{2}R^{2}+e\sqrt{8U^{-1}}R+4U^{-1} & \text{if}\ LR^{2}\leq8,\\
8\sqrt{2\pi}R^{-1}L^{-1/2}(L^{-1}+U^{-1})\exp\Bigl(\frac{LR^{2}}{8}\Bigr)+32R^{-2}U^{-2} & \text{otherwise.}
\end{cases}
\]
\end{prop}

These two conditions also conveniently lead to the dissipativity condition
(Assumption \ref{assu:dissipativity}) defined above. 

\subsubsection{Explicit bound on the pseudo-Lipschitz metric}

The previous sections set the stage for our main result, presented in this section. Our first result below relates the IPM $d_{\mathcal{F}_{\polyorder}}$
and the KSD.

\begin{thm}[KSD bound on $d_{{\cal \mathcal{F}}_{\polyorder}}$; an informal statement\protect\footnote{A precise but less easily interpretable bound is available in \cref{thm:plip-ksd-bound-id-kernel-translation-invariant}.
}]
\label{thm:ksd-bound-df-abstract-informal}
Fix $\polyorder\in [1, \infty)$. 
Let $K=k\idmat$ with a scalar-valued kernel $k$. 
Following the notation in \cref{def:kernel-form}, define $k$ as follows: 
\[
k(x,y) = \weight_{\polyorder-1}(x)\left(\frac{\Phi(x-y)}{\weight_{\polyorder_m}(x)\weight_{\polyorder_m}(y)}+\bar{k}_{\mathrm{lin}}(x,y)\right)\weight_{\polyorder-1}(y), 
\]
where 
$\Phi\in\calC^{2}$ is a positive definite function with continuous generalized Fourier transform $\hat{\Phi}$, 
$\polyorder_m\in\{0, 1\}$ is the growth exponent of $\Verts{m(x)}_{\mathrm{op}}$ (\cref{assu:poly-grow-coef}). 
Then, under \cref{assu:poly-grow-coef,assu:dissipativity-diffusion},
\ref{assu:diff-wasserstein-rate}, for any $Q\in\finitemomentspace{}$  
and $\varepsilon,\rho>0,$ we have 
\begin{align*}
 & d_{{\cal F}_{\polyorder}}(P,Q)\\
 & \leq C_{P,\datdim}\left\{ \rho\cdot r_{\varepsilon}^{\polyorder+\polyorder_{m}}+r_{\varepsilon}^{\polyorder_{m}+\datdim/2}\left(\sqrt{\sup_{\Verts{\omega}_{2}\le2(\rho_{\varepsilon}\land\rho)^{-1}}\hat{\Phi}(\omega)^{-1}}\cdot\ksd{K_{m}}P(Q)\right)+\varepsilon\right\} ,
\end{align*}
where $r_{\varepsilon}$ and $\rho_{\varepsilon}$ are increasing
functions of $\varepsilon^{-1}$ and $C_{P,d}$ is an explicit constant specified in 
\opt{aapnolink}{\nolink{\cref{thm:plip-ksd-bound-id-kernel-translation-invariant}. }}%
\opt{aap,arxiv}{\cref{thm:plip-ksd-bound-id-kernel-translation-invariant}. }%
\end{thm}

\begin{proof}
We provide a proof sketch (see 
\opt{aapnolink}{\nolink{\cref{sec:KSD-bounds-plip}}}\opt{aap,arxiv}{\cref{sec:KSD-bounds-plip}}
for a complete proof). 
From \eqref{eq:stein-bound-df}, for any $r>0$ and $\tilde{v}_f\in \rkhs{K}$, 
we have 
\begin{align}
\begin{aligned}d_{\mathcal{F}_{\polyorder}}(P,Q) & \leq\sup_{f\in\mathcal{F}_{\polyorder}}\int_{\Verts x_{2}\leq r}\verts{{\cal T}_{P}^{m}v_{f}-{\cal T}_{P}^{m}\tilde{v}_{f}}\dd Q+\sup_{f\in\mathcal{F}_{\polyorder}}\int_{\Verts x_{2}>r}\left(\verts{{\cal T}_{P}^{m}v_{f}}+\verts{{\cal T}_{P}^{m}\tilde{v}_{f}}\right)\dd Q\\
 & \hphantom{\leadsto}\quad+\sup_{f\in\mathcal{F}_{\polyorder}}\Verts{\tilde{v}_{f}}_{\rkhs K}\cdot\ksd{K_{m}}P(Q).
\end{aligned}
\label{eq:plipbound-proof-sketch}
\end{align}
Each term is evaluated as follows. We construct the approximation $\tilde{v}_{f}$
by mollifying the solution $v_{f}$ to the Stein equation with a Fourier
multiplier. This makes the first term small, since ${\cal T}_{P}^{m}\tilde{v}_{f}$
can approximate the target ${\cal T}_{P}^{m}v_{f}$ uniformly inside
the ball of radius $r$ up to precision $\rho.$ Under the stated assumptions,
we have that both ${\cal T}_{P}^{m}v_{f}$ and ${\cal T}_{P}^{m}\tilde{v}_{f}$
are of $\polyorder$-growth, and there exists a function $v_{\varepsilon}$ in ${\cal T}_{P}^{m}({\cal H}_{K})$
that can approximate $\Verts x_{2}^{\polyorder}$ for sufficiently
large $r=r_{\varepsilon}$. Thus, the second term can be bounded by 
$\Verts{v_{\varepsilon}}_{\rkhs{K}}\ksd{K_m}{P}(Q) + \varepsilon$. 
The norms $\Verts{\tilde{v}_{f}}_{{\cal H}_{K}}$ and $\Verts{v_{\varepsilon}}_{{\cal H}_{K}}$ 
can be evaluated
using \citet[Theorem 10.21]{Wendland2004}, which expresses the norm
in terms of Fourier transforms of $\tilde{v}_{f}$ and $\Phi;$ this
norm is upper-bounded uniformly over $\mathcal{F}_{\polyorder}$, as all functions in ${\cal F}_{\polyorder}$
have uniformly bounded $\polyorder$- and $(\polyorder-1)$-growth coefficients for their
values and their derivatives, respectively. 
\end{proof}
\begin{rem}\label{rem:coercive}
\citet{GorMac2017} established a bound on the bounded-Lipschitz metric,
to which our proof similarly applies. In essence, their argument relies
on the existence of a coercive function to obtain a KSD bound on the
second term of the RHS in (\ref{eq:plipbound-proof-sketch}). Our
proof uses the less strict $\polyorder$-growth approximation property. %
More specifically, \citet{GorMac2017} uses
a distribution-dependent quantity called a \emph{tightness rate} to
determine the radius $r$ in (\ref{eq:plipbound-proof-sketch}), whereas
our choice of $r$ is distribution-independent. While our result is more broadly applicable, the assumption of a higher-order moment can yield an even tighter bound. 
In 
\opt{aapnolink}{\nolink{\cref{sec:Coercive-functions-and}, }}%
\opt{aap,arxiv}{\cref{sec:Coercive-functions-and}, }%
we generalize the coercivity strategy of \citet{GorMac2017} to obtain a similar quantity specifying the radius $r$.  This in turn can be used
to enforce $\polyorder$th moment uniform integrability. 
\end{rem}

Once $\Phi$ is specified, Theorem \ref{thm:ksd-bound-df-abstract-informal}
enables us to trade off the complexity of the approximation and its
precision. We carry out this task below; the following two results
present concrete bounds for the two growth cases of diffusion matrix $m$: $\polyorder_{m}\in\{0,1\}$. 
The proofs of these theorems are deferred to 
\opt{aapnolink}{\nolink{\cref{subsec:proof-Matern-KSD-linear,subsec:proof-Matern-KSD-quad}. }}%
\opt{aap,arxiv}{\cref{subsec:proof-Matern-KSD-linear,subsec:proof-Matern-KSD-quad}. }%
\begin{restatable}[Mat\'{e}rn KSD bound for the linear growth case]{thm}{MaternKSDLinear}
\label{thm:Matern-bound-linear}Fix  $\polyorder\in [1,\infty)$. 
Define kernel $K=k\idmat$ as in Theorem \ref{thm:ksd-bound-df-abstract-informal}
with $\Phi$ specified by the Mat\'{e}rn kernel with smoothness parameter
$\nu>1$ 
\begin{equation}
\Phi_{\mathrm{Mat}}(x)=\frac{2^{1-(\datdim/2+\nu)}}{\Gamma(\datdim/2+\nu)}\norm{\Sigma x}_{2}^{\nu}K_{-\nu}\bigl(\Verts{\Sigma x}_{2}\bigr), \label{eq:Matern}
\end{equation}
where 
$\Gamma$ is the Gamma function, 
$\Sigma$ a strictly positive definite matrix, 
and $K_{-\nu}$ the modified Bessel function of the second kind of order $-\nu$. 
Suppose Assumption \ref{assu:poly-grow-coef}
holds with $\polyorder_{m}=0.$ 
Suppose that \cref{assu:dissipativity-diffusion,assu:diff-wasserstein-rate} hold. 
Then, there exists an explicit  
constant $A_{P,\datdim}>0$ 
\opt{aapnolink}{(given in \nolink{\cref{subsec:proof-Matern-KSD-linear}}) }%
\opt{aap,arxiv}{(given in \cref{subsec:proof-Matern-KSD-linear}) }%
such that, for any $Q\in \finitemomentspace{}$,
\[
d_{{\cal F}_{\polyorder}}(P,Q)\leq A_{P,\datdim}\Bigl(1\lor\ksd{K_{m}}P(Q)^{^{\frac{t_0}{1+t_0}\lor\frac{\polyorder(\polyorder/3+1)}{1+t_0}}}\Bigr)\ksd{K_{m}}P(Q)^{\frac{1}{1+t_0}}
\]
where $t_0=(1+\polyorder/6)\datdim+\nu(\polyorder+5)/3.$ 

\end{restatable}
\begin{restatable}[Mat\'{e}rn KSD bound for the quadratic growth case]{thm}{MaternKSDQuad}
\label{thm:Matern-bound-quad}
Define kernel $K=k\idmat$ as in Theorem \ref{thm:ksd-bound-df-abstract-informal} using the Mat\'{e}rn kernel \eqref{eq:Matern}. 
Suppose Assumption \ref{assu:poly-grow-coef} holds with $\polyorder_{m}=1.$
Suppose that \cref{assu:dissipativity-diffusion,assu:diff-wasserstein-rate} hold. 
Then,  
there exists an explicit constant $B_{P,\datdim}>0$ 
\opt{aapnolink}{\nolink{(given in \cref{subsec:proof-Matern-KSD-quad}) }}%
\opt{aap,arxiv}{(given in \cref{subsec:proof-Matern-KSD-quad}) }%
such that, for any $Q\in \finitemomentspace{}$,
\begin{align*}
d_{{\cal F}_{\polyorder}}(P,Q)\leq B_{P,\datdim}\Bigl(1\lor\ksd{K_{m}}P(Q)^{\frac{t_1+1/3}{1+t_1}\lor\frac{(\polyorder^{2}+3\polyorder-2)}{3(1+t_1)}}\Bigr)\ksd{K_{m}}P(Q)^{\frac{2}{3}\frac{1}{1+t_1}}
\end{align*}
with 
\[
t_1=\frac{\{(\polyorder+1)(1+\datdim^{-1})+6\}\datdim 
}{6} 
+ \frac{\{(\polyorder+1)(1+\datdim^{-1})+5\}\nu + 1
}{3}.
\]
\end{restatable}
The dimension dependence in \cref{thm:Matern-bound-linear,thm:Matern-bound-quad} reflects the difficulty of  controlling the 
IPM $d_{\mathcal{F}_{\polyorder}}$ in high dimensions. 
To see this, consider the case
$\polyorder=1$ and an approximation of $P$ by the empirical measure $Q_{n}=\seqidx^{-1}\sum_{i=1}^{\seqidx}\delta_{X_{i}}$
with independent sample points $X_{i}\sim P$. 
\citet[Lemma 2.2]{Mackey_2016} established the following relationship between the IPM $d_{\mathcal{F}_{1}}$ and the $1$-Wasserstein distance (with $G$ a standard normal vector in $\bbR^\datdim$): 
\begin{equation}
\textstyle
   \min\big(\frac{1}{3} W_1(P, Q),  \frac{1}{27\sqrt{2}\Ex[\Verts{G}_2]^2} W_1(P, Q)^3\big) \leq 
        d_{\mathcal{F}_1}(P, Q).
    \label{eq:W1-bound}
\end{equation}
Moreover, \citet[Prop.~2.1]{Dudley_1969} showed that $\EE[W_1(P, Q_{\seqidx})] = \Omega(n^{-1/\datdim})$ whenever $P$ has a Lebesgue density. 
In contrast, the KSD is relatively easy to decrease: $\EE[\ksd{K_m}{P}(Q_{\seqidx})]=O(n^{-1/2})$ whenever $\EE[\ksd{K_m}{P}(Q_{\seqidx})^2]<\infty$~\citep[p.194]{Ser2009}. 
With this rate substituted, our \cref{thm:Matern-bound-linear,thm:Matern-bound-quad} upper
bounds imply roughly 
$n^{-6/(7d)}$ and $n^{-1/(2d)}$ rates of decay for $\EE[d_{\mathcal{F}_1}(P, Q_{\seqidx})]$. 
While this dimension dependence may not be optimal, it reasonably illustrates 
the relation between the two quantities and reflects the challenge of decreasing 
$d_{\mathcal{F}_{\polyorder}}$ in
high dimensions. 
We suspect that a tighter relationship can be deduced between the KSD and a less stringent IPM with a smaller test function class
than $\mathcal{F}_{\polyorder}$.
This would require a new  analysis
of the Stein equation and its solutions over the IPM's test function class, and we leave such refinements for future
work. 

One might also be interested in a direct comparison between the KSD and the $\polyorder$-Wasserstein distance $W_\polyorder$ (\ref{eq:Wass-dist-def}). 
For $q=1$, the relation \eqref{eq:W1-bound} allows us to translate our bounds on $d_{\mathcal{F}_{1}}$ into KSD bounds directly on $W_1$. 
Moreover, our next result, proved in 
\opt{aapnolink}{\nolink{\cref{subsec:proof-Wass-ksd-bound}, }}%
\opt{aap,arxiv}{\cref{subsec:proof-Wass-ksd-bound}, }%
provides direct control of $W_q$ under the assumptions of \cref{thm:Matern-bound-linear} or \cref{thm:Matern-bound-quad}.

\begin{thm}[KSD bounds on $W_{\polyorder}$] 
\label{thm:Wass-ksd-bound}
    Fix $\polyorder \in (1, \infty)$ and 
    assume 
    $P\in \finitemomentspace{\polyorder}$. 
    Instantiate the assumptions and notation of  \cref{thm:Matern-bound-linear} with $\polyorder_m = 0$ or \cref{thm:Matern-bound-quad} with $\polyorder_m = 1$. 
    Then 
    there exists an explicit constant $C_{P, \datdim}(\polyorder_m)$, given in 
    \opt{aapnolink}{\nolink{\cref{subsec:proof-Wass-ksd-bound}}, }%
    \opt{aap,arxiv}{\cref{subsec:proof-Wass-ksd-bound}, }%
    such that, for any $Q\in \finitemomentspace{}$,
    \begin{align*}
    W_{\polyorder}(P, Q)\leq C_{P,\datdim}(\polyorder_m) \left(1\lor \ksd{K_m}{P}(Q)^{1-\gamma(\polyorder_m)}\right)\ksd{K_m}{P}(Q)^{\gamma(\polyorder_m)}
    \end{align*}
    where 
    \begin{align*}
    \gamma(\polyorder_{m})=
        \begin{cases}
        \frac{1}{\polyorder(\polyorder+2)}\cdot 
        \frac{1}{1+t_{0}} & \text{if}\ \polyorder_{m}=0,\\
        \frac{1}{\polyorder\{(1+\datdim^{-1})\polyorder+2-\datdim^{-1}\}}\cdot \frac{2}{3}\frac{1}{1+t_{1}} & \text{if}\ \polyorder_{m}=1,
        \end{cases}
    \end{align*}
    where $t_0$ and $t_1$ are constants from \cref{thm:Matern-bound-linear,thm:Matern-bound-quad}. 
\end{thm}
Finally, we mention related work by \citet{Vayer2023}, which provides MMD bounds on $W_\polyorder$ with a similar dependence on dimension $\datdim$. 
Although the KSD is an instance of MMD, their results do not directly lead to a KSD bound, as they require the MMD to be defined by the RKHS of a translation-invariant kernel, which the Stein kernel does not satisfy. 

\section{Numerical illustration} \label{sec:Numerical-illustration}
We conduct numerical experiments to examine the theory developed above.
In the first two experiments, we investigate how kernel choice affects
the KSD's ability to detect non-convergence in moments, using simple
light-tailed and heavy-tailed target distributions. We then present
a cautionary case study, in which, even though the KSD will asymptotically
detect moment discrepancies by Theorem \ref{thm:ksd-equiv-to-Wass-conv},
a large sample size is needed for the KSD to detect discrepancies
arising from isolated modes.

\subsection{Detecting second moment non-convergence} \label{subsec:non-conv-variance}
\cref{prop:ksd-conv-control-failure} implies that the standard Langevin KSD with IMQ base kernel---often recommended for weak convergence control---can fail to detect non-convergence of second moments (corresponding to the case $\polyorder=2$). 
Here we provide a concrete example verifying this failure mode.

We construct a simple sequence that does not converge to its target in
the second moment. 
We use the $\datdim$-dimensional standard Gaussian distribution ${\cal N}(\bfzero,\idmat)$ as a target $P$ with $\datdim=5$. 
We then choose an approximating sequence $(Q_{\seqidx})_{\seqidx \geq 1}$
as follows:
\begin{equation}
    Q_{\seqidx}=\left(1-\frac{1}{\seqidx+1}\right)P_{\seqidx}+
    \frac{1}{\seqidx+1}\delta_{x_\seqidx}\ \text{ with }\  P_{\seqidx} = \frac{1}{\seqidx}\sum_{j=1}^{\seqidx}\delta_{X_j}, \label{eq:approx-seq-second}
\end{equation}
where $x_{\seqidx} = \sqrt{\seqidx+1}\cdot \bfone$ with $\bfone=(1,\dots,1)$, and $\{X_{1},\dots,X_{\seqidx}\}\iidsim P$. 
By \citet[Theorem 11.4.1]{Dudley2002}, the sequence $Q_{\seqidx}$ converges weakly to $P$ almost surely, 
but it has the following (almost sure) biased limit: 
\begin{align*}
    \lim_{\seqidx\to\infty}\EE_{Y\sim Q_{\seqidx}}[Y\otimes Y] 
 & =\EE_{X\sim P}[X\otimes X]+\bfone \otimes \bfone. 
\end{align*}

We investigate how the KSD between $P$
and $Q_{\seqidx}$ changes along the sequence. 
We examine the kernel choice recommended in \cref{cor:default-kernel} by setting $\polyorder_m=0$ and $\polyorder=2$ (referred to as \imqp). 
We compare this choice against using $k_{\imq}$ alone, which corresponds to the failure case illustrated in \cref{prop:ksd-conv-control-failure}. 

\opt{arxiv}{\begin{figure}[H]}%
\opt{aap,aapnolink}{\begin{figure}[H]}%
\centering\includegraphics[scale=0.45]{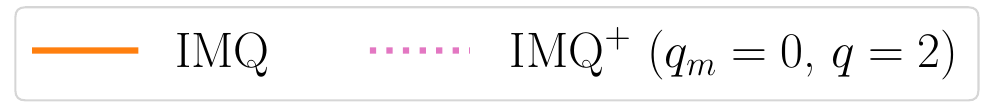}\\
\subfloat[Off-target sequence with non-converging second moments.]{\centering\includegraphics[width=0.45\columnwidth]{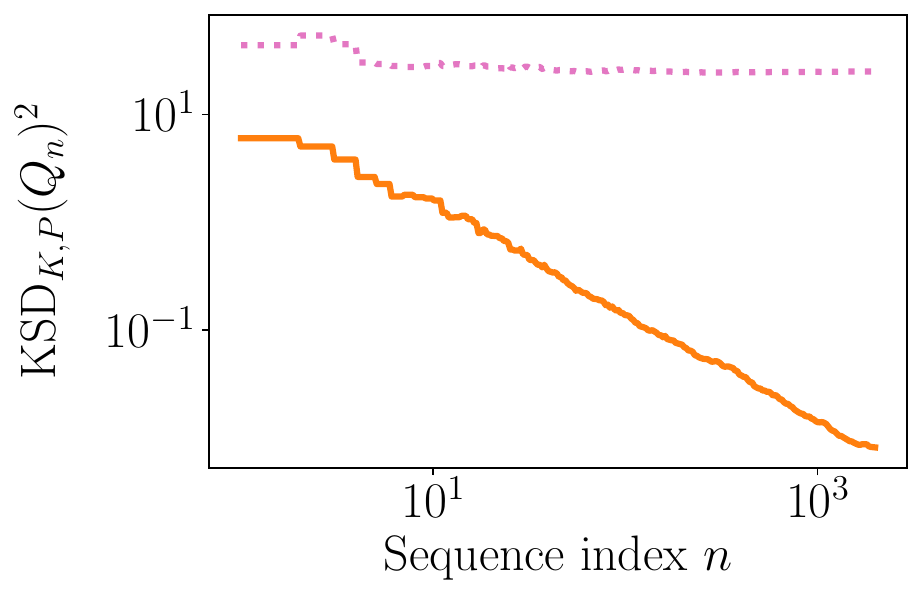}\label{fig:ksd-gauss-var-shift}}
\hfill{}\subfloat[On-target sequence formed by i.i.d samples from the target.]{\begin{centering}
\includegraphics[width=0.45\columnwidth]{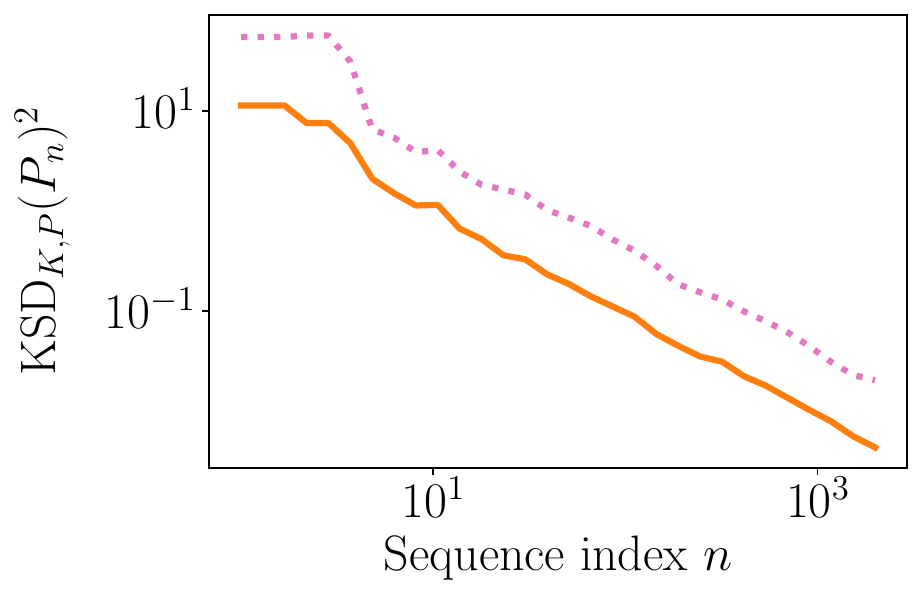}\label{fig:ksd-gauss-var-ontarget}
\par\end{centering}
}\caption{
\textbf{Importance of kernel choice for controlling second moments.}
(a) For a multivariate Gaussian target $P$, the standard IMQ KSD converges to $0$ for an off-target sequence $Q_{\seqidx}$ that does \textbf{not} converge to $P$ in $2$-Wasserstein distance.
Meanwhile, a KSD with the proposed \imqp kernel remains bounded away from $0$ as required by \cref{cor:default-kernel}. 
(b) For an on-target sequence $P_{\seqidx}$ drawn i.i.d.\ from $P$, both KSDs decay to zero, correctly detecting the $2$-Wasserstein convergence. 
See \cref{subsec:non-conv-variance} for more details. 
}
\end{figure}

\cref{fig:ksd-gauss-var-shift} shows the KSD's transition along
the approximating sequence for our two kernel choices. 
We can see that using the IMQ kernel alone, the KSD decays to zero even though the second moment is not converging to $P$. 
Fortunately, as proved in \cref{thm:ksd-equiv-to-Wass-conv}, our proposed base kernel does not suffer from this failing. 
Rather, its induced KSD enforces second-moment uniform integrability and hence remains bounded away from zero for the off-target sequence $Q_{\seqidx}$.
\cref{fig:ksd-gauss-var-ontarget} verifies that both KSDs converge to zero 
for an on-target sample sequence $P_{\seqidx}$ drawn i.i.d. from the target (see Eq.~\ref{eq:approx-seq-second}).   
This result concords with \cref{prop:Wass-conv-implies-KSD}, which guarantees that the KSD converges to zero whenever the approximating sequence converges to $P$ in the 2-Wasserstein sense.

\subsection{Detecting first moment non-convergence} \label{subsec:heavy-tailed}
Next, we turn to a heavy-tailed target distribution. 
Past work~\citep{GorMac2017, Barp2024} has shown that the IMQ Langevin KSD can fail to detect weak convergence for this target and hence can also fail to detect convergence in mean.
We exemplify this failure mode using a specific example and show how we can surmount it with a diffusion KSD. 

As our target $P,$ we use the standard multivariate \textit{t}-distribution
with the degrees of freedom $\nu>1$, which is defined by the density 
\[
p(x)=\frac{\Gamma\left(\frac{\nu+\datdim}{2}\right)}{\Gamma\left(\frac{\nu}{2}\right)\nu^{\frac{\datdim}{2}}\pi^{\frac{\datdim}{2}}}\left(1+\frac{\Verts x_{2}^{2}}{\nu}\right)^{-\frac{\nu+\datdim}{2}}; 
\]
we set $\datdim=5$ and $\nu=6$ in our experiment. 
We consider the following approximation sequence: 
\[
    Q_{\seqidx}=\left(1-\frac{1}{\seqidx+1}\right)P_{\seqidx}+\frac{1}{\seqidx+1}\delta_{x_{\seqidx}}, 
\]
where $x_{\seqidx} = (\seqidx+1)\cdot \bfone$. As in the previous section, this sequence has a limiting mean with bias $\bfone$. 

Here, we compare the IMQ Lanvegin KSD (LKSD) against a diffusion KSD (DKSD). 
As in the previous section, LKSD uses $K=k_{\imq}\idmat$ as a base kernel. 
To define the DKSD, we consider a diffusion matrix $m(x)=(1+\nu^{-1}\Verts x_{2}^{2})\idmat$. 
Recall that the DKSD may be seen as the Langevin KSD (Eq.~\ref{eq:ksd-def}) defined by the tilted kernel $K(x,y) = m(x) K_0(x,y) m(y)^{\top}$, where we use $K_0=k_{\imqp}\idmat$, the base kernel recommended in \cref{cor:default-kernel} with $\polyorder_m=1$ and $\polyorder=1$ (see the last paragraph of \cref{sec:Diffusion-kernel-Stein}). 

\begin{figure}[t]
\centering{}\includegraphics[scale=0.4]{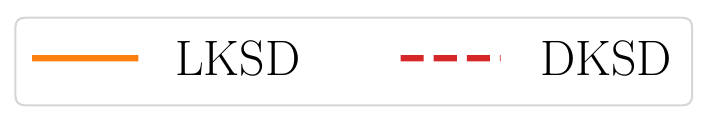}\bigskip{}
\\
\subfloat[Off-target sequence with non-converging means.]
{\begin{centering}
\includegraphics[width=0.45\columnwidth]{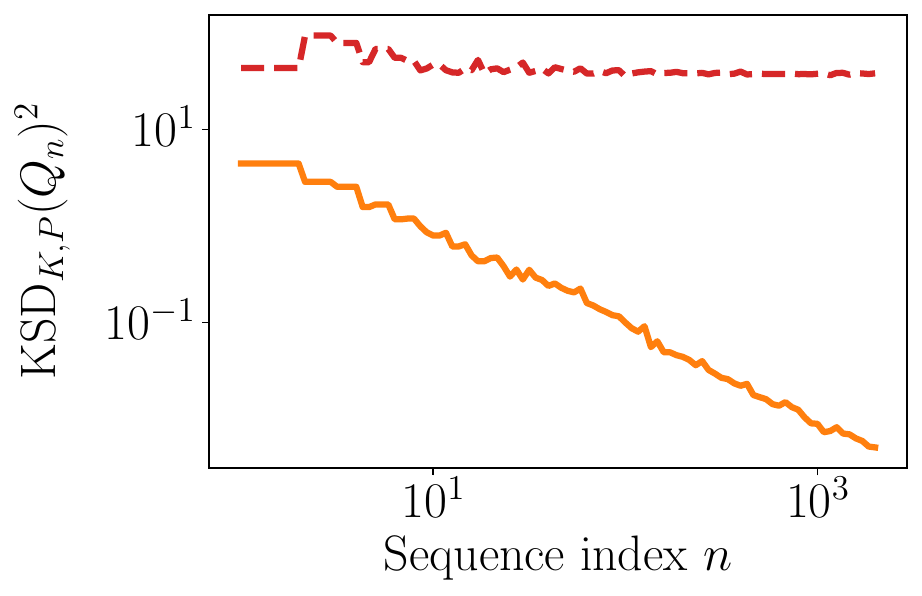}
\par\end{centering}
\label{fig:ksd-tdist-mean-shift}}
\hfill{}\subfloat[On-target sequence formed by i.i.d samples from the target.]{\centering{}\includegraphics[width=0.45\columnwidth]{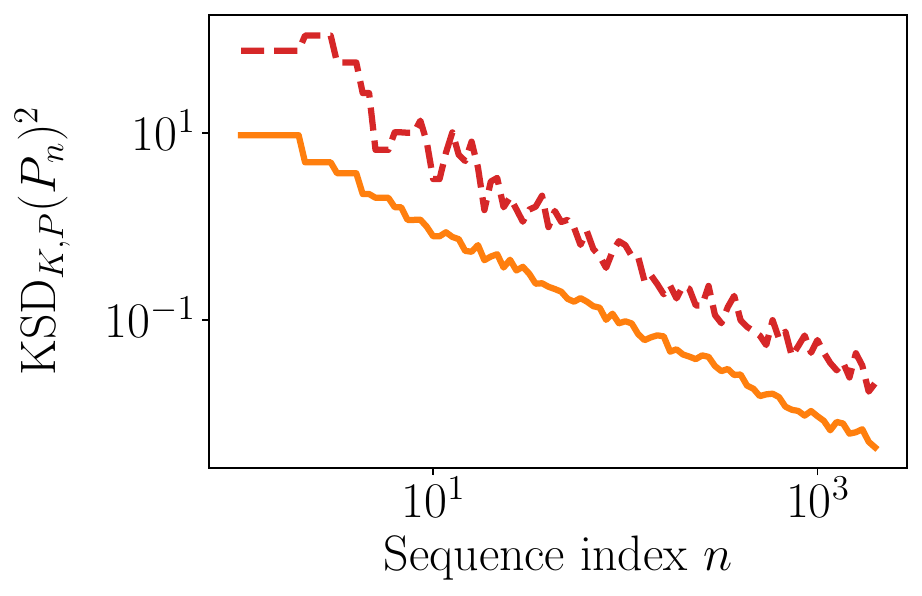}\label{fig:ksd-studentt-on-target}}
\caption{
\textbf{Importance of kernel choice for controlling first moments.}
(a) For a multivariate Student's \textit{t} target $P$, the standard IMQ Langevin KSD (LKSD) converges to $0$ for an off-target sequence $Q_{\seqidx}$ that does \textbf{not} converge to $P$ in $1$-Wasserstein distance.
Meanwhile, a diffusion KSD (DKSD) with an appropriate diffusion matrix $m$ and an \imqp base kernel remains bounded away from $0$ as required by \cref{thm:ksd-equiv-to-Wass-conv}. 
(b) For an on-target sequence $P_{\seqidx}$ drawn i.i.d.\ from $P$, both KSDs decay to zero, correctly detecting the $1$-Wasserstein convergence. 
See \cref{subsec:heavy-tailed} for more details. 
}
\label{fig:mean-shift-studentt}
\end{figure}

Figure \ref{fig:ksd-tdist-mean-shift} plots the two KSDs along the off-target sequence $Q_{\seqidx}$. 
We can see that the LKSD decays to zero. 
As \cref{prop:ksd-conv-control-failure} suggests, 
since the Student's \textit{t}-distribution has a decaying score function $\nabla\log p$, 
the corresponding Stein RKHS cannot have a growing function with a bounded base kernel, and thus is unable to enforce first-moment uniform integrability. 
In contrast, the DKSD sucessfuly detects the mean non-convergence. 
Indeed, our choice of the diffusion matrix $m$ satisfies our assumptions and enables us to cancel the decay of the score function: 
as we show in 
\opt{aapnolink}{\nolink{\cref{lem:dissipative-student}, }}%
\opt{aap,arxiv}{\cref{lem:dissipative-student}, }%
if $\nu>2,$ the \textit{t}-distribution is dissipative (Assumption
\ref{assu:dissipativity-diffusion}) %
and $\alpha=(1-2\nu^{-1});$ in fact, it satisfies the uniform dissipativity
condition in \cref{prop:wasserstein-decay-uniform-diss}.\footnote{
According to Theorem \ref{thm:ksd-equiv-to-Wass-conv},
the allowed choice of the growth $\polyorder$ of test functions is $\polyorder<\nu-4$. 
In fact, we may take $\polyorder<\nu+\theta-2$ with $0<\theta<1,$
since $P$ has moments up to order $\nu+\theta-2;$ see Remark \ref{rem:allowed-growth-testfunction}.
}
This result confirms our theory for the DKSD, as proved in \cref{thm:ksd-equiv-to-Wass-conv}. 
Finally, as in the previous section, \cref{fig:ksd-studentt-on-target} shows that both KSDs vanish for an on-target sample sequence $P_{\seqidx}$, demonstrating that the KSD detects the $1$-Wasserstein convergence, as \cref{prop:Wass-conv-implies-KSD} implies. 

\subsection{Cautionary tale: mixtures with isolated components} \label{subsec:Failure-mode:-distribution}

Our final experiment concerns the following distributions: 
\begin{equation}
P=\frac{1}{2}{\cal N}(\mu_{1},\idmat)+\frac{1}{2}{\cal N}(\mu_{2},\idmat),\ Q_{\pi}=\pi{\cal N}(\mu_{1},\idmat)+(1-\pi){\cal N}(\mu_{2},\idmat), \label{eq:Gaussian-mixtures}
\end{equation}
where $0\leq\pi\leq1$, $\mu_1, \mu_2 \in \bbR^{\datdim}$, and  
${\cal N}(\mu,\idmat)$ denotes a Gaussian distribution with mean $\mu$ and diagonal covariance $\idmat$. 
The two distributions $P$ and $Q_{\pi}$ share the mixture components, with the difference arising from the variation in the mixing weight $\pi$. 
Varying $\pi$ shifts the mean of $Q_{\pi}$, and one might hope that the KSD will detect this discrepancy. 
The target $P$ is indeed supported by our theory,
since Gaussian mixtures are known to be distantly dissipative \citep[Example 3]{gorham2016measuring}.
When the distance $\Verts{\mu_{1}-\mu_{2}}_{2}$ between
the two modes is large, however, the KSD is unable to capture the discrepancy
of the mixture ratio $\pi$ unless a very large sample size is observed.
Figure \ref{fig:mixture-sample-size} illustrates this point, using the kernel of \cref{cor:default-kernel}, where $\polyorder_m=0$ and $\polyorder=1$. 
This issue has been noted by \citet[Section 5.1]{gorham2016measuring}
and \citet{Wenliang2020}. 

\begin{figure}[t]
\centering{}\includegraphics[width=\columnwidth]{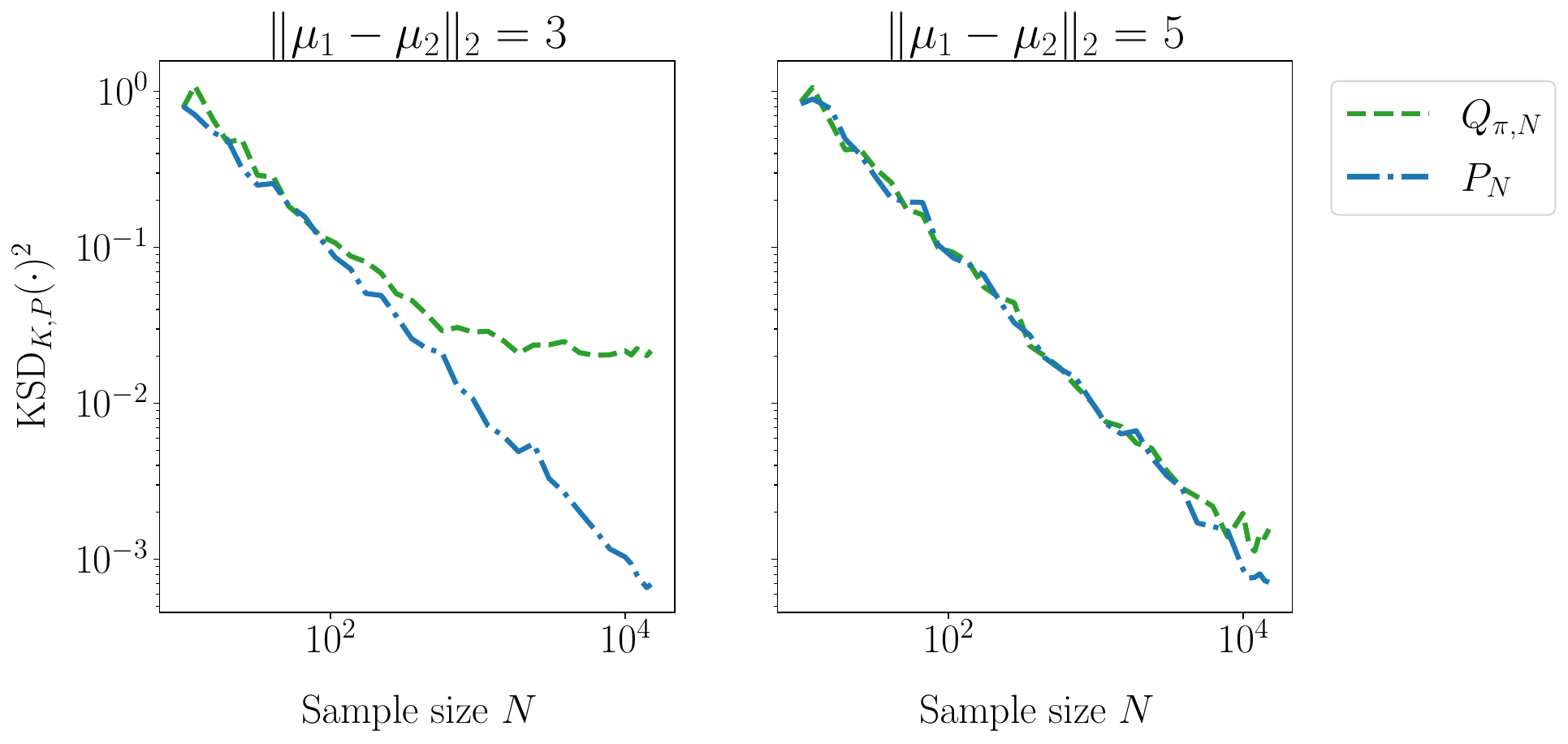}
\caption{
\textbf{Impact of isolated modes on non-convergence detection.}
For a two-component multivariate Gaussian mixture $P$ with mode separation 
$\protect\Verts{\mu_{1}-\mu_{2}}_{2}$, we compare
the \imqp KSD of an on-target sequence $P_N$ drawn i.i.d.\ from $P$ and an 
off-target sequence $Q_{\pi,N}$ drawn i.i.d.\ from a distribution that undersamples mode $1$ (with $\pi=1/10$). 
The larger the mode separation $\protect\Verts{\mu_{1}-\mu_{2}}_{2},$
the larger the sample size required to distinguish the on- and off-target sequences.
See \cref{subsec:Failure-mode:-distribution} for more details.}
\label{fig:mixture-sample-size}
\end{figure}

The insensitivity can be in part attributed to the RKHS's lack of adaptivity to the sample. 
We first examine this claim using fixed kernels, as in the experiment from \cref{fig:mixture-sample-size}. 
To this end, we consider the Langevin KSD with the base kernel of the form of \cref{def:kernel-form}, where $L=L^{(2)}$ in Theorem
\ref{thm:ksd-equiv-to-Wass-conv}, $\polyorder_m=0$, and $\polyorder=1$; 
this choice ensures that the Stein RKHS has a function of linear growth. 
We study two choices of the translation invariant kernel $\ell$~: the IMQ
kernel $k_{\imq}(x,x')=\bigl(1+\Verts{x-x'}_{2}^{2}/\sigma^{2}\bigr)^{-1/2}$
and the following Mat\'{e}rn ($\nu = 3/2$) kernel 
\[
k_{\mathrm{Mat}}(x,x')=\left(1+\frac{\sqrt{3}\Verts{x-x'}_{2}}{\sigma}\right)\exp\left(-\frac{\sqrt{3}}{\sigma}\Verts{x-x'}_{2}\right). 
\]
As in \cref{cor:default-kernel}, we term these two kernel choices \imqp and \matp, respectively. 
Here, to capture a sensible scale, we set the bandwidth $\sigma$ equal to the median (Euclidean) distance calculated from samples drawn from the target distribution $P$.

In this setting, we use an independent sample $\{X_{i}\}_{i=1}^{N}$ from $Q_{\pi}$ to form a sequence of empirical distributions $Q_{\pi,N}=N^{-1}\sum_{i\leq N}\delta_{X_{i}}$
and exactly compute the KSD between $P$ and $Q_{\pi,N}.$ In
the following, we set $\mu_{1}=-30\cdot\bfone,$ $\mu_{2}=-10\cdot\bfone$
and $\datdim=5.$ With $N$ fixed at $500,$ we vary the mixture ratio $\pi$
from $0$ to $1/2$ using a regular grid of size $30.$ We noticed
that the KSD's trajectory has different trends depending on the drawn
samples. We therefore repeat this procedure 100 times and provide
a summary. 

\cref{fig:ksd-mixture-noopt} presents density estimates of the distribution of KSD values computed from different sample draws, plotted against the mixture ratio. 
We observe that for both kernels, their KSDs do not
change along the sequence. 

\opt{aap,aapnolink}{
    \begin{figure}[t]
    \centering%
    \subfloat[KSDs with fixed length scale base kernels suffer from an insensitivity to changes in the Gaussian mixture proportion $\pi$. \label{fig:ksd-mixture-noopt}]
    {
        \centering\includegraphics[width=0.9\columnwidth]{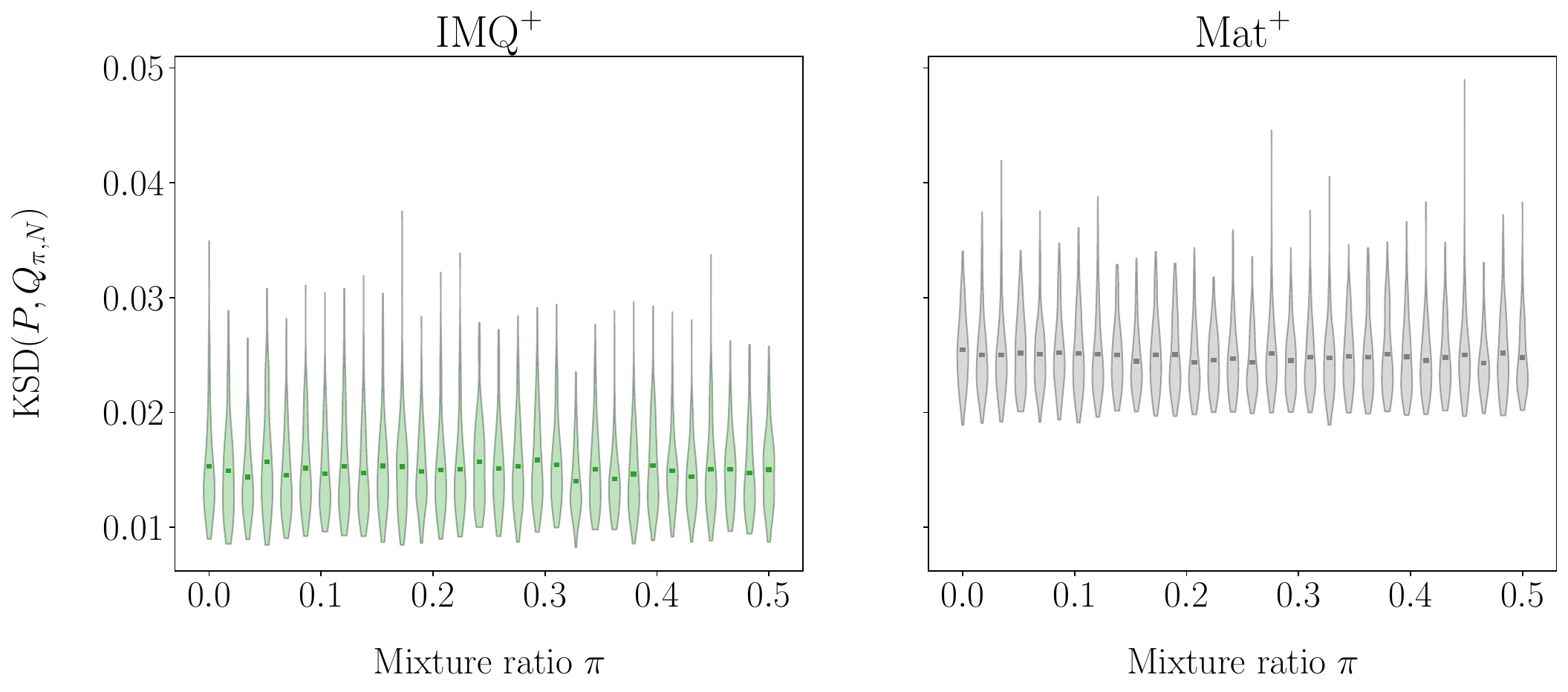}
    }\\ 
    \subfloat[
        KSDs with optimized length scale enjoy improved sensitivity to changes in the Gaussian mixture proportion $\pi$.\label{fig:ksd-mixture-opt}
    ]{
        \includegraphics[width=0.9\columnwidth]{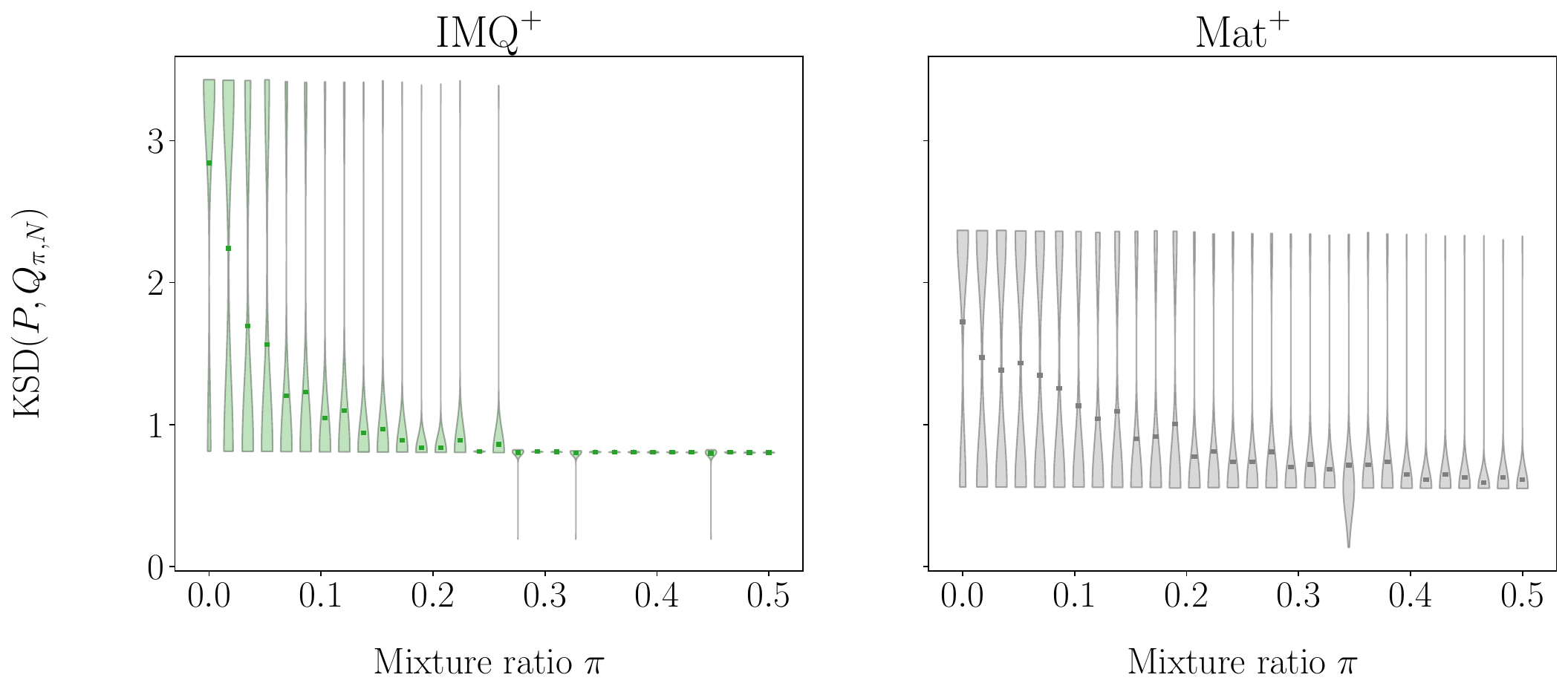}
    }\\
    \caption{\textbf{Impact of kernel length scale on detecting undersampled mixture components.}
        To assess the impact of adaptively selecting the length scale of the KSD base kernel, we plot both the distribution of KSD values and their mean value (as a bold dot) over independent samples of $N=500$ points drawn i.i.d.\ from the Gaussian mixture $Q_\pi$ in \cref{eq:Gaussian-mixtures} when the target $P=Q_{1/2}$. See \cref{subsec:Failure-mode:-distribution} for more details.
    }
    \end{figure}
}%

\opt{arxiv}{
    \begin{figure}[H]
    \centering\includegraphics[scale=0.4]{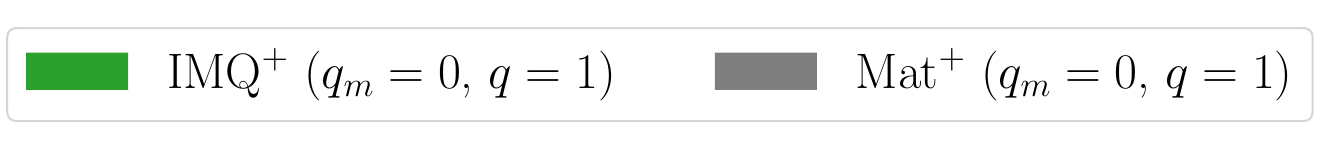}\\
    \subfloat[KSDs with fixed length scale base kernels suffer from an insensitivity to changes in the Gaussian mixture proportion $\pi$. \label{fig:ksd-mixture-noopt}]
    {
        \centering\includegraphics[width=0.9\columnwidth]{img/ksd_gauss_mixture_imq_vs_matern}
    }\\ 
    \subfloat[
        KSDs with optimized length scale enjoy improved sensitivity to changes in the Gaussian mixture proportion $\pi$.\label{fig:ksd-mixture-opt}
    ]{
        \includegraphics[width=0.9\columnwidth]{img/ksd_gauss_mixture_opt_imq_vs_matern}
    }\\
    \caption{\textbf{Impact of kernel length scale on detecting undersampled mixture components.}
        To assess the impact of adaptively selecting the length scale of the KSD base kernel, we plot both the distribution of KSD values and their mean value (as a bold dot) over independent samples of $N=500$ points drawn i.i.d.\ from the Gaussian mixture $Q_\pi$ in \cref{eq:Gaussian-mixtures} when the target $P=Q_{1/2}$. See \cref{subsec:Failure-mode:-distribution} for more details.
    }
    \end{figure}
}%
One can hope to improve this sensitivity by adjusting the length scale $\sigma$ adaptively. 
Here, following the approaches
in statistical hypothesis testing \citep{Gretton2012,Sutherland2016,Jitkrittum2016,JitXuSzaFukGre2017},
we consider choosing a bandwidth by optimizing the
objective $\ksd KP(Q_{\pi,N})^{2}/\sqrt{v_{H_{1}}},$ where $v_{H_{1}}=\var_{X\sim Q_{\pi,N}}\bigl[\EE_{X'\sim Q_{\pi,N}}[k_{p,K}(X,X')]\bigr].$
This objective is a known proxy for the power of the goodness-of-fit
test based on a KSD test statistic \citep{ChwStrGre2016,LiuLeeJor2016} and can be computed exactly.
The optimization procedure may be interpreted as seeking a bandwidth
that allows us to conclude $P\neq Q_{\pi,N}$ using as few sample
points as possible, without using the whole of $Q_{\pi,N}.$ We choose
a bandwidth that maximizes the objective from a regular grid between
$10^{-3}$ and $10^{3}$ (in the logarithmic scale with base $10$) with the grid size $20.$ 

Figure \ref{fig:ksd-mixture-opt} shows the result. Both kernels successfully manifest
decreasing trends. The IMQ kernel tends to take larger values near
$\pi=0,$ but stays at the same value past the point $\pi=0.3.$ In
contrast, the Mat\'{e}rn kernel keeps the decreasing trend and captures
the discrepancy. The Mat\'{e}rn kernel may therefore be thought of as
more stable and as inducing a more stringent discrepancy measure. 

\section{Discussion and conclusion}\label{sec:Conclusion}
In this paper, we presented necessary and sufficient conditions for the diffusion kernel Stein discrepancy to control both weak convergence and convergence of moments. 
These theoretical results provide guidance for choosing a reproducing kernel when applying the diffusion kernel Stein discrepancy in practice.

Nevertheless, several challenges remain for practical application. 
One limitation is the need to validate the assumptions on the target distribution, requiring analytical work. 
For example, dissipativity (\cref{assu:dissipativity}) may be challenging to establish if the target's analytical expression is not known (e.g., when the target is specified through a computer program). 
Developing a discrepancy measure that is adaptive to the target is desirable and left for future work. 

Another practical limitation concerns computational complexity. 
As shown in \eqref{eq:ksdexact}, computing the diffusion KSD requires evaluating pairwise interactions between sample points, leading to a computational cost that scales quadratically with the sample size. 
While the temporal costs can be substantially diminished by leveraging modern CPU and GPU parallelism, the quadratic complexity can still be onerous for very large sample sizes. 
In such circumstances, one can accurately approximate the KSD in near-linear time using techniques like Compress \citep{shetty2022distribution} and kernel thinning \citep{dwivedi2022generalized,dwivedi2024kernel} or follow the recipes of \citet{HugMac2018} to design a random feature Stein discrepancy that inherits the convergence control properties of a target KSD while being computable in near-linear time. 
Finally, our theoretical framework could be extended to settings that replace the (potentially expensive) exact evaluation of the log density gradient with a cheaper stochastic estimator, as in \citet{Gorham2020}.

%% file: supp.tex
\section{Proof of 
    Proposition~\lowercase{\ref{prop:ksd-conv-control-failure}}:\\ \ksdconvcontrolfailurename
}
\label{sec:q-wass-conv-necessary}
This appendix presents a proof of \cref{prop:ksd-conv-control-failure}. 
We provide a result that holds for a more general maximum mean discrepancy~\citep[MMD,][]{GreBorRasSchSmo2012}. 
The claim for the KSD follows once we address the MMD case, since the KSD is an specific instance of the MMD (see \cref{sec:Diffusion-kernel-Stein}). 

Recall that for a scalar-valued kernel, the MMD is defined as an IPM (see Eq. \ref{eq:ipm-definition} for the definition): $\mmd{k}(P, Q) = d_{\mathcal{B}_k} (P, Q)$, where 
$\mathcal{B}_k$ is the unit ball $ \{h\in \rkhs{k}: \Verts{h}_{\rkhs{k}}\leq 1 \}$ of the RKHS $\rkhs{k}$. 
If $\rkhs{k} \subset L^1(P) \cap L^1(Q)$, we have 
\[
    \mmd{k}(P,Q)^2 =  \int \int k(x,y) \dd \bigl(P-Q\bigr)(x) \dd \bigl(P-Q\bigr)(y). 
\]
For a finite signed measure $\mu$ such that $\rkhs{k}\subset L^1(\verts{\mu})$, we also denote 
\[
    \Verts{\mu}_{k} \coloneqq \sqrt{ \int \int k(x,y) \dd \mu(x) \dd \mu(y) }
\]
so that $\mmd{k}(P, Q) = \Verts{P-Q}_k$. 

Here, we generalize our setup to a metric space $(\calX, d)$, and consider $\polyorder$-Wasserstein convergence accordingly: 
\begin{defn}[$\polyorder$-Wasserstein convergence in metric spaces]
Let $({\cal X},d)$ be a separable metric space. We define the set
of Borel probability measures with finite $\polyorder$th moments
by 
\[
{\cal P}_{\polyorder}({\cal X})=\left\{ \mu:\int d(x,\bar{x})^{\polyorder}\dd\mu(x)<\infty\ \text{for some}\ \bar{x}\in{\cal X}\right\},
\]
and denote $\finitemomentspace{0}(\mathcal{X})$ by $\finitemomentspace{}(\mathcal{X})$.
A function $f:{\cal X}\to\bbR$ is said to be of $\polyorder$-growth
if there exists a constant $C>0$ such that 
\[
    \verts{f(x)}\leq C(1+d(x,x_{0})^{\polyorder})
\]
for some (and then any) $x_{0}\in{\cal X}.$ A sequence $(P_{\seqidx})_{\seqidx\geq1}\subset\finitemomentspace{}({\cal X})$
is said to converge to $P$ in $\polyorder$-Wasserstein convergence
if
\[
\int \verts{f}\dd P_{\seqidx}\to\int \verts{f}\dd P
\]
for any continuous function $f$ of $\polyorder$-growth. 
\end{defn}

We are ready to provide the result. 
The following proposition demonstrates that the MMD cannot detect moment non-convergence when its RKHS lacks $\polyorder$-growth functions. 
\begin{prop}[$\polyorder$-growth is necessary for enforcing $\polyorder$th moment convergence]
Let $(\calX, d)$ be an unbounded separable metric space. 
Let $\polyorder\in (0,\infty)$ and $P \in \finitemomentspace{\polyorder}(\calX)$. 
Let $k:\calX\times\calX \to \bbR$ be a kernel satisfying the following growth condition: 
there exist some $C, R>0$ and some $x_0$ such that 
\[
    \sqrt{k(x, x)} \leq C \{ 1+ d(x_0, x)^{\polyorder'} \}
\]
for all $x\in \calX$ with $d(x, x_0) > R$, where $0 < \polyorder' <\polyorder$. 
Then, vanishing $\mmd{k}(P, \cdot)$ does not imply $\polyorder$-Wasserstein convergence to $P$. 
\end{prop}
Before presenting a proof, we briefly remark that 
\cref{prop:ksd-conv-control-failure} holds as a corollary of this proposition. 
Indeed, under \cref{assu:poly-grow-coef}, from the expression of the Stein kernel (Eq. \ref{eq:Stein-kernel-expanded}), we have $\sqrt{k_{p,K}(x,x)} \leq C (1+\Verts{x}_2)$ for some constant $C>0$. 
As a result, for $\polyorder > 1$, the KSD cannot control $\polyorder$-Wasserstein convergence. 

\begin{proof}
We construct a sequence in $\finitemomentspace{\polyorder}(\calX)$ that converges to $P$ in MMD but not in $\polyorder$-Wasserstein convergence. 
Let $n\geq 1$ be an integer such that there exists a point $x_{\seqidx} \in \calX$ with 
$  (\seqidx-1)^{1/\polyorder} \leq d(x_0, x_{\seqidx}) \leq  \seqidx^{1/\polyorder}$. 
There are infinitely many such $n$ as $\calX$ is bounded otherwise. 
Let $\{\seqidx(j)\}_{j\geq1}$ be a sequence of such integers with $x_j$ the corresponding point. 
Then, define a sequence 
\[ 
    Q_{j} = \left( 1-\frac{1}{\seqidx(j)+1}\right) P + \frac{1}{\seqidx(j)+1}\delta_{x_{j}}. 
\]
For $j\geq1$ with $d(x_j, x_0) > R$, we have 
\begin{align*}
    \mmd{k}(P, Q_{j}) 
    & \leq  \frac{1}{\seqidx(j)+1}\Verts{P}_k + \frac{1}{\seqidx(j)+1} \Verts{\delta_{x_j}}_k\\
    & = \frac{1}{\seqidx(j)+1}\Verts{P}_k + \frac{1}{\seqidx(j)+1} \sqrt{k(x_j, x_j)}\\
    & \leq \frac{1}{\seqidx(j)+1}\Verts{P}_k + \frac{C}{\seqidx(j)+1} \left( 1+n(j)^{\polyorder'/\polyorder} \right) 
    \to 0\ (j \to \infty). 
\end{align*}
Thus, $Q_{j}$ converges to $P$ in MMD. 
On the other hand, for any continuous bounded function $f$, 
\begin{align*}
    \int f \dd Q_j 
    &=  \left( 1- \frac{1}{\seqidx(j)+1}\right )\int f \dd P + \underbrace{\frac{1}{n(j)+1} f(x_j)}_{\to 0} \\
    &\to \int f \dd P\ (j\to \infty), 
\end{align*}
where we have used the boundedness of $f$ to diminish the second term in the limit. Hence, $Q_{j}$ weakly converges to $P$. 
Moreover, we have 
\begin{align*}
    \int d(x_0, \cdot)^{\polyorder} \dd Q_{j} 
    & = \left( 1- \frac{1}{\seqidx(j)+1}\right )\int d(x_0, \cdot)^{\polyorder} \dd P + \frac{1}{\seqidx(j)+1} d(x_0, x_j)^\polyorder \\
    & \geq \left( 1- \frac{1}{\seqidx(j)+1}\right)\int d(x_0, \cdot)^{\polyorder} \dd P + \frac{ \seqidx(j) - 1}{\seqidx(j)+1}. 
\end{align*}
Taking the limit on both sides implies that $Q_j$ does not converge to $P$ in the $\polyorder$th moment. 
Thus, we have constructed the desired sequence, and the proof is complete. 
\end{proof}

\section{Approximating  $\protect\polyorder$-growth} \label{sec:results-q-growth-approximation}
In this appendix, it is shown that we can construct a Stein RKHS $\mathcal{T}_P^m (\rkhs{K})$ that approximates $\polyorder$-growth. 
Our strategy is as follows. 
First, we construct a vector-valued function $v_{\varepsilon}$ such that $\mathcal{T}_P^m v_{\varepsilon}$ approximates $\polyorder$-growth up to precition $\varepsilon$, i.e., $\mathcal{T}_P^m v_{\varepsilon}$ satisfies 
\[ \mathcal{T}_P^{m} v_{\varepsilon}(x) \geq \Verts{x}_2^{\polyorder} 1\{\Verts{x}_2 > r_{\varepsilon}\} - \varepsilon\]
for some ${r_{\varepsilon}}>0$ 
--- \cref{lem:norm-indicator-dominator} serves this purpose. 
    Note that $v_{\epsilon}$ may not belong to the RKHS. 
Thus, we approximate it using a function of the RKHS; we provide two such results in \cref{subsec:Proof-of-Lemma-existence-qgrowth-approx}. 

\subsection{Preparatory results}
\begin{lem}[Existence of $\polyorder$-growth approximation]
\label{lem:norm-indicator-dominator}For $r,\delta>0,$ define a $\calC^{2}$-function
$1_{r,\delta}(x)=h_{r,\delta}(\Verts x_{2})$ with 
\[
h_{r,\delta}(t)=\frac{f(r+\delta-t)}{f(r+\delta-t)+f(t-r)}
\]
and $f(t)=t^{3}1_{[0,\infty)}(t)$, $t\in \bbR$. 
Let $\polyorder\in [0,\infty).$ 
Assume the dissipativity condition in Assumption \ref{assu:dissipativity} with
$\alpha,\beta_{1},\beta_{0}.$ Suppose Assumption \ref{assu:poly-grow-coef}
holds for some constants $c_{m}>0$ and $\polyorder_{m}\in\{0,1\}$;
if $\polyorder_{m}=1,$ additionally assume $\alpha>\lambda_{m}(\polyorder+2).$
Then, for any $\varepsilon>0,$ there exist positive constants $r,\delta,\eta$
such that 
\[
\mathcal{T}_{P}^{m}(\tilde{1}_{r,\delta}v)(x)\geq\eta\Verts x_{2}^{\polyorder}1\{\Verts x_{2}\geq r+\delta\}-\varepsilon1\{r<\Verts x_{2}<r+\delta\}
\]
where $\tilde{1}_{r,\delta}=1-1_{r,\delta},$ $v(x)=-\weight_{\polyorder-2}(x)x$
with $\weight_{\polyorder-2}(x)=\bigl(\tau^{2}+\Verts x_{2}^{2}\bigr)^{(\polyorder-2)/2}$
and the constant $\eta$
is determined by the following quantities: $\alpha,$ $\beta_{1},$
$\beta_{0},$ $\lambda_{m},$ $\tau,$ $\polyorder_{m},$ and $\polyorder.$ 
In particular, we may take $r+\delta\geq1$ . 
\end{lem}

\begin{proof}
Given positive constants $r,\delta,$ our proof
separately considers the three regimes: $\Verts x_{2}\leq r,$ $r<\Verts x_{2}<r+\delta,$
and $\Verts x_{2}\geq r+\delta.$ We then choose suitable $r,\delta.$ 
Below we make use of the following relation that holds for differentiable functions $\gamma: \bbR^{\datdim} \to \bbR$ and $g: \bbR^{\datdim}\to \bbR^{\datdim}$: 
\begin{equation}
\mathcal{T}_P^m \bigl(\gamma g \bigr)(x)
=\gamma(x){\cal T}_{P}^{m}g(x)+\la m(x),\nabla \gamma(x)\otimes g(x)\ra.\label{eq:Stein-op-formula-tilting}
\end{equation}

First, if $\Verts x_{2}\leq r,$ we have $\tilde{1}_{r,\delta}(x)=0$
and thus $\mathcal{T}_{P}^{m}(\tilde{1}_{r,\delta}v)(x)=0,$ for any
choice of $r$ and any $v$ in the domain of $\mathcal{T}_{P}^{m}.$
In the following, we consider $v(x)=-\weight_{\polyorder-2}(x)x$ as claimed. 

In the regime $r<\Verts x_{2}<r+\delta,$ we show that for any $\varepsilon>0,$
we can choose $r,\delta>0$ such that $\mathcal{T}_{P}^{m}(\tilde{1}_{r,\delta}v)(x)\geq-\varepsilon.$
With $\Verts x_{2}=t$ and $\tilde{h}_{r,\delta}(t) = 1 - h_{r,\delta}(t)$,  we have 
\begin{align}
&\mathcal{T}_{P}^{m}(\tilde{1}_{r,\delta}v)(x) \nonumber \\
& =\tilde{1}_{r,\delta}(x){\cal T}_{P}^{m}v(x)-h_{r,\delta}^{'}\bigl(\Verts x_{2}\bigr)\la m(x),x\otimes v(x)\ra\frac{1}{\Verts x_{2}}\nonumber \\
&=\tilde{\weight}(t)\tilde{h}_{r,\delta}(t)\left\{ -\mathcal{T}_{P}^{m}x-\frac{2(\polyorder/2-1)t^{2}}{\tau^{2}+t^{2}}\frac{\inner{m(x)}{x\otimes x}}{t^{2}}\right. \nonumber \\
& \hphantom{=\tilde{\weight}(t)\bigl(1-h_{r,\delta}(t)\bigr)}\quad\left. -\frac{h'_{r,\delta}(t)}{\tilde{h}_{r,\delta}(t)}\frac{\inner{m(x)}{x\otimes x}}{t^{2}}t\right\} \nonumber \\
& \geq \tilde{\weight}(t)\tilde{h}_{r,\delta}(t)\left[\alpha t^{2}-\beta_{1}t-\beta_{0}-\lambda_{m}\left( \polyorder-2+\frac{h'_{r,\delta}(t)}{\tilde{h}_{r,\delta}(t)}t\right) (c_{m}+t^{\polyorder_{m}+1})\right]  \nonumber \\ 
\nonumber \\
&=\begin{aligned}\tilde{\weight}(t)\left[\tilde{h}_{r,\delta}(t)\cdot\left\{ \alpha t^{2}-\beta_{1}t-\beta_{0}-\lambda_{m}\left(\polyorder-2+\omega_{r,\delta}(t)\right)(c_{m}+t^{\polyorder_{m}+1})\right\} \right. \nonumber\\
\hphantom{=\tilde{w}(t)}\quad\left.-\lambda_{m}r \cdot h'_{r,\delta}(t)(c_{m}+t^{\polyorder_{m}+1})\right]
\end{aligned}
\nonumber \\
 & \geq\begin{aligned}\tilde{\weight}(t)\left[\frac{(t-r)^{3}}{(r+\delta-t)^{3}+(t-r)^{3}}\cdot\left\{ \alpha t^{2}-\beta_{1}t-\beta_{0}-\lambda_{m}(\polyorder+2)(c_{m}+t^{\polyorder_{m}+1})\right\} \right.\\
\hphantom{=\tilde{w}(t)}\quad\left.-\lambda_{m}r\frac{3\delta(r+\delta-t)^{2}(t-r)^{2}}{\{(r+\delta-t)^{3}+(t-r)^{3}\}^{2}}(c_{m}+t^{\polyorder_{m}+1})\right]
\end{aligned}
\label{eq:lowerbound-truncation}
\end{align}
where $\tilde{\weight}(t)=\bigl(\tau^{2}+t^{2}\bigr)^{\polyorder/2-1}$ 
and $\omega_{r,\delta}(t)\coloneqq3\delta(r+\delta-t)^{2}/\{(r+\delta-t)^{3}+(t-r)^{3}\}$; 
the first inequality is derived from the dissipativity
(\cref{assu:dissipativity})
and the growth condition on
$\Verts{m(x)}_{\mathrm{op}}$ (\cref{assu:poly-grow-coef}); to obtain the second inequality, we
have used $\omega_{r,\delta}(t)\leq4$ for $r < t < r+\delta$. 
In the lower bound \eqref{eq:lowerbound-truncation}, we have a quadratic function inside the braces in the first term,
where, by assumption, the quadratic term has positive coefficients
$\alpha$ and $\alpha-\lambda_{m}(\polyorder+2)$ for $\polyorder_{m}=0$
and $\polyorder_{m}=1,$ respectively. Thus, by choosing $r>0$ (independent
of $\delta)$ appropriately, we can make this quadratic function both
positive and increasing in the interval $(r,r+\delta).$ We denote
this value of $r$ by $r_{0}.$

The lower bound (\ref{eq:lowerbound-truncation}) may become negative,
but we show that its magnitude can be made arbitrarily small by choosing
$r$ and $\delta$ judiciously. To see this, note that the second
term of (\ref{eq:lowerbound-truncation}) dominates the first iff
\begin{equation}
\frac{\alpha}{3\lambda_{m}}\frac{\theta(1+\theta^{3})}{(1+\theta)^{2}}\frac{\delta}{r}\cdot\frac{t^{2}-\lambda_{m}(\polyorder+2)/\alpha\cdot t^{\polyorder_{m}+1}-(\beta_{1}/\alpha)t-\beta_{0}/\alpha-\lambda_{m}c_{m}(\polyorder+2)/\alpha}{t^{\polyorder_{m}+1}+c_{a}}\leq1,\label{eq:lowerbound-trunc-negative-cond}
\end{equation}
where $\theta=(t-r)/(r+\delta-t)\in(0,\infty).$ By taking sufficiently
large $r,$ we may assume that for any $t>r,$ 
\[
\frac{t^{\polyorder_{m}+1}+c_{m}}{t^{2}-\lambda_{m}(\polyorder+2)/\alpha\cdot t^{\polyorder_{m}+1}-(\beta_{1}/\alpha)t-\beta_{0}/\alpha-\lambda_{m}c_{m}(\polyorder+2)/\alpha}\leq C
\]
for some constant $C>0.$ If $\polyorder_{m}=0,$ we may choose any
$C>0,$ whereas if $\polyorder_{m}=1,$ we need $C>\{1-\lambda_{m}(\polyorder+2)/\alpha\}^{-1}.$
Indeed, such $r$ can be chosen by examining the following conditions:
\[
\begin{cases}
0\leq Ct^{2}-\bigl[C\bigl\{\frac{\beta_{1}}{\alpha}+\lambda_{m}\frac{\polyorder+2}{\alpha}\bigr\}+1\bigr]t-C\left\{ \frac{\beta_{0}}{\alpha}+\frac{\lambda_{m}c_{m}(\polyorder+2)}{\alpha}+\frac{c_{m}}{C}\right\}  & \text{if }\polyorder_{m}=0,\\
0\leq\bigl[C\{1-\frac{\lambda_m(\polyorder+2)}{\alpha}\}-1\bigr]t^{2}-C\frac{\beta_{1}}{\alpha}t-C\left\{ \frac{\beta_{0}}{\alpha}+\frac{\lambda_{m}c_{m}(\polyorder+2)}{\alpha}+\frac{c_{m}}{C} \right\}  & \text{if }\polyorder_{m}=1.
\end{cases}
\]
Let $r_{1}$ be a value of such $r.$ Using $r=r_{0}\lor r_{1},$
the inequality (\ref{eq:lowerbound-trunc-negative-cond}) implies
\begin{equation}
    \frac{\theta(1+\theta^{3})}{(1+\theta)^{2}}\leq\frac{3\lambda_{m}C}{\alpha}\frac{r}{\delta}. \label{eq:norm-indicator-theta-eval}
\end{equation}
Then, the magnitude of the second term is evaluated as follows: we
have 
\begin{align*}
 & \tilde{\weight}(t)=\left\{ \tau^{2}+\left(\frac{\theta(1+\theta^{3})}{(1+\theta)^{2}}\frac{1+\theta}{1+\theta^{3}}\delta+r\right)^{2}\right\} ^{\polyorder/2-1}\\
 & \leq\begin{cases}
\frac{1}{\tau^{\polyorder-2}} & \polyorder\leq2\\
\left\{ \tau^{2}+\left(\frac{4\lambda_{m}C}{\alpha}+1\right)^{2}r^{2}\right\} ^{\polyorder/2-1} & q>2
\end{cases},
\end{align*}
and 
\begin{align}
 & 3\lambda_{m}r\delta\frac{(r+\delta-t)^{2}(t-r)^{2}}{\{(r+\delta-t)^{3}+(t-r)^{3}\}^{2}}(c_{m}+t^{\polyorder_{m}+1})\nonumber \\
 & =3\lambda_{m}\frac{\theta^{2}}{(1+\theta)^{4}}\frac{(1+\theta)^{6}}{(1+\theta^{3}){}^{2}}\left\{ c_{m}+\left(\frac{\theta}{1+\theta}\delta+r\right)^{\polyorder_{m}+1}\right\} \frac{r}{\delta}\nonumber \\
 & =3\lambda_{m}\left(\frac{\theta(1+\theta^{3})}{(1+\theta)^{2}}\right)^{2}\frac{(1+\theta)^{6}}{(1+\theta^{3}){}^{4}}\left\{ c_{m}+\left(\frac{1+\theta}{1+\theta^{3}}\frac{\theta(1+\theta^{3})}{(1+\theta)^{2}}\delta+r\right)^{\polyorder_{m}+1}\right\} \frac{r}{\delta}\nonumber \\
 & \overset{\text{\eqref{eq:norm-indicator-theta-eval}}}{\leq}\frac{27\lambda_{m}^{3}C^{2}}{\alpha^{2}}\left(\frac{123+55\sqrt{5}}{32}\right)\left\{ c_{m}+\left(\frac{4\lambda_{m}C}{\alpha}+1\right)^{\polyorder_{m}+1}r^{\polyorder_{m}+1}\right\} \cdot\left(\frac{r}{\delta}\right)^{3}.\label{eq:lowerbound-trunc-magnitude}
\end{align}
The product of these provides the desired estimate of the magnitude.
For a given $r,$ as $\delta$ grows, the evaluation (\ref{eq:lowerbound-trunc-magnitude})
decays at the rate of $\delta^{-3}.$ Hence, for any $\varepsilon>0,$
we can choose sufficiently large $\delta$ such that the domination
by the second term in (\ref{eq:lowerbound-truncation}) is at most
by $\varepsilon;$ i.e., $\mathcal{T}_{P}^{m}\bigl(\tilde{1}_{r,\delta}v\bigr)(x)\geq-\varepsilon$
for $r<\Verts x_{2}<r+\delta.$ We denote such choice of $\delta$
by $\delta_{\varepsilon}(r).$ 

Finally, consider the regime $\Verts x_{2}\geq r+\delta,$ where $\mathcal{T}_{P}^{m}(\tilde{1}_{r,\delta}v)(x)=\mathcal{T}_{P}^{m}v(x).$
We show that a suitable value of $r+\delta$ yields $\mathcal{T}_{P}^{m}v(x)\geq\eta\Verts x_{2}^{\polyorder}$
for some $\eta>0.$ By Assumption \ref{assu:dissipativity} and the
growth condition on $\Verts{m(x)}_{\mathrm{op}},$ we have 
\begin{align*}
\mathcal{T}_{P}^{m}v(x) & \geq\tilde{w}(\Verts{x}_2)\left(\alpha\Verts x_{2}^{2}-\beta_{1}\Verts x_{2}-\beta_{0}-\lambda_{m}(q-2)\frac{\Verts x_{2}^{2}}{(\tau^{2}+\Verts x_{2}^{2})}\bigl(c_{m}+\Verts x_{2}^{\polyorder_{m}+1}\bigr)\right)\\
 & \geq\tilde{f}\bigl(\Verts x_{2}\bigr)\left(1+\frac{\tau^{2}}{\Verts x_{2}^{2}}\right)^{\frac{\polyorder}{2}-1}\Verts x_{2}^{\polyorder},
\end{align*}
where 
\[
\tilde{f}(t)=\left\{ \alpha-\frac{1}{t^{2}}\left(\beta_{1}t+\beta_{0}+\lambda_{m}(q-2)\bigl(c_{m}+t^{\polyorder_{m}+1}\bigr)1\{\polyorder>2\}\right)\right\} .
\]
Due to the monotonicity of $\tilde{f},$ we have $\delta_{0}$ such
that $\tilde{f}(\delta_{0})=0.$ By the requirement on $r_{0},$ it
turns out $r_{0}>\delta_{0}.$ Therefore, if $\polyorder\geq2,$ we
have ${\cal T}_{P}g(x)\geq\eta\Verts x_{2}^{\polyorder}$ for $\Verts x_{2}\geq r_{0}$
and $\eta=\tilde{f}(r_{0});$ if $\polyorder<2,$ we have $\mathcal{T}_{P}^{m}g(x)\geq\eta\Verts x^{\polyorder}$
for $\Verts x_{2}\geq r_{0}\lor\tau^{2}\bigl(2^{2/\verts{\polyorder-2}}-1\bigr)^{-1}$
and $\eta=2^{-1}\text{\ensuremath{\tilde{f}(r_{0})}}.$

The consideration for the above three regimes of $\Verts x_{2}$ suggests
that we should choose $\delta=\delta_{\varepsilon}(r),$ and $r=r_{0}\lor r_{1}$
if $\polyorder\ge2$ or else $r=r_{0}\lor r_{1}\lor\tau^{2}\bigl(2^{2/\verts{\polyorder-2}}-1\bigr)^{-1}$. This choice yields 
\begin{align*}
\mathcal{T}_{P}^{m}\bigl(\tilde{1}_{r,\delta}v\bigr)(x) & \geq\eta\Verts x_{2}^{\polyorder}1\{\Verts x_{2}\geq r+\delta\}-\varepsilon1\{r<\Verts x_{2}<r+\delta\}
\end{align*}
for any $x\in\bbR^{\datdim}.$
\end{proof}
\begin{cor}[Concrete choices of $r$ and $\delta$ in \cref{lem:norm-indicator-dominator}]
\label{cor:r-delta-expression}
Lemma \ref{lem:norm-indicator-dominator}
holds for 
\begin{align*}
r & =\begin{cases}
r_{1} & \polyorder\geq2\\
r_{1}\lor\tau^{2}\bigl(2^{2/\verts{\polyorder-2}}-1\bigr)^{-1} & 0\leq\polyorder<2
\end{cases}\\
\delta & =W_{r}r\cdot\varepsilon^{-\frac{1}{3}},
\end{align*}
where 
\[
r_{1}=\begin{cases}
\frac{\{\beta_{1}+\lambda_{m}(\polyorder+2)\}+1+\sqrt{[\{\beta_{1}+\lambda_{m}(\polyorder+2)\}+1]^{2}+4\alpha\left\{ \beta_{0}+\lambda_{m}c_{a}(\polyorder+2)+c_{m}\right\} }}{2\alpha} & \polyorder_{m}=0\\
\frac{2\beta_{1}+2\sqrt{\beta_{1}^{2}+2\tilde{\alpha}\{\lambda_{m}c_{m}(\polyorder+2)+\beta_{0}\}+\tilde{\alpha}c_{m}}}{2\tilde{\alpha}} & \polyorder_{m}=1,
\end{cases}
\]

\begin{align*}
&W_{r}^{3}=\\
&\begin{cases}
\bigl(\tau^{-(q-2)}\lor\bar{w}_{0}(r)\bigr)27\lambda_{m}^{3}\left(\frac{123+55\sqrt{5}}{32}\right)\left\{ c_{m}+\left(4\lambda_{m}+1\right)r\right\}  & \polyorder_{m}=0\\
\bigl(\tau^{-(q-2)}\lor\bar{w}_{1}(r)\bigr)\frac{4\cdot27\lambda_{m}}{\{\alpha-\lambda_{m}(\polyorder+2)\}^{2}}\left(\frac{123+55\sqrt{5}}{32}\right)\left\{ c_{m}+\left(\frac{8\lambda_{m}}{\alpha-\lambda_{m}(\polyorder+2)}+1\right)^{2}r^{2}\right\}  & \polyorder_{m}=1,
\end{cases}
\end{align*}
$\bar{w}_{0}(r)=\{\tau^{2}+\left(4\lambda_{m}+1\right)^{2}r^{2}\}^{\polyorder/2-1},$
and $\bar{w}_{1}(r)=\bigl[\tau^{2}+\bigl\{2\lambda_{m}/\bigl(\alpha-\lambda_{m}(\polyorder+2)\bigr)+1^{2}\bigr\} r^{2}\bigr]^{\polyorder/2-1}.$
Thus, $r+\delta=(1+W_{r}\varepsilon^{-\frac{1}{3}})r.$ In particular,
we may take $r\geq1$ to ensure $r+\delta>1.$ 
\end{cor}

\begin{proof}
We follow the notation in the proof of \cref{lem:norm-indicator-dominator}.
We consider the two cases $\polyorder_{m}=0$ and $\polyorder_{m}=1$
separately. 

We first examine the case $\polyorder_{m}=0.$ Here, by checking the
requirement on $r_{0},$ we obtain 
\begin{align*}
r_{0} & =\frac{\beta_{1}+\lambda_{m}(\polyorder+2)+\sqrt{\left(\beta_{1}+\lambda_{m}(\polyorder+2)\right)^{2}+4\alpha\left(\lambda_{m}c_{m}(\polyorder+2)+\beta_{0}\right)}}{2\alpha}
\end{align*}
 and with $C=\alpha$ in the proof of \cref{lem:norm-indicator-dominator}, 
the condition on $r_{1}$ implies 
\begin{align*}
r_{1} & =\frac{\{\beta_{1}+\lambda_{m}(\polyorder+2)\}+1+\sqrt{[\{\beta_{1}+\lambda_{m}(\polyorder+2)\}+1]^{2}+4\alpha\left\{ \beta_{0}+\lambda_{m}c_{a}(\polyorder+2)+c_{m}\right\} }}{2\alpha}>r_{0}.
\end{align*}
We then obtain $\delta_{\varepsilon}(r)=W_{r}r\varepsilon^{-1/3}$
with 
\[
W_{r}^{3}=\bar{w}(r)\cdot27\lambda_{m}^{3}\left(\frac{123+55\sqrt{5}}{32}\right)\left\{ c_{m}+\left(4\lambda_{m}+1\right)r\right\} 
\]
 and $\bar{w}(r)=\left\{ \tau^{2}+\left(4\lambda_{m}+1\right)^{2}r^{2}\right\} ^{\polyorder/2-1}\lor\tau^{-(q-2)}.$

Next, we similarly investigate the case $\polyorder_{m}=1.$ With
$\tilde{\alpha}=\alpha-\lambda_{m}(\polyorder+2),$ we have 
\[
r_{0}=\frac{\beta_{1}+\sqrt{\beta_{1}^{2}+4\tilde{\alpha}\{\lambda_{m}c_{m}(\polyorder+2)+\beta_{0}\}}}{2\tilde{\alpha}}
\]
and by choosing $C=2\alpha(\tilde{\alpha})^{-1},$ we obtain 
\begin{align*}
r_{1} & =\frac{2\beta_{1}+2\sqrt{\beta_{1}^{2}+2\tilde{\alpha}\{\lambda_{m}c_{m}(\polyorder+2)+\beta_{0}\}+\tilde{\alpha}c_{m}}}{2\tilde{\alpha}}\\
 & >r_{0}
\end{align*}
Thus, $\delta_{\varepsilon}=W_{r}r\varepsilon^{-\frac{1}{3}}$ with
\[
W_{r}^{3}=\bar{w}(r)\frac{4\cdot27\lambda_{m}^{3}}{\{\alpha-\lambda_{m}(\polyorder+2)\}^{2}}\left(\frac{123+55\sqrt{5}}{32}\right)\left\{ c_{m}+\left(\frac{8\lambda_{m}}{\alpha-\lambda_{m}(\polyorder+2)}+1\right)^{2}r^{2}\right\} 
\]
and 
\[
\bar{w}(r)=\left\{ \tau^{2}+\left(\frac{8\lambda_{m}}{\alpha-\lambda_{m}(\polyorder+2)}+1\right)^{2}r^{2}\right\} ^{\polyorder/2-1}\lor\tau^{-(\polyorder-2)}.
\]
\end{proof}

\subsection{\pcref{lem:existence-qgrowth-approx-universality}}
\label{subsec:Proof-of-Lemma-existence-qgrowth-approx}

Lemma \ref{lem:norm-indicator-dominator} provides a concrete function to approximate $q$-growth. 
To prove Lemma \ref{lem:existence-qgrowth-approx-universality}, we
approximate this function using an RKHS function. 
Our proof is divided into two parts:
the first proof deals with the general universal kernel, while the
second proof deals with the translation-invariant kernel case (see
Corollary \ref{cor:q-growth-trans-inv}). 

\subsubsection{Proof via universality}
\label{subsec:proof-via-universality}
\begin{lem}[Stein RKHS with universal kernel approximates $\polyorder$-growth]
Let $\polyorder\in[0,\infty).$ Let $K:\bbR^{\datdim}\times\bbR^{\datdim}\to\bbR^{\datdim\times\datdim}$
be a matrix-valued kernel defined by 
\[
K(x,y)=\weight_{\polyorder-1}(x)\left(L(x,y)+\bar{k}_{\mathrm{lin}}(x,y)\idmat\right) \weight_{\polyorder-1}(y),
\]
where $\weight_{\polyorder-1}(x)=\bigl(\tau^{2}+\Verts x_{2}^{2}\bigr)^{(\polyorder-1)/2}$
with $\tau>0;$ $L$ is universal $\calC_{0}^{1}\bigl(\bbR^{\datdim}\bigr);$
and 
\[
\bar{k}_{\mathrm{lin}}(x,y)=\frac{\tau^{2}+\la x,y\ra}{\sqrt{\tau^{2}+\Verts x_{2}^{2}}\sqrt{\tau^{2}+\Verts y_{2}^{2}}}.
\]
Then, under the same assumptions as in \cref{lem:norm-indicator-dominator}, 
${\cal T}_{P}^{m}(\rkhs K)$ approximates $\polyorder$-growth. 
\end{lem}

\begin{proof}
\cref{lem:norm-indicator-dominator} provides a concrete function approximating $\polyorder$-growth: 
for any $\varepsilon>0,$, there
exist positive constants $r,\delta,\eta$ such that 
\begin{equation}
\mathcal{T}_{P}^m\bigl(v-1_{r,\delta}v\bigr)(x)\geq\eta\Verts x_{2}^{\polyorder}1\{\Verts x_{2}\geq r+\delta\}-\frac{\eta\varepsilon}{4}1\{r<\Verts x_{2}<r+\delta\}\label{eq:indicator-dominator-target-property}
\end{equation}
where 
$v(x)=-\weight_{\polyorder-2}(x)x$, 
$r+\delta\geq1,$ and $1_{r,\delta}\in\calC^{2}$ is a function
defined in \cref{lem:norm-indicator-dominator} satisfying
$1_{r,\delta}(x)=1$ if $\Verts x_{2}\leq r+\delta,$ $1_{r,\delta}(x)=0$
if $\Verts x_{2}\geq r+\delta,$ and $0<1_{r,\delta}(x)<1,$ otherwise.
The proof would be complete if $v-1_{r,\delta}v$ were an element of $\rkhs{K}$. 
Our choice of kernel indeed implies $v\in\rkhs K$, as we detail in the next paragraph; 
however, the function $1_{r,\delta}v$ may not be an element of $\rkhs K$.
Thus, we approximate the latter using a function in the RKHS defined by kernel $L_{\weight_{\polyorder-1}}(x, y)=\weight_{\polyorder-1}(x)L(x,y)\weight_{\polyorder-1}(y)$. 

Let us first examine the claim $v\in\rkhs K$. 
First recall that for two kernels $K_1$, $K_2$, the RKHS $\rkhs{K_1+K_2}$ of the sum kernel $K_1 + K_2$ is given by functions $\{ h_1 + h_2: h_1\in \rkhs{K_1}, h_2 \in \rkhs{K_2}\}$~\citep[see, e.g.,][Proposition 3.1]{carmeli2010vector}; 
in particular, we have $\rkhs{K_1} \subset \rkhs{K_1+K_2}$ (the same goes for $\rkhs{K_2})$. 
Note that the RKHS of the kernel $\bar{k}_{\weight,\mathrm{lin}}(x,y)\coloneqq \weight_{\polyorder-1}(x)\weight_{\polyorder-1}(y) \bar{k}_{\mathrm{lin}}(x,y)$ contains tilted linear functions of the form $ \weight_{\polyorder-2}(x) \la a, x \ra$ with $a \in \mathbb{R}^{\datdim}$, since it is the sum of two scalar kernels, one of which is a linear kernel tilted by $\weight_{\polyorder-2}$. 
Each component of $v$ therefore belongs to $\rkhs{\bar{k}_{\weight,\mathrm{lin}}}$. 
Since $\rkhs{\bar{k}_{\weight, \mathrm{lin}} \idmat}$ consists of vector-valued functions with each component in $\rkhs{\bar{k}_{\weight, \mathrm{lin}}}$, we have $v \in \rkhs{\bar{k}_{\weight, \mathrm{lin}} \idmat} \subset \rkhs{K}$. 

We choose our approximation to $1_{r,\delta}v$ as follows. Define $\tilde{v}(x)=$$-1_{r,\delta}(x)\cdot \weight_{-1}(x)x$. 
so that $1_{r,\delta}v(x)=w_{\polyorder-1}(x)\tilde{v}(x).$ Since $\nabla1_{r,\delta}(x)$
is supported in $\{\Verts x\leq r+\delta\}$, $\tilde{v}$ is an element
of ${\cal C}_{0}^{1}(\bbR^{\datdim}).$ By the ${\cal C}_{0}^{1}(\bbR^{\datdim})$-universality
of the RKHS $\rkhs L,$ for $\rho>0$ (specified below), we have an
element $v_{\rho}=(v_{\rho}^{1},\dots,v_{\rho}^{\datdim})\in\rkhs L$
approximating $\tilde{v}$ such that 
\[
\sup_{x\in\bbR^{\datdim}}\left(\verts{v_{\rho}^{\dimidx}(x)-\tilde{v}^{\dimidx}(x)}\lor\max_{j\in\{1,\dots,\datdim\}}\verts{\partial_{x^{j}}v_{\rho}^{\dimidx}(x)-\partial_{x^{j}}\tilde{v}^{\dimidx}(x)}\right)\leq\rho.
\]
This in turn implies 
\begin{equation}
\sup_{x\in\bbR^{\datdim}}\Verts{v_{\rho}(x)-\tilde{v}(x)}_{2}\leq\sqrt{\datdim}\rho\ \text{and\ }\sup_{x\in\bbR^{\datdim}}\Verts{\nabla v_{\rho}(x)-\nabla\tilde{v}(x)}_{\mathrm{F}}\leq\datdim\rho.\label{eq:indicator-dominator-universality-precision}
\end{equation}
Note that we have, with the shorthand $\weight = \weight_{\polyorder-1}$, 
\begin{align*}
\mathcal{T}_{P}^{m}\bigl(v-\weight v_{\rho}\bigr)(x) & =\mathcal{T}_{P}^{m}\bigl(v-\weight\tilde{v}\bigr)(x)+\mathcal{T}_{P}^{m}\bigl(\weight\tilde{v}-\weight v_{\rho}\bigr)(x)\\
 & \geq\eta\Verts x_{2}^{\polyorder}1\{\Verts x_{2}\geq r+\delta\}-\frac{\eta\varepsilon}{4}1\{r<\Verts x_{2}<r+\delta\}\\
 & \hphantom{\geq}\quad-\verts{\mathcal{T}_{P}^{m}\bigl(\weight\tilde{v}-\weight v_{\rho}\bigr)(x)}, 
\end{align*}
and $\weight v_{\rho} \in \rkhs{L_{\weight_{\polyorder}}} \subset \rkhs{K}$. 
Thus, our goal is to choose $\rho$ such that the error $\verts{\mathcal{T}_{P}^{m}\bigl(\weight\tilde{v}-\weight v_{\rho}\bigr)(x)}$
grows more slowly than $\eta\Verts x_{2}^{\polyorder}$ if $\Verts x_{2}>r+\delta$
and can be made small if $\Verts{x}_2 \leq r+\delta$.

We evaluate the error of $\weight v_{\rho}$ when applied to the Stein
operator $\mathcal{T}_{P}^{m}$ (see also Eq.\ref{eq:Stein-op-formula-tilting}): 
\begin{align}
 & \left|\mathcal{T}_{P}^{m}\weight v_{\rho}(x)-\mathcal{T}_{P}^{m}\weight\tilde{v}(x)\right|\nonumber \\
 & \leq\begin{aligned} & \weight(x)\left(2\Verts{b(x)}_{2}\Verts{v_{\rho}(x)-\tilde{v}(x)}_{2}+\verts{\polyorder-1}\frac{\Verts x_{2}}{\tau^{2}+\Verts x^{2}}\Verts{m(x)}_{\mathrm{op}}\Verts{v_{\rho}(x)-\tilde{v}(x)}_{2}\right.\\
 & \hphantom{w(x)}\left.\quad+\text{\ensuremath{\sqrt{\datdim}}}\Verts{m(x)}_{\mathrm{op}}\Verts{\nabla v_{\rho}(x)-\nabla\tilde{v}(x)}_{\mathrm{F}}\right)
\end{aligned}
\label{eq:indicator-dominator-error-upperbound}
\end{align}
We bound the estimate (\ref{eq:indicator-dominator-error-upperbound})
to prove the statement for $\polyorder_{m}=0.$ 

Using the universality properties (\ref{eq:indicator-dominator-universality-precision}),
we can evaluate (\ref{eq:indicator-dominator-error-upperbound}) as
\begin{align*}
 & \verts{{\cal T}_{P}\weight v_{\rho}(x)-{\cal T}_{P}\weight\tilde{v}(x)}\\
 & \leq\sqrt{\datdim}\rho\cdot\weight(x)\left(2\Verts{b(x)}_{2}+\frac{\Verts x_{2}}{\tau^{2}+\Verts x^{2}}\verts{\polyorder-1}\Verts{m(x)}_{\mathrm{op}}+\datdim\Verts{m(x)}_{\mathrm{op}}\right)\\
 & \leq\rho\cdot\weight(x)\pi(\Verts x_{2}),
\end{align*}
where we have defined $\pi$ by 
\begin{align*}
\pi(\Verts x_{2}) & \coloneqq\sqrt{\datdim}\left\{ 2\lambda_{b}\bigl(1+\Verts x_{2}\bigr)+\lambda_{m}\frac{\Verts x_{2}}{\tau^{2}+\Verts x^{2}}(\polyorder-1)_{+}\bigl(c_{m}+\Verts x_{2}\bigr)+\lambda_{m}\datdim\bigl(c_{m}+\Verts x_{2}\bigr)\right\} \\
 & \leq\sqrt{\datdim}\left\{ 2\lambda_{b}\bigl(1+\Verts x_{2}\bigr)+\lambda_{m}(\polyorder-1)_{+}\bigl(c_{m}\tau^{-2}+1\bigr)+\lambda_{m}\datdim\bigl(c_{m}+\Verts x_{2}\bigr)\right\} 
\end{align*}
The function $\pi$ satisfies $\weight(x)\pi\bigl(\Verts x_{2}\bigr)<C\Verts x_{2}^{\polyorder}$
with 
\[
C=\bigl(1+\tau^{2}\bigr)^{\polyorder/2}\left[2\lambda_{b}+\lambda_{m}\datdim+\left\{ 2\lambda_{b}+\lambda_{m}(\polyorder-1)_{+}\left(c_{m}\tau^{-2}+1\right)+\lambda_{m}c_{m}\datdim\right\} \right]
\]
for $\Verts x_{2}>r+\delta\geq1.$ Consequently, 
\begin{equation}
\mathcal{T}_{P}^{m}\weight\tilde{v}(x)-\mathcal{T}_{P}^{m}\weight v_{\rho}(x)>-\rho C\Verts x_{2}^{\polyorder}1\{\Verts x_{2}>r+\delta\}-\rho\max_{\Verts x_{2}\leq r+\delta}\weight(x)\pi\bigl(\Verts x_{2}\bigr).\label{eq:moment-domination-intermediate-c01}
\end{equation}

With the evaluation (\ref{eq:moment-domination-intermediate-c01}),
we conclude the proof for the linear case $\polyorder_{m}=0.$ We
choose $\rho>0$ such that 
\[
\rho\leq\frac{1}{2}\eta C^{-1}\land\frac{\eta\varepsilon}{4}\left(\max_{\Verts x_{2}\leq r+\delta}\weight(x)\pi\bigl(\Verts x_{2}\bigr)\right)^{-1}
\]
to ensure $\mathcal{T}_{P}^{m}\weight\tilde{v}(x)-\mathcal{T}_{P}^{m}\weight v_{\rho}(x)\geq1/2\cdot\eta\Verts x_{2}^{\polyorder}1\{\Verts x_{2}>r+\delta\}-\eta\varepsilon/4.$
Therefore, 
\begin{align*}
\mathcal{T}_{P}^{m}\bigl(v-\weight v_{\rho}\bigr)(x) & =\mathcal{T}_{P}^{m}\bigl(v-\weight\tilde{v}\bigr)(x)+\mathcal{T}_{P}^{m}\bigl(\weight\tilde{v}-\weight v_{\rho}\bigr)(x)\\
 & \geq\eta\Verts x_{2}^{\polyorder}1\{\Verts x_{2}\geq r+\delta\}-\frac{\eta\varepsilon}{4}1\{r<\Verts x_{2}<r+\delta\}\\
 & \hphantom{\geq}\ -\frac{\eta}{2}\Verts x_{2}^{\polyorder}1\{\Verts x_{2}>r+\delta\}-\frac{\eta\varepsilon}{4}\\
 & \geq\frac{\eta}{2}\Verts x_{2}^{\polyorder}1\{\Verts x_{2}\geq r+\delta\}-\frac{\eta\varepsilon}{2}.
\end{align*}
The conclusion follows for 
$v_{\varepsilon}=2\eta^{-1}(v-\weight_{\polyorder-1}v_{\rho})$ 
and $r_{\varepsilon}=r+\delta.$ 

\end{proof}

\subsubsection{Convolution construction of $\protect\polyorder$-growth function}
In this section, we provide an alternative, constructive proof of RKHS $\polyorder$-growth approximation. 

\begin{lem}[Constructive approximation to the function from \cref{lem:norm-indicator-dominator}] 
\label{lem:q-approx-constructive} Let $\polyorder\geq0,$ and 
$\weight_{\polyorder-2}(x)=\bigl(\tau^{2}+\Verts x_{2}^{2}\bigr)^{(\polyorder-2)/2}$
where $\tau>0.$ 
Let $v(x)=-\weight_{\polyorder-2}(x)x$. 
For $\varepsilon>0,$ choose $r,\delta,\eta$ such that 
\[
\mathcal{T}_{P}^{m}(\tilde{1}_{r,\delta}v)(x)\geq\eta\Verts x_{2}^{\polyorder}1\{\Verts x_{2}\geq r+\delta\}-\varepsilon,
\]
where $\tilde{1}_{r,\delta}=1-1_{r,\delta}(x)$ and $1_{r,\delta}(x)$
is defined as in Lemma \ref{lem:norm-indicator-dominator}. Let $\varphi:\bbR^{\datdim}\to[0,\infty)$
be a probability density function. For $\rho>0$, define 
\[
T_{\varphi,\rho}\tilde{v}(x)=\rho^{-\datdim}\int\tilde{v}(x-y)\varphi(y/\rho)\dd y, 
\]
where $\tilde{v}(x)=-1_{r,\delta}(x)\weight_{-\theta}(x)\cdot x$ with $\theta \in [0,1]$. 
Assume the dissipativity condition in Assumption \ref{assu:dissipativity}
with $\alpha>0$. Assume that the growth conditions
in Assumption \ref{assu:poly-grow-coef} with $\lambda_m>0$ and $\polyorder_{m}\in\{0,1\}$. 
If $\polyorder_{m}=1,$ assume the following two conditions: (a) $\alpha>\lambda_{m}(\polyorder+2),$
and (b) there exists positive constants $r_{\varphi},$ $C_{\varphi},$
$\gamma\geq\datdim-1$ such that 
\[
\verts{\varphi(x)}\leq C_{\varphi}\Verts x_{2}^{-(1+\gamma)}
\]
for $\Verts x_{2}\geq r_{\varphi}.$ Suppose $\mu_{1}\coloneqq\int\Verts x_{2}\varphi(x)\dd x<\infty.$
Then, for any $\varepsilon>0,$ we may choose $\rho$ such that the
function $\mathring{v}_{\varepsilon}=2\eta^{-1}\bigl(v-\weight_{\polyorder+\theta-2} T_{\varphi,\rho}\tilde{v}\bigr)$
satisfies 
\[
\mathcal{T}_{P}^{m}\mathring{v}_{\varepsilon}(x)\geq\Verts x_{2}^{\polyorder}1\bigl\{\Verts x_{2}>r_{\varepsilon}\bigr\}-\varepsilon
\]
for any $x\in\bbR^{\datdim}$ and some $r_{\varepsilon}>0.$ In particular,
$\rho$ and $r_{\varepsilon}$ are given as follows: 
\[
\rho=\begin{cases}
    \frac{\eta}{4}(1\land\delta^{2})\Bigl(2\land\varepsilon r_{\varepsilon}^{-\{\polyorder\lor(1-\theta)\}}\Bigr)U_{3}^{-1} & \text{if}\ \polyorder_{m}=0,\\
    \frac{\eta}{4}(1\land\delta^{2})\left(2\tilde{C}_{\theta}\land U_{3,\theta}^{-1}\varepsilon r_{\varepsilon}^{-(\polyorder+1)}\right)\land1, & \text{if}\ \polyorder_{m}=1,
\end{cases}
\]
and 
\begin{align*}
r_{\varepsilon}=\begin{cases}
r+\delta & \polyorder_{m}=0,\\
2\bigl\{(r+\delta)\lor r_{\varphi}\bigr\}\lor r_{\eta}, & \polyorder_{m}=1,
\end{cases}
\end{align*}
respectively, where 
\begin{align*}
    \tilde{C}_{\theta}=\left\{ (\tau^{2}+1)^{\polyorder/2}\left(4\lambda_{b}+\lambda_{m}\verts{\polyorder+\theta-2}(c_{m}+1)\right)U_{1,\theta}\right\} ^{-1},
\end{align*}

$U_{1,\theta}=\mu_{1}\{ (1+\theta)\tau^{-\theta}+3\}$,
\[
U_{2,\theta}=\mu_{1}\left\{ (5\theta+\theta^2)\tau^{-\theta-1} + 6(1+\theta)\tau^{-\theta} + 72+\frac{35+13\sqrt{13}}{12}+6+2^{3+2/3} \right\}
\]
and 
\[
    U_{3,\theta}=\bigl(\tau^{2}+1\bigr)^{\verts{\polyorder+\theta-1}/2}\left\{ 2\lambda_{b}U_{1,\theta}\bigl(\tau^{-1}+1\bigr)+\lambda_{m}\left(\frac{\verts{\polyorder+\theta-2}}{2\tau} U_{1,\theta} + U_{2,\theta}\datdim\right)\bigl(c_{m}\tau^{-1}+1\bigr)\right\} .
\]

 and 
\[
    r_{\eta}=\left\{ 2^{2+\gamma}\cdot3(1+(1+\theta)\tau^{-\theta})\left(\tau^{2}+1\right)^{\polyorder/2}\cdot\lambda_{m}\datdim \bigl(c_{m}+1\bigr)(1\lor\delta^{-1})\eta^{-1}\cdot C_{\varphi}\frac{\pi^{\datdim/2}(r+\delta)^{\datdim}}{\Gamma(\datdim/2+1)}\right\} ^{1/\gamma}.
\]
\end{lem}

\begin{proof}
The proof essentially proceeds as in the proof of Lemma \ref{lem:existence-qgrowth-approx-universality}
in \cref{subsec:Proof-of-Lemma-existence-qgrowth-approx},
and we use the same notation. 
Our objective is to construct an approximation to $x\mapsto -1_{r,\delta}(x)\weight_{\polyorder-2}(x)x$. 
Let $v_{0}(x)=-\weight_{-\theta}(x)\cdot x$. 
The proof here differs in that we approximate $\tilde{v}=1_{r,\delta}v_{0}$
with 
\[
    T_{\varphi,\rho}\tilde{v}(x)=\frac{1}{\rho^{\datdim}}\int\tilde{v}(x-y)\varphi(y/\rho)\dd y,
\]
where $\rho>0$. 
We then show that $\weight_{\polyorder+\theta-2}T_{\varphi,\rho}\tilde{v}$ is the desired approximation, 
replacing its counterpart in the proof in Section \ref{subsec:Proof-of-Lemma-existence-qgrowth-approx};
we obtain $\rho$ and $r_{\varepsilon}$ explicitly. 

We first clarify approximation properties of $T_{\varphi,\rho}\tilde{v}.$
Note that $v_{0}$ satisfies 
\[
\Verts{v_{0}(x)}_{2}\leq \Verts{x}_2^{1-\theta}, \Verts{\nabla v_{0}(x)}_{\mathrm{op}}\leq (1+\theta)\tau^{-\theta},\ \text{and}\ \Verts{\nabla^{2}v_{0}(x)}_{\mathrm{op}}\leq (5\theta+\theta^2)\tau^{-\theta-1}. 
\]
In particular, $\nabla\tilde{v}$ is continuous and uniformly bounded,
and hence we have $\nabla T_{\varphi,\rho}\tilde{v}=T_{\varphi,\rho}\nabla\tilde{v}$
by the mean value theorem and \citep[Corollary A.5]{Dudley_1999}.
Applying Lemmas \ref{lem:convolution-diff-bound}, \ref{lem:convolution-grad-diff-bound},
\ref{lem:cutoff-ball-bound-C2} with $g=v_0$ and $\weight\equiv1$ (according
to the notation therein) implies that there exist some constants $U_{1,\theta}$, 
$U_{2,\theta}$ such that 
\begin{align*}
    & \Verts{T_{\varphi,\rho}\tilde{v}(x)-\tilde{v}(x)}_{2}  \leq \rho\bigl(1\lor\delta^{-1}\bigr)U_{1,\theta}(r+\delta)^{1-\theta}\ \text{and} \\
    & \Verts{\nabla T_{\varphi,\rho}\tilde{v}(x)-\nabla\tilde{v}(x)}_{\mathrm{op}}\leq\rho \bigl(1\lor\delta^{-2}\bigr)U_{2,\theta} (r+\delta)^{1-\theta}, 
\end{align*}
where
\begin{align*}
    & U_{1,\theta}=\mu_{1}\{ (1+\theta)\tau^{-\theta}+3\}\\
    & U_{2,\theta}=\mu_{1}\left\{ (5\theta+\theta^2)\tau^{-\theta-1} + 6(1+\theta)\tau^{-\theta} + 72+\frac{35+13\sqrt{13}}{12}+6 + 2^{3+2/3} \right\}
\end{align*}
and $\mu_{1}=\int\Verts x_{2}\varphi(x)\dd x$ (here, we have used $r\geq1$ and thus $r+\delta\geq1$, see \cref{cor:r-delta-expression}). 

In both cases $\polyorder_{m}=0$ and $\polyorder_{m}=1,$ we need
to choose $\rho$ such that we can bound the error of $T_{\varphi,\rho}\tilde{v}$
when applied to the Stein operator by $\eta\varepsilon/4,$ within
the ball of radius $r_{\varepsilon}.$ The error is evaluated as 
\begin{align*}
    \left|\mathcal{T}_{P}^{m}\bigl(\weight_{\polyorder+\theta-2}\tilde{v}\bigr)(x)-\mathcal{T}_{P}^{m}\bigl(\weight_{\polyorder+\theta-2}T_{\varphi,\rho}\tilde{v}\bigr)(x)\right|\leq\rho\bigl(1\lor\delta^{-2}\bigr)\cdot e\bigl(\Verts x_{2}\bigr),
\end{align*}
where 
\begin{align*}
    & e\bigl(\Verts x_{2}\bigr)\\
	& =(r+\delta)^{1-\theta}\bigl(\tau^{2}+\Verts x_{2}^{2}\bigr)^{(\polyorder+\theta-1)/2}\left\{ 2\lambda_{b}U_{1,\theta}(\tau^{-1}+1)\right.\\
    & \hphantom{(r+\delta)^{1-\theta}\bigl(\tau^{2}+\Verts x_{2}^{2}\bigr)^{(\polyorder+\theta-1)/2}}\left.\qquad+\lambda_{m}\left(\frac{\verts{\polyorder+\theta-2}}{2\tau}U_{1,\theta}+U_{2,\theta}\datdim\right)(c_{m}\tau^{-1}+\Verts x_{2}^{\polyorder_{m}})\right\}.
\end{align*}
Furthermore, we have $e\bigl(\Verts x_{2}\bigr)\leq U_{3,\theta}(r+\delta)^{1-\theta}(1\lor\Verts{x}_{2}^{\polyorder+\theta-1+\polyorder_{m}})$
with 
\[
    U_{3,\theta}=\bigl(\tau^{2}+1\bigr)^{\verts{\polyorder+\theta-1}/2}\left\{ 2\lambda_{b}U_{1}\bigl(\tau^{-1}+1\bigr)+\lambda_{m}\left(\frac{\verts{\polyorder+\theta-2}}{2\tau} U_{1,\theta} + U_{2,\theta}\datdim\right)\bigl(c_{m}\tau^{-1}+1\bigr)\right\} .
\]
If $r_{\varepsilon} \geq r+\delta \geq 1$, we have 
\[
    \max_{\Verts x_{2}\leq r_{\varepsilon}}e(\Verts x_{2})\leq U_{3,\theta} r_{\varepsilon}^{(\polyorder+\polyorder_{m}) \lor (1-\theta)}. 
\]
Therefore, for $\varepsilon>0$, we require 
\[
\rho\leq\frac{\eta\varepsilon}{4}U_{3,\theta}^{-1}(1\land\delta^{2})r_{\varepsilon}^{-\{(\polyorder+\polyorder_{m})\lor(1-\theta)\}}.
\]

An additional condition is imposed on $\rho$ to make sure that the absolute
error of ${\cal T}_{P}^{m}\bigl(\weight_{\polyorder+\theta-2}T_{\varphi,\rho}\tilde{v}\bigr)(x)$
is bounded by $\eta\Verts x_{2}^{\polyorder}/2$ for $\Verts x_{2}>r_{\varepsilon}\geq1.$
We first address the linear case $\polyorder_m=0$. 
Since $e\bigl(\Verts x_{2}\bigr) < U_{3,\theta}\Verts x_{2}^{\polyorder}$
for $\Verts x_{2} > r_{\varepsilon},$ 
we require $\rho\leq\eta U_{3,\theta}^{-1}(1\land\delta^{2})/2.$ With $\rho=\eta U_{3}^{-1}(1\land\delta^{2})(2\land\varepsilon r_{\varepsilon}^{-\{\polyorder\lor (1-\theta)\}})/4,$
the conclusion holds for $\mathring{v}_{\varepsilon}=2\eta^{-1}\bigl(v-\weight_{\polyorder+\theta-2}T_{\varphi,\rho}\tilde{v}\bigr)$
and $r_{\varepsilon}=(r+\delta).$

We next deal with the quadratic case $\polyorder_{m}=1.$ We first
require $\rho\leq1.$ Then, by the assumption on $\varphi,$ for $\Verts x_{2}\geq r_{\varphi},$
we have 
\begin{align*}
\rho^{-\datdim}\varphi(x/\rho) & \leq C_{\varphi}\rho^{\gamma+1-\datdim}\Verts x_{2}^{-(1+\gamma)}\\
 & \le C_{\varphi}\Verts x_{2}^{-(1+\gamma)}.
\end{align*}
By Lemma \ref{lem:convolution-decay-rate}, we have for $\Verts x_{2}\geq2\{(r+\delta)\lor r_{\varphi}\},$
\begin{align*}
&\Verts{\nabla T_{\varphi,\rho}\tilde{v}(x)}_{\mathrm{\mathrm{op}}} \\
& \leq2^{(1+\gamma)}\sup_{y\in\bbR^{\datdim}}\Verts{\nabla\tilde{v}(y)}_{\mathrm{op}}C_{\varphi}\lambda(\{\Verts x_{2}\leq r+\delta\})\Verts x_{2}^{-(1+\gamma)}\\
& \leq\underbrace{2^{(1+\gamma)}3(1+(1+\theta)\tau^{-\theta})C_{\varphi}\lambda(\{\Verts x_{2}\leq r+\delta\})}_{C_{\varphi,1}}(1\lor\delta^{-1})(r+\delta)^{1-\theta}\Verts x_{2}^{-(1+\gamma)},
\end{align*}
where $\lambda$ is the Lebesgue measure. 
For $\Verts x_{2}\geq r+\delta\geq1,$ letting $v_{\rho} \coloneqq T_{\varphi, \rho}\tilde{v}$ and using \cref{eq:Stein-op-formula-tilting} give 

\begin{align*}
    & \left|{\cal T}_{P}^{m}(\weight_{\polyorder+\theta-2}\tilde{v})(x)-{\cal T}_{P}^{m}(\weight_{\polyorder+\theta-2}v_{\rho})(x)\right| \\
	& \leq\weight_{\polyorder+\theta-2}(x)\left(2\Verts{b(x)}_{2}\Verts{\tilde{v}(x)-v_{\rho}(x)}_{2}+\Verts{m(x)}_{\mathrm{op}}\verts{\polyorder+\theta-2}\frac{\Verts x_{2}}{\tau^{2}+\Verts x_{2}^{2}}\Verts{\tilde{v}(x)-v_{\rho}(x)}_{2}\right) \\ 
	& \hphantom{\leq}\quad+\weight_{\polyorder+\theta-2}(x)\datdim\Verts{m(x)}_{\mathrm{op}}\Verts{\nabla v_{\rho}(x)}_{\mathrm{op}} \\
	& \leq\weight_{\polyorder+\theta-2}(x)\left(2\lambda_{b}(1+\Verts x_{2})+\frac{\lambda_{m}\verts{\polyorder+\theta-2}\Verts x_{2}}{\tau^{2}+\Verts x_{2}^{2}}\bigl(c_{m}+\Verts x_{2}^{2}\bigr)\right)U_{1,\theta}\rho(1\lor\delta^{-1})(r+\delta)^{1-\theta} \\ 
	& \hphantom{\leq}\quad+\weight_{\polyorder+\theta-2}(x)\lambda_{m}\datdim\bigl(c_{m}+\Verts x_{2}^{2}\bigr)C_{\varphi,1}(1\lor\delta^{-1})(r+\delta)^{1-\theta}\Verts x_{2}^{-(1+\gamma)} \\
	& \leq\Verts x_{2}^{\polyorder}\left(\tau^{2}+1\right)^{\polyorder/2}\left\{ 4\lambda_{b}+\lambda_{m}\verts{\polyorder+\theta-2}\left(c_{m}+1\right)\right\} U_{1,\theta}\rho(1\lor\delta^{-1}) \\
    & \hphantom{\leq}\quad+\Verts x_{2}^{\polyorder}\left(\tau^{2}+1\right)^{\polyorder/2}\cdot\lambda_{m}\datdim\left(c_{m}+1\right)C_{\varphi,1}(1\lor\delta^{-1})\Verts x_{2}^{-\gamma}.\\
\end{align*}
Then, we choose $r_{\eta}\geq r_{\varphi}\lor1$ such that 
\[
\left(\tau^{2}+1\right)^{\polyorder/2}\cdot\lambda_{m}\datdim\bigl(c_{m}+1\bigr)C_{\varphi,1}(1\lor\delta^{-1})r_{\eta}^{-\gamma}\leq\frac{\eta}{2};
\]
specifically, we use 
\begin{equation}
\begin{aligned}
 r_{\eta} & =\left\{ 2\left(\tau^{2}+1\right)^{\polyorder/2}\cdot\lambda_{m}\datdim\bigl(c_{m}+1\bigr)C_{\varphi,1}(1\lor\delta^{-1})\eta^{-1}\right\} ^{1/\gamma}\\
    & =\left[ \frac{2^{2+\gamma}}{\eta}\cdot3\left\{1+\frac{(1+\theta)}{\tau^{\theta}}\right\}\left(\tau^{2}+1\right)^{\polyorder/2}\cdot\lambda_{m}\datdim\bigl(c_{m}+1\bigr)(1\lor\delta^{-1})\cdot C_{\varphi}\frac{\pi^{\datdim/2}(r+\delta)^{\datdim}}{\Gamma(\datdim/2+1)}\right] ^{1/\gamma}.\\
\end{aligned}
\label{eq:rgamma}
\end{equation}
Following the proof of Lemma \ref{lem:existence-qgrowth-approx-universality},
we may choose $r_{\varepsilon}=2\{(r+\delta)\lor r_{\varphi}\}\lor r_{\eta}.$
With this choice, we require $\rho\leq\eta\tilde{C}(1\land\delta)/2,$
where
\[
    \tilde{C}=\left\{ (\tau^{2}+1)^{\polyorder/2}\left(4\lambda_{b}+\lambda_{m}\verts{\polyorder+\theta-2}(c_{m}+1)\right)U_{1,\theta}\right\} ^{-1}.
\]
Thus, we set $\rho$ as 
\[
\rho=\frac{\eta}{4}(1\land\delta^{2})\left(2\tilde{C}\land U_{3}^{-1}\varepsilon r_{\varepsilon}^{-(\polyorder+1)}\right)\land1,
\]
where $r_{\varepsilon}=2\{(r+\delta)\lor r_{\varphi}\}\lor r_{\eta}.$
With this value of $\rho,$ the conclusion holds for $\mathring{v}_{\varepsilon}=2\eta^{-1}\bigl(v-\weight_{\polyorder+\theta-2}T_{\varphi,\rho}\tilde{v}\bigr)$
and $r_{\varepsilon}=2\{(r+\delta)\lor r_{\varphi}\}\lor r_{\eta}.$
\end{proof}
\begin{lem}[Concrete expressions of constants in \cref{lem:q-approx-constructive}]
\label{lem:r-eta-expression}Let $\polyorder\geq0.$ Specify the convolution
function $\varphi$ in Lemma \ref{lem:q-approx-constructive} by $\varphi_{\polyorder}$
in Lemma \ref{lem:Bochner-Riesz-density}. Then, the quadratic case
of Lemma \ref{lem:q-approx-constructive} holds with 
\[
    r_{\varepsilon}=\left\{ 2\lor c_{P,\datdim,\polyorder,\tau,\theta}(1\lor\delta^{-1/\datdim})\right\} (r+\delta),
\]
where $c_{P,\datdim,\polyorder,\tau,\theta}$ is a constant given by 
\[
    \left[ 2^{\datdim+\polyorder+1}3\{1+(1+\theta)\tau^{-\theta}\}\left(\tau^{2}+1\right)^{\polyorder/2}\lambda_{m}\bigl(c_{m}+1\bigr)\eta^{-1}\frac{\Gamma(\polyorder/2+1)^{2}}{\mathrm{B}(\datdim/2,\polyorder+1)}\right]^{1/\datdim}
\]
\end{lem}

\begin{proof}
By the upper bound on the Bessel function \citep[\href{https://dlmf.nist.gov/10.14.1}{Eq. 10.14.1}]{DLMF}, we have 
\begin{align*}
\hat{\Phi}_{\polyorder}(x)^{2} & =2^{\polyorder}\Gamma(\polyorder/2+1)^{2}\Verts x_{2}^{-(\datdim+\polyorder)}J_{(\datdim+\polyorder)/2}\bigl(\Verts x_{2}\bigr)^{2}\\
& \leq2^{\polyorder-1}\Gamma(\polyorder/2+1)^{2}\Verts x_{2}^{-(\datdim+\polyorder)}.
\end{align*}
Thus, for $\Verts x_{2}\geq r_{\varphi_{\polyorder}}\coloneqq1,$
we have 
\[
\varphi_{\polyorder}(x)\leq C_{\varphi_{\polyorder}}\Verts x_{2}^{-(\datdim+1)}
\]
with 
\[
    C_{\varphi_{\polyorder}} = 2^{\polyorder-1}\Gamma(\polyorder/2+1)^{2}\frac{\Gamma(\datdim/2)}{\pi^{\datdim/2}\mathrm{B}(\datdim/2, \polyorder+1)}. 
\]
Hence, substituting these estimates ($\gamma = \datdim$ and $C_\varphi = C_{\varphi_{\polyorder}}$) into \eqref{eq:rgamma} gives 
\begin{align*}
    &(r+\delta)^{-1}r_{\eta}\\
    & =\left\{ 2^{2+\datdim}3\{1+(1+\theta)\tau^{-\theta}\}\left(\tau^{2}+1\right)^{\polyorder/2}\lambda_{m}\datdim\bigl(c_{m}+1\bigr)\eta^{-1} C_{\varphi_{\polyorder}}\frac{\pi^{\datdim/2}}{\Gamma(\datdim/2+1)}\right\} ^{1/\datdim}(1\lor\delta^{-1/\datdim})\\
    & =\left\{ 2^{\datdim+\polyorder+1}3\{1+(1+\theta)\tau^{-\theta}\}\left(\tau^{2}+1\right)^{\polyorder/2}\lambda_{m}\bigl(c_{m}+1\bigr)\eta^{-1}\frac{\Gamma(\polyorder/2+1)^{2}}{\mathrm{B}(\datdim/2,\polyorder+1)}\right\} ^{1/\datdim}(1\lor\delta^{-1/\datdim}).
\end{align*}
\end{proof}
\begin{cor}[Approximation from \cref{lem:q-approx-constructive} belongs to an RKHS] 
    \label{cor:q-growth-trans-inv}
    Let $\theta \in [0,1]$. 
    Let $\ell(x,y ) = \weight_{\theta-1}(x)\weight_{\theta-1}(y)\Phi(x-y)$ be a kernel  
defined by $\Phi \in \calC_0^1 \cap \calC^2$ with continuous non-vanishing generalized Fourier transform $\hat{\Phi}$. 
Consider the approximation $T_{\varphi,\rho}\tilde{v}$ in Lemma \ref{lem:q-approx-constructive}
defined by the density $\varphi=\varphi_{\polyorder}$ in Lemma \ref{lem:Bochner-Riesz-density}.
Then, each component of $\weight_{\theta-1}T_{\varphi,\rho}\tilde{v}$ is an element
of RKHS $\rkhs{\ell}.$ 
\end{cor}

\begin{proof}
The conclusion follows from Lemma \ref{lem:Bochner-Riesz-density}
and the RKHS norm expression of \citet[Theorem 10.21]{Wendland2004}
(see Lemma \ref{lem:RKHS-norm-translation-invariant} for a concrete bound
on the norm). 
\end{proof}

\section{Topological equivalence between KSD- and Wasserstein
convergence}
\label{sec:Topological-equivalence}

In Section \ref{subsec:KSD-convergence-and}, we established the equivalence
between KSD- and $\polyorder$-Wasserstein convergence. 
This appendix presents proofs for the results presented therein. 

\subsection{\pcref{thm:ksd-equiv-to-Wass-conv}}
\label{subsec:proof-of-KSD-implies-Wass}
Here, we provide a proof for Theorem \ref{thm:ksd-equiv-to-Wass-conv}. 
Recall that $\polyorder$-Wasserstein convergence of a sequence is
equivalent to its weak convergence and $\polyorder$th moment uniform
integrability. We therefore aim to show that the suggested kernel
choices induce KSDs such that their convergence to zero implies the latter pair of conditions. 
\cref{lem:existence-qgrowth-approx-universality} (with \cref{cor:q-approx-uniform-integrability}) specifies conditions
for enforcing uniform integrability. Thus, the proof is complete
if we additionally establish conditions for the KSD to control weak
convergence. 

Our proof relies on the following result of \citet{}: 
\begin{thm}[{\citealp[Adaptation of][Theorem 6]{Barp2024}}]
\label{thm:bounded-separation-adaptation}
Suppose $\mathcal{T}_{P}^{m}({\cal H}_{K})$
is bounded $P$-separating, i.e., if $Q\in{\cal P},$ and $\int h\dd Q=\int h\dd P$
for all $h\in\mathcal{T}_{P}^{m}(\rkhs K)$ with $\sup_{x\in\bbR^{\datdim}}\verts{h(x)}<\infty$,
then we have $Q=P.$ Then, for any tight sequence $(Q_{\seqidx})_{\seqidx\geq1}\subset\mathcal{P}$, 
we have $\ksd{K_{m}}P(Q_{\seqidx})\to0$ iff $Q_{\seqidx}$ converges
weakly to $P$ as $\seqidx\to\infty.$ 
\end{thm}

Since the uniform integrability implies the tightness of a sequence,
it suffices to show that the kernel of the form (\ref{eq:q-growth-approx-kernel})
can also induce the bounded $P$-separation property. We apply the
recipes provided in \citet[Theorems 7 and 9]{Barp2024}. 
We divide the proof into two parts according to the growth exponent $\polyorder_m$ of diffusion matrix $\Verts{m(x)}_{\mathrm{op}}$ 
(see \cref{assu:poly-grow-coef}). 

Before proceeding with the proofs, we define some concepts and then provide a support lemma. 
Recall that $\calC_c^1(\bbR^{\datdim})$ (resp. $\calC_b^1(\bbR^{\datdim})$) denotes the set of continuously differentiable $\bbR^{\datdim}$-valued functions that are compactly supported (resp. bounded) and have compactly supported partial derivatives (resp. bounded partial derivatives). 
We equip $\calC_c^1(\bbR^{\datdim})$ with the same norm as $\calC_0^1(\bbR^{\datdim})$ (see \cref{sec:intro}). 
We denote by $\calC_b^1(\bbR^{\datdim})_{\beta}$ the set $\calC_b^1(\bbR^{\datdim})$ equipped with the strict topology~\citep[Appendix A]{Buck_1958, Barp2024}. 
The topological dual of $\calC_0^1(\bbR^{\datdim})$ is denoted by $\mathcal{D}_{L^1}^1(\bbR^{\datdim})$. 
We say $\rkhs{K}$ is $\mathcal{D}_{L^1}^1(\bbR^{\datdim})$-characteristic if it uniquely embeds the elements of $\mathcal{D}_{L^1}^1(\bbR^{\datdim})$ into itself~\citep[see][Appendix C for the definition of embeddability]{Barp2024}. 

\begin{lem}[Matrix-valued tilting preserves characteristicity]
\label{lem:matrix-multi-characteristic}Let $A\in\calC_b^{1}(\bbR^{\datdim\times\datdim})$
such that $A(x)$ is invertible for each $x\in\bbR^{\datdim}.$ Suppose
a matrix-valued kernel $K$ with ${\cal H}_{K}\subset{\cal C}_{b}^{1}(\bbR^{\datdim})$
is $\mathcal{D}_{L^{1}}^{1}(\bbR^{\datdim})$-characteristic. Then,
$K_{A}(x,y)=A(x)K(x,y)A(y)^{\top}$ is also $\mathcal{D}_{L^{1}}^{1}(\bbR^{\datdim})$-characteristic. 
\end{lem}

\begin{proof}
Define $\phi_{A}: f\mapsto Af$, which is a continuous map from $\calC_b^1(\bbR^{\datdim})_{\beta}$  to itself. 
We also have $\phi_A(\calC^1_c(\bbR^{\datdim}))=\calC_c^1(\bbR^{\datdim})$, since for any $f\in\calC^1_c(\bbR^{\datdim})$, we have $g \in \calC_c^1(\bbR^{\datdim}): x\mapsto A(x)^{-1} f(x) $ satisfying $\phi_A(g) =f$. 
The claim follows from  $\calC_c^1(\bbR^{\datdim})$ being dense in $\calC_0^1(\bbR^{\datdim})$ and applying \citet[Theorem 14]{Barp2024} with $\phi=\phi_A$. 
\end{proof}
\subsubsection{Linear case $\protect\polyorder_m=0$}
We prove the claim for the two suggested kernels.

\paragraph{Proof for kernel $L^{(1)}$}
We first observe that \cref{assu:poly-grow-coef}, together with 
the assumptions $\limsup_{\Verts{x}_2 \to \infty} \Verts{m^{-1}(x)}_{\mathrm{op}}<\infty$
and $\Verts{\la\nabla, m(x)\ra}_2 = O(\Verts{x}_2)$, 
implies that we may take some constant $C>0$ such that 
\begin{equation}
\Verts{\score p(x)}_{2}\leq C\weight_1(x)
\label{eq:score-upper-bound}
\end{equation}
where $\score p(x)=\nabla\log p(x)$ and $\weight_1(x) = (\tau^2 + \Verts{x}_2^2)^{1/2}$. 

We use \citet[Theorem 7]{Barp2024} to show that $\mathcal{T}_{P}^{m}(\rkhs{K})$ is bounded $P$-separating. 
Since $\rkhs{L_0} \subset \calC_0^1(\bbR^{\datdim})$, by \citet[Lemma 8]{Barp2024}, 
the RKHS $\rkhs{L_0}$ is continuously embedded into $\calC_0^1(\bbR^{\datdim})$; therefore, the $\calC_0^1(\bbR^{\datdim})$-universality assumption is equivalent to $\mathcal{D}_{L^1}^1(\bbR^{\datdim})$-characteristicity by \citet[Theorem 6]{SimonGabriel2018}. 
Then, we examine the following kernel 
\[
    m(x)\weight_{\polyorder-1}(x)L^{(1)}(x,y)\weight_{\polyorder-1}(y)m(y)^{\top}=\frac{m(x)L_{0}(x,y)m(y)^{\top}}{\weight_{1}(x)^2\weight_{1}(y)^2} = \frac{L_{m/\weight_{1}}(x,y)}{\weight_1(x)\weight_1(y)}
\]
where $L_{m/\weight_1}(x,y)\coloneqq m(x)L_0(x,y)m(y)^{\top}/\{\weight_1(x)\weight_1(y)\}$. 
The scaled diffusion matrix $m/\weight_1$ is bounded and has bounded derivatives; the kernel $L_{m/\weight_1}$ has its RKHS included in ${\cal C}_{0}^{1}(\bbR^{\datdim})$, 
and it is again ${\cal D}_{L^{1}}^{1}(\mathbb{R}^{\datdim})$-characteristic by Lemma \ref{lem:matrix-multi-characteristic}. 
Using \eqref{eq:score-upper-bound}, we apply \citet[Theorem 7]{Barp2024} with $\theta(x)=C\weight_1(x)$
and $K(x,y)=C^{2}L_{m/\weight_1}(x,y)$, which implies the
$P$-bounded separability of $\mathcal{T}_{P}^{m}(\rkhs{\tilde{L}^{(1)}})$ 
where $\tilde{L}^{(1)}(x, y) = \weight_{\polyorder-1}(x) L^{(1)}(x, y) \weight_{\polyorder-1}(y)$. 
It follows from $\mathcal{T}_{P}^{m}(\rkhs{\tilde{L}^{(1)}})\subset\mathcal{T}_{P}^{m}(\rkhs{K})$
that $\mathcal{T}_{P}^{m}(\rkhs{K})$ is also $P$-bounded separating.
Hence, we have established that $\ksd{K_{m}}P(\cdot)$ controls tight
weak convergence. 

It then suffices to show that \cref{lem:existence-qgrowth-approx-universality}
applies, as it then follows that the KSD also enforces $\polyorder$th
moment uniform integrability. 
This claim immediately holds by the $\calC_0^1(\bbR^{\datdim})$-universality of kernel 
\[ 
\frac{L_{0}(x,y)}{{\weight}_{\polyorder+1}(x){\weight}_{\polyorder+1}(y)}, 
\] 
which can be verified by the same argument for $L_{m/\weight_1}$ above. 
Hence, we have proved the claim for $L=L^{(1)}$. 
\qed

\paragraph{Proof for kernel $L^{(2)}$}

Weak convergence control is shown as in the proof of \citet[Theorem 9]{Barp2024}.
By \citet[Theorem 8]{Barp2024}, there exists a translation-invariant
${\cal D}_{L}^{1}$-characteristic kernel $\ell_{s}\in{\cal C}^{(1,1)},$
and a positive-definite function $f$, satisfying $1/f\in{\cal C}^{1}$
and $\verts{f(x)}\lor\Verts{\nabla f(x)}_{2}=O\bigl(\exp(-\sum_{\dimidx=1}^{\datdim}\sqrt{\verts{x_{\dimidx}}})\bigr)$, 
such that $\rkhs{\ell_{f}}\subset\rkhs{\ell}$ with $\ell_{f}(x,y)=f(x)\ell_{s}(x,y)f(y).$
Let $L_f = \ell_{f}\idmat$ and $L_s = \ell_{s}\idmat$. 
Then, we have 
\begin{equation}
 m(x)\weight_{\polyorder-1}(x)L_{f}(x,y)\weight_{\polyorder-1}(y)m(y)^{\top}=\weight_1(x)^{-1}
\underbrace{\tilde{m}(x)L_{s}(x,y)\tilde{m}(y)^{\top}}_{\tilde{L}_{s}(x,y)}\weight_1(y)^{-1},    
\label{eq:kernel-transform-translation-invariant}
\end{equation}
where $\tilde{m}(x) = f(x)\weight_1(x)\weight_{\polyorder-1}(x)m(x)$. 
By the root-exponential decay of $f$ and the growth conditions 
on $m$ (Assumption \ref{assu:poly-grow-coef} and $\Verts{\la\nabla, m(x)\ra}_2 = O(\Verts{x}_2)$), 
the kernel $\tilde{L}_{s}$ is bounded and has bounded derivatives; 
it also follows from \citet[Proposition 4 (c)]{Barp2024}
and Lemma \ref{lem:matrix-multi-characteristic} that $\tilde{L}_{s}$
is $\mathcal{D}_{L^{1}}^{1}(\bbR^{\datdim})$-characteristic. 
Since $\Verts{\score p(x)}_2 \leq C\weight_1(x)$, 
applying \citet[Theorem 7]{Barp2024} as in the proof for $L^{(1)}$, 
we have that the Stein kernel induced by the LHS of \eqref{eq:kernel-transform-translation-invariant} 
is bounded, and the corresponding RKHS is bounded $P$-separating,
thus controlling tight weak convergence to $P.$ 
It follows from $\rkhs{\ell_{f}}\subset\rkhs{\ell}$ that $\mathcal{T}_{P}^{m}(\rkhs{K})$ is also bounded $P$-separating. 

The $\polyorder$-growth approximation property of $\mathcal{T}_{P}^{m}(\rkhs K)$
follows from the $\calC_0^1(\bbR^{\datdim})$-universality of $L^{(2)}$ and \cref{lem:existence-qgrowth-approx-universality}. 
The proof is complete. 
\qed

\subsubsection{Quadratic case $\protect\polyorder_m=1$}
The generalized Fourier transform $\hat{\Phi}$ acts as the density (with respect to the Lebesgue measure) of the finite measure appearing in the integral representation of $\Phi$ via Bochner's theorem~\citep[Theorem 6.6]{Wendland2004}. 
With $\hat{\Phi}$ non-vanishing, the measure has full support, and therefore the kernel $\ell$ is $\mathcal{C}_0^1$-universal~\citep[Theorem 17]{SimonGabriel2018}. 
Moreover, there exists a translation-invariant kernel with root-exponentially tilting~\citep[Theorem 8]{Barp2024}. 
The suggested kernel form $L^{(2)}$ is the same as the linear growth case up to the tilting $1/\sqrt{\tau^2+\Verts{x}_2^2}$, which does not affect the root-exponential tilting.
Thus, the rest of the proof follows essentially the same. 
\qed

\subsection{Wasserstein convergence implies KSD convergence}
\label{subsec:Wass-conv-imply-KSD}
This appendix provides proofs for \cref{prop:Wass-conv-implies-KSD} and \cref{cor:Wass-imply-KSD-conv-kernel-cond,cor:ksd-Wass-equiv}

\subsubsection{\pcref{prop:Wass-conv-implies-KSD}}
\label{subsec:proof-wass-conv-imply-KSD}
The proof essentially follows \citet[Theorem 2]{Barp2024}. 
Suppose $Q_{\seqidx}\to P$ in $\polyorder$-Wasserstein convergence. 
Then, for sufficiently large $\seqidx_0\geq1$, we have $Q_{\seqidx}\in \finitemomentspace\polyorder$ for each $\seqidx\geq \seqidx_0$. Thus, without loss of generality, we may assume $Q_{\seqidx}\in \finitemomentspace\polyorder$ for all $\seqidx\geq1$. 
For $Q\in \finitemomentspace{\polyorder}$, let $\tilde{Q}$ be a measure defined as $\tilde{Q}(A) = \int_A (1+\Verts{x}_2^{\polyorder})\dd Q(x)$. 
Then, the sequence $(\tilde{Q}_1, \tilde{Q}_2,\dots)$ weakly converges to $\tilde{P}$. 
Since the product measure of finite nonnegative measures is weakly continuous~\citep[Theorem 3.3, p.47]{Berg_1984}, 
we have  
\begin{align*}
 &\ksd{K}{P}(Q_{\seqidx})^2 \\
 &= \int \frac{k_{p,K}(x,y)}{(1+\Verts{x}_2^{\polyorder}) (1+\Verts{y}_2^{\polyorder})} \tilde{Q}_{\seqidx}\otimes \tilde{Q}_{\seqidx}\bigl(\dd(x,y)\bigr)\\
&\to \int \frac{k_{p,K}(x,y)}{(1+\Verts{x}_2^{\polyorder}) (1+\Verts{y}_2^{\polyorder})} \tilde{P}\otimes \tilde{P}\bigl(\dd(x,y)\bigr)= \ksd{K}{P}(P)^2 = 0\ (\seqidx\to \infty), 
\end{align*}
where, in the limit, we have used the continuity and the boundedness of the integrand, derived respectively from the continuity of $k_{p, K}$ and the $\polyorder$-growth of $x\mapsto \sqrt{k_{p,K}(x,x)}$. 
\qed

\subsubsection{\pcref{cor:Wass-imply-KSD-conv-kernel-cond}}
For a matrix-valued kernel $K$, let us define 
$K_1:\bbR^{\datdim}\times \bbR^{\datdim} \to \bbR^{\datdim\times\datdim\times\datdim}$ and 
$K_{12}:\bbR^{\datdim}\times \bbR^{\datdim} \to \bbR^{\datdim\times\datdim\times\datdim\times\datdim}$ by 
\[ 
K_1(a,b) = \nabla_x K(x,y)\vert_{x=a, y=b}\ \text{and}\ K_{12}(a,b ) = \nabla_x \otimes \nabla_y K(x,y)\vert_{x=a, y=b}. 
\]
By the derivative reproducing property~\citep[Lemma 4]{Barp2024}, these functions satisfy the following:
\begin{equation}
    \begin{aligned}
    (K_1(a, b))_{i,j,l} &\coloneqq \partial_{1,i}K_{jl}(a,b)\\
    &\leq \sqrt{ \la e_j, \partial_{1, i}\partial_{2, i}K(a, a) e_j\ra} \sqrt{ \la e_l, K(b,b)e_l \ra }\\
    \end{aligned}
    \label{eq:K1-upperbound}
\end{equation}
and
\begin{equation}
    \begin{aligned}
    (K_{12}(a, b))_{p,i,q,j} & \coloneqq \partial_{1,p}\partial_{2,i}K_{qj}(a,b)\\
    &\leq \sqrt{\la e_q, \partial_{ 1,p}\partial_{2,p}K(a,a) e_q\ra}  \sqrt{ \la e_j, \partial_{1, i }\partial_{2,i}K(b, b) e_j\ra}, \\
    \end{aligned}
    \label{eq:K12-upperbound}
\end{equation}
where $e_{\dimidx}$ denotes the $\dimidx$th standard basis in $\bbR^{\datdim}$. 

For a matrix-valued function $A:\bbR^{\datdim}\to\bbR^{\datdim\times\datdim}$, we define $\mathcal{T}_P^m A= \la\nabla, pmA\ra/p$.  
The Stein kernel is then expressed as $k_{p,K_m} = \mathcal{T}_{P}^{m,2}\mathcal{T}_{P}^{m,1} K$, where $\mathcal{T}_P^{m,1}$ and $\mathcal{T}_P^{m,2}$ denote applying $\mathcal{T}_P^m$ to the first and second arguments, respectively. 

Recall $2b(x)=\la\nabla,p(x)m(x)\ra/p(x)$. Since $\mathcal{T}_P^m A(x)= 2\la b(x), A(x)\ra + \la m(x)^\top\nabla, A(x)\ra$, we have 
\begin{align}
    \begin{aligned}
        k_{p, K_m}(x, y) &= 4\la b(x),K(x,y)b(y)\ra + 2\bigl\la b(y),\la m(x)^{\top}\nabla_{1},K(x,y)\ra\bigr\ra\\
        &\hphantom{=}+2\bigl\la b(x),\la m(y)^{\top}\nabla_{2},K(x,y)\ra\bigr\ra+\bigl\la m(y)^{\top}\nabla_{2},\la m(x)^{\top}\nabla_{1},K(x,y)\ra\bigr\ra.
    \end{aligned}
\label{eq:Stein-kernel-expanded}
\end{align}
In particular, from \cref{eq:Stein-kernel-expanded}, we obtain 
\begin{align*}
&k_{p,K_m}(x,x)\\
&\leq 4\Verts{b(x)}_{2}^2\Verts{K(x,x)}_{\mathrm{op}} + 4 \datdim \Verts{b(x)}_{2}\Verts{m(x)}_{\mathrm{op}} \Verts{K_1(x,x)}_{\mathrm{F}} + \datdim \Verts{m(x)}_{\mathrm{op}}^2 \Verts{K_{12}(x,x)}_{\mathrm{F}}\\
&\leq 4\Verts{b(x)}_{2}^2 \Verts{K(x,x)}_{\mathrm{op}} + 4\datdim^{2} \Verts{b(x)}_{2}
\sqrt{\Verts{K(x,x)}_{\mathrm{op}}}\Verts{m(x)}_{\mathrm{op}} 
\sqrt{ \sum_{\dimidx=1}^{\datdim} \Verts{ \partial_{1,\dimidx}\partial_{2,\dimidx}K(x, x) }_{\mathrm{op}}}\\
& \hphantom{\leq} \quad + \datdim^2 \Verts{m(x)}_{\mathrm{op}}^2 \sum_{\dimidx=1}^{\datdim} \Verts{\partial_{1, \dimidx}\partial_{2, \dimidx}K(x, x)}_{\mathrm{op}},
\end{align*}
where the second inequality follows from \eqref{eq:K1-upperbound} and \eqref{eq:K12-upperbound}. Thus, 
under \cref{assu:poly-grow-coef} and the supposed conditions on the kernel, we have 
the following evaluation: 
\begin{align*}
\sqrt{k_{p,K_m}(x,x)} &\leq 2C_0\lambda_b (1+\Verts{x}_2) (1+\Verts{x}_2)^{\polyorder-1} \\
&\hphantom{\leq}+ 2 \sqrt{\lambda_b \lambda_m C_0 C_1} \datdim^{5/4}\sqrt{ 
(1+\Verts{x}_2^{\polyorder_m+1}) (1+\Verts{x}_2)^{2\polyorder-\polyorder_m-1}
}\\
& \hphantom{\leq}+ C_1 \datdim \lambda_m (1+\Verts{x}_2^{\polyorder_m+1}) (1+\Verts{x}_2)^{\polyorder-\polyorder_m-1},
\end{align*}
which concludes the proof. 
\qed

\subsubsection{\pcref{cor:ksd-Wass-equiv}}
\label{subsec:proof-KSD-Wass-equiv}
Note that translation invariant kernels $\ell$ have both $\ell(x,x)$ and $\partial_{1,i}\partial_{2,i}\ell(x,x)$  bounded. 
Also, since $\rkhs{L_0}\subset \calC_0^1$, the boundedness of $L_0(x,x)$ and $\partial_{1,i}\partial_{2,i}L_0(x,x)$ holds~\citep[Lemma 3 (a)-(b)]{Barp2024}. The required growth conditions for the three kernel choices can be easily checked from the form (\cref{def:kernel-form}). 
\qed
\subsubsection{\pcref{cor:default-kernel}}
\label{subsec:proof-user-friendly-KSD-Wass-equiv}
The suggested kernel form $k_{\imqp}(x,y)$ corresponds to setting $\ell$ in $L^{(2)}$ to the IMQ kernel $k_{\imq}$. 
The IMQ kernel is a translation-invariant kernel defined by an infinitely differentiable function $\Phi$, and according to \citet[Theorem 8.15]{Wendland2004}, 
has a continuous nonvanishing generalized Fourier transform
\[
    \hat{\Phi}(\omega) = \frac{\sqrt{2}}{\Gamma(-\beta)} \Verts{\omega}_2^{-(\datdim-1)/2} K_{(\datdim-1)/2}(\Verts{\omega}_2),
\]
where $\Gamma$ is the Gamma function, and  $K_{(\datdim-1)/2}$ is the modified Bessel function of the second kind of order $(\datdim-1)/2$. 
Thus, $k_{\mathrm{IMQ}}$ satisfies the requirements in \cref{thm:ksd-equiv-to-Wass-conv}. 
\qed
\section{Coercive functions and $\protect\polyorder$-Wassserstein convergence
\label{sec:Coercive-functions-and}}

As mentioned in \cref{rem:coercive} in the main text, we may obtain a KSD that enforces $\polyorder$th moment uniform integrability by assuming a higher-order moment. 
This appendix provides results relevant to this claim. 

We begin with two definitions: 
\begin{defn}[Coercive functions]
Let $\polyorder_{0}\in[0,\infty).$ A function $f:\bbR^{\datdim}\to\bbR$
is said to be coercive of order $\polyorder_{0}$ if there exists
$r_{0}>0$ such that $\inf_{\Verts x_{2}\leq r_{0}}f(x)>\nu$ for
some $\nu\in\bbR;$ and $f(x)\geq\eta\Verts x_{2}^{\polyorder_{0}}$
for some $\eta>0$ if $\Verts x_{2}>r_{0}.$ 
\end{defn}

\begin{defn}[Integrability rate] 
\label{def:integrability-rate}
For $Q\in{\cal P}_{\polyorder}$ and $\varepsilon>0,$ we define the
$\polyorder$th moment integrability rate by 
\[
R_{\polyorder}(Q,\varepsilon)\coloneqq\inf\left\{ r\geq1:\int_{\{\Verts x_{2}>r\}}\Verts x_{2}^{\polyorder}\dd Q(x)\leq\varepsilon\right\} 
\]
\end{defn}

An order-$\polyorder_{0}$ coercive function approximates the $\polyorder_{0}$-growth
outside a ball, while it is only assumed to be bounded below inside
the ball. The integrability rate above represents the radius of
a ball, outside of which the tail moment integral becomes negligible.
Note that $Q\in{\cal P}_{\polyorder}$ is equivalent to having $R_{\polyorder}(Q,\varepsilon)<\infty$
for each $\varepsilon>0.$ In particular, if a sequence $\{Q_{\seqidx}\}_{n\geq1}$
does not have uniformly integrable $\polyorder$-th moments, the integrability
rate $R_{\polyorder}(Q_{\seqidx},\varepsilon)$ diverges. The case
$\polyorder=0$ corresponds to the tightness rate used by \citet[Appendix G]{GorMac2017}. 

We first show that if the Stein RKHS admits a coercive function, the
KSD then enforces the uniform integrability. 
\begin{lem}[KSD upper-bounds the integrability rate]
\label{lem:ksd-ubound-intrate} Let $\polyorder\in[0,\infty).$ Let
$\rkhs K$ be the RKHS of $\bbR^{\datdim}$-valued functions defined
by a matrix-valued kernel $K:\bbR^{\datdim}\times\bbR^{\datdim}\to\bbR^{\datdim\times\datdim}$
for which ${\cal T}_{P}^{m}({\cal H}_{K})$ exists. Assume that the
Stein RKHS ${\cal T}_{P}^{m}({\cal H}_{K})$ contains a coercive function
of order $\polyorder+\theta$ for some $\theta>0.$ Then, for sufficiently
small $\varepsilon>0,$ we have 
\begin{align*}
R_{\polyorder}(Q,\varepsilon) & \leq\left\{ 2\left(1+\frac{\polyorder}{\theta}\right)\left(\frac{\ksd{K_{m}}P(Q)-\nu}{\eta\varepsilon}\right)\right\} ^{\frac{1}{\theta}\lor\frac{\polyorder}{\theta}}. 
\end{align*}
Thus, for a sequence $(Q_{1},Q_{2}\dots)\subset{\cal P}$
we have 
\[
\limsup_{\seqidx\to\infty}\ksd{K_{m}}P(Q_{\seqidx})<\infty\Rightarrow\limsup_{\seqidx\to\infty}R_{\polyorder}(Q_{\seqidx},\varepsilon)<\infty.
\]
In particular, if the sequence $\{Q_{1},Q_{2}\dots\}$ does not have
uniformly integrable $\polyorder$th moments, then $\ksd{K_{m}}P(Q_{\seqidx})$
diverges as $\seqidx\to\infty.$ 
\end{lem}

\begin{proof}
Let $f(x)=\Verts x_{2}^{\polyorder}1\{\Verts x_{2}>r\}.$ We consider
the integral 
\[
\int_{\{\Verts x_{2}>r\}}\Verts x_{2}^{\polyorder}\dd Q(x)=\int f(x)\dd Q(x)=\int_{0}^{\infty}Q(\{f(x)>t\})\dd t.
\]
By dividing the range of the integral, we obtain 
\begin{align*}
\int_{0}^{\infty}Q(\{f(x)>t\})\dd t & =r^{\polyorder}Q(\{\Verts x_{2}>r\})+\int_{r^{\polyorder}}^{\infty}Q(\{f(x)>t\})\dd t\\
 & =r^{\polyorder}Q(\{\Verts x_{2}>r\})+\int_{r}^{\infty}Q\bigl\{\Verts x_{2}>t^{1/\polyorder}\bigr\}\dd t,
\end{align*}
where we regard the second term as zero when $\polyorder=0.$ 

We evaluate the tail probabilities in terms of the KSD. By assumption,
we have a function $v\in\rkhs K$ such that ${\cal T}_{P}^{m}v$ is
a coercive function of order $\polyorder+\theta;$ i.e., there exists
$r_{0}>0$ such that $\mathcal{T}_{P}^{m}v(x)\geq\eta\Verts x_{2}^{\polyorder+\theta}$
for $\Verts x_{2}>r_{0}$ and $\inf_{\Verts x_{2}\leq r_{0}}\mathcal{T}_{P}^{m}v(x)>\nu,$
where $\eta>0$ and $\nu\in\bbR.$ Following the proof of \citet[Lemma 17]{GorMac2017},
we define $\gamma(r)=\inf\{{\cal T}_{P}^{m}v(x)-\nu:\Verts x_{2}\geq r\}.$
Define $\text{\ensuremath{r_{\gamma}}}\coloneqq r_{0}\lor2(\verts{\nu}/\eta)^{1/(\polyorder+\theta)}.$
It is straightforward to check that for $r\geq r_{\gamma},$ we have
$\gamma(r)\geq\eta r^{\polyorder+\theta}/2.$ By Markov's inequality,
\begin{align*}
Q(\{\Verts x_{2} & >r\})\leq\frac{\EE_{X\sim Q}\gamma(\Verts X_{2})}{\gamma(r)}\leq\frac{\EE_{X\sim Q}[\mathcal{T}_{P}^{m}v(X)-\nu]}{\gamma(r)}\leq\frac{\ksd{K_{m}}P(Q)-\nu}{\gamma(r)}.
\end{align*}
These observations yield the following estimate of the above integral:
\begin{align*}
\int_{\{\Verts x_{2}>r\}}\Verts x_{2}^{\polyorder}\dd Q(x) & =r^{\polyorder}Q(\{\Verts x_{2}>r\})+1_{\{\polyorder>0\}}\int_{r}^{\infty}Q(\{\Verts x_{2}>t^{1/\polyorder}\})\dd t\\
 & \leq r^{\polyorder}\frac{\ksd{K_{m}}P(Q)-\nu}{\gamma\bigl(r\bigr)}+1_{\{\polyorder>0\}}\int_{r}^{\infty}\frac{\ksd{K_{m}}P(Q)-\nu}{\gamma\bigl(t\bigr)}\dd t\\
 & \leq2\frac{\ksd{K_{m}}P(Q)-\nu}{\eta r^{\theta}}+21_{\{\polyorder>0\}}\int_{r}^{\infty}\frac{\ksd{K_{m}}P(Q)-\nu}{\eta t^{1+\theta/\polyorder}}\dd t\\
 & \leq2\frac{\ksd{K_{m}}P(Q)-\nu}{\eta}\left(\frac{1}{r^{\theta}}+\frac{\polyorder}{\theta}\frac{1}{r^{\theta/\polyorder}}1_{\{\polyorder>0\}}\right)\\
 & \leq2\left(1+\frac{\polyorder}{\theta}\right)\frac{\ksd{K_{m}}P(Q)-\nu}{\eta}\frac{1}{r^{\theta\land\theta/\polyorder}}
\end{align*}
where we assume $r\ge r_{\gamma}\lor1$ and use the convention $\theta/0=\infty.$
Therefore, for $\epsilon>0,$ by taking sufficiently large $r_{\epsilon}\geq1$
such that 
\[
2\left(1+\frac{\polyorder}{\theta}\right)\frac{\ksd{K_{m}}P(Q)-\nu}{\eta}\frac{1}{r_{\varepsilon}^{\theta\land\theta/\polyorder}}\leq\varepsilon\ \text{and}\ r_{\varepsilon}\geq r_{\gamma},
\]
we have 
\[
\int_{\Verts x_{2}>r_{\epsilon}}\Verts x_{2}^{\polyorder}\dd Q(x)\leq{\cal \varepsilon}.
\]
Therefore, the $\polyorder$th moment integrability rate $R_{\polyorder}(Q,\varepsilon)$
satisfies
\[
R_{\polyorder}(Q,\varepsilon)\leq\left\{ 2\left(1+\frac{\polyorder}{\theta}\right)\left(\frac{\ksd{K_{m}}P(Q)-\nu}{\eta\varepsilon}\right)\right\} ^{\frac{1}{\theta}\lor\frac{\polyorder}{\theta}}\vee r_{\gamma}.
\]
For sufficiently small $\varepsilon,$ the KSD term dominates $r_{\gamma}.$
Thus, the claim has been proved. 
\end{proof}
A similar bound on the tightness rate was obtained by \citet[Lemma 17]{GorMac2017}. 
\cref{lem:ksd-ubound-intrate} establishes a result similar to
\cref{cor:q-approx-uniform-integrability}. A major difference
is that \cref{cor:q-approx-uniform-integrability} does not require $P \in \finitemomentspace{\polyorder+\theta}$.  
\cref{lem:ksd-ubound-intrate} essentially states that the KSD controls a higher-order moment; 
that is, under the coercivity assumption, the condition $\limsup_{\seqidx\to\infty} \ksd{K_m}{P}(Q_{\seqidx})<\infty$ implies $\limsup_{\seqidx\to \infty} \EE_{X\sim Q_{\seqidx}} [\Verts{X}_2^{\polyorder+\theta}]<\infty$, and hence controls $\polyorder$th moment uniform integrability. 
Note that while we may take as small $\theta$ as possible for enforcing uniform integrablity, assuming a higher order moment provides a benefit of a tighter bound on $R_{\polyorder}(Q, \varepsilon)$. 

We can construct an RKHS admitting a coercive function.
\begin{lem}[Existence of coercive functions] 
\label{lem:existence-coercive} Let $\polyorder\in[0,\infty).$ Assume
the dissipativity condition in Assumption \ref{assu:dissipativity}
with $\alpha,\beta_{1},\beta_{0}.$ Suppose Assumption \ref{assu:poly-grow-coef}
holds for some constants $c_{m}>0$ and $\polyorder_{m}\in\{0,1\}$;
if $\polyorder_{m}=1,$ additionally assume $\alpha>\lambda_{m}(\polyorder-2).$
Let $v(x)=-w_{\polyorder-2}(x)x$ with $\weight_{\polyorder-2}(x)=\bigl(\tau^{2}+\Verts x_{2}^{2}\bigr)^{(\polyorder-2)/2}$. 
Then, there exists $r>0$ $\mathcal{T}_{P}^{m}v(x)\geq\eta\Verts x_{2}^{\polyorder},$
and $\min_{\Verts x_{2}\leq r_{0}}\mathcal{T}_{P}^{m}v(x)>-\infty,$
where $\eta$ is determined by the following quantities: $\alpha,$
$\beta_{1},$ $\beta_{0},$ $\lambda_{m},$ $\tau,$ $\polyorder_{m},$
and $\polyorder.$ 
\end{lem}

\begin{proof}
The proof proceeds as in that of Lemma \ref{lem:norm-indicator-dominator};
specifically, the argument for the regime $\Verts x_{2}>r+\delta$
in Lemma \ref{lem:norm-indicator-dominator} shows that there exists
$\tilde{r}>0$ such that $\mathcal{T}_{P}^{m}v(x)\geq\eta\Verts x_{2}^{\polyorder}$
for some $\eta>0$ if $\Verts x_{2}>\tilde{r}.$ Since $v$ is
continuously differentiable, $\mathcal{T}_{P}^{m}v(x)$ is continuous,
which implies that there exists a minimum inside the ball of radius
$\tilde{r}.$ 
\end{proof}
\begin{lem}[Stein RKHS admits coercive functions]
\label{lem:uniform-integrability-coercive}
Fix $\polyorder \in [0, \infty)$. 
Define $K:\bbR^{\datdim}\times\bbR^{\datdim}\to\bbR^{\datdim\times\datdim}$ by 
\begin{equation}
    K(x,y)=L(x,y)+\weight_{\polyorder-2}(x)k_{\mathrm{lin}}(x,y) \weight_{\polyorder-2}(y)\idmat,\label{eq:rkhs-coercive}
\end{equation}
where $\weight_{\polyorder-2}(x)=\bigl(\tau^{2}+\Verts x_{2}^{2}\bigr)^{(\polyorder-2)/2}$
with $\tau>0$; 
$L$ is a matrix-valued kernel with $\rkhs{L} \subset \calC^1(\bbR^{\datdim})$, and 
$k_{\mathrm{lin}}(x,y)= \la x,y\ra + \tau^{2}$. 
Assume the dissipativity condition in Assumption \ref{assu:dissipativity}
with $\alpha,\beta_{1},\beta_{0}.$ Assume the growth conditions in
Assumption \ref{assu:poly-grow-coef}. If $\polyorder_{m}$ in Assumption
\ref{assu:poly-grow-coef} is equal to $1,$ additionally assume that
$\alpha>\lambda_{m}(\polyorder-2).$ Then, ${\cal T}_{P}^{m}(\rkhs K)$
admits a coercive function of order $\polyorder$. 
\end{lem}

\begin{proof}
The function $v=-\tilde{w}(x)x$ is an element of the RKHS of the
second kernel in (\ref{eq:rkhs-coercive}). The conclusion follows
from Lemma \ref{lem:existence-coercive}. 
\end{proof}
We end this appendix by establishing a result analogous to Theorem
\ref{thm:ksd-equiv-to-Wass-conv}. 
\begin{thm}[KSD convergence implies $\polyorder$-Wasserstein convergence] 
Fix $\polyorder, \theta \in (0,\infty)$. 
Suppose $P \in \finitemomentspace{\polyorder+\theta}$. 
Define base kernel K as in \eqref{eq:q-growth-approx-kernel} 
using tilting $\weight_{\polyorder+\theta-1}(x)=\bigl(\tau^{2}+\Verts x_{2}^{2}\bigr)^{(\polyorder+\theta-1)/2}$. 
Specify the kernel $L$ by either of the following two kernels: 
\begin{align*}
 & L^{(1)}(x,y)=\frac{L_{0}(x,y)}{{\weight}_{\polyorder+\theta+\polyorder_m}(x)\weight_{\polyorder+\theta+\polyorder_m}(y)},\ L^{(2)}(x,y)=\frac{\ell(x,y)}{\weight_{\polyorder_m}(x)\weight_{\polyorder_m}(y)} \idmat,
\end{align*}
where 
$L_{0}$ is $\mathcal{C}_{0}^{1}(\bbR^{\datdim})$-universal, and 
$\ell$ is a scalar-valued translation-invariant ${\calC}_{0}^{1}$-universal 
kernel. 
If $\polyorder_{m}=1$ in \cref{assu:poly-grow-coef}, assume $\alpha>\lambda_{m}(\polyorder+\theta-2)$. 
Then, for any sequence $(Q_{\seqidx})_{\seqidx\geq1}\subset\mathcal{P}$, under \cref{assu:poly-grow-coef} and \cref{assu:dissipativity}, 
$Q_{\seqidx}\toL{\polyorder} P$ if $\ksd{K_{m}}P(Q_{\seqidx})\to0$ ($\seqidx\to\infty)$. 
\end{thm}

\begin{proof}
The proof is essentially the same as that of Theorem \ref{thm:ksd-equiv-to-Wass-conv}:
the kernel choice ensures weak convergence control, along with the
uniform integrability guarantee by Lemma \ref{lem:uniform-integrability-coercive}. 
\end{proof}

\section{Polynomial functions are pseudo-Lipschitz functions \label{sec:Polynomial-functions-are}}

We show that the pseudo-Lipschitz class ${\cal F}_{\polyorder}$, used in \cref{subsec:Rate-of--Wasserstein}, suffices for characterizing
convergence in moments. 
\begin{lem}[Monomial functions are pseudo-Lipschitz]
\label{lem:monomial-plip-const}
Let $\polyorder\geq1$ be an integer. Let $\bm{\polyorder}=(\polyorder_{1},\dots,\polyorder_{\datdim})\in\{0,\dots,\polyorder\}^{\datdim}$
be a multi-index such that $\sum_{\dimidx=1}^{\datdim}\polyorder_{\dimidx}=\polyorder.$
Then, $x^{\bm{\polyorder}}\coloneqq\prod_{\dimidx=1}^{\datdim}x_{\dimidx}^{\polyorder_{\dimidx}}$
is pseudo-Lipschitz of order $\polyorder-1.$ Its pseudo-Lipschitz
constant $\plipconst{x^{\bm{\polyorder}}}{\polyorder-1}$ is bounded
by $(1\lor\polyorder/2) \Verts{\bm{\polyorder}}_2$. 
Moreover, it satisfies 
\begin{align*}
\sup_{x\in\bbR^{\datdim}}\frac{\Verts{\nabla^{j}x^{\bm{\polyorder}}}_{\mathrm{op}}}{1+\Verts x_{2}^{\polyorder-1}} & \leq \datdim^{j/2} \max_{\dimidx\in\{1,\dots, \datdim\}} \frac{q_{\dimidx}!}{\{1_{\{q_{\dimidx} \geq j\}}(q_i-j)\}!}. 
\end{align*} 
Thus, $x\mapsto c_{\polyorder,\datdim}^{-1}x^{\bm{\polyorder}}$ is an element of $\calF_{\polyorder}$ (see \cref{subsec:plipfuncs})
where 
\[
c_{\polyorder, \datdim} = (1\lor\polyorder/2) \Verts{\bm{\polyorder}}_2
\lor \max_{j=2,3} d^{j/2} 1\{\polyorder\geq j\}\max_{i\in\{1, \dots, \datdim\}} \frac{ \polyorder_i! }{ \{1_{\{q_i\geq j\}}(q_i-j)\}! },
\]
and $0!\coloneqq 1$.
\end{lem}

\begin{proof}
Let $f_{\bm{\polyorder}}: x \mapsto x^{\bm{\polyorder}}$. Note that we have 
$
\Verts{\nabla f_{\bm{\polyorder}}(x)}_2 \leq 
\Verts{\bm{\polyorder}}_2 \Verts{x}^{\polyorder-1}_2 
$.
Thus, by the mean value theorem, we obtain 
\begin{align*}
    \verts{x^{\bm{\polyorder}} - y^{\bm{\polyorder}}} 
    & \leq \Verts{\bm{\polyorder}}_2 \Verts{x-y}_2 \int_0^1 \Verts{tx + (1-t)y}^{\polyorder-1}_2 \dd t\\
    & \leq \Verts{\bm{\polyorder}}_2 (1\lor \polyorder/2) \bigl(1 + \Verts{x}_2^{\polyorder-1} + \Verts{y}_2^{\polyorder-1} \bigr)\Verts{x-y}_2. 
\end{align*}

Next we check the growth of the 
$j$-th derivatives. We assume $\polyorder\geq j$ below, as the derivatives
are zero otherwise. Note that we have
\begin{align*}
(\nabla^{j}x^{\bm{n}})_{l_{1},\dots,l_{j}} 
& = \prod_{\dimidx=1}^{\datdim}\frac{\polyorder_{\dimidx}!}{(\polyorder_{\dimidx}-m_{\dimidx})!}\cdot x_{\dimidx}^{\polyorder_{\dimidx}-m_{\dimidx}}\cdot 1{\{\polyorder_{\dimidx}\geq m_{\dimidx}\}} \\
& \leq \Verts{x}_2^{\polyorder-1} \max_{\dimidx\in\{1,\dots, \datdim\}} \frac{q_{\dimidx}!}{\{1_{\{q_{\dimidx} \geq j\}}(q_i-j)\}!}
\end{align*}
where $m_{\dimidx} = \sum_{k=1}^j 1\{\polyorder_{k}=\dimidx\}$, for at least one index $\dimidx$ satisfies $\polyorder_{\dimidx} \geq m_{\dimidx}$. 
Therefore, 
\begin{align*}
\sup_{x\in\bbR^{\datdim}}\frac{\Verts{\nabla^{j}x^{\bm{\polyorder}}}_{\mathrm{op}}}{1+\Verts x_{2}^{\polyorder-1}} & \leq \datdim^{j/2} \max_{\dimidx\in\{1,\dots, \datdim\}} \frac{q_{\dimidx}!}{\{1_{\{q_{\dimidx} \geq j\}}(q_i-j)\}!}. 
\end{align*} 
 
\end{proof}

\section{Characterization of pseudo-Lipschitz metrics\label{sec:Characterization-of-pseudo-Lipsc}}
In this appendix, we first prove that pseudo-Lipschitz IPMs on a general metric space metrize $\polyorder$-Wasserstein convergence 
(see \cref{sec:q-wass-conv-necessary} for the definition of $\polyorder$-Wasserstein convergence). 
We then provide a proof of \cref{prop:plip-IPM-convergence}. 

We first define necessary concepts. 

\begin{defn}[Uniform integrability]
Let $({\cal X},d)$ be a separable metric space. 
A sequence $(P_{\seqidx})_{\seqidx\geq1}\subset\finitemomentspace{}({\cal X})$
is said to have uniformly integrable $\polyorder$th moments if
\[
\lim_{r\to\infty}\limsup_{\seqidx\to\infty}\int_{d(x,x_{0})\geq r}d(x,x_{0})^{\polyorder}\dd P_{\seqidx}(x)=0
\]
for some (and then any) $x_{0}\in{\cal X}.$ 
\end{defn}

\begin{defn}[Pseudo-Lipschitz functions on general metric spaces]
\label{def:plip-metric-space}Let $({\cal X},d)$ be a metric space.
For $V:{\cal X}\to[0,\infty),$ a real-valued function $f$ on ${\cal X}$
is said to be $V$-pseudo-Lipschitz if it satisfies 
\[
\Verts f_{\mathrm{L},V}\coloneqq\sup_{x\neq y}\frac{\verts{f(x)-f(y)}}{\left(1+V(x)+V(y)\right)d(x,y)}<\infty.
\]
For $\polyorder\geq0,$ we say that a $V$-pseudo-Lipschitz function
$f$ is of \emph{order} $\polyorder$ if $V$ satisfies the
following conditions for some point $x_{0}$ in ${\cal X}$: (a) for
any $x\in{\cal X},$ and some $A_{0}\geq0,$ $A_{1}>0,$ we have $V(x)\leq A_{0}+A_{1}d(x,x_{0})^{\polyorder},$
and (b) if $\polyorder>0,$ there exists $B>0$ and $R_{V}\geq0$
such that $V(x)>Bd(x,x_{0})^{\polyorder}$ if $d(x,x_{0})>R_{V}.$
\end{defn}

The following lemma provides a trivial example of a pseudo-Lipschitz function. 
\begin{lem}[Metric power is pseudo-Lipschitz]
\label{lem:metric-is-pLip}Let $d$ be a metric on a set ${\cal X}.$
For $\polyorder\in [1,\infty)$ and $x_{0}\in{\cal X},$ the function $x\mapsto d(x,x_{0})^{\polyorder}$
is $V$-pseudo-Lipschitz with $V(x)=d(x,x_{0})^{\polyorder-1};$ thus,
it is of order $\polyorder-1.$ We have $\Verts{d(\cdot,x_{0})}_{\mathrm{L},V}=1$
if $1\leq\polyorder<2,$ and $\Verts{d(\cdot,x_{0})}_{\mathrm{L},V}\leq\polyorder/2$
if $\polyorder\geq2.$ 
\end{lem}

\begin{proof}
The case $\polyorder=1$ follows from the triangle inequality. Consider
$\polyorder>1.$ By the mean value theorem, for any two real numbers
$a,b$, we have 
\begin{align*}
\verts{a^{\polyorder}-b^{\polyorder}} & \leq\polyorder\verts{a-b}\int_{0}^{1}\verts{(1-t)a+tb}^{\polyorder-1}\dd t
\end{align*}
If $1<\polyorder<2,$ by the triangle inequality and $(\alpha+\beta)^{\polyorder-1}\leq\alpha^{\polyorder-1}+\beta^{\polyorder-1}$
for $\alpha,\beta\geq0,$ we have 
\begin{align*}
\verts{a^{\polyorder}-b^{\polyorder}} & \leq\left(1+\verts a^{\polyorder-1}+\verts b^{\polyorder-1}\right)\verts{a-b},
\end{align*}
whereas if $\polyorder>2,$ by Jensen's inequality, 
\[
\verts{a^{\polyorder}-b^{\polyorder}}\leq\frac{\polyorder}{2}\left(1+\verts a^{\polyorder-1}+\verts b^{\polyorder-1}\right)\verts{a-b}. 
\]
Thus, since $\verts{d(x,x_{0})-d(y,x_{0})}\leq d(x,y),$ letting $a=d(x,x_{0})$
and $b = d(y,x_{0})$ gives 
\[
\verts{d(x,x_{0})^{\polyorder}-d(y,y_{0})^{\polyorder}}\leq C\bigl(1+d(x,x_{0})^{\polyorder-1}+d(y,y_{0})^{\polyorder-1}\bigr)d(x,y),
\]
where $C=1$ if $1\leq\polyorder<2,$ and $C=\polyorder/2$ if $\polyorder\geq2.$ 
\end{proof}
The following proposition shows that $\polyorder$-Wasserstein convergence is induced by the IPM that is defined
by a pseudo-Lipschitz class. 
\begin{prop}[Pseudo-Lipschitz metric metrizes $\polyorder$-Wasserstein convergence]
\label{prop:plip-metric-Wasserstein-conv}
Let $({\cal X},d)$ be a separable metric space. 
Fix $\polyorder\geq1$. 
Let $P \in \finitemomentspace{\polyorder}(\mathcal{X})$, 
and $d_{\plipspace{\polyorder-1}(V)}$
be the IPM defined by the subset $\plipspace{\polyorder-1}(V)=\{f:\calX\to\bbR:\Verts f_{\mathrm{L},V}\leq1\}$
of $V$-pseudo Lipschitz functions of order $\polyorder-1.$ For a
sequence of probability measures $(P_{\seqidx})_{\seqidx\geq 1}\subset\finitemomentspace{}({\cal X}),$
the following statements are equivalent: (a) $d_{\plipspace{\polyorder-1}(V)}(P, P_{\seqidx})\to0$
as $\seqidx\to\infty,$ and (b) $P_{\seqidx}$ converges to $P$ in
the sense of $\polyorder$-Wasserstein convergence. 
\end{prop}

\begin{proof}
We may assume $(P_n)_{\seqidx \geq 1} \subset \finitemomentspace{\polyorder}(\mathcal{X})$, since if either of (a) or (b) holds, then $(P_{\seqidx})_{\seqidx \geq \seqidx_0} \subset \finitemomentspace{\polyorder}(\mathcal{X})$ for some $\seqidx_0 \geq 1$.

$\text{(a)}\ \Rightarrow\ \text{(b):}$ By \citet[Theorem 7.12]{Villani_2003},
it suffices to show that $(P_{\seqidx})_{\seqidx\geq1}$ converges
weakly to $P$ and $\EE_{P_{\seqidx}}[d(\cdot,x_{0})^{\polyorder}]\to\EE_{P}[d(\cdot,x_{0})^{\polyorder}].$
For the first claim, note that the set $\plipspace{\polyorder-1}(V)$
contains the set $\mathrm{BL}=\{f:{\cal X}\to\bbR:\Verts f_{\mathrm{BL}}\coloneqq\Verts f_{\infty}+\Verts f_{\mathrm{L}}\leq1\}$
where $\Verts f_{\infty}=\sup_{x\in{\cal X}}\verts{f(x)}$ and $\Verts f_{\mathrm{L}}=\sup_{x\neq y}\verts{f(x)-f(y)}/d(x,y).$
The convergence $d_{\plipspace{\polyorder-1}}(P, P_{\seqidx})\to0$
thus implies $d_{\mathrm{BL}}(P, P_{\seqidx})\to0.$ The sequence
converges to $P$ weakly, since $d_{\mathrm{BL}}$ metrizes  
weak convergence \citep[Theorem 11.3.3]{Dudley2002}. To prove the second
claim, note that by Definition \ref{def:plip-metric-space}, for $\polyorder>1,$
there exist a point $x_{0}\in{\cal X},$ and constants $B>0,$ $R_{V}\geq0$
such that $V(x)\geq Bd(x,x_{0})^{\polyorder-1}$ if $d(x,x_{0})>R_{V}.$
Then, $f_{\polyorder}(x)=d(x,x_{0})^{\polyorder}$ satisfies, for
any $x,y\in{\cal X},$ 
\begin{align*}
\frac{\verts{f_{\polyorder}(x)-f_{\polyorder}(y)}}{\bigl(1+V(x)+V(y)\bigr)} & \leq\frac{1+f_{\polyorder-1}(x)+f_{\polyorder-1}(x)}{1+V(x)+V(y)}\frac{\verts{f_{\polyorder}(x)-f_{\polyorder}(y)}}{1+f_{\polyorder-1}(x)+f_{\polyorder-1}(y)}\\
 & \leq\left(1+2\sup_{x\in{\cal X}}\frac{d(x,x_{0})^{\polyorder-1}}{1+V(x)}\right)\frac{\verts{f_{\polyorder}(x)-f_{\polyorder}(y)}}{1+f_{\polyorder-1}(x)+f_{\polyorder-1}(y)}\\
 & \leq C_{\polyorder}\cdot\frac{\verts{f_{\polyorder}(x)-f_{\polyorder}(y)}}{1+f_{\polyorder-1}(x)+f_{\polyorder-1}(y)}
\end{align*}
with 
\[
C_{\polyorder}=\begin{cases}
3 & \text{if}\ \polyorder=1\\
\left(1+2\bigl(R_{V}^{\polyorder-1}\lor B^{-1})\right) & \text{if}\ \polyorder>1.
\end{cases}
\]
This observation together with \cref{lem:metric-is-pLip} shows
that $f_{\polyorder}$ is $V$-pseudo-Lipschitz. The conclusion follows
from $d_{\plipspace{\polyorder-1}(V)}(P, P_{\seqidx})\to0$ as $\seqidx\to\infty.$ 

$\text{(b)}\ \Rightarrow\ \text{(a):}$ Our goal is to show that the
following quantity can be made arbitrarily small by taking sufficiently
large $\seqidx:$
\begin{align*}
 &\sup_{f\in\plipspace{\polyorder-1}(V)}\left|\int f\dd P_{\seqidx}-\int f\dd P\right| \\
& =\sup_{f\in \plipspace{\polyorder-1}(V)}\left|\int\bigl(f-f(x_{0})\bigr)\dd(P_{\seqidx}-P)+\underbrace{\int f(x_{0})\dd(P_{\seqidx}-P)}_{=0}\right|\\
 & =\sup_{f\in \plipspace{\polyorder}(V)}\left|\int\bar{f}\dd(P_{\seqidx}-P)\right|,    
\end{align*}
where $\bar{f}=f-f(x_{0}),$ and $x_{0}$ is a point such that $V(x)\leq A_{0}+A_{1}d(x,x_{0})^{\polyorder-1}$
for any $x\in{\cal X},$ and some $A_{0}\geq0,$ $A_{1}>0.$ 

As a preparatory step, we clarify some properties of $\bar{f}.$ For
any $x,y\in{\cal X},$ we have 
\begin{align*}
\verts{\bar{f}(x)} & \leq\bigl(1+V(x_{0})+V(x)\bigr)d(x,x_{0})\\
 & \leq\bigl(1+A_{0}+V(x_{0})+A_{1}d(x,x_{0})^{\polyorder-1}\bigr)d(x,x_{0})
\end{align*}
and
\begin{align*}
\verts{\bar{f}(x)-\bar{f}(y)} & \leq\bigl(1+V(x)+V(y)\bigr)d(x,y)\\
 & \leq\bigl(1+2A_{0}+A_{1}d(x,x_{0})^{\polyorder-1}+A_{1}d(y,x_{0})^{\polyorder-1}\bigr)d(x,y).
\end{align*}
These estimates imply the following: (i) the growth of $\bar{f}$
is of $d(\cdot,x_{0})^{\polyorder},$ and (ii) the restriction of $\bar{f}$
to a closed ball ${\cal B}_{r}(x_{0})=\{x\in{\cal X}:d(x,x_{0})\leq r\}$
is a bounded Lipschitz function. By \citep[Proposition 11.2.3]{Dudley2002},
the restricted function $\bar{f}_{\given{\cal B}_{r}(x_{0})}$ can
be extended to a function $\tilde{f}$ on ${\cal X}$ without changing
its bounded-Lipschitz norm. In particular, the bounded-Lipschitz norm
$\Verts{\tilde{f}}_{\mathrm{BL}}$ can be bounded by a constant depending
only on $r.$ 

For $R>0,$ consider $g_{R}(x)\coloneqq\bigl(1-d(x,{\cal B}_{R}(x_{0}))\bigr)_{+},$
where $d(x,F)\coloneqq\inf_{x\in F}d(x,F)$ for a set $F;$ the function
satisfies $1_{{\cal B}_{R}(x_{0})}\leq g_{R}\leq1_{{\cal B}_{R+1}(x_{0})},$
and is Lipschitz by \citep[Proposition 11.2.2]{Dudley2002}. For $f\in\plipspace{\polyorder-1}(V),$
consider the bounded-Lipschitz extension $\tilde{f}$ of $f-f(x_{0})$
restricted to the closed ball ${\cal B}_{R+1}(x_{0}).$ Since the
product of bounded Lipschitz functions is again bounded-Lipschitz
\citep[Proposition 11.2.1]{Dudley2002}, the function $\tilde{f}g_{R}$
is bounded-Lipschitz with 
\[
    \Verts{\tilde{f}g_{R}}_{\mathrm{BL}}\leq 
    \Verts{\tilde{f}}_{\mathrm{BL}}\Verts{g_R}_{\mathrm{BL}}\leq  C_R,
\]
where $C_R > 0$ is a constant that does not depend on the choice of $f$. 

We are ready to bound the quantify of interest. For a given $R>0$
and any $\varepsilon>0,$ using the weak convergence assumption, take
$\seqidx$ large enough so that 
\[
d_{\mathrm{BL}}(P_{\seqidx},P)=\sup_{f\in\mathrm{BL}}\left|\int f\dd P_{\seqidx}-\int f\dd P\right|<\frac{\varepsilon}{C_{R}},
\]
which is possible as the bounded Lipschitz metric $d_{\mathrm{BL}}$
metrizes weak convergence \citep[Theorem 11.3.3]{Dudley2002}. The
definition of $R$ has been left unspecified; by \citet[Theorem 7.12]{Villani_2003},
the sequence $\{P_{1},P_{2},\cdots,\}$ has uniformly integrable $\polyorder$-th
moments, and we may therefore take $R\geq1$ such that 
\[
\sup_{\seqidx\geq1}\int_{\{d(x,x_{0})>R\}}d(x,x_{0})^{\polyorder}\dd P_{\seqidx}(x)\lor\int_{\{d(x,x_{0})>R\}}d(x,x_{0})^{\polyorder}\dd P(x)<\varepsilon.
\]
Then, for any $\varepsilon>0,$ we obtain 
\begin{align*}
 & \sup_{f\in\plipspace{\polyorder-1}(V)}\left|\int\bar{f}\dd(P_{\seqidx}-P)\right|\\
 & =\sup_{f\in\plipspace{\polyorder-1}(V)}\left|\int\bar{f}\{(1-g_{R})+g_{R}\}\dd(P_{\seqidx}-P)\right|\\
 & \leq\sup_{f\in\plipspace{\polyorder-1}(V)}\int_{{\cal X}\setminus{\cal B}_{R}(x_{0})}\verts{\bar{f}}\dd(P_{\seqidx}+P)+\sup_{f\in\plipspace{\polyorder-1}(V)}\left|\int\tilde{f}g_{R}\dd(P_{\seqidx}-P)\right|\\
 & \leq C\sup_{f\in\plipspace{\polyorder-1}(V)}\int_{{\cal X}\setminus{\cal B}_{R}(x_{0})}d(x,x_{0})^{\polyorder}\dd(P_{\seqidx}+P)(x)+C_{R}d_{\mathrm{BL}}(P_{\seqidx},P)\\
 & \leq(2C+1)\varepsilon,
\end{align*}
where $C=1+A_{0}+A_{1}+V(x_{0}),$ and the first inequality is due
to $\tilde{f}\equiv\bar{f}$ on $\mathcal{B}_{R+1}(x_{0}).$ 
\end{proof}
\begin{rem*}
    When $d$ is the discrete metric $d(x,y)=1\{x\neq y\}$, the IPM defined by 1-$V$-pseudo-Lipschitz functions is (up to a multiplicative constant) the weighted total variation distance with weight $1+2V$~\citep[Lemma 2.1]{Hairer_2011}. 
The above use of order-$(\polyorder-1)$-$V$-pseudo Lipschitz functions may be considered as an alternative weighted-seminorm generalization of the 1-Wasserstein distance.  
\end{rem*}

\subsection{\pcref{prop:plip-IPM-convergence}}
\label{subsec:proof-plip-IPM-metrization}
We address the following result for the class ${\cal F}_{\polyorder}$ presented in the main text. 
We first provide a proof for this result and then prove a lemma used in the proof. 

\plipdistchar*

\begin{proof}
Suppose $P_{\seqidx} \toL{\polyorder} P$. 
Then, by \cref{prop:plip-metric-Wasserstein-conv} with $V(x)=\Verts x_{2}^{\polyorder-1}$, 
we have $d_{\plipspace{\polyorder-1}(V)}(P, P_\seqidx) \to 0$, leading to $d_{\calF_{\polyorder}}(P, P_\seqidx)\to 0$ as 
$d_{\calF_{\polyorder}}(P, P_\seqidx) \leq d_{\plipspace{\polyorder-1}(V)}(P, P_\seqidx)$. 
For the other direction, suppose $d_{\mathcal{F}_{\polyorder}}(P, P_\seqidx)\to 0$. 
Then, by \cref{lem:smooth-plip-upperbounds-plip} with $\mu=P$,  
$d_{\plipspace{\polyorder-1}(V)}(P, P_{\seqidx}) \to 0$. 
By another application of \cref{prop:plip-metric-Wasserstein-conv}, we have $P_{\seqidx} \toL{\polyorder} P$. 
\end{proof}

\begin{lem}[{IPM $d_{\mathcal{F}_{\polyorder}}$ upper-bounds $d_{\plipspace{\polyorder-1}}$; an extension of \citealp[Lemma 2.2]{Mackey_2016}}]
\label{lem:smooth-plip-upperbounds-plip}
Fix $\polyorder \in [1, \infty)$. 
Assume $P \in \mathcal{P}_{\polyorder}$.
Let $\plipspace{\polyorder-1} = \plipspace{q-1}(\Verts{\cdot}_2^{\polyorder-1})$ be a set of $1$-pseudo-Lipscitz functions on $\bbR^{\datdim}$. 
Then, the following relation holds: 
if $\polyorder=1$, 
\begin{align*}
   W_1(P,Q) = d_{\plipspace{0}}(P, Q) \leq 3 \left\{
        d_{\mathcal{F}_1}(P, Q) \lor 
        \left( d_{\mathcal{F}_1}(P, Q) \sqrt{2}\Ex[\Verts{G}_2]^2\right)^{1/3}
    \right\},
\end{align*}
and if $\polyorder>1$, 
\begin{align*}
   d_{\plipspace{\polyorder-1}}(P, Q) \leq C_{\mu}
    \bigl(1\lor d_{\mathcal{F}_\polyorder}(P,Q)^{1\lor(\polyorder+1)/3}\bigr)
    d_{\mathcal{F}_\polyorder}(P,Q)^{1/3}.
\end{align*}
Here, $G$ is a standard normal vector with $\Ex[\Verts{G}_2^k] = 2^{k/2} \Gamma\bigl( (\datdim+k)/2 \bigr)/\Gamma(\datdim/2)$ for any $k \geq  1$; 
and $C_{\mu}$ is a constant that depends on $\mu \in \{P,Q\}\cap \finitemomentspace{\polyorder-1}$, defined by 
\begin{align*}
    C_{\mu} 
    =  \left(\frac{A}{M_\mu}\right)^{1/3}\left[
           2M_\mu\left\{
                1 \lor \left( \frac{A}{M_\mu} \right)^{(\polyorder+1)/3}
           \right\}+ B (1+c_{\polyorder-1})\Ex_G \bigl[ \Verts{G}_2\bigr]  
        \right],
\end{align*} 
with 
$A = c_{\polyorder-1}\sqrt{2}(1+\Ex[\Verts{G}_2^{2(\polyorder-1)}]^{1/2})$, 
$c_{\polyorder-1} = 1\lor2^{(\polyorder-1)-1}$, 
\begin{equation*}
B = \frac{(\polyorder-1)(\polyorder+1)}{2} \lor \frac{\verts{(\polyorder-1)(\polyorder-3)(\polyorder+7)}}{4}, 
\end{equation*}
and 
\begin{align*}
   M_\mu = 2\left[ \Ex[\Verts{G}_2]   \left\{1 + (1+c_{\polyorder-1})\int (1+\Verts{x}_2^2)^{(\polyorder-1)/2}\dd \mu(x) \right\}  +  c_{\polyorder-1} \Ex[\Verts{G}_2^{\polyorder}] \right].
\end{align*}
\end{lem}

\begin{proof}
    We address the case $\polyorder>1$ since $\polyorder=1$ is the statement of \citet[Lemma 2.2]{Mackey_2016}.
    We may assume $Q \in \finitemomentspace{\polyorder}$ as otherwise the conclusion is trivial. 
    Fix $f \in \plipspace{\polyorder-1}$. 
    We have 
    \begin{align*}
        \left|\int f\dd(P -Q)\right| 
        \leq \left|\int (f - f_t)\dd(P -Q)\right| + \left\lvert \int f_t  \dd (P-Q)\right\rvert,
   \end{align*}
   where, for $t>0$, $f_t \in \calC^{\infty}$ is a smoothed version of $f$, defined by $f_t(x) = \Ex_G[f(x+ t G)]$ using a standard normal vector $G$. 
   Below, we evaluate each term in the upper bound. 

   For the first term, by the pseudo-Lipschitzness of $f$, we have 
   \begin{align*}
       \verts{ f_t(x) - f(x)} 
       &\leq  \Ex_G \bigl[ (1 + \Verts{x}_2^{\polyorder-1} + \Verts{x+tG}_2^{\polyorder-1})\Verts{tG}_2\bigr]\\
       &\leq  t (C_1 \Verts{x}_2^{\polyorder-1} + \Ex_{G}[\Verts{G}_2]) +  t^{\polyorder} C_2
   \end{align*}
   where  $C_1 = \Ex_G \bigl[ \Verts{G}_2\bigr]  (1+c_{\polyorder-1})$, 
   $C_{2} = c_{\polyorder-1}\Ex_G \bigl[\Verts{G}_2^{\polyorder}\bigr]$, 
   and $c_{\polyorder-1} = 1\lor2^{(\polyorder-1)-1}$.
   Also, with $\weight(x) = (1+\Verts{x}_2^2)^{(\polyorder-1)/2}$, 
   \begin{align*}
       \int \Verts{x}_2^{\polyorder-1} \dd Q(x) 
       \leq \int  \weight(x) \dd Q(x) 
       \leq B d_{\mathcal{F}_{\polyorder}}(P, Q)
            + \int \weight(x) \dd P(x)
   \end{align*}
   where $B$ is a constant such that $B^{-1}\weight \in \mathcal{F}_{\polyorder}$; 
   we may choose $B = \max((\polyorder-1)(\polyorder+1)/2, \verts{(\polyorder-1)(\polyorder-3)(\polyorder+7)}/4)$ by \cref{lem:multi-quad-derivatives}. 
   Note that the same argument applies with the roles $Q$ and $P$ swapped. 
   Thus, 
    \begin{align*}
        &\left \lvert \int \bigl(f_t(x) - f(x)\bigr)\dd (P-Q) (x)\right\rvert
        \leq \int \verts{f_t(x) - f(x)}\dd (P+Q) (x)\\
        &\leq tC_1 \int \Verts{x}_2^{\polyorder-1}\dd (P+Q) (x)
            + 2t \Ex[\Verts{G}_2] + 2t^{\polyorder} C_2\\
        &\leq tC_1 \int \{\Verts{x}_2^{\polyorder-1} + \weight(x)\}\dd P (x) 
            + t B C_1  d_{\mathcal{F}_{\polyorder}}(P, Q)
            + 2t \Ex[\Verts{G}_2] + 2t^{\polyorder} C_2\\
        &\leq M_P(t\lor t^{\polyorder}) + tBC_1 d_{\mathcal{F}_{\polyorder}}(P, Q),
   \end{align*}
   where $M_P = 2( \Ex[\Verts{G}_2] +   \Ex_P[\weight]  C_1 +  C_2)$. 

    To evaluate the second term, define 
    $
        b_t = C (t^{\polyorder-1} \lor \sqrt{2} t^{-2})
    $,  where $ C = c_{\polyorder-1} (1 + \Ex[\Verts{G}_2^{2(\polyorder-1)}]^{1/2}$. 
    The constant $b_t$ represents a bound on the maximum of (weighted) derivatives of $f_t$ to guarantee $b_t^{-1} f_t \in \mathcal{F}_{\polyorder}$ (see \cref{lem:plip-Gauss-conv-derivative-bounds}). 
    The second term is thus bounded as 
    \begin{equation*}
        \left\lvert \int f_t  \dd (P-Q)\right\rvert \leq b_t d_{\mathcal{F}_{\polyorder}}(P, Q). 
    \end{equation*}

    The estimates above yield, with $D = d_{\mathcal{F}_{\polyorder}}(P,Q)$, 
    \begin{align*}
        d_{\plipspace{\polyorder-1}}(P, Q) 
        &\leq \inf_{t>0} M_P(t\lor t^{\polyorder})+ B C_1 Dt +  b_t D\\
        &\leq \inf_{t>0} M_Pt (1 \lor t^{\polyorder-1})+ BC_1 Dt +  \sqrt{2} CDt^{-2}(1 \lor t^{q+1})\\
        &\leq 2M_Pt_* (1 \lor t^{\polyorder+1})+ B C_1 Dt_* \\
        &\leq \left(\frac{\sqrt{2}C}{M_P}\right)^{1/3}\left[
           2M_P\left\{
                1 \lor \left( \frac{\sqrt{2}C}{M_P}D \right)^{(\polyorder+1)/3}
           \right\}+ B C_1 D
        \right] D^{1/3},
    \end{align*}
    where, in the third line, $t_*$ is chosen such that $M_Pt_* = \sqrt{2}CD t_*^{-2}$. 
\end{proof}

\section{KSD bounds on pseudo-Lipschitz metric}
\label{sec:KSD-bounds-plip}

Here, we collect proofs of 
\cref{thm:ksd-bound-df-abstract-informal,thm:Matern-bound-linear,thm:Matern-bound-quad}, 
which provide KSD bounds on the pseudo-Lipschitz metric $d_{\mathcal{F}_{\polyorder}}$ (see their respective proofs in \cref{subsec:proof-plip-KSD-bound-ti,subsec:proof-Matern-KSD-linear,subsec:proof-Matern-KSD-quad}). 
In the following, without loss of generality, we assume $Q \in \finitemomentspace{\polyorder}$, since otherwise both sides of the inequalities (in \cref{thm:ksd-bound-df-abstract-informal,thm:Matern-bound-linear,thm:Matern-bound-quad}) become $\infty$. 

For ease of presentation, we restate the proof strategy. If we assume
the existence of a solution $v_{f}$ to the Stein equation (\ref{eq:SteinEq}),
we can obtain a bound on the worst-case expectation error with respect
to ${\cal F}_{\polyorder}:$
\begin{align*}
d_{{\cal F}_{\polyorder}}(P,Q)=\sup_{f\in{\cal F}_{\polyorder}}\left|\int f\dd P-\int f\dd Q\right|\leq\sup_{f\in{\cal F}_{\polyorder}}\left|\int{\cal T}_{P}^{m}v_{f}\dd Q\right|.
\end{align*}
Our goal is to express the upper bound in terms of the KSD. To this
end, we approximate each $v_{f}$ by an RKHS function $\tilde{v}_{f}\in{\cal H}_{K}$,
which yields an estimate 
\begin{align*}
 & d_{{\cal F}_{\polyorder}}(P,Q)\\
 & \leq\sup_{f\in{\cal F}}\int\left|{\cal T}_{P}^{m}v_{f}-{\cal T}_{P}^{m}\tilde{v}_{f}\right|\dd Q+\sup_{f\in{\cal F}}\Verts{\tilde{v}_{f}}_{\rkhs K}\cdot\ksd{K_{m}}P(Q)\\
 & \leq\sup_{f\in{\cal F}}\underbrace{\int_{\Verts x_{2}>r}\left(\verts{{\cal T}_{P}^{m}v_{f}}+\verts{{\cal T}_{P}^{m}\tilde{v}_{f}}\right)\dd Q}_{\mathrm{(A)}}+\sup_{f\in{\cal F}}\underbrace{\int_{\Verts x_{2}\leq r}\left|{\cal T}_{P}^{m}v_{f}-{\cal T}_{P}^{m}\tilde{v}_{f}\right|\dd Q}_{\mathrm{(B)}}\\
 & \vphantom{\leq}\quad+\sup_{f\in{\cal F}}\underbrace{\Verts{\tilde{v}_{f}}_{\rkhs K}}_{(\mathrm{C})}\cdot\ksd{K_{m}}P(Q). 
\end{align*}
Each term is evaluated as follows. In Lemmas \ref{lem:integral-eval-stein-solution}
\ref{lem:integral-eval-rkhs}, we show that the term (A) can be expressed
in terms of the KSD assuming the $\polyorder$-growth approximation
property of the Stein RKHS
$\mathcal{T}_{P}^{m}(\rkhs K)$  (\cref{def:q-growth-approx}). In \cref{lem:convolution-Stein-error} %
we show that the term (B) can be made arbitrarily small if we choose
approximation $\tilde{v}_{f}$ appropriately. Note that a better approximation
typically increases the complexity (the term C) of the function, measured
in the RKHS norm. We explicitly construct $\tilde{v}_{f}$ in Section
\ref{subsec:Constructive-approximation-of-Stein} so that we can evaluate
the norm and compare it against the approximation precision. Section
\ref{subsec:Proof-of-Theorem-KSD-bound-IPM} combines these results
to provide our final results. 

\subsection{Preparatory lemmas }

\begin{lem}[Tail integral bound by KSD; the Stein equation]
\label{lem:integral-eval-stein-solution} Suppose that ${\cal T}_{P}^{m}(\rkhs K)$
approximates $\polyorder$-growth; for $\varepsilon>0,$ let $v_{\varepsilon}$
be a function such that ${\cal T}_{P}v_{\varepsilon}$ is $\polyorder$-growth
approximation function, and $r_{\varepsilon}$ be the corresponding
radius of the indicator set. Suppose $f\in{\cal F}_{\polyorder}.$
Then, the solution $v_{f}$ to ${\cal T}_{P}^{m}v_{f}=f-\EE_{P}[f]$
in Lemma \ref{lem:finite-stein-factors} satisfies 
\[
\int_{\Verts x_{2}>r_{\varepsilon}\lor1}\left|{\cal T}_{P}^{m}v_{f}(x)\right|\dd Q(x)\leq C \left(\Verts{v_{\varepsilon}}_{\rkhs K}\ksd{K_{m}}P(Q)+\varepsilon\right),
\]
where 
\[
C=2\left( 1+\lambda_{b}\zeta_{1}\sqrt{\datdim}+\lambda_{m}c_{m}\zeta_{2}\datdim\right). 
\]
\end{lem}

\begin{proof}
First, note that the solution $v_{f}$ satisfies 
\begin{align*}
\verts{\mathcal{T}_{P}^{m}v_{f}(x)-\mathcal{T}_{P}^{m}v_{f}(0)}=\verts{f(x)-f(0)}\leq\bigl(1+\Verts x_{2}^{\polyorder-1}\bigr)\Verts x_{2},
\end{align*}
where the inequality is due to the $1$-pseudo-Lipschitzness of $f.$
By the growth conditions in Assumption \ref{assu:poly-grow-coef},
we have
\begin{align*}
\verts{\mathcal{T}_{P}^{m}v_{f}(0)} & \leq2\Verts{b(0)}_{2}\Verts{v_{f}(0)}_{2}+\datdim\Verts{m(0)}_{\mathrm{op}}\Verts{\nabla v_{f}(0)}_{\mathrm{op}}\\
 & \leq2\left(\lambda_{b}\zeta_{1}\sqrt{\datdim}+\lambda_{m}\zeta_{2}c_m\datdim \right).
\end{align*}
These imply 
\begin{align*}
 & \int_{\Verts x_{2}>r_{\varepsilon}\lor1}\left|\mathcal{T}_{P}^{m}v_{f}(x)\right|\dd Q(x)\\
 & \leq\int_{\Verts x_{2}>r_{\varepsilon}\lor1}\left|\mathcal{T}_{P}^{m}v_{f}(x)-\mathcal{T}_{P}^{m}v_{f}(0)\right|\dd Q(x)+\verts{\mathcal{T}_{P}^{m}v_{f}(0)}\int_{\Verts x_{2}>r_{\varepsilon}\lor1}\dd Q(x)\\
 & \leq2\left( 1+\lambda_{b}\zeta_{1}\sqrt{\datdim}+\lambda_{m}\zeta_{2}c_m\datdim\right) \left(\int_{\Verts x_{2}>r_{\varepsilon}}\Verts x_{2}^{\polyorder}\dd Q(x)\right)\\
 & \leq 2\left( 1+\lambda_{b}\zeta_{1}\sqrt{\datdim}+\lambda_{m}\zeta_{2}c_m\datdim\right) 
 \left(\int\mathcal{T}_{P}^{m}v_{\varepsilon}\dd Q+\varepsilon\right)\\
 & \leq 2\left( 1+\lambda_{b}\zeta_{1}\sqrt{\datdim}+\lambda_{m}\zeta_{2}c_m\datdim\right)
 \left(\Verts{v_{\varepsilon}}_{\rkhs K}\ksd{K_{m}}P(Q)+\varepsilon\right),
\end{align*}
where the second inequality follows from $\Verts x_{2}^{\polyorder'}<\Verts x_{2}^{\polyorder}$
for $\Verts x_{2}>1$ and $\polyorder'<\polyorder;$ the third inequality
is due to the $\polyorder$-growth approximation property of $\mathcal{T}_{P}^{m}v_{\varepsilon}.$
\end{proof}
\begin{lem}[Tail integral bound by KSD; general functions)]
\label{lem:integral-eval-rkhs} Suppose that ${\cal T}_{P}^{m}(\rkhs K)$
approximates $\polyorder$-growth; for $\varepsilon>0,$ let $v_{\varepsilon}$
be a function such that ${\cal T}_{P}^{m}v_{\varepsilon}$ is a $\polyorder$-growth
approximation function, and $r_{\varepsilon}$ be the corresponding
radius of the indicator set. Let $v\in{\cal C}^{1}(\bbR^{\datdim})$
be fixed. Assume that $v$ satisfies $\Verts{v(x)}_{2}\leq C_{0}\Verts x_{2}^{\polyorder-1}$
for $\Verts x_{2}\geq r_{0}$ and $C_{0}>0.$ For $\polyorder_{m}$
in Assumption \ref{assu:poly-grow-coef}, for $\Verts x_{2}\geq r_{0},$
assume further $\Verts{\nabla v(x)}_{\mathrm{F}}\leq C_{1}\Verts x_{2}^{\polyorder-1}$if
$\polyorder_{m}=0;$ and $\Verts{\nabla v(x)}_{\mathrm{F}}\leq C_{1}\Verts x_{2}^{\polyorder-2}$
if $\polyorder_{m}=1.$ Then, under Assumption \ref{assu:poly-grow-coef},
we have 
\[
\int_{\Verts x_{2}>r_{\varepsilon}\lor r_{0}\lor1}\verts{\mathcal{T}_{P}^{m}v(x)}\dd Q(x)\leq C\left(\Verts{v_{\varepsilon}}_{\rkhs K}\ksd{K_{m}}P(Q)+\varepsilon\right),
\]
where 
\[
C = 4\lambda_{b}C_{0}+\lambda_{m}(c_{m}+1)C_{1}\sqrt{\datdim}.
\]
\end{lem}

\begin{proof}
By assumption, we have for $\Verts x_{2}\geq r_{\varepsilon}\lor r_{0}$
\begin{align*}
\verts{\mathcal{T}_{P}^{m}v(x)} & \leq2\Verts{b(x)}_{2}\Verts{g(x)}_{2}+\sqrt{\datdim}\Verts{m(x)}_{\mathrm{op}}\Verts{\nabla g(x)}_{\mathrm{F}}\\
 & \leq2\lambda_{b}C_{0}\bigl(\Verts x_{2}^{\polyorder-1}+\Verts x_{2}^{\polyorder}\bigr)+\lambda_{m}C_{1}\sqrt{\datdim}\bigl(c_{m}\Verts x_{2}^{\polyorder-1-\polyorder_{m}}+\Verts x_{2}^{\polyorder}\bigr).
\end{align*}
Therefore, we have 
\begin{align*}
 & \int_{\Verts x_{2}>r_{\varepsilon}\lor r_{0}\lor1}\verts{\mathcal{T}_{P}^{m}v(x)}\dd Q(x)\\
 & \leq\int_{\Verts x_{2}>r_{\varepsilon}\lor r_{0}\lor1}2\lambda_{b}C_{0}\bigl(\Verts x_{2}^{\polyorder-1}+\Verts x_{2}^{\polyorder}\bigr)\dd Q(x)+\lambda_{m}C_{1}\sqrt{\datdim}\int_{\Verts x_{2}>r_{\varepsilon}\lor r_{0}\lor1}c_{m}\Verts x_{2}^{\polyorder-1-\polyorder_{m}}+\Verts x_{2}^{\polyorder}\dd Q(x)\\
 & <\left(4\lambda_{b}C_{0}+(c_{m}+1)\lambda_{m}C_{1}\sqrt{\datdim}\right)\int_{\Verts x_{2}>r_{\varepsilon}}\Verts x_{2}^{\polyorder}\dd Q(x)\\
 & \leq\left(4\lambda_{b}C_{0}+(c_{m}+1)\lambda_{m}C_{1}\sqrt{\datdim}\right)\left(\int\mathcal{T}_{P}^{m}v_{\varepsilon}(x)\dd Q+\varepsilon\right)\\
 & \leq\left(4\lambda_{b}C_{0}+(c_{m}+1)\lambda_{m}C_{1}\sqrt{\datdim}\right)\left(\Verts{v_{\varepsilon}}_{\rkhs K}\ksd{K_{m}}P(Q)+\varepsilon\right),
\end{align*}
where the second inequality follows from $\Verts x_{2}^{\polyorder'}<\Verts x_{2}^{\polyorder}$
for $\Verts x_{2}>1$ and $\polyorder'<\polyorder;$ the third inequality
is due to the $\polyorder$-growth approximation property of ${\cal T}_{P}^{m}v_{\varepsilon}.$ 
\end{proof}

\subsection{Constructive approximation of solutions to the Stein equation \label{subsec:Constructive-approximation-of-Stein}}
\begin{lem}[Constructive approximation via convolution]
    \label{lem:convolution-Stein-error} Let $\polyorder\geq1,$ and $\theta \in [0,1]$. 
    Let $\weight$ be the shorthand for $w_{\polyorder+\theta-2}(x)=\bigl(\tau^{2}+\Verts x_{2}^{2}\bigr)^{(\polyorder+\theta-2)/2}$
with $\tau>0.$ Let $\varphi:\bbR^{\datdim}\to[0,\infty)$ be a probability
density function satisfying $\int\Verts x_{2}^{\polyorder}\varphi(x)\dd x<\infty.$
Let $v$ be a function satisfying the growth Lemma \ref{lem:finite-stein-factors}.
For $\rho>0$ and $r\geq1,$ define 
\[
T_{\varphi,\rho}\tilde{v}_{r}^{\weight}(x)=\rho^{-\datdim}\int\tilde{v}_{r}^{\weight}(x-y)\varphi(y/\rho)\dd y,\ \text{and}\ \tilde{v}_{r}^{\weight}(x)=\frac{1_{r,1}(x)}{\weight(x)}v(x),
\]
where $1_{r,1}\in{\cal C}^{2}:\bbR^{\datdim}\to[0,1]$ defined in
Lemma \ref{lem:cutoff-ball-bound-C2} vanishing outside $\{\Verts x_{2}\leq r+1\}$
for $r>0.$ Then, for $\Verts x_{2}\leq r,$ 
\begin{align*}
\verts{\mathcal{T}_{P}^{m}v(x)-\mathcal{T}_{P}^{m}\bigl(wT_{\varphi,\rho}\tilde{v}_{r}^{\weight}\bigr)(x)}\leq C\rho(1\lor\rho)^{(\polyorder-1)}r^{\polyorder+\polyorder_{m}},
\end{align*}
where $C$ is a constant defined by 
\[
C=2^{\polyorder}\left\{ 2\tilde{u}_{P,\datdim,\weight}^{(1)}\lambda_{b}+\lambda_{m}\left(\sqrt{\datdim}\tilde{u}_{P,\datdim,\weight}^{(2)}+\tilde{u}_{P,\datdim,\weight}^{(1)}\frac{\polyorder-1}{2\tau}\right)\bigl(c_{m}+1\bigr)\right\} 
\]
using constants 
\[
\tilde{u}_{P,\datdim,\weight}^{(1)}=\sqrt{\datdim}\zeta_{1}\frac{\polyorder-1}{2\tau}\mu_{\polyorder}+\bigl(2^{-1}\zeta_{2}+3\sqrt{\datdim}\zeta_{1}\big)\mu_{1},
\]
 
\begin{align*}
\tilde{u}_{P,\datdim,\weight}^{(2)} & =\left[\frac{\polyorder-1}{2\tau}\left\{ \mu_{\polyorder}+(\polyorder-1)\left((4+\tau)\tau^{\polyorder-3}+ \frac{(\polyorder-1)^2}{\tau^{\polyorder}(1\land\tau)} \right)\mu_{1}\right\} \left(2^{-1}\zeta_{2}+3\zeta_{1}\sqrt{\datdim}\right)\right.\\
 & \hphantom{=}\quad+\left.\left(\mu_{1}+(\polyorder-1)\mu_{\polyorder}/(2\tau)\right)\left\{ 2^{-1}\zeta_{3}+3\zeta_{2}+\left(72+(35+13\sqrt{13})/12+6+3^{4}7^{-1}\right)\sqrt{\datdim}\zeta_{1}\right\} \right],
\end{align*}
 and 
\[
\mu_{\polyorder}=\int\Verts x_{2}\left\{ 1+2\left(1+\Verts{x/\tau}_{2}^{2}\right)\right\} ^{(\polyorder-1)/2}\varphi(x)\dd x.
\]
\end{lem}

\begin{proof}
By Lemma \ref{lem:multiquadratic-helper}, we have 
\begin{align*}
B_{\weight}(\rho z) & \coloneqq\sup_{x\in\bbR^{\datdim},u\in[0,1]}\frac{\weight(x)}{\weight(x-u\rho z)}\\
 & \leq\left(1+2\left(1+\Verts{\rho z/\tau}_{2}^{2}\right)\right)^{(\polyorder-1)/2}\\
 & \leq(1\lor\rho)^{\polyorder-1}\left\{ 1+2\left(1+\Verts{z/\tau}_{2}^{2}\right)\right\} ^{(\polyorder-1)/2}, 
\end{align*}
\[
M_{\weight}\coloneqq\sup_{x\in\bbR^{\datdim}}\Verts{\nabla\log\weight(x)}_{2}\leq\frac{\polyorder-1}{2\tau},
\]
\[
N_{\weight}\coloneqq\sup_{x\in\bbR^{\datdim}}\weight(x)^{-1}\norm{\nabla^{2}\log\weight(x)}_{2}=(\polyorder-1)(4+\tau)\tau^{-2},
\]
and 
\[
O_{\weight}\coloneqq\sup_{x\in\bbR^{\datdim}}\weight(x)^{-1}\norm{\nabla\log\weight(x)}_{2}^2\leq \frac{(\polyorder-1)^2}{\tau^{\polyorder}(1\land\tau)}.
\]
By Corollaries \ref{cor:convolution-diff-bound-stein}, \ref{cor:convolution-grad-diff-bound-stein}
and Lemma \ref{lem:cutoff-ball-bound-C2}, for any $x\in\bbR^{\datdim},$
we have  

\begin{align*}
\Verts{T_{\varphi,\rho}\tilde{v}^{\weight}(x)-\tilde{v}^{\weight}(x)}_{\mathrm{2}} & \leq\frac{\rho(1\lor\rho)^{(\polyorder-1)}\{1+(r+1){}^{\polyorder-1}\}}{\weight(x)}\tilde{u}_{P,\datdim,\weight}^{(1)},
\end{align*}
and 
\[
\Verts{\nabla T_{\varphi,\rho}\tilde{v}^{\weight}(x)-\nabla\tilde{v}^{\weight}(x)}_{\mathrm{op}}\leq\frac{\rho(1\lor\rho)^{(\polyorder-1)}\{1+(r+1)^{\polyorder-1}\}}{\weight(x)}\tilde{u}_{P,\datdim,\weight}^{(2)},
\]
where 
\[
\tilde{u}_{P,\datdim,\weight}^{(1)}=\sqrt{\datdim}\zeta_{1}\frac{\polyorder-1}{2\tau}\mu_{\polyorder}+\bigl(2^{-1}\zeta_{2}+3\sqrt{\datdim}\zeta_{1}\big)\mu_{1},
\]
 
\begin{align*}
 &\tilde{u}_{P,\datdim,\weight}^{(2)}\\
 & =\left[\frac{\polyorder-1}{2\tau}\left\{ \mu_{\polyorder}+(\polyorder-1)\left((4+\tau)\tau^{\polyorder-3}+ \frac{(\polyorder-1)^2}{\tau^{\polyorder}(1\land\tau)} \right)\mu_{1}\right\} \left(2^{-1}\zeta_{2}+3\zeta_{1}\sqrt{\datdim}\right)\right.\\
 & \hphantom{=}\quad+\left.\left(\mu_{1}+(\polyorder-1)\mu_{\polyorder}/(2\tau)\right)\left\{ 2^{-1}\zeta_{3}+3\zeta_{2}+\left(72+(35+13\sqrt{13})/12+6+3^{4}7^{-1}\right)\sqrt{\datdim}\zeta_{1}\right\} \right],
\end{align*}
and 
\[
\mu_{\polyorder}=\int\Verts x_{2}\left\{ 1+2\left(1+\Verts{x/\tau}_{2}^{2}\right)\right\} ^{(\polyorder-1)/2}\varphi(x)\dd x.
\]
For $\Verts x_{2}\leq r,$ with $v^{\weight}(x)=v(x)/\weight(x),$
\begin{align*}
 & \verts{\mathcal{T}_{P}^{m}\bigl(wv^{\weight}\bigr)(x)-\mathcal{T}_{P}^{m}\bigl(wT_{\varphi,\rho}\tilde{v}^{\weight}\bigr)(x)}\\
 & \leq w(x)\left\{ \bigl|\mathcal{T}_{P}^{m}\bigl(v^{\weight}(x)-T_{\varphi,\rho}\tilde{v}^{\weight}(x)\bigr)\bigr|+\frac{\polyorder-1}{2\tau}\Verts{m(x)}_{\mathrm{op}}\Verts{T_{\varphi,\rho}\tilde{v}^{\weight}(x)-\tilde{v}^{\weight}(x)}_{2}\right\} \\
 & \leq2\rho(1\lor\rho)^{(\polyorder-1)}(r+1)^{\polyorder-1}\left\{ 2\lambda_{b}\tilde{u}_{P,\datdim,\weight}^{(1)}\bigl(1+r\bigr)+\lambda_{m}\left(\datdim\tilde{u}_{P,\datdim,\weight}^{(2)}+\frac{\polyorder-1}{2\tau}\tilde{u}_{P,\datdim,\weight}^{(1)}\right)\bigl(c_{m}+r{}^{\polyorder_{m}+1}\bigr)\right\} .
\end{align*}
Thus, if $\Verts x_{2}\leq r$ and $r\geq1,$ 
\begin{align*}
 & \verts{\mathcal{T}_{P}^{m}\bigl(wv^{\weight}\bigr)(x)-\mathcal{T}_{P}^{m}\bigl(wT_{\varphi,\rho}\tilde{v}^{\weight}\bigr)(x)}\\
 & \leq\rho(1\lor\rho)^{(\polyorder-1)}2^{\polyorder}\left\{ 2\tilde{u}_{P,\datdim,\weight}^{(1)}\lambda_{b}+\lambda_{m}\left(\datdim\tilde{u}_{P,\datdim,\weight}^{(2)}+\tilde{u}_{P,\datdim,\weight}^{(1)}\frac{\polyorder-1}{2\tau}\right)\bigl(c_{m}+1\bigr)\right\} r^{\polyorder+\polyorder_{m}}.
\end{align*}
 
\end{proof}
\begin{lem}[Growth of convolution approximation]
\label{lem:Stein-convolution-decay-linear}
Let $\polyorder\geq1$ and $\theta \in [0,1].$ 
Let $w$ be the short hand for $\weight_{\polyorder+\theta-1}(x)=\bigl(\tau^{2}+\Verts x_{2}^{2}\bigr)^{(\polyorder+\theta-2)/2}.$
Let $\varphi$ be a probability density.  For
$\rho>0$ and $r\geq1,$ define 
\[
T_{\varphi,\rho}\tilde{v}_{f}^{\weight}(x)=\rho^{-\datdim}\int\tilde{v}_{f}^{\weight}(x-y)\varphi(y/\rho)\dd y,\ \text{where}\ \tilde{v}_{f}^{\weight}(x)=\frac{1_{r,1}(x)}{\weight(x)}v_{f}(x)
\]
and $1_{r,1}\in{\cal C}^{2}$ is a differentiable indicator defined
in Lemma \ref{lem:cutoff-ball-bound-C2}, $v_{f}$ is in the solution
to the Stein equation $\mathcal{T}_{P}^{m}v=f-\EE_{P}[f]$ satisfying
the conditions in Lemma \ref{lem:finite-stein-factors}, and $v_{f}^{w}(x)=v_{f}(x)/w(x).$
Then, for $\Verts x_{2}\geq r,$ 
\[
\Verts{\weight(x)T_{\varphi,\rho}\tilde{v}_{f}^{\weight}(x)}_{2}\leq A_{0}\Verts x_{2}^{\polyorder-1}\ \text{and}\ \left\Vert \nabla\Bigl(\weight(x)T_{\varphi,\rho}\tilde{v}_{f}^{\weight}(x)\Bigr)\right\Vert _{\mathrm{F}}\leq A_{1}\Verts x_{2}^{\polyorder-1},
\]
where 
\[
    A_{0}=2^{2-\theta}\zeta_{1}\sqrt{\datdim}\bigl(\tau^{2}+1)^{(\polyorder-\theta)/2},
\]
and 
\[
    A_{1}=2^{2-\theta}\left\{ \left(\frac{\polyorder-1}{\tau}+3\right)\zeta_{1}\sqrt{\datdim}+2^{-1}\zeta_{2}\right\} (\tau^{2}+1)^{(\polyorder-\theta)/2}.
\]
\end{lem}

\begin{proof}
For $\Verts x_{2}\geq r,$ we have 
\begin{align*}
\Verts{w(x)T_{\varphi,\rho}\tilde{v}_{f}^{\weight}(x)}_{2}
& \leq\weight(x)\sup_{x\in\bbR^{\datdim}}\frac{\Verts{v_{f}(x)}_{2}}{\weight(x)}1_{r,1}(x)\\
& \leq\weight(x) \cdot 2^{2-\theta}\zeta_{1}\sqrt{\datdim} (\tau^2+1)^{(1-\theta)/2}r^{1-\theta}\\
& \leq 2^{2-\theta}\zeta_{1}\sqrt{\datdim}\bigl(\tau^{2}+1)^{(\polyorder-\theta)/2} \Verts x_{2}^{\polyorder-1}.
\end{align*}
We also have 
\begin{align*}
\left\Vert \nabla\tilde{v}_{f}^{\weight}(x)\right\Vert _{\mathrm{F}} & \leq\frac{\Verts{\nabla\log\weight(x)}_{2}\Verts{v_{f}(x)}_{2}1_{r,1}(x)}{\weight(x)}+\frac{1_{r,1}(x)}{w(x)}\Verts{\nabla v_{f}(x)}_{\mathrm{F}}+\frac{\Verts{\nabla1_{r,1}(x)}_{2}\Verts{v_{f}(x)}_{2}}{w(x)}\\
 & \leq\weight(x)^{-1}\left\{ \left(\frac{\polyorder-1}{2\tau}+3\right)\zeta_{1}\sqrt{\datdim}+2^{-1}\zeta_{2}\right\} \bigl(1+\Verts x_{2}^{\polyorder-1}\bigr) 1\{\Verts{x}_2\leq r+1\}\\
 & \leq 2^{2-\theta}\left\{ \left(\frac{\polyorder-1}{2\tau}+3\right)\zeta_{1}\sqrt{\datdim}+2^{-1}\zeta_{2}\right\} (\tau^2+1)^{(1-\theta)/2} r^{1-\theta}
\end{align*}

Since each element of $\nabla\tilde{v}_{f}^{\weight}$ is continuous
and bounded, we can differentiate under the integral sign \citep[Corollary A.5]{Dudley_1999};
i.e., 
\[
\nabla T_{\varphi,\rho}\tilde{v}_{f}^{\weight}(x)=T_{\varphi,\rho}\nabla\tilde{v}_{f}^{\weight}(x).
\]
Thus, 
\begin{align*}
&\left\Vert \nabla\Bigl(\weight(x)T_{\varphi,\rho}\tilde{v}_{f}^{\weight}(x)\Bigr)\right\Vert _{\mathrm{F}} \\
 & \leq\weight(x)\left(\left\Vert \nabla\log\weight(x)\otimes T_{\varphi,\rho}\tilde{v}_{f}^{\weight}(x)\right\Vert _{\mathrm{F}}+\left\Vert \nabla T_{\varphi,\rho}\tilde{v}_{f}^{\weight}(x)\right\Vert _{\mathrm{F}}\right)\\
 & \leq2^{2-\theta}\weight(x)\left\{\left(\frac{\polyorder-1}{\tau}+3\right)\zeta_{1}\sqrt{\datdim} + 2^{-1}\zeta_{2}\right\}(\tau^2+1)^{(1-\theta)/2}r^{1-\theta}\\
 & \leq2^{2-\theta}\left\{ \sqrt{\datdim}\left(\frac{\polyorder-1}{\tau}+3\right)\zeta_{1}+2^{-1}\zeta_{2}\right\} (\tau^{2}+1)^{(\polyorder-\theta)/2}\Verts x_{2}^{\polyorder-1}.
\end{align*}
\end{proof}
\begin{lem}[Slower gradient growth of convolution approximation]
\label{lem:Stein-convolution-decay-quad}
Let $\polyorder\geq1$ and $\theta\in[0,1]$.
Let $\weight$ be the shorthand for $\weight_{\polyorder +\theta-2}(x)=\bigl(\tau^{2}+\Verts x_{2}^{2}\bigr)^{(\polyorder+\theta-2)/2}.$
Let $\varphi_{\polyorder}$
defined as in Lemma \ref{lem:Bochner-Riesz-density}. For $\rho>0$
and $r\geq1,$ define 
\[
T_{\varphi_{\polyorder},\rho}\tilde{v}_{f}^{\weight}(x)=\rho^{-\datdim}\int\tilde{v}_{f}^{\weight}(x-y)\varphi_{\polyorder}(y/\rho)\dd y,\ \text{where}\ \tilde{v}_{f}^{\weight}(x)=\frac{v_{f}(x)}{\weight(x)}1_{r,1}(x)
\]
and $1_{r,1}\in{\cal C}^{2}$ is a differentiable indicator defined
in Lemma \ref{lem:cutoff-ball-bound-C2}, $v_{f}$ is in the solution
to the Stein equation $\mathcal{T}_{P}^{m}v=f-\EE_{P}[f]$ satisfying
the conditions in Lemma \ref{lem:finite-stein-factors}, and $v_{f}^{w}(x)=v_{f}(x)/w(x)$.
Then, for $\Verts x_{2}>2(r+1),$ 
\[
\left\Vert \nabla\Bigl(\weight(x)T_{\varphi_{\polyorder},\rho}\tilde{v}_{f}^{\weight}(x)\Bigr)\right\Vert _{\mathrm{F}}\leq C_{P,\datdim,\weight}\Verts x_{2}^{\polyorder-2}
\]
 where $C_{P,\datdim,\weight}$ is a constant given by 
\[
    \frac{2^{1+\theta}\Gamma(\polyorder/2+1)^{2}}{\mathrm{B}(\datdim/2,\polyorder+1)}\left\{ \left(\frac{\polyorder-1}{\tau}+3\right)\frac{\zeta_{1}}{\sqrt{\datdim}}+\frac{\zeta_{2}}{2\datdim}\right\} (\tau^{1-\polyorder}+1)(\tau^{1-\theta}+1)\left(1+\tau^{2}\right)^{\frac{\polyorder+\theta-2}{2}}.
\]
$\Gamma$ is the Gamma function, and $\mathrm{B}$ is the beta function.
\end{lem}

\begin{proof}
We collect some preliminary evaluations required for the required estimate. 
First, by the upper bound on the Bessel function \citep[\href{https://dlmf.nist.gov/10.14.1}{Eq. 10.14.1}]{DLMF}, we have 
\begin{align*}
\hat{\Phi}_{\polyorder}(x)^{2} & =2^{\polyorder}\Gamma(\polyorder/2+1)^{2}\Verts x_{2}^{-(\datdim+\polyorder)}J_{(\datdim+\polyorder)/2}\bigl(\Verts x_{2}\bigr)^{2}\\
& \leq 2^{\polyorder-1}\Gamma(\polyorder/2+1)^{2}\Verts x_{2}^{-(\datdim+\polyorder)}.
\end{align*}
Thus, for $\Verts x_{2}>r+1,$ we have 
\begin{align*}
    \varphi_{\polyorder}(x) & \leq \frac{2^{\polyorder-1}\pi^{-\datdim/2}\Gamma(\datdim/2)\Gamma(\polyorder/2+1)^{2}}{\mathrm{B}(\datdim/2, \polyorder+1)} \Verts x_{2}^{-(\datdim+\polyorder)}\\
    & \leq \underbrace{ \frac{2^{\polyorder-1}\pi^{-\datdim/2}\Gamma(\datdim/2)\Gamma(\polyorder/2+1)^{2}}{\mathrm{B}(\datdim/2, \polyorder+1)} (r+1)^{-(\datdim+\polyorder-\theta)}}_{\eqqcolon C} \Verts x_{2}^{-\theta}.
\end{align*}
Since we have 
\begin{align*}
    \Verts{\tilde{v}_{f}^{\weight}(x)}_{2} & \leq \sqrt{\datdim}\zeta_{1}(\tau^{-(\polyorder-1)}+1\bigr)1\{\Verts x_{2}\leq r+1\} \weight_{1-\theta}(x), 
\end{align*}
by Lemma \ref{lem:convolution-decay-rate}, we obtain, for $\Verts x_{2}\geq2(r+1),$
\begin{align*}
    \norm{T_{\varphi_{\polyorder},\rho}\tilde{v}_{f}^{\weight}(x)}_{2} & \leq2^{\theta}C\int\Verts{\tilde{v}_{f}^{\weight}(x)}_{2}\dd x\cdot\Verts x_{2}^{-\theta}\\
    & \leq2^{\theta+1}C\zeta_{1}\datdim^{-1/2}(\tau^{1-\polyorder}+1)(\tau^{1-\theta}+1)\frac{\pi^{\datdim/2}}{\Gamma(\datdim/2)}(r+1)^{\datdim+1-\theta}\Verts x_{2}^{-\theta}.
\end{align*}
Also, 
\begin{align*}
\left\Vert \nabla\tilde{v}_{f}^{\weight}(x)\right\Vert _{\mathrm{F}} & \leq\frac{\Verts{\nabla\log\weight(x)}_{2}\Verts{v_{f}(x)}_{2}1_{r,1}(x)}{\weight(x)}+\frac{1_{r,1}}{w(x)}\Verts{\nabla v_{f}(x)}_{\mathrm{F}}+\frac{\Verts{\nabla1_{r,1}(x)}_{2}\Verts{v_{f}(x)}_{2}}{w(x)}\\
 & \leq\weight(x)^{-1}\left\{ \left(\frac{\polyorder-1}{2\tau}+3\right)\zeta_{1}\sqrt{\datdim} + 2^{-1}\zeta_{2}\right\} (1+\Verts x_{2}^{\polyorder-1}\bigr)1\{\Verts x_{2}\leq r+1\}\\
 & \leq \weight_{1-\theta}(x)\left\{\left(\frac{\polyorder-1}{2\tau}+3\right)\zeta_{1}\sqrt{\datdim}+2^{-1}\zeta_{2}\right\} (\tau^{1-\polyorder}+1\bigr)1\{\Verts x_{2}\leq r+1\}. 
\end{align*}
Since each element of $\nabla\tilde{v}_{f}^{\weight}$ is continuous
and bounded, we can differentiate under the integral sign \citep[Corollary A.5]{Dudley_1999};
i.e., $\nabla T_{\varphi,\rho}\tilde{v}_{f}^{\weight}(x)=T_{\varphi,\rho}\nabla\tilde{v}_{f}^{\weight}(x).$
Thus, Lemma \ref{lem:convolution-decay-rate} implies 
\begin{align*}
 & \left\Vert \nabla T_{\varphi_{\polyorder},\rho}\tilde{v}_{f}^{\weight}(x)\right\Vert _{\mathrm{F}}\\
 & \leq 2^{\theta}C\int\Verts{\nabla\tilde{v}_{f}^{\weight}(x)}_{\mathrm{F}}\dd x\cdot\Verts x_{2}^{-\theta}\\
 & \leq 2^{\theta+1}C\left\{ \sqrt{\datdim}\left(\frac{\polyorder-1}{2\tau}+3\right)\zeta_{1}+\frac{\zeta_{2}}{2}\right\} \frac{ \bigl(\tau^{1-\polyorder}+1\bigr)\bigl(\tau^{1-\theta}+1\bigr)\pi^{\datdim/2}}{\datdim \Gamma(\datdim/2)}(r+1)^{\datdim+1-\theta}\Verts x_{2}^{-\theta}.
\end{align*}

Using the above estimates, the quantity of interest is therefore evaluated as 
\begin{align*}
\left\Vert \nabla\Bigl(\weight(x)T_{\varphi_{\polyorder},\rho}\tilde{v}_{f}^{\weight}(x)\Bigr)\right\Vert _{\mathrm{F}} & =\weight(x)\left(\left\Vert \nabla\log\weight(x)\otimes T_{\varphi_{\polyorder},\rho}\tilde{v}_{f}^{\weight}(x)\right\Vert _{\mathrm{F}}+\left\Vert \nabla T_{\varphi_{\polyorder},\rho}\tilde{v}_{f}^{\weight}(x)\right\Vert _{\mathrm{F}}\right)\\
 & \leq\weight(x)\left(\frac{\polyorder-1}{2\tau}\left\Vert T_{\varphi_{\polyorder},\rho}\tilde{v}_{f}^{\weight}(x)\right\Vert _{2}+\left\Vert \nabla T_{\varphi_{\polyorder},\rho}\tilde{v}_{f}^{\weight}(x)\right\Vert _{\mathrm{F}}\right)\\
 & \leq C_{P,\datdim,\weight}\Verts x_{2}^{\polyorder-2}
\end{align*}
for $\Verts x_{2}>2(r+1),$ where $C_{P,\datdim,\weight}$ is a constant given by 
\[
    \frac{2^{1+\theta}\Gamma(\polyorder/2+1)^{2}}{\mathrm{B}(\datdim/2,\polyorder+1)}\left\{ \left(\frac{\polyorder-1}{\tau}+3\right)\frac{\zeta_{1}}{\sqrt{\datdim}}+\frac{\zeta_{2}}{2\datdim}\right\} (\tau^{1-\polyorder}+1)(\tau^{1-\theta}+1)\left(1+\tau^{2}\right)^{\frac{\polyorder+\theta-2}{2}}.
\]
\end{proof}

\subsection{\pcref{thm:ksd-bound-df-abstract-informal}\label{subsec:Proof-of-Theorem-KSD-bound-IPM}}
\label{subsec:proof-plip-KSD-bound-ti}
This section elaborates on the informal statement made in Theorem \ref{thm:ksd-bound-df-abstract-informal}
in the main text; our proof proceeds as outlined in its proof sketch. 

\subsubsection{Expectation difference bound for a single function }
\begin{lem}[Expectation difference bound for a single function] 
\label{thm:plip-ksd-bound-basic} 
Let $\polyorder\geq1$, $\theta \in [0,1]$, 
and $\weight(x)=\bigl(\tau^{2}+\Verts x_{2}^{2}\bigr)^{(\polyorder+\theta-2)/2}.$
Let $K:\bbR^{\datdim}\times\bbR^{\datdim}\to\bbR^{\datdim\times\datdim}$
be a matrix-valued kernel. Let $\varphi_{\polyorder}$ be a probability
density function defined in Lemma \ref{lem:Bochner-Riesz-density}.
Suppose $\mathcal{T}_{P}^{m} (\rkhs K)$ approximates $\polyorder$-growth;
for $\varepsilon>0,$ let $v_{\varepsilon}$ be a function such that
${\cal T}_{P}^{m}v_{\varepsilon}$ is a $\polyorder$-growth approximation
function, and $r_{\varepsilon}\geq1$ be the corresponding radius of the
indicator. Then, under Assumptions \ref{assu:poly-grow-coef},
\ref{assu:dissipativity-diffusion}, \ref{assu:diff-wasserstein-rate},
for any $f\in{\cal F}_{\polyorder},$ $\rho>0,$ and $\varepsilon>0,$
we have 
\begin{align*}
 & \left|\int f\dd P-\int f\dd Q\right|\\
 & \leq C\left(\Verts{v_{\varepsilon}}_{\rkhs K}\ksd{K_{m}}P(Q)+\varepsilon\right)+C_{P,\datdim,\weight}^{(3)}\rho(1\lor\rho)^{(\polyorder-1)}\cdot r_{\varepsilon}^{\polyorder+\polyorder_{m}}+\norm{\weight T_{\varphi_{\polyorder},\rho}\tilde{v}_{f}^{\weight}}_{\rkhs K}\ksd{K_{m}}P(Q)
\end{align*}
where 
\[
C=\begin{cases}
C_{P,\datdim}^{(1)}+C_{P,\datdim}^{(2)} & \text{if}\ \polyorder_{m}=0,\\
C_{P,\datdim}^{(1)}+\dot{C}_{P,\datdim}^{(2)}r_{\varepsilon}+\ddot{C}_{P,\datdim}^{(2)} & \text{if}\ \polyorder_{m}=1,
\end{cases}
\]
with $C_{P,\datdim}^{(1)},$ $C_{P,\datdim}^{(2)},$ $C_{P,\datdim,\weight}^{(3)},$
$\dot{C}_{P,\datdim}^{(2)},$ and $\ddot{C}_{P,\datdim}^{(2)}$ positive
constants independent of $f$ (given in the proof); and 
\[
T_{\varphi_{\polyorder},\rho}\tilde{v}_{f}^{\weight}(x)=\rho^{-\datdim}\int\tilde{v}_{f}^{\weight}(x-y)\varphi_{\polyorder}(y/\rho)\dd y,\ \text{and}\ \tilde{v}_{f}^{\weight}(x)=\frac{1_{r_{\varepsilon},1}(x)}{\weight(x)}v_{f}(x),
\]
with $1_{r_{\varepsilon},1}\in{\cal C}^{2}:\bbR^{\datdim}\to[0,1]$
defined in Lemma \ref{lem:cutoff-ball-bound-C2} vanishing outside
$\{\Verts x_{2}\leq r_{\varepsilon}+1\}.$ 
\end{lem}

\begin{proof}
For $f\in{\cal F}_{\polyorder},$ let $v_{f}$ be the solution to
the Stein equation $\mathcal{T}_{P}^{m}v_{f}=f-\EE_{P}[f]$ as in
Lemma \ref{lem:finite-stein-factors}. For any function $v\in\rkhs K$ with $ (\mathcal{T}_P^mv)_{+}\in L^1(Q)$, 
we have
\begin{align}
\left|\int\mathcal{T}_{P}^{m}v_{f}\dd Q\right| & \leq\left|\int\bigl(\mathcal{T}_{P}^{m}v_{f}-\mathcal{T}_{P}^{m}v\bigr)\dd Q\right|+\left|\int\mathcal{T}_{P}^{m}v\dd Q\right|\nonumber \\
 & \leq\left|\int\bigl(\mathcal{T}_{P}^{m}v_{f}-\mathcal{T}_{P}^{m}v\bigr)\dd Q\right|+\Verts v_{\rkhs K}\ksd{K_{m}}P(Q).\label{eq:integral-eval-solution-upperbound}
\end{align}
We split the domain of the integral (the first term) and evaluate as follows: 

\begin{align}
\left|\int\bigl(\mathcal{T}_{P}^{m}v_{f}-\mathcal{T}_{P}^{m}v\bigr)\dd Q\right| & =\left|\int_{\Verts x_{2}>r_{\varepsilon}}+\int_{\Verts x_{2}\leq r_{\varepsilon}}\bigl(\mathcal{T}_{P}^{m}v_{f}(x)-\mathcal{T}_{P}^{m}v(x)\bigr)\dd Q(x)\right|\nonumber \\
 & \begin{aligned}\leq & \int_{\Verts x_{2}>r_{\varepsilon}}\verts{\mathcal{T}_{P}^{m}v_{f}(x)}\dd Q(x)+\int_{\Verts x_{2}>r_{\varepsilon}}\verts{\mathcal{T}_{P}^{m}v(x)}\dd Q(x)\\
 & \quad+\int_{\Verts x_{2}\leq r_{\varepsilon}}\verts{\mathcal{T}_{P}^{m}v_{f}(x)-\mathcal{T}_{P}^{m}v(x)}\dd Q(x).
\end{aligned}
\label{eq:Stein-bound-split-integral}
\end{align}

We first deal with the first term of the upper bound \eqref{eq:Stein-bound-split-integral}.
By Lemma \ref{lem:integral-eval-stein-solution}, we have 
\begin{align*}
\int_{\Verts x_{2}>r_{\varepsilon}}\verts{\mathcal{T}_{P}^{m}v_{f}(x)}\dd Q(x)\leq C_{P,\datdim}^{(1)}\left(\Verts{v_{\varepsilon}}_{\rkhs K}\ksd{K_{m}}P(Q)+\varepsilon\right),
\end{align*}
where $C_{P,\datdim}^{(1)}=2(1+\lambda_{b}\zeta_{1}\sqrt{\datdim}+\lambda_{m}\zeta_{2}c_m\datdim).$ 

We next evaluate the third term of (\ref{eq:Stein-bound-split-integral})
by specifying $v$ to be an approximation of $v_{f}.$ For $\rho>0,$
define 
\[
T_{\varphi_{\polyorder},\rho}\tilde{v}_{f}^{\weight}(x)=\rho^{-\datdim}\int\tilde{v}_{f}^{\weight}(x-y)\varphi_{\polyorder}(y/\rho)\dd y,\ \text{where}\ \tilde{v}_{f}^{\weight}(x)=\frac{1_{r_{\varepsilon},1}(x)}{\weight(x)}v_{f}(x)
\]
and $1_{r_{\varepsilon},1}\in{\cal C}^{2}:\bbR^{\datdim}\to[0,1]$
is a differentiable indicator defined in Lemma \ref{lem:cutoff-ball-bound-C2}
satisfying $1_{r_{\varepsilon},1}(x)=1$ if $\Verts x_{2}\leq r_{\varepsilon}$
and $1_{r_{\varepsilon},1}(x)=0$ if $\Verts x_{2}\ge r_{\varepsilon}+1.$
Then, by Lemma \ref{lem:convolution-Stein-error}, the function $\weight T_{\varphi_{\polyorder},\rho}\tilde{v}_{f}^{\weight}$
satisfies, for $\Verts x_{2}\leq r_{\varepsilon},$ 
\begin{align*}
\verts{\mathcal{T}_{P}^{m}v_{f}(x)-\mathcal{T}_{P}^{m}\bigl(\weight T_{\varphi_{\polyorder},\rho}\tilde{v}_{f}^{\weight}\bigr)(x)}\leq C_{P,\datdim,\weight}^{(3)}\rho(1\lor\rho)^{(\polyorder-1)}\cdot r_{\varepsilon}^{\polyorder+\polyorder_{m}},
\end{align*}
where 
\[
C_{P,\datdim,\weight}^{(3)}=2^{\polyorder}\left\{ 2\tilde{u}_{P,\datdim,\weight}^{(1)}\lambda_{b}+\lambda_{m}\left(\datdim\tilde{u}_{P,\datdim,\weight}^{(2)}+\frac{\polyorder-1}{2\tau}\tilde{u}_{P,\datdim,\weight}^{(1)}\right)\bigl(c_{m}+1\bigr)\right\} .
\]

Finally, we address the second term of (\ref{eq:Stein-bound-split-integral}).
We separately consider the two cases of the growth of $\Verts{m(x)}_{\mathrm{op}}:$
$\polyorder_{m}=0$ and $\polyorder_{m}=1.$ For the linear growth
case $\polyorder_{m}=0$, note that by \cref{lem:Stein-convolution-decay-linear},
there exist positive constants $A_{0}$ and $A_{1}$ such that for
$\Verts x_{2}\geq r_{\varepsilon}$, 
\[
\norm{w(x)T_{\varphi_{\polyorder},\rho}\tilde{v}_{f}^{\weight}(x)}_{2}\leq A_{0}\Verts x_{2}^{\polyorder-1}\ \text{and}\ \left\Vert \nabla\bigl(w(x)T_{\varphi_{\polyorder},\rho}\tilde{v}_{f}^{\weight}(x)\bigr)\right\Vert _{\mathrm{F}}\leq A_{1}\Verts x_{2}^{\polyorder-1}.
\]
Thus, by Lemma \ref{lem:integral-eval-rkhs}, $v=wT_{\varphi_{\polyorder},\rho}\tilde{v}_{f}^{\weight}$
satisfies 
\begin{align*}
\int_{\Verts x_{2}>r_{\varepsilon}}\verts{\mathcal{T}_{P}^{m}v(x)}\dd Q(x)\leq C_{P,\datdim}^{(2)}\left(\Verts{v_{\varepsilon}}_{\rkhs K}\ksd{K_{m}}P(Q)+\varepsilon\right),
\end{align*}
where $C_{P,\datdim}^{(2)}=2\lambda_{b}A_{0}+(c_{m}+1)\lambda_{m}A_{1}\sqrt{\datdim}.$

In the case of the quadratic growth $\polyorder_{m}=1,$ by 
\cref{lem:Stein-convolution-decay-quad}, we have 
\[
\left\Vert \nabla\Bigl(w(x)T_{\varphi_{\polyorder},\rho}\tilde{v}_{f}^{\weight}(x)\Bigr)\right\Vert _{\mathrm{F}}\leq C_{P,\datdim,\weight}\Verts x_{2}^{\polyorder-2}
\]
for some constant $C_{P,\datdim,\weight}>0$ and $\Verts x_{2}>2(r_{\varepsilon}+1).$
We decompose the integral as follows: 
\begin{align*}
 & \int_{\Verts x_{2}>r_{\varepsilon}}\verts{\mathcal{T}_{P}^{m}v(x)}\dd Q(x)\\
 & =\int_{r_{\varepsilon}<\Verts x_{2}\leq2(r_{\varepsilon}+1)}\verts{\mathcal{T}_{P}^{m}v(x)}\dd Q(x)+\int_{\Verts x_{2}>2(r_{\varepsilon}+1)}\verts{\mathcal{T}_{P}^{m}v(x)}\dd Q(x).
\end{align*}
 Using  \cref{lem:Stein-convolution-decay-linear}, the first
integral on the RHS is evaluated as 
\begin{align*}
 & \int_{r_{\varepsilon}<\Verts x_{2}\leq2(r_{\varepsilon}+1)}\verts{\mathcal{T}_{P}^{m}v(x)}\dd Q(x)\\
 & \leq\int_{r_{\varepsilon}<\Verts x_{2}\leq2(r_{\varepsilon}+1)}\left(2\Verts{b(x)}_{2}\Verts{v(x)}_{2}+\sqrt{\datdim}\Verts{m(x)}_{\mathrm{op}}\Verts{\nabla v(x)}_{\mathrm{F}}\right)\dd Q(x)\\
 & \leq\int_{r_{\varepsilon}<\Verts x_{2}\leq2(r_{\varepsilon}+1)}\left(2\lambda_{b}A_{0}\bigl(1+\Verts x_{2}\bigr)\Verts x_{2}^{\polyorder-1}+\lambda_{m}A_{1}\sqrt{\datdim}\bigl(c_{m}+\Verts x_{2}^{2}\bigr)\Verts x_{2}^{\polyorder-1}\right)\dd Q(x)\\
 & \leq\left(4\lambda_{b}A_{0}+2\lambda_{m}A_{1}\sqrt{\datdim}(1+c_{m})(2r_{\varepsilon}+1)\right)\int_{\Verts x_{2}>r_{\varepsilon}}\Verts x_{2}^{\polyorder}\dd Q(x)\\
 & \leq\dot{C}_{P,\datdim}^{(2)}r_{\varepsilon}\left(\Verts{v_{\varepsilon}}_{\rkhs K}\ksd{K_{m}}P(Q)+\varepsilon\right),
\end{align*}
where the transition from the third line to the fourth follows from
$\Verts x_{2}\leq2(r_{\varepsilon}+1)$ and $\Verts x_{2}>r_{\varepsilon}\geq1,$
and $\dot{C}_{P,\datdim}^{(2)}=6\{2\lambda_{b}A_{0}+\lambda_{m}A_{1}\sqrt{\datdim}(1+c_{m})\}.$
The second integral is evaluated using Lemma \ref{lem:integral-eval-rkhs}
as 
\[
\int_{\Verts x_{2}>2(r_{\varepsilon}+1)}\verts{\mathcal{T}_{P}^{m}v(x)}\dd Q(x)\leq\ddot{C}_{P,\datdim}^{(2)}\left(\Verts{v_{\varepsilon}}_{\rkhs K}\ksd{K_{m}}P(Q)+\varepsilon\right),
\]
 where $\ddot{C}_{P,\datdim}^{(2)}=2\lambda_{b}A_{0}+(c_{m}+1)\lambda_{m}C_{P,\datdim,\weight}\sqrt{\datdim}.$

Combining these estimates, we obtain
\begin{align*}
 & \left|\int\mathcal{T}_{P}^{m}v_{f}-\mathcal{T}_{P}^{m}\Bigl(wT_{\varphi_{\polyorder},\rho}\tilde{v}_{f}^{\weight}\Bigr)\dd Q\right|\\
 & \leq\begin{cases}
\bigl(C_{P,\datdim}^{(1)}+C_{P,\datdim}^{(2)}\bigr)\left(\Verts{v_{\varepsilon}}_{\rkhs K}\ksd{K_{m}}P(Q)+\varepsilon\right)+C_{P,\datdim,\weight}^{(3)}\rho(1\lor\rho)^{(\polyorder-1)}r_{\varepsilon}^{\polyorder} & \text{if}\ \polyorder_{m}=0,\\
\left(C_{P,\datdim}^{(1)}+\dot{C}_{P,\datdim}^{(2)}r_{\varepsilon}+\ddot{C}_{P,\datdim}^{(2)}\right)\left(\Verts{v_{\varepsilon}}_{\rkhs K}\ksd{K_{m}}P(Q)+\varepsilon\right)+C_{P,\datdim,\weight}^{(3)}\rho(1\lor\rho)^{(\polyorder-1)}\cdot r_{\varepsilon}^{\polyorder+1} & \text{if}\ \polyorder_{m}=1.
\end{cases}
\end{align*}
\end{proof}

\subsubsection{Evaluating the RKHS norm of convolution approximations }

We denote by $L^{1}$ and $L^{2}$ by the respective Lebesgue spaces
of absolutely integrable and square integrable functions; i.e., $L^{i}=L^{i}(\bbR,\lambda)$
for $i\in\{1,2\}$ with $\lambda$ the $\datdim$-dimensional Lebesgue
measure. 
Recall that by $\hat{f}$, we denote the Fourier transform of $f\in L^1 \cup L^2$. 

\begin{lem}[Modified Bochner-Riesz multiplier kernel]
\label{lem:Bochner-Riesz-density}Let 
\[
\varphi_{\polyorder}(x)=\frac{\hat{\Phi}_{\polyorder}(x)^{2}}{\Verts{\hat{\Phi}_{\polyorder}}_{L^{2}}^{2}}
\]
where 
\[
\hat{\Phi}_{\polyorder}(x)=2^{\polyorder/2}\Gamma(\polyorder/2+1)\Verts x_{2}^{-(\datdim+\polyorder)/2}J_{(\datdim+\polyorder)/2}\bigl(\Verts x_{2}\bigr),
\]
and $J_{(\datdim+\polyorder)/2}$ denotes the Bessel function of the first kind with order $(\datdim+\polyorder)/2$. 
Then, 
\[
\hat{\varphi}_{\polyorder}(\omega)\leq\bigl(2\pi\bigr)^{-\datdim/2}\frac{\mathrm{B}(\datdim/2,\polyorder/2+1)}{\mathrm{B}(\datdim/2,\polyorder+1)}1\{\Verts{\omega}_{2}\leq2\},
\]
where $\mathrm{B}$ is the beta function. 
\end{lem}

\begin{proof}
By Theorem \ref{thm:Bochner-Riesz-mean-func}, the function $\hat{\Phi}_{\polyorder}$
is the Fourier transform of 
\[
\Phi_{\polyorder}(t)=\begin{cases}
\bigl(1-\Verts t_{2}^{2}\bigr)^{\polyorder/2} & \text{if}\ \Verts t_{2}\leq1,\\
0 & \text{if}\ \Verts t_{2}>1.
\end{cases}
\]
By the Plancherel's theorem, we have 
\begin{align*}
\Verts{\hat{\Phi}_{\polyorder}}_{L^{2}}^{2} & =\Verts{\Phi_{\polyorder}}_{L^{2}}^{2}\\
 & =\int_{\Verts x_{2}\leq1}(1-\Verts x_{2}^{2})^{\polyorder}\dd x\\
 & =\frac{2\pi^{\datdim/2}}{\Gamma(\datdim/2)}\int_{[0,1]}r^{\datdim-1}(1-r^{2})^{\polyorder}\dd r\\
 & =\frac{\pi^{\datdim/2}}{\Gamma(\datdim/2)}\mathrm{B}(\datdim/2,\polyorder+1).
\end{align*}
Thus, by the convolution theorem \citep[Theorem 5.16]{Wendland2004}, 
we obtain for $\Verts{\omega}_{2}\leq2,$
\begin{align*}
\hat{\varphi}_{\polyorder}(\omega) & =\bigl(2\pi\bigr)^{-\datdim/2}\frac{\Gamma(\datdim/2)}{\pi^{\datdim/2}}\frac{1}{\mathrm{B}(\datdim/2,\polyorder+1)}\cdot\Phi_{\polyorder}*\Phi_{\polyorder}(\omega)\\
 & \leq\bigl(2\pi\bigr)^{-\datdim/2}\frac{\mathrm{B}(\datdim/2,\polyorder/2+1)}{\mathrm{B}(\datdim/2,\polyorder+1)},
\end{align*}
where the inequality is obtained by observing 
\begin{align*}
\Phi_{\polyorder}*\Phi_{\polyorder}(\omega) & =\int\Phi_{\polyorder}(\omega')\Phi_{\polyorder}(\omega-\omega')\dd\omega'\\
 & \leq\int\Phi_{\polyorder}(\omega')\dd\omega'\\
 & =\frac{2\pi^{\datdim/2}}{\Gamma(\datdim/2)}\int_{[0,1]}r^{\datdim-1}(1-r^{2})^{\polyorder/2}\dd r\\
 & =\frac{\pi^{\datdim/2}}{\Gamma(\datdim/2)}\mathrm{B}(\datdim/2, \polyorder/2+1).
\end{align*}
The claim for the case $\Verts{\omega}_{2}>2$ holds since $\Phi_{\polyorder}*\Phi_{\polyorder}(\omega)=0$
in the given regime. 
\end{proof}
\begin{lem}[RKHS norm evaluation of tilted convolution]
\label{lem:RKHS-norm-translation-invariant} Let $\weight:\bbR^{\datdim}\to(0,\infty).$
Let 
\[
k(x,y)=\text{\ensuremath{\Phi_{\weight}(x-y)+\ell(x,y)}},
\]
where 
\[
\Phi_{\weight}(x,y)=\weight(x)\weight(y)\Phi(x-y),
\]
with $\Phi$ a positive definite function with a generalized Fourier transform $\hat{\Phi}$, and $\ell$ is an arbitrary scalar-valued positive
definite kernel. For $\polyorder\geq1,$ let $\varphi_{\polyorder}$
defined as in Lemma \ref{lem:Bochner-Riesz-density}. For $\rho>0,$
define 
\[
T_{\varphi_{\polyorder},\rho}f(x)=\rho^{-\datdim}\int f(x-y)\varphi_{\polyorder}(y/\rho)\dd y,
\]
where $f:\bbR^{\datdim}\to\bbR$ is in $L^{1}\cap L^{2}.$ Then, 
\[
\norm{\weight T_{\varphi_{\polyorder},\rho}f}_{\rkhs k}\leq(2\pi)^{-\datdim/4}\Verts f_{L^{2}}\cdot\sqrt{\sup_{\Verts{\omega}_{2}\le2\rho^{-1}}\hat{\Phi}(\omega)^{-1}}.
\]
\end{lem}

\begin{proof}
According to \citet{aronszajnTheoryReproducingKernels1950}, because
$k$ is the sum of two kernels, the RKHS norm of $\rkhs k$ of a (scalar-valued)
function $h\in\rkhs k$ is given by
\[
\Verts h_{\rkhs k}^{2}
=\min\{\Verts{h_{1}}_{\rkhs{\Phi_{\weight}}}^{2}+\Verts{h_{2}}_{\rkhs{\ell}}^{2}:h_1\in\rkhs{\Phi_{\weight}},h_{2}\in\rkhs{\ell} \text{ s.t. } h = h_1+h_2\}.
\]
Thus, we can evaluate the norm as 
\[
\Verts{\weight T_{\varphi_{\polyorder},\rho}f}_{\rkhs k}\leq\Verts{\weight T_{\varphi_{\polyorder},\rho}f}_{\rkhs{\Phi_{\weight}}}=\norm{T_{\varphi_{\polyorder},\rho}f}_{\rkhs{\Phi}}.
\]
Using the representation of the RKHS norm $\Verts{\cdot}_{\rkhs{\Phi}}$
\citep[Theorem 10.21]{Wendland2004}, we obtain 
\begin{align*}
\norm{T_{\varphi_{\polyorder},\rho}f}_{\rkhs{\Phi}}^{2} & =(2\pi)^{-\datdim/2}\int\frac{\verts{\widehat{T_{\varphi_{\polyorder},\rho}f}(\omega)}^{2}}{\hat{\Phi}(\omega)}\dd\omega\\
 & =(2\pi)^{-\datdim/2}\int\frac{(2\pi)^{\datdim}\verts{\widehat{f}(\omega)}^{2}\verts{\hat{\varphi}_{\polyorder}(\rho\omega)}^{2}}{\hat{\Phi}(\omega)}\dd\omega\\
 & \leq(2\pi)^{-\datdim/2}\sup_{\Verts{\omega}_{2}\le2\rho^{-1}}\hat{\Phi}(\omega)^{-1}\cdot\Verts{(2\pi)^{\datdim/2}\hat{f}\hat{\varphi}_{\polyorder}}_{L^{2}}^{2}\\
 & =(2\pi)^{-\datdim/2}\sup_{\Verts{\omega}_{2}\le2\rho^{-1}}\hat{\Phi}(\omega)^{-1}\cdot\Verts{\widehat{T_{\varphi_{\polyorder},\rho}f}}_{L^{2}}^{2}
\end{align*}
where the second and last equalities are due to the convolution theorem
\citep[Theorem 5.16]{Wendland2004}; the inequality follows
from the bounded support of $\hat{\varphi}_{\polyorder}$ (Lemma \ref{lem:Bochner-Riesz-density}).
Since $\varphi_{\polyorder}$ is in $L^{1}$, and $f\in L^{1}\cap L^{2},$
by \citep[Theorem 1.3]{EliasM.Stein1971}, we have 
\[
\norm{T_{\varphi_{\polyorder},\rho}f}_{L^{j}}\leq\Verts f_{L^{j}}\Verts{\varphi_{\polyorder}}_{L^{1}}<\infty\ \text{for}\ j\in\{1,2\},
\]
i.e., $T_{\varphi_{\polyorder},\rho}f\in L^{1}\cap L^{2}.$ Thus, the Fourier transform in the first line is well-defined. 
Moreover, by the Plancherel theorem \citep[Corollary 5.25]{Wendland2004}, 
the norm of $\widehat{T_{\varphi_{\polyorder},\rho}f}$ is evaluated as 
\begin{align*}
\Verts{\widehat{T_{\varphi_{\polyorder},\rho}f}}_{L^{2}} & =\Verts{T_{\varphi_{\polyorder},\rho}f}_{L^{2}}\\
 & \leq\Verts f_{L^{2}}\Verts{\varphi_{\polyorder}}_{L^{1}}=\Verts f_{L^{2}}.
\end{align*}
\end{proof}
\begin{lem}[Norm evaluation of $\polyorder$-growth approximator]
\label{lem:norm-q-growth-approx}Consider the same assumptions as
in Lemma \ref{lem:q-approx-constructive}; define the same symbols.
Let $\polyorder\geq1.$ For $\varepsilon>0,$ let $\mathring{v}_{\varepsilon}$
be the function in Lemma \ref{lem:q-approx-constructive} obtained
using the density $\varphi_{\polyorder}$ in Lemma \ref{lem:Bochner-Riesz-density}
such that ${\cal T}_{P}\mathring{v}_{\varepsilon}(x)\geq\Verts x^{\polyorder}1\{\Verts x_{2}>r_{\varepsilon}\}-\varepsilon$
with $r_{\varepsilon}\geq1.$ Let $K=k\idmat$ with 
\[
    k(x,y)=\weight_{\polyorder-2}(x)\weight_{\polyorder-2}(y)\bigl(\weight_{\theta}(x)\weight_{\theta}(y)\Phi(x-y)+k_{\mathrm{lin}}(x,y)\bigr),
\]
for $\theta \in [0,1]$. 
Then, 
\begin{align*}
\Verts{\mathring{v}_{\varepsilon}}_{\rkhs K} 
&\leq2\eta^{-1}\left(\sqrt{\datdim}+C^{(4)}_{\datdim}r_{\varepsilon}^{\datdim/2+1-\theta}\sqrt{\sup_{\Verts{\omega}_{2}\le2\rho^{-1}}\hat{\Phi}(\omega)^{-1}}\right)
\end{align*}
where 
\[
    C_{\datdim}^{(4)}=2^{3/2-\theta}\sqrt{\frac{2^{\datdim/2}}{\Gamma(\datdim/2)}},
\]
and $\rho>0$ is given as in Lemma \ref{lem:q-approx-constructive}. 
\end{lem}

\begin{proof}
Recall that $\mathring{v}_{\varepsilon}$ is given as 
\[
\mathring{v}_{\varepsilon}(x)= 2\eta^{-1}\weight_{\polyorder-2}(x)\left(-x - \weight_{\theta}(x)T_{\varphi_{\polyorder},\rho}\tilde{v}_{0}(x)\right)
\]
where $\tilde{v}_0(x)$ is a truncated version of $x\mapsto -\weight_{\theta}(x)^{-1} x$ with 
\[
\sup_{x\in\bbR^{\datdim}}\Verts{\tilde{v}_{0}(x)}_{2}\leq (r_{\varepsilon}+1)^{1-\theta}
\]
supported in $\{\Verts x_{2}\leq r_{\varepsilon}+1\}.$ 
We have
\begin{align*}
\Verts{(\tilde{v}_{0})_{\dimidx}}_{L^{2}} & \leq  (r_{\varepsilon}+1)^{1-\theta}\sqrt{\lambda\{\Verts x_{2}\leq r_{\varepsilon}+1\}}\\
 & =\sqrt{\frac{\pi^{\datdim/2}}{\Gamma(\datdim/2+1)}(r_{\varepsilon}+1)^{\datdim+2(1-\theta)}},
\end{align*}
where $(\tilde{v}_{0})_{\dimidx}$ is $\dimidx$-th component of $\tilde{v}_{0},$
and $\lambda$ is the $\datdim$-dimensional Lebesgue measure. Also,
for each $\dimidx\in\{1,\dots,\datdim\},$ we have 
\[
    \norm{x\mapsto \weight_{-1}(x)x_i}_{\rkhs{\bar{k}_{\mathrm{lin}}}}=\norm{x_{\dimidx}}_{\rkhs{k_{\mathrm{lin}}}}=1,
\]
where $x_{\dimidx}$ is the $\dimidx$-th coordinate of $x,$ and
$k_{\mathrm{lin}}(x,y)=\la x,y\ra+\tau^{2}.$ Note that for an $\bbR^{\datdim}$-valued
function $v$, we have 
\[
\Verts v_{\rkhs{k\idmat}}=\sqrt{\sum_{\dimidx=1}^{\datdim}\Verts{v_{\dimidx}}_{\rkhs k}^{2}},
\]
where $v_{\dimidx}$ is the $\dimidx$-th component function. Thus,
by Lemma \ref{lem:RKHS-norm-translation-invariant} and the norm characterization
for sum kernels, we have 
\begin{align*}
\Verts{\mathring{v}_{\varepsilon}}_{\rkhs K} & \leq2\eta^{-1}\left(\sqrt{\datdim}+(2\pi)^{-\datdim/4}\sqrt{\datdim}\cdot\sqrt{\frac{\pi^{\datdim/2}}{\Gamma(\datdim/2+1)}(2r_{\varepsilon})^{\datdim+2(1-\theta)}}\cdot\sqrt{\sup_{\Verts{\omega}_{2}\le2\rho^{-1}}\hat{\Phi}(\omega)^{-1}}\right)\\
& \leq2\eta^{-1}\left(\sqrt{\datdim}+2^{(3/2-\theta)}\sqrt{\frac{2^{\datdim/2}}{\Gamma(\datdim/2)}}r_{\varepsilon}^{\datdim/2+(1-\theta)}\sqrt{\sup_{\Verts{\omega}_{2}\le2\rho^{-1}}\hat{\Phi}(\omega)^{-1}}\right).
\end{align*}
\end{proof}
\begin{lem}[Norm of approximate solution to Stein equation]
\label{lem:norm-stein-eq-approx}
Define symbols as in \cref{thm:plip-ksd-bound-basic}. For $f\in{\cal F}_{\polyorder}$
and $\rho>0,$ we have 
\[
\norm{wT_{\varphi_{\polyorder},\rho}\tilde{v}_{f}^{\weight}}_{\rkhs{k\idmat}}\leq C_{P,\datdim,\weight}^{(5)}\cdot r_{\varepsilon}^{\datdim/2+1-\theta}\sqrt{\sup_{\Verts{\omega}_{2}\le2\rho^{-1}}\hat{\Phi}(\omega)^{-1}},
\]
where 
\[
    C_{P,\datdim,\weight}^{(5)}=2^{3/2-\theta}\cdot\zeta_{1}\bigl(\tau^{-(\polyorder-1)}+1\bigr)\sqrt{\frac{2^{\datdim/2}}{\Gamma(\datdim/2)}} (\tau^2+1)^{(1-\theta)/2}.
\]
\end{lem}

\begin{proof}
Since $\tilde{v}_{f}^{\weight}$ is supported in $\{\Verts x_{2}\leq r_{\varepsilon}+1\}$
and satisfies 
\[
    \sup_{x\in\bbR^{\datdim}}\Verts{\tilde{v}_{f}^{\weight}(x)}_{2}\leq 2^{1-\theta}\sqrt{\datdim}\zeta_{1}(\tau^{-(\polyorder-1)}+1)(\tau^2+1)^{(1-\theta)/2}r_{\varepsilon}^{1-\theta}
\]

the claim follows from Lemma \ref{lem:RKHS-norm-translation-invariant}
and 
\[
\Verts v_{\rkhs{k\idmat}}=\sqrt{\sum_{\dimidx=1}^{\datdim}\Verts{v_{\dimidx}}_{\rkhs k}^{2}},
\]
where $v$ is an $\bbR^{\datdim}$-valued function, and $v_{i}$ is
the $\dimidx$-th component function. 
\end{proof}

\subsubsection{Bound for translation-invariant kernels }
\begin{thm}[KSD bound for general translation-invariant kernels]
\label{thm:plip-ksd-bound-id-kernel-translation-invariant} 
Let $\polyorder\geq1$ and $\theta \in [0,1]$. 
Let $K=k\idmat$ with scalar-valued kernel 
\[
    k(x,y)=\weight_{\polyorder-2}(x)\weight_{\polyorder-2}(y)\bigl(\weight_{\theta}(x)\weight_{\theta}(y)\Phi(x-y)+k_{\mathrm{lin}}(x,y)\bigr),
\]
where $\Phi$ is assumed to have continuous non-vanishing generalized Fourier transform. 
Under Assumptions \ref{assu:poly-grow-coef}, \ref{assu:dissipativity-diffusion},
\ref{assu:diff-wasserstein-rate}, there exist $r_{\varepsilon}\geq1$
and $\rho_{\varepsilon}>0$ such that for any $\varepsilon,\rho>0,$
\begin{align*}
 & d_{{\cal F}_{\polyorder}}(P,Q)\\
 & \leq C_{P,\datdim}(\polyorder_{m})r_{\varepsilon}^{\polyorder_{m}}\left(2\eta^{-1}\left( \sqrt{\datdim}+C_{\datdim}^{(4)}r_{\varepsilon}^{\datdim/2+1-\theta}\sqrt{\sup_{\Verts{\omega}_{2}\le2\rho_{\varepsilon}^{-1}}\hat{\Phi}(\omega)^{-1}}\right)\cdot\ksd{K_{m}}P(Q)+\varepsilon\right)\\
 & \hphantom{\leq}\quad+C_{P,\datdim,\weight}^{(3)}\cdot\rho(1\lor\rho)^{\polyorder-1}\cdot r_{\varepsilon}^{\polyorder+\polyorder_{m}}+C_{P,\datdim,\weight}^{(5)}\cdot r_{\varepsilon}^{\datdim/2+1-\theta}\sqrt{\sup_{\Verts{\omega}_{2}\le2\rho^{-1}}\hat{\Phi}(\omega)^{-1}}\cdot\ksd{K_{m}}P(Q),
\end{align*}
where 
\[
C_{P,\datdim}(\polyorder_{m})=\begin{cases}
C_{P,\datdim}^{(1)}+C_{P,\datdim}^{(2)} & \text{if}\ \polyorder_{m}=0,\\
C_{P,\datdim}^{(1)}+\dot{C}_{P,\datdim}^{(2)}+\ddot{C}_{P,\datdim}^{(2)} & \text{if}\ \polyorder_{m}=1;
\end{cases}
\]
the constants defining $C_{P,\polyorder}(\polyorder_{m})$ are
given in Theorem \ref{thm:plip-ksd-bound-basic};
$C_{\datdim}^{(4)}$ and $C_{P,\datdim, \weight}^{(5)}$ are given in \cref{lem:norm-q-growth-approx,lem:norm-stein-eq-approx}, respectively. 
In particular, for
a sequence $\{Q_{\seqidx}\}_{\seqidx=1}^{\infty}\subset{\cal P}$, 
the convergence $\ksd{K_{m}}P(Q_{n})\to0$ implies $d_{{\cal F}_{\polyorder}}(P,Q_{\seqidx})\to0$
as $n\to\infty.$ 
\end{thm}

\begin{proof}
The conclusion follows from \cref{thm:plip-ksd-bound-basic} using
Lemmas \ref{lem:q-approx-constructive}, 
\ref{lem:norm-q-growth-approx}, 
and \ref{lem:norm-stein-eq-approx}. 
\end{proof}

\subsubsection{\pcref{thm:Matern-bound-linear}}
\label{subsec:proof-Matern-KSD-linear}

\MaternKSDLinear*
\begin{proof}
We use Theorem \ref{thm:plip-ksd-bound-id-kernel-translation-invariant}.
Since, according to \citet[Theorem 6.13]{Wendland2004}, 
we have $\hat{\Phi}(\omega)=\verts{\Sigma}^{-1}\bigl(1+\Verts{\Sigma^{-1}\omega}_{2}^{2}\bigr)^{-\datdim/2-\nu}$
with $\verts{\Sigma}$ the determinant of $\Sigma$,  
we obtain 
\begin{align*}
\sup_{\Verts{\omega}_{2}\le\lambda}\hat{\Phi}(\omega)^{-1} & =\sup_{\Verts{\omega}_{2}\le\lambda}\verts{\Sigma}\bigl(1+\Verts{\Sigma^{-1}\omega}_{2}^{2}\bigr)^{\datdim/2+\nu}\\
 & \leq\verts{\Sigma}\bigl(1+\Verts{\Sigma^{-1}}_{\mathrm{op}}^{2}\lambda^{2}\bigr)^{\datdim/2+\nu}
\end{align*}
for any $\lambda>0.$ By Lemma \ref{lem:q-approx-constructive} with $\theta=1$, we
have 
\begin{align*}
\rho_{\varepsilon} & =\frac{\eta}{4}(1\land\delta^{2})(2\land\varepsilon r_{\varepsilon}^{-\polyorder})U_{3}^{-1}
\end{align*}
for some constant $U_{3}>0.$ Then, with $s=\datdim/2+\nu,$ 
\begin{align*}
 & \sqrt{\sup_{\Verts{\omega}_{2}\le8\eta^{-1}\rho_{\varepsilon}^{-1}}\hat{\Phi}(\omega)^{-1}}\\
 & =\sqrt{\verts{\Sigma}}\left\{ 1+2^{6}\eta^{-2}\Verts{\Sigma^{-1}}_{\mathrm{op}}^{2}U_{3}^{2}(1\lor\delta^{-4})\bigr(1\lor\varepsilon^{-2}r_{\varepsilon}^{2\polyorder}\bigr)\right\} ^{s/2}.
\end{align*}
 By Corollary \ref{cor:r-delta-expression}, we have $\delta=rW_{r}\varepsilon^{-\frac{1}{3}},$
and $r_{\varepsilon}=r+\delta=(1+W_{r}\varepsilon^{-\frac{1}{3}})r$
for some constant $r\ge1$ and $W_{r}>0.$ Thus, 
\begin{align*}
1\lor\delta^{-4} & \leq\{1\lor(rW_{r})^{-4}\}(1\lor\varepsilon^{-4/3}),
\end{align*}
and 
\[
r_{\varepsilon}\leq r\bigl(1+W_{r}\bigr)\bigl(1\lor\varepsilon^{-1/3}\bigr).
\]
Therefore, 
\begin{align*}
 & \sqrt{\sup_{\Verts{\omega}_{2}\le8\eta^{-1}\rho_{\varepsilon}^{-1}}\hat{\Phi}(\omega)^{-1}}\\
 & \leq C_{P,\Phi,r}\bigl\{(1\lor\varepsilon^{-4/3})(1\lor\varepsilon^{-2}\lor\varepsilon^{-2(\polyorder/3+1)}\bigr)\bigr\}^{s/2}\\
 & \leq C_{P,\Phi,r}(1\lor\varepsilon^{-s(\polyorder+5)/3}\bigr)
\end{align*}
 with 
\[
C_{P,\Phi,r}=\sqrt{\verts{\Sigma}}\left(1\lor2^{6}\eta^{-2}\Verts{\Sigma^{-1}}_{\mathrm{op}}^{2}U_{3}^{2}\right)\{1\lor(rW_{r})^{-4}\}^{s/2}\{r(1+rW_{r})\}^{s\polyorder}.
\]

Also, 
\begin{align*}
 & \sqrt{\sup_{\Verts{\omega}_{2}\le2\rho^{-1}}\hat{\Phi}(\omega)^{-1}}\\
 & \leq\sqrt{\verts{\Sigma}}\bigl(1+4\Verts{\Sigma^{-1}}_{\mathrm{op}}^{2}\rho^{-2}\bigr)^{s/2}\\
 & \leq\sqrt{\verts{\Sigma}}\rho^{-s}\bigl(\rho^{2}+4\Verts{\Sigma^{-1}}_{\mathrm{op}}^{2}\bigr)^{s/2}\\
 & \leq\sqrt{\verts{\Sigma}}\bigl(1+4\Verts{\Sigma^{-1}}_{\mathrm{op}}^{2}\bigr)^{s/2}\bigl(1\lor\rho^{-s}\bigr).
\end{align*}
Define the following constants: 
\begin{align*}
A^{(0)} & =\bigl(C_{P,\datdim}^{(1)}+C_{P,\datdim}^{(2)}\bigr)\\
 & =2(1+\lambda_{b}\zeta_1\sqrt{\datdim}+4\lambda_{m}\zeta_{2}c_m\datdim)\\
 & \hphantom{+}+2\sqrt{\datdim}\bigl(\tau^{2}+1)^{(\polyorder-1)/2}\left[2\lambda_{b}\zeta_{1}+\lambda_m(c_{m}+1)\left\{ \left(\frac{\polyorder-1}{\tau}+3\right)\zeta_{1}\sqrt{\datdim}+\frac{\zeta_{2}}{2}\right\} \right], 
\end{align*}

\begin{align*}
A^{(1)} & =2\eta^{-1}\sqrt{\datdim}A^{(0)}C_{\datdim}^{(4)}\{r(1+W_{r})\}^{\datdim/2}C_{P,\Phi,r}\\
 & =2\eta^{-1}\sqrt{\datdim}A^{(0)}\sqrt{\frac{2^{\datdim/2}}{\Gamma(\datdim/2+1)}}\{r(1+W_{r})\}^{\datdim/2}\\
 & \hphantom{=}\quad\cdot\sqrt{\verts{\Sigma}}\left(1\lor2^{6}\eta^{-2}\Verts{\Sigma^{-1}}_{\mathrm{op}}^{2}U_{3}^{2}\right)\{1\lor(rW_{r})^{-4}\}^{s/2}\{r(1+rW_{r})\}^{s\polyorder},
\end{align*}
 
\begin{align*}
A^{(2)} & =C_{P,\datdim,\weight}^{(3)}r\bigl(1+W_{r}\bigr)\bigr\}^{\polyorder},\\
\end{align*}
 and 
\begin{align*}
A^{(3)} & =C_{P,\datdim.\polyorder,\weight}^{(5)}r^{\datdim/2}\bigl(1+W_{r}\bigr)^{\datdim/2}\sqrt{\verts{\Sigma}}\bigl(1+4\Verts{\Sigma^{-1}}_{\mathrm{op}}^{2}\bigr)^{s/2}\\
 & =\sqrt{\datdim}\zeta_{1}\bigl(\tau^{-(\polyorder-1)}+1\bigr)\sqrt{\frac{\bigl\{2^{1/2}r\bigl(1+W_{r}\bigr)\bigr\}^{\datdim}\bigl(1+4\Verts{\Sigma^{-1}}_{\mathrm{op}}^{2}\bigr)^{s}}{\Gamma(\datdim/2+1)}}.
\end{align*}

By Theorem \ref{thm:plip-ksd-bound-id-kernel-translation-invariant},
with $S=\ksd{K_{m}}P(Q),$ we obtain 
\begin{align}
 & \sup_{f\in{\cal F}_{\polyorder}}\left|\int\mathcal{T}_{P}^{m}v_{f}\dd Q\right|\nonumber \\
 & \begin{aligned} & \leq2\eta^{-1}\sqrt{\datdim}A^{(0)}S+A^{(1)}\bigl(1\lor\varepsilon^{-\datdim/6}\bigr)(1\lor\varepsilon^{-s(\polyorder/3+5/3)})S+A^{(0)}\varepsilon\\
 & \hphantom{\leq}\quad+A^{(2)}\rho(1\lor\rho)^{\polyorder-1}\cdot\bigl(1\lor\varepsilon^{-\polyorder/3}\bigr)+A^{(3)}\varepsilon^{-\datdim/6}\bigl(1\lor\rho^{-s}\bigr)S.
\end{aligned}
\label{eq:Stein-bound-Matern-intermediate-linear-growth}
\end{align}
We let $\rho=\varepsilon^{\polyorder/3+1}.$ Then, the upper bound
becomes 
\begin{align*}
 & 2\eta^{-1}\sqrt{\datdim}A^{(0)}S+A^{(1)}\bigl(1\lor\varepsilon^{-\datdim/6}\bigr)(1\lor\varepsilon^{-s(\polyorder/3+5/3)}\bigr)S+A^{(0)}\varepsilon\\
 & \quad+A^{(2)}(1\lor\varepsilon^{\polyorder/3+1})^{\polyorder-1}\cdot\bigl(\varepsilon^{\polyorder/3}\lor1\bigr)\varepsilon+A^{(3)}\varepsilon^{-\datdim/6}\bigl(1\lor\varepsilon^{-s(\polyorder/3+1)}\bigr)S.
\end{align*}
We choose 
\[
\varepsilon=S^{\frac{1}{1+t}}\ \text{with}\ t=\datdim/6+s(\polyorder+5)/3. 
\]
Then, we obtain the final bound
\begin{align*}
 & d_{{\cal F}_{\polyorder}}(P,Q)\\
 & \leq\varepsilon\left\{ 2\eta^{-1}\sqrt{\datdim}A^{(0)}\varepsilon^{t}+A^{(1)}\bigl(1\lor\varepsilon^{-t}\bigr)\varepsilon^{t}+A^{(0)}\right.\\
 & \hphantom{\leq\mathcal{S}^{\frac{1}{1+\datdim/6+s(\polyorder/3+1)}}}\quad\left.+A^{(2)}(1\lor\varepsilon^{\polyorder/3+1})^{\polyorder-1}\cdot\bigl(\varepsilon^{\polyorder/3}\lor1\bigr)+A^{(3)}\bigl(\varepsilon^{-s(\polyorder/3+1)}\lor1\bigr)\varepsilon^{t-\datdim/6}\right\} \\
 & \leq S^{\frac{1}{1+t}}\left\{ 2\eta^{-1}\sqrt{\datdim}A^{(0)}S^{\frac{t}{1+t}}+A^{(1)}\bigl(S^{\frac{t}{1+t}}\lor1\bigr)+A^{(0)}\right.\\
 & \hphantom{\leq\mathcal{S}^{\frac{1}{1+\datdim/6+s(\polyorder/3+1)}}}\quad\left.+A^{(2)}(S^{\frac{(\polyorder/3+1)(\polyorder-1)}{1+t}}\lor1)\cdot\bigl(S^{\frac{\polyorder/3}{1+t}}\lor1\bigr)+A^{(3)}\bigl(S^{\frac{t-\datdim/6-s(\polyorder/3+1)}{1+t}}\lor S^{\frac{t-\datdim/6}{1+t}}\bigr)\right\} \\
 & \leq\bigl\{(2\eta^{-1}\sqrt{\datdim}+1)A^{(0)}+A^{(1)}+A^{(2)}+A^{(3)}\bigr\}\bigl(S^{\frac{t}{1+t}\lor\frac{\polyorder(\polyorder/3+1)}{1+t}}\lor1\bigr)S^{\frac{1}{1+t}}.
\end{align*}
Letting $A_{P,\datdim}=(2\eta^{-1}\sqrt{\datdim}+1)A^{(0)}+A^{(1)}+A^{(2)}+A^{(3)}$ concludes the proof. 
\end{proof}

\subsubsection{\pcref{thm:Matern-bound-quad}} %
\label{subsec:proof-Matern-KSD-quad}

\MaternKSDQuad*
\begin{proof}
The proof proceeds as in Theorem \ref{thm:Matern-bound-linear}. By
applying Lemma \ref{lem:q-approx-constructive} with $\theta=0$, we have 
\begin{align*}
\rho_{\varepsilon} & =\frac{\eta}{4}\left((1\land\delta^{2})\left(2\tilde{C}\land U_{3}^{-1}\varepsilon r_{\varepsilon}^{-(\polyorder+1)}\right)\land\frac{4}{\eta}\right)\\
 & \geq\frac{\eta}{4}(1\land\delta^{2})\left(2\tilde{C}\land U_{3}^{-1}\varepsilon r_{\varepsilon}^{-(\polyorder+1)}\land\frac{4}{\eta}\right)
\end{align*}
for some constants $U_{3},\tilde{C}>0;$ we also have. 
\begin{align*}
r_{\varepsilon} & =\left\{ 2\lor c_{P,\datdim,\polyorder,\tau}(1\lor\delta^{-1/\datdim})\right\} (r+\delta)\\
 & \leq\tilde{c}(1\lor\varepsilon^{-1/3\datdim})(1\lor\varepsilon^{-1/3})
\end{align*}
with $\tilde{c}=(2\lor c_{P,\datdim,\polyorder,\tau})r(1+W_{r})\bigl(1\lor W_{r}^{-1/\datdim}\bigr).$
Then, with $s=\datdim/2+\nu$, 
\begin{align*}
 & \sqrt{\sup_{\Verts{\omega}_{2}\le2\rho_{\varepsilon}^{-1}}\hat{\Phi}(\omega)^{-1}}\\
 & \leq\verts{\Sigma}\left\{ 1+\Verts{\Sigma^{-1}}_{\mathrm{op}}^{2}2^{6}\eta^{-2}(1\lor\delta^{-4})\bigl(4^{-1}\tilde{C}^{-2}\lor U_{3}^{2}\varepsilon^{-2}r_{\varepsilon}^{2(\polyorder+1)}\lor\eta^{2}/2^{4}\bigr)\right\} ^{s/2}\\
 & \leq C_{P,\Phi,r}\left\{ (1\lor\delta^{-4})\left(1\lor\varepsilon^{-2}r_{\varepsilon}^{2(\polyorder+1)}\right)\right\} ^{s/2}\\
 & \leq C_{P,\Phi,r}\{1\lor(rW_{r})^{-2s}\}(1\lor\tilde{c}^{s(\polyorder+1)})\left(1\lor\varepsilon^{-s\{(\polyorder+1)(1+\datdim^{-1})+5\}/3}\right)
\end{align*}
where 
\[
C_{P,\Phi,r}=\verts{\Sigma}\left[1+\Verts{\Sigma^{-1}}_{\mathrm{op}}^{2}2^{3}\eta^{-2}\left(\tilde{C}^{-2}\lor U_{3}^{2}\lor\eta^{2}/2^{4}\right)\right]^{s/2}.
\]

Define the following constants: 
\begin{align*}
 & A^{(0)}=\tilde{c}\left(C_{P,\datdim}^{(1)}+\dot{C}_{P,\datdim}^{(2)}+\ddot{C}_{P,\datdim}^{(2)}\right),\\
 & A^{(1)}=2\eta^{-1}\sqrt{\datdim}A^{(0)}C_{\datdim}^{(4)}\tilde{c}^{\datdim/2+1}C_{P,\Phi,r}\{1\lor(rW_{r})^{-2s}\}(1\lor\tilde{c}^{s(\polyorder+1)}),\\
 & A^{(2)}=C_{P,\datdim,\weight}^{(3)}\tilde{c}^{\polyorder+1}\\
 & A^{(3)}=C_{P,\datdim,\weight}^{(5)}\tilde{c}^{\datdim/2+1}\sqrt{\verts{\Sigma}}\bigl(1+4\Verts{\Sigma^{-1}}_{\mathrm{op}}^{2}\bigr)^{s/2}.
\end{align*}
Theorem \ref{thm:plip-ksd-bound-id-kernel-translation-invariant}
implies, with $S=\ksd{K_{m}}P(Q),$

\begin{align*}
 & d_{{\cal F}_{\polyorder}}(P,Q)\\
 & \leq\left\{ 2\eta^{-1}\sqrt{\datdim}A^{(0)}(1\lor\varepsilon^{-1/3})+A^{(1)}(1\lor\varepsilon^{-t})(1\lor\varepsilon^{-1/3})\right\} \cdot S+A^{(0)}\varepsilon(1\lor\varepsilon^{-1/3})\\
 & \hphantom{\leq}\quad+A^{(2)}\rho(1\lor\rho)^{\polyorder-1}\cdot(1\lor\varepsilon^{-(\polyorder+1)/3})+A^{(3)}\cdot(1\lor\varepsilon^{-\datdim/6})\bigl(1\lor\rho^{-s}\bigr)\cdot S
\end{align*}
where 
\[
t=\frac{s\{(\polyorder+1)(1+\datdim^{-1})+5\}}{3}+\frac{\datdim+2}{6}.
\]
With $\rho=\text{\ensuremath{\varepsilon^{\polyorder/3+1}}}$ and
$\varepsilon=S^{1/(1+t)},$ the above upper bound on $d_{{\cal F}_{\polyorder}}$
is evaluated as follows: 
\begin{align*}
 & \left\{ 2\eta^{-1}\sqrt{\datdim}A^{(0)}(1\lor\varepsilon^{1/3})\varepsilon^{t}+A^{(1)}(1\lor\varepsilon^{t+1/3})\right\} \varepsilon^{2/3}+A^{(0)}\varepsilon^{2/3}(1\lor\varepsilon^{1/3})\\
 & \quad+A^{(2)}(1\lor\varepsilon^{\polyorder/3+1})^{\polyorder-1}\cdot(\varepsilon^{(\polyorder+1)/3}\lor1)\varepsilon^{2/3}+A^{(3)}\cdot(1\lor\varepsilon^{-s(\polyorder/3+1)-\datdim/6}\bigr)\cdot\varepsilon^{t+1}\\
 & \leq\varepsilon^{2/3}\left[\left\{ 2\eta^{-1}\sqrt{\datdim}A^{(0)}(\varepsilon^{t}\lor\varepsilon^{t+1/3})+A^{(1)}(1\lor\varepsilon^{t+1/3})\right\} +A^{(0)}(1\lor\varepsilon^{1/3})\right.\\
 & \hphantom{\leq\varepsilon^{2/3}}\quad\left.+A^{(2)}(1\lor\varepsilon^{(\polyorder^{2}+3\polyorder-2)/3})+A^{(3)}\cdot(\varepsilon^{t+1/3}\lor\varepsilon^{t+1/3-s(\polyorder/3+1)-\datdim/6}\bigr)\right]\\
 & \leq\left(2\eta^{-1}\sqrt{\datdim}A^{(0)}+A^{(1)}+A^{(0)}+A^{(2)}+A^{(3)}\right)(1\lor S^{\frac{t+1/3}{1+t}\lor\frac{(\polyorder^{2}+3\polyorder-2)}{3(1+t)}})S^{\frac{2}{3}\frac{1}{1+t}}.
\end{align*}
Letting $B_{P, \datdim} = 2\eta^{-1}\sqrt{\datdim}A^{(0)}+A^{(1)}+A^{(0)}+A^{(2)}+A^{(3)}$ concludes the proof. 
\end{proof}
\section{Bounding Wasserstein distances}
In this appendix, we provide a proof for \cref{thm:Wass-ksd-bound}. 
We begin with a supporting lemma and then provide the proof. 

\begin{lem}
    For $\polyorder= (1, \infty)$, 
    let $P, Q \in \finitemomentspace{\polyorder}$. 
    For any coupling $\pi\in\Pi(P,Q)$ between $P$ and $Q$, 
    let $(X, Y)$ be random variables jointly distributed according to $\pi$, 
    and $Z = \Verts{X-Y}_2$. 
    Suppose we are given a function set $\mathcal{F}$ that approximates $\polyorder$-growth with respect to $P$. 
    Then, for any $\varepsilon > 0$, 
    \begin{equation}
        \Ex[ Z^{\polyorder} ] \leq 2^{\polyorder-1} \left( 
        3 r_{\varepsilon}^{\polyorder-1} \Ex[Z] + 2 \varepsilon + \Ex_Q[f_{\varepsilon} ] - \Ex_P[f_{\varepsilon}]
        \right). \label{eq:bound-Wq-W1-prelim}
    \end{equation}
    where $f_{\varepsilon} \in \mathcal{F}$ is $\polyorder$-growth approximation,  
    and $r_{\varepsilon} > 0 $ is the corresponding radius. 
\end{lem}
\begin{proof}
First note that, with $r=r_{\varepsilon}$, 
\begin{align*}
    &\Ex\bigl[ \Verts{Y}_2^{\polyorder}1\{Z > r\} \bigr] \\
    &= \Ex\bigl[ \Verts{Y}_2^{\polyorder}1\{Z > r, \Verts{Y}_2 \leq r \} \bigr] 
        + \Ex\bigl[ \Verts{Y}_2^{\polyorder}1\{Z > r, \Verts{Y}_2 > r \} \bigr] \\
    &\leq r^{\polyorder} Q (Z > r)  + \Ex\bigl[ \underbrace {\Verts{Y}_2^{\polyorder}1\{\Verts{Y}_2 > r \}  -f_{\varepsilon}(Y)}_{\leq \varepsilon - \Ex_{P}[f_{\varepsilon}]}+ f_{\varepsilon}(Y) \bigr] \\
    &\leq  r^{\polyorder-1} \Ex \bigl[ Z \bigr] + \varepsilon + \Ex_{Q}[f_{\varepsilon}] - \Ex_{P}[f_{\varepsilon}]. 
\end{align*}
where we have used the $\polyorder$-growth approximation property and Markov's inequality to obtain the final line. 
The inequality for $\Ex[\Verts{X}_2^{\polyorder}$ is given with $P$ substituted for $Q$. 
The evaluation above leads to the desired inequality: 
\begin{align*}
    \Ex[ Z^{\polyorder} ]
    &= \Ex[ Z^{\polyorder} 1\{Z \leq r_{\varepsilon}\} ] + \Ex[ Z^{\polyorder}1\{Z > r_{\varepsilon}\} ] \\
    &\leq r_{\varepsilon}^{\polyorder-1} \Ex[Z] + 2^{\polyorder-1}\left (
        \Ex[ \Verts{X}_2^{\polyorder}1\{Z > r_{\varepsilon}\} ] + \Ex[ \Verts{Y}_2^{\polyorder}1\{Z > r_{\varepsilon}\} ]
      \right)\\
    &\leq 2^{\polyorder-1} \left( 
        3 r_{\varepsilon}^{\polyorder-1} \Ex[Z] + 2 \varepsilon + \Ex_Q[f_{\varepsilon} ] - \Ex_P[f_{\varepsilon}]
        \right). 
\end{align*}
\end{proof}

Taking the infimum over all couplings on both sides of \eqref{eq:bound-Wq-W1-prelim} therefore yields a bound on the $\polyorder$-Wasserstein distance $W_{\polyorder}$: 
\begin{align*}
    W_{\polyorder}(P, Q)^{\polyorder}
    &\leq 2^{\polyorder-1}  
        \left( 
        3 r_{\varepsilon}^{\polyorder-1} W_1(P, Q) + 2\varepsilon +  
        \Ex_Q[f_{\varepsilon}] - \Ex_P[f_{\varepsilon}]
        \right ). 
\end{align*}
This indicates that using an IPM induced by (possibly a subset of) $\mathcal{F}$, 
one can further upper bound $W_1(P,Q)$ and control 
$\Ex_Q[f_{\varepsilon}] - \Ex_P[f_{\varepsilon}]$ to obtain an estimate of $W_{\polyorder}$. 
We will carry out this task below for a Stein RKHS to derive a KSD bound. 

\begin{rem}
In fact, one can obtain an inequality similar to \eqref{eq:bound-Wq-W1-prelim} that holds for any $\polyorder > 0$: 
\begin{align*}
    \Ex[ Z^{\polyorder} ]
    &\leq 2^{\polyorder} \left( 
        2^{-1}r_{\varepsilon}^{\polyorder-1} \Ex[ (Z\land 2r_{\varepsilon})] 
        +  2\varepsilon + \Ex_{Q}[f_{\varepsilon}] - \Ex_P[f_{\varepsilon}] 
    \right)
\end{align*}
After taking the infimum, we can control the first term in the round brackets by the (scaled) bounded-Lipschitz metric $r_{\varepsilon}^{\polyorder} d_{\mathrm{BL}}(P, Q)$ (see the proof of \cref{prop:plip-metric-Wasserstein-conv} for the definition). 
Therefore, an upper bound is given by any IPM that controls the bounded-Lipschitz metric and is defined by a function class approximating $\polyorder$-growth. 
For our purpose, we use the slightly tighter \eqref{eq:bound-Wq-W1-prelim}. 
\end{rem}

\subsection{\pcref{thm:Wass-ksd-bound}}\label{subsec:proof-Wass-ksd-bound}
Taking the infimum over all possible couplings in \eqref{eq:bound-Wq-W1-prelim}, with $\mathcal{F} = \mathcal{T}_P^m (\rkhs{K})$, yields 
\begin{equation}
    W_{\polyorder}(P, Q)^{\polyorder}
    \leq 2^{\polyorder-1}  
        \left( 
        3 r_{\varepsilon}^{\polyorder-1} W_1(P, Q) + 2\varepsilon +   \Verts{v_{\varepsilon}}_{\rkhs{K}} \ksd{K_m}{P}(Q) 
        \right ) \label{eq:KSD-Wq-bound-prelim}
\end{equation}
for any $\varepsilon > 0$, 
where $v_{\varepsilon} \in \rkhs{K}$ is a function that induces a $\polyorder$-growth approximation $f_{\varepsilon} = \mathcal{T}_P^m(v_{\varepsilon})$. 

We obtain a bound by choosing an appropriate $\varepsilon$. 
In the following, we separately consider the two growth cases of $\Verts{m(x)}_{\mathrm{op}}$: $\polyorder_m \in \{0, 1\}$. 
Before proceeding, we collect results used to evaluate \eqref{eq:KSD-Wq-bound-prelim}. 
For our kernel choices, by applying \cref{lem:norm-q-growth-approx} with $\theta = 1-\polyorder_m$, 
\begin{equation}
    \begin{aligned}
\Verts{v_{\varepsilon}}_{\rkhs K} 
&\leq2\eta^{-1}\left(\sqrt{\datdim}+
    2^{1/2+\polyorder_m}\sqrt{\frac{2^{\datdim/2}}{\Gamma(\datdim/2)}}
    r_{\varepsilon}^{\datdim/2+\polyorder_m}\sqrt{\sup_{\Verts{\omega}_{2}\le2\rho^{-1}}\hat{\Phi}(\omega)^{-1}}\right)
    \end{aligned}\label{eq:rkhs-norm-estimate-Wq-bound}
\end{equation}
Here, $v_{\varepsilon}$, $r_{\varepsilon}$, 
and $\rho>0$ are given as in Lemma \ref{lem:q-approx-constructive}. 
These objects depend on $\polyorder_m$, and we derive bounds accordingly. 
We also make use of \citet[Lemma 2.2]{Mackey_2016}, which provides the following bound on $W_1$: 
\begin{equation}
    \begin{aligned}
     W_1(P,Q) 
     &\leq 3 \left\{
        d_{\calF_1}(P, Q)\lor 
        \left( \sqrt{2}\Ex[\Verts{G}_2]^2 d_{\calF_{1}}(P, Q)\right)^{1/3}
    \right\}. %
    \end{aligned}
   \label{eq:W1-bound-alt}
\end{equation}

\subsubsection{Linear case $\polyorder_m=0$}
Let $s = \datdim /2 + \nu$.
From \cref{subsec:proof-Matern-KSD-linear}, 
we have 
\[
r_{\varepsilon}\leq C_{1}\bigl(1\lor\varepsilon^{-1/3}\bigr)
\]
and 
\begin{align*}
 \sqrt{\sup_{\Verts{\omega}_{2}\le8\eta^{-1}\rho_{\varepsilon}^{-1}}\hat{\Phi}(\omega)^{-1}}
 \leq C_{2}(1\lor\varepsilon^{-s(\polyorder+5)/3}\bigr)
\end{align*}
where 
\begin{align*}
    &C_{1} = r\bigl(1+W_{r}\bigr), \\
    &C_{2} =\sqrt{\verts{\Sigma}}\left(1\lor2^{6}\eta^{-2}\Verts{\Sigma^{-1}}_{\mathrm{op}}^{2}U_{3}^{2}\right)\{1\lor(rW_{r})^{-4}\}^{s/2}\{r(1+rW_{r})\}^{s\polyorder}, 
\end{align*}
with $\eta$, $r$, $U_3$ and $W_r$ being constants from \cref{lem:q-approx-constructive} and  \cref{cor:r-delta-expression}.

Substituting these estimates into the RHS in \eqref{eq:KSD-Wq-bound-prelim} gives 
\begin{align*}
        &r_{\varepsilon}^{\polyorder-1} W_1(P, Q) + 2\varepsilon +   \Verts{v_{\varepsilon}}_{\rkhs{K}} \ksd{K_m}{P}(Q) \\
        &\leq C_{1}^{\polyorder-1} (1\lor \varepsilon^{-(\polyorder-1)/3})
        W_1(P, Q)
        + 2\varepsilon \\
        &\quad + 2\eta^{-1} \left\{
            \sqrt{\datdim} + 
            2^{1/2}\sqrt{\frac{2^{\datdim/2}}{\Gamma(\datdim/2)}}
            C_1^{\datdim/2}
            C_2 
            (1\lor \varepsilon^{-t})
        \right\}S,
\end{align*}
where $t = (1+\polyorder/6)\datdim + \nu (\polyorder+5)/3$. 

As $\calF_1 \subset \calF_{\polyorder}$, the estimate \eqref{eq:W1-bound-alt} together with 
\cref{thm:Matern-bound-linear} yields  
\begin{equation}
    \begin{aligned}
     W_1(P,Q) 
    &\leq 3 \Bigl( A_{P,\datdim} \lor (\sqrt{2}A_{P,\datdim} \Ex[\Verts{G}_2]^2)^{1/3}\Bigr)
    (1 \lor S^{\frac{2}{3(1+t)}})
    S^{\frac{1}{3(1+t)}}
    \end{aligned} \label{eq:W1-KSD-bound-linear}
\end{equation}

Taking $\varepsilon=S^{\frac{1}{(\polyorder+2)(t+1)}}$ together with \eqref{eq:W1-KSD-bound-linear} results in the following bound on $W_{\polyorder}$
\begin{align*}
    &W_{\polyorder}(P,Q)^{\polyorder}\\
    &\leq 2^{\polyorder-1} \left[
    9C_{1}^{\polyorder-1}\Bigl(A_{P,\datdim}\lor(\sqrt{2}A_{P,\datdim}\Ex[\Verts G_{2}]^{2})^{1/3}\Bigr)(1\lor S^{-\frac{\polyorder-1}{3(\polyorder+2)(1+t)}})(1\lor S^{\frac{2}{3(1+t)}})S^{\frac{1}{3(1+t)}} 
    \right.\\
    &\hphantom{\leq}\left.
        +2S^{\frac{1}{(\polyorder+2)(1+t)}}+2\eta^{-1}\left\{ \sqrt{\datdim}+2^{1/2}\sqrt{\frac{2^{\datdim/2}}{\Gamma(\datdim/2)}}C_{1}^{\datdim/2}C_{2}(1\lor S^{-\frac{t}{1+t}\frac{1}{\polyorder+2}})\right\} S\right]\\
    &\leq C^{(0)}_{P, \datdim} (1\lor S^{1-\frac{1}{(\polyorder+2)(1+t)}})S^{\frac{1}{(\polyorder+2)(1+t)}}
\end{align*}
where 
\[
    C^{(0)}_{P,\datdim}=2^{\polyorder}\left\{ \frac{9}{2}C_{1}^{\polyorder-1}\Bigl(A_{P,\datdim}\lor(\sqrt{2}A_{P,\datdim}\Ex[\Verts G_{2}]^{2})^{1/3}\Bigr)+1+\frac{1}{\eta}\Bigl(\sqrt{\datdim}+2^{1/2}\sqrt{\frac{C_1^{\datdim}2^{\datdim/2}}{\Gamma(\datdim/2)}}C_{2}\Bigr)\right\} 
\]
Taking $C_{P,\datdim}(\polyorder_m) = \left(C_{P,\datdim}^{(0)}\right)^{1/\polyorder}$ completes the proof. 
\qed
\subsubsection{Quadratic case $\polyorder_m=1$}
The proof proceeds as in \cref{subsec:proof-Matern-KSD-quad}.
By \cref{lem:r-eta-expression}, we have 
\begin{align*}
r_{\varepsilon} & =\left\{ 2\lor c_{P,\datdim,\polyorder,\tau}(1\lor\delta^{-1/\datdim})\right\} (r+\delta)\\
 & \leq\tilde{c}(1\lor\varepsilon^{-1/3\datdim})(1\lor\varepsilon^{-1/3})
\end{align*}
with $\tilde{c}=(2\lor c_{P,\datdim,\polyorder,\tau})r(1+W_{r})\bigl(1\lor W_{r}^{-1/\datdim}\bigr).$
With $s=\datdim/2+\nu$, by applying \cref{lem:q-approx-constructive} with $\theta=0$, we have 
\begin{align*}
 \sqrt{\sup_{\Verts{\omega}_{2}\le2\rho_{\varepsilon}^{-1}}\hat{\Phi}(\omega)^{-1}}\leq C_3\left(1\lor\varepsilon^{-s\{(\polyorder+1)(1+\datdim^{-1})+5\}/3}\right)
\end{align*}
where 
\[
    C_3 = C_{P,\Phi,r}\{1\lor(rW_{r})^{-2s}\}(1\lor\tilde{c}^{s(\polyorder+1)})
\]
is defined by constants
\[
    C_{P,\Phi,r}=\verts{\Sigma}\left[1+\Verts{\Sigma^{-1}}_{\mathrm{op}}^{2}2^{3}\eta^{-2}\left(\tilde{C}^{-2}\lor U_{3}^{2}\lor\eta^{2}/2^{4}\right)\right]^{s/2}, 
\]
$U_{3}$, and $\tilde{C}=\tilde{C}_{0}>0$ from \cref{lem:q-approx-constructive}. 
With $S = \ksd{K_m}{P}(Q)$, substituting these estimates into \eqref{eq:KSD-Wq-bound-prelim} and \eqref{eq:rkhs-norm-estimate-Wq-bound} provide the following intermedidate bound on $W_{\polyorder}$: 
\begin{equation}
    \begin{aligned}
    W_{\polyorder}(P,Q)^{\polyorder}  
    &\leq2^{\polyorder-1}\left(3r_{\varepsilon}^{\polyorder-1}W_{1}(P,Q)+2\varepsilon+\Verts{v_{\varepsilon}}_{\rkhs K}S\right)\\
    &\leq2^{\polyorder-1}\left\{ 3\tilde{c}^{\polyorder-1}(1\lor\varepsilon^{-\tilde{t}})W_{1}(P,Q)+2\varepsilon\right.\\
    &\hphantom{\leq2^{\polyorder-1}}\quad \left.+\frac{2}{\eta}\left(\sqrt{\datdim}+2^{3/2}\sqrt{\frac{2^{\datdim/2}}{\Gamma(\datdim/2)}}\tilde{c}^{\datdim/2+1}C_{3}(1\lor\varepsilon^{-t})\right)S\right\} 
    \end{aligned}\label{eq:Wq-KSD-bound-intermediate}
\end{equation}
where 
\[
    t=\frac{s\{(\polyorder+1)(1+\datdim^{-1})+5\}}{3}+\frac{\datdim+2}{6}, 
\]
and $\tilde{t} = (\polyorder-1)(1+d^{-1})/3$. 
Moreover, the estimate \eqref{eq:W1-bound-alt} together with 
\cref{thm:Matern-bound-quad} yields  
\begin{equation}
    \begin{aligned}
     W_1(P,Q) 
    &\leq 3 \Bigl( B_{P,\datdim} \lor (\sqrt{2}B_{P,\datdim} \Ex[\Verts{G}_2]^2)^{1/3}\Bigr)
    (1 \lor S^{\frac{4}{9(1+t)}})
    S^{\frac{2}{9(1+t)}}
    \end{aligned} \label{eq:W1-KSD-bound-quad}.
\end{equation}

Combining \eqref{eq:Wq-KSD-bound-intermediate} and \eqref{eq:W1-KSD-bound-quad} provides the final bound: 
\begin{align*}
    W_{\polyorder}(P,Q)^{\polyorder}	
    &\leq2^{\polyorder-1}\left\{ 9\tilde{c}^{\polyorder-1}\Bigl(B_{P,\datdim}\lor(\sqrt{2}B_{P,\datdim}\Ex[\Verts G_{2}]^{2})^{1/3}\Bigr)(1\lor S^{\frac{6(1+\tilde{t})-1}{9(1+t)(1+\tilde{t})}})S^{\frac{2}{9(1+t)}\cdot\frac{1}{1+\tilde{t}}}\right.\\
        &\hphantom{\leq2^{\polyorder-1}}\left.2S^{\frac{2}{9(1+t)}\cdot\frac{1}{1+\tilde{t}}}+\frac{2}{\eta}\Bigl(\sqrt{\datdim}+2^{3/2}\sqrt{\frac{2^{\datdim/2}}{\Gamma(\datdim/2)}}\tilde{c}^{\datdim/2+1}C_{3}(1\lor S^{1-\frac{2}{9(1+t)}\cdot\frac{1}{1+\tilde{t}}})\Bigr)S^{\frac{2}{9(1+t)}\cdot\frac{1}{1+\tilde{t}}}\right\} \\
    &\leq C_{P,\datdim}^{(1)}(1\lor S^{1-\frac{1}{\tilde{\polyorder}}\frac{2}{3(1+t)}\cdot})S^{\frac{1}{\tilde{q}}\frac{2}{3}\frac{1}{1+t}\cdot},
\end{align*}
where 
\begin{align*}
    C_{P,\datdim}^{(1)}=2^{\polyorder}\left\{ \frac{9}{2}\tilde{c}^{\polyorder-1}\Bigl(B_{P,\datdim}\lor(\sqrt{2}B_{P,\datdim}\Ex[\Verts G_{2}]^{2})^{1/3}\Bigr)+1+\frac{1}{\eta}\left(\sqrt{\datdim}+2^{3/2}\sqrt{\frac{2^{\datdim/2}}{\Gamma(\datdim/2)}}\tilde{c}^{\datdim/2+1}C_{3}\right)\right\},
\end{align*}
and
\[
    \tilde{\polyorder} = (1+d^{-1})\polyorder+(2-d^{-1}).
\]
Taking $C_{P,\datdim}(\polyorder_m) = \left(C_{P,\datdim}^{(1)}\right)^{1/\polyorder}$ completes the proof. 
\qed

\section{Results concerning approximation via convolution }

This section collects results concerning convolution. Specifically,
we consider the convolution operator $T_{\varphi,\rho}$ defined as
\[
T_{\varphi,\rho}f(x)=\frac{1}{\rho^{\datdim}}\int f(x-y)\varphi(y/\rho)\dd y,
\]
where $f\in L^{1}(\bbR^{\datdim'},\lambda),$ $\varphi\in L^{1}\bigl(\lambda\bigr)$
with $\int\varphi(x)\dd x=1,$ $\lambda$ the $\datdim$-dimensional
Lebesgue measure, and $\rho>0.$ 
\begin{lem}[Approximation error] 
\label{lem:convolution-diff-bound} Let $g\in{\cal C}^{1}\bigl(\bbR^{\datdim}).$
For $r\in(0,\infty],$ let $\iota_{r}\in{\cal C}^{1}:\bbR^{\datdim}\to [0,1]$ 
such that $\iota_{r}(x)=0$ for $\Verts x_{2}>r.$ Let $\tilde{g}_{r}=\iota_{r}g.$
For $w\in{\cal C}^{1}:\bbR^{\datdim}\to(0,\infty),$ define 
\[
B_{\weight}(z)=\sup_{x\in\bbR^{\datdim},u\in[0,1]}\weight(x)/\weight(x-uz)\ \text{and}\ M_{\weight}=\sup_{x\in\bbR^{\datdim}}\Verts{\nabla\log\weight(x)}_{2}.
\]
Let $\varphi:\bbR^{\datdim}\to[0,\infty]$ be a function such that
$\int\varphi(x)\dd x=1,$ $\mu_{\varphi,\rho,\weight}=\int\Verts x_{2}B_{\weight}(\rho x)\varphi(x)\dd x<\infty,$
and $\mu_{\varphi}=\int\Verts x_{2}\varphi(x)\dd x.$ With fixed $\rho>0,$
define the convolution 
\[
T_{\varphi,\rho}\tilde{g}_{r}^{\weight}(x)=\frac{1}{\rho^{\datdim}}\int\tilde{g}_{r}^{w}(x-z)\varphi(z/\rho)\dd z,\ \text{where}\ \tilde{g}_{r}^{\weight}(x)\coloneqq\frac{\tilde{g}_{r}(x)}{w(x)}.
\]
Then, for any $x\in\bbR^{\datdim},$ we have 
\begin{align*}
 & \Verts{T_{\varphi,\rho}\tilde{g}_{r}^{\weight}(x)-\tilde{g}_{r}^{\weight}(x)}_{2}\\
 & \leq\frac{\rho}{\weight(x)}\left\{ M_{\weight}\mu_{\varphi,\rho,\weight}\sup_{\Verts y_{2}\leq r}\Verts{g(y)}_{2}+\mu_{\varphi}\Bigl(\sup_{\Verts y_{2}\leq r}\norm{\nabla g(y)}_{\mathrm{op}}+C_{r,1}\sup_{\Verts y_{2}\leq r}\norm{g(y)}_{\mathrm{2}}\Bigr)\right\} ,
\end{align*}
where $C_{r,1}=\sup_{x\in\bbR^{\datdim}}\Verts{\nabla\iota_{r}(x)}_{2}.$ 
\end{lem}

\begin{proof}
Let $Z$ be a random variable whose law is a distribution having $\varphi$
as the density with respect to the Lebesgue measure. By the definition
of $T_{\varphi,\rho}\tilde{g}_{r}^{\weight},$ we have 
\[
T_{\varphi,\rho}\tilde{g}_{r}^{\weight}(x)-\tilde{g}_{r}^{\weight}(x)=\EE_{Z}\left[\tilde{g}_{r}(x-\rho Z)\left\{ \frac{1}{\weight(x-\rho Z)}-\frac{1}{\weight(x)}\right\} +\frac{\tilde{g}_{r}(x-\rho Z)-\tilde{g}_{r}(x)}{\weight(x)}\right].
\]
To bound each quantity inside the expectation, we derive their norm
estimates as follows. First, we have 

\begin{align*}
\Verts{\tilde{g}_{r}(x-\rho Z)-\tilde{g}_{r}(x)}_{2} & \leq\rho\Verts{\nabla\tilde{g}_{r}(x-t\rho Z)}_{\mathrm{op}}\Verts Z_{2}\\
 & \leq\rho\Verts Z_{2}\sup_{y\in\bbR^{\datdim}}\Verts{\nabla\tilde{g}_{r}(y)}_{\mathrm{op}}\\
 & \leq\rho\Verts Z_{2}\left(\max_{\Verts y_{2}\leq r}\norm{\nabla g(y)}_{\mathrm{op}}+\max_{\Verts y_{2}\leq r}\norm{g(y)}_{\mathrm{2}}\cdot\sup_{y\in\bbR^{\datdim}}\Verts{\nabla\iota_{r}(y)}_{2}\right),
\end{align*}
where the first line is due to the mean value theorem with $t\in(0,1).$
Similarly, we obtain 
\begin{align}
\left|\frac{1}{\weight(x-\rho Z)}-\frac{1}{\weight(x)}\right| & \leq\rho\Verts Z_{2}\norm{\nabla_{x}\frac{1}{w(x-t\rho Z)}}_{2}\nonumber \\
 & \leq\frac{\rho\Verts Z_{2}}{\weight(x)}\frac{\weight(x)}{\weight(x-t\rho Z)}\norm{\nabla_{x}\log w(x-t\rho Z)}_{2}\\
 & \leq\frac{\rho\norm Z_{2}}{\weight(x)}B_{\weight}(\rho Z)M_{\weight},\nonumber 
\end{align}
where $t\in(0,1).$ Combining these evaluations yields
\begin{align*}
\Verts{T_{\varphi,\rho}\tilde{g}_{r}^{\weight}(x)-\tilde{g}_{r}^{\weight}(x)}_{2} & \leq\EE_{Z}\left[\Verts{\tilde{g}_{r}^{\weight}(x-\rho Z)}_{2}\left|\frac{1}{\weight(x-\rho Z)}-\frac{1}{\weight(x)}\right|\right]\\
 & \hphantom{\leq}\quad+\frac{1}{\weight(x)}\EE_{Z}\left[\Verts{\tilde{g}_{r}^{\weight}(x-\rho Z)-\tilde{g}_{r}^{\weight}(x)}_{2}\right]\\
 & \leq\frac{\rho}{\weight(x)}\left[M_{\weight}\EE_{Z}\Bigl[\norm Z_{2}B_{\weight}(\rho Z)\Bigr]\max_{\Verts y_{2}\leq r}\Verts{g(y)}\right.\\
 & \hphantom{\leq\frac{\rho}{\weight(x)}}\left.\quad+\mathbb{E}_{Z}\Verts Z_{2}\left(\max_{\Verts y_{2}\leq r}\norm{\nabla g(y)}_{\mathrm{op}}+\max_{\Verts y_{2}\leq r}\norm{g(y)}_{\mathrm{2}}\cdot\sup_{y\in\bbR^{\datdim}}\Verts{\nabla\iota_{r}(y)}_{2}\right)\right]
\end{align*}
\end{proof}
\begin{cor}[Approximation error with respect to solutions of Stein equations]
\label{cor:convolution-diff-bound-stein}Let $v:\bbR^{\datdim}\to\bbR^{\datdim}$
be a function satisfying the growth conditions in Lemma \ref{lem:finite-stein-factors}.
Define symbols as in Lemma \ref{lem:convolution-diff-bound}. Then, we have
\begin{align*}
 & \Verts{T_{\varphi,\rho}\tilde{v}_{r}^{\weight}(x)-\tilde{v}_{r}^{\weight}(x)}_{2}\\
 & \leq\frac{\rho(1+r^{\polyorder-1})}{\weight(x)}\left[\sqrt{\datdim}\zeta_{1}M_{\weight}\mu_{\varphi,\rho,\weight}+\mu_{\varphi}\bigl(2^{-1}\zeta_{2}+\sqrt{\datdim}\zeta_{1}C_{r,1}\big)\right]
\end{align*}
for any $x\in\bbR^{\datdim}.$ 
\end{cor}

\begin{lem}[Derivative approximation error]
\label{lem:convolution-grad-diff-bound} Let $g\in{\cal C}^{2}(\bbR^{\datdim}).$
For $r\in(0,\infty],$ let $\iota_{r}\in{\cal C}^{2}:\bbR^{\datdim}\to[0,1]$
such that $\iota_{r}(x)=0$ for $\Verts x_{2}>r.$ Let $\tilde{g}_{r}=\iota_{r}g.$
For $w\in{\cal C}^{2}:\bbR^{\datdim}\to(0,\infty),$ define the following quantities: 
\[
B_{\weight}(z)=\sup_{x\in\bbR^{\datdim},u\in[0,1]}\weight(x)/\weight(x-uz)
\]
\[
M_{\weight}=\sup_{x\in\bbR^{\datdim}}\Verts{\nabla\log\weight(x)}_{2},\ N_{\weight}=\sup_{x\in\bbR^{\datdim}} \frac{\Verts{\nabla^{2}\log\weight(x)}_{\mathrm{op}}}{\weight(x)}
\]
and
\[
O_{\weight} = \sup_{x\in\bbR^{\datdim}} \frac{\Verts{\nabla\log\weight(x)}_2^2}{\weight(x)}. 
\]
Let $\varphi:\bbR^{\datdim}\to[0,\infty]$ be a function such that
$\int\varphi(x)\dd x=1,$ $\mu_{\varphi,\rho,\weight}=\int\Verts x_{2}B_{\weight}(\rho x)\varphi(x)\dd x<\infty,$
and $\mu_{\varphi}=\int\Verts x_{2}\varphi(x)\dd x.$ With fixed $\rho>0,$
define the convolution 
\[
T_{\varphi,\rho}\tilde{g}_{r}^{\weight}(x)=\frac{1}{\rho^{\datdim}}\int\tilde{g}_{r}^{w}(x-z)\varphi(z/\rho)\dd z,\ \text{where}\ \tilde{g}_{r}^{\weight}(x)\coloneqq\frac{\tilde{g}_{r}(x)}{w(x)}.
\]
Assume $\nabla T_{\varphi,\rho}\tilde{g}_{r}^{\weight}(x)=T_{\varphi,\rho}\nabla\tilde{g}_{r}^{\weight}(x)$
for each $x\in\bbR^{\datdim}.$ Then, for any $x\in\bbR^{\datdim},$
we have 
\begin{align*}
 & \Verts{\nabla T_{\varphi,\rho}\tilde{g}_{r}^{\weight}(x)-\nabla\tilde{g}_{r}^{\weight}(x)}_{\mathrm{op}}\\
 & \leq\frac{\rho}{\weight(x)}\left[ \left\{M_{\weight}\mu_{\varphi,\rho,\weight}+\left(N_{\weight}+O_{\weight}\right)\mu_{\varphi}\right\}\left(\sup_{\Verts y_{2}\leq r}\norm{\nabla g(y)}_{\mathrm{op}}+C_{r,1}\sup_{\Verts y_{2}\leq r}\norm{g(y)}_{2}\right)\right.\\
 & \hphantom{\leq\frac{\rho}{\weight(x)}}\quad\left.+(\mu_{\varphi}+M_{\weight}\mu_{\varphi,\rho,\weight})\sup_{\Verts y_{2}\leq r}\left(\Verts{\nabla^{2}g(y)}_{\mathrm{op}}+2C_{r,1}\Verts{\nabla g(y)}_{\mathrm{op}}+C_{r,2}\Verts{g(y)}_{2}\right)\right] ,
\end{align*}
where $C_{r,1}=\sup_{x\in\bbR^{\datdim}}\Verts{\nabla\iota_{r}(x)}_{2},$
$C_{r,2}=\sup_{x\in\bbR^{\datdim}}\Verts{\nabla^{2}\iota_{r}(x)}_{\mathrm{op}}.$ 
\end{lem}

\begin{proof}

First, note that 
\begin{align}
\begin{aligned} & \Verts{\nabla T_{\varphi,\rho}\tilde{g}_{r}^{\weight}(x)-\nabla\tilde{g}_{r}^{\weight}(x)}_{\mathrm{op}}=\norm{\Ex_{Z}\Bigl[\nabla\tilde{g}_{r}^{\weight}(x-\rho Z)-\nabla\tilde{g}_{r}^{\weight}(x)\Bigr]}_{\mathrm{op}}\\
 & \leq\underbrace{\norm{\Ex_{Z}\left[\frac{\nabla\tilde{g}_{r}(x-\rho Z)}{\weight(x-\rho Z)}-\frac{\nabla\tilde{g}_{r}(x)}{\weight(x)}\right]}_{\mathrm{op}}}_{\mathrm{(a)}}\\
 & \hphantom{\leq\quad}+\underbrace{\norm{\Ex_{Z}\left[\frac{\nabla\log\weight(t-\rho Z)}{\weight(x-\rho Z)}\otimes\nabla\tilde{g}_{r}(x-\rho Z)-\frac{\nabla\log\weight(x)}{\weight(x)}\otimes\nabla\tilde{g}_{r}(x)\right]}_{\mathrm{op}}}_{\mathrm{(b)}},
\end{aligned}
\label{eq:gv-grad-diff}
\end{align}
In the first line, we have interchanged the gradient and the expectation
operation. We evaluate each term below.

The term (a) is evaluated as 
\begin{align}
\norm{\Ex_{Z}\left[\frac{\nabla\tilde{g}_{r}(x-\rho Z)}{\weight(x-\rho Z)}-\frac{\nabla\tilde{g}_{r}(x)}{\weight(x)}\right]}_{\mathrm{op}} & \leq\underbrace{\norm{\Ex_{Z}\left[\nabla\tilde{g}_{r}(x-\rho Z)\left(\frac{1}{\weight(x-\rho Z)}-\frac{1}{\weight(x)}\right)\right]}_{\mathrm{op}}}_{(\mathrm{a1})}\nonumber \\
 & \hphantom{\leq}\quad+\underbrace{\frac{1}{\weight(x)}\norm{\Ex_{Z}\left[\nabla\tilde{g}_{r}(x-\rho Z)-\nabla\tilde{g}_{r}(x)\right]}_{\mathrm{op}}}_{\mathrm{(a2)}}.\label{eq:grad-diff-bound-first}
\end{align}
The term (a1) is bounded as 
\begin{align*}
\mathrm{(a1)} & \leq\Ex_{Z}\left[\left|\frac{1}{\weight(x-\rho Z)}-\frac{1}{\weight(x)}\right|\Verts{\nabla\tilde{g}_{r}(x-\rho Z)}_{\mathrm{op}}\right]\\
 & \leq\frac{\rho}{\weight(x)}M_{\weight}\Ex_{Z}\left[\norm Z_{2}B_{\weight}(\rho Z)\right]\sup_{x\in\bbR^{\datdim}}\Verts{\nabla\tilde{g}_{r}(x)}_{\mathrm{op}}
\end{align*}
The term (a2) is evaluated as follows. By the mean value theorem,
for some $t\in(0,1),$ 
\[
\nabla\tilde{g}_{r}(x-\rho Z)-\nabla\tilde{g}_{r}(x)=-\rho\nabla^{2}\tilde{g}_{r}(x-t\rho Z)\nmodemul 1Z,
\]
where $(A\times_{1}v)_{ij}=\sum_{k}A_{kij}v_{k}$ for $A\in\bbR^{\datdim_{1}\times\datdim_{2}\times \datdim_{3}}$
and $v\in\bbR^{\datdim_{1}}.$ This implies 
\begin{align*}
\Verts{\nabla\tilde{g}_{r}(x-\rho Z)-\nabla\tilde{g}_{r}(x)}_{\mathrm{op}} & =\rho\Verts{\nabla^{2}\tilde{g}_{r}(x-t\rho Z)\nmodemul 1Z}_{\mathrm{op}}\\
 & =\rho\sup_{\Verts{u_{2}}_{2}=1,\Verts{u_{3}}_{2}=1}\left|\la\nabla^{2}\tilde{g}_{r}(x-t\rho Z),Z\otimes u_{2}\otimes u_{2}\ra\right|\\
 & \leq\rho\Verts Z_{2}\sup_{y\in\bbR^{\datdim}}\Verts{\nabla^{2}\tilde{g}_{r}(y)}_{\mathrm{op}}.
\end{align*}
Thus, the operator norm in the term (a2) is bounded as 

\begin{align*}
\norm{\Ex_{Z}\bigl[\nabla\tilde{g}_{r}(x-\rho Z)-\nabla\tilde{g}_{r}(x)\bigr]}_{\mathrm{op}} & \leq\Ex_{Z}\bigl[\Verts{\nabla\tilde{g}_{r}(x-\rho Z)-\nabla\tilde{g}_{r}(x)}_{\mathrm{op}}\bigr]\\
 & \leq\rho\Ex_{Z}\Verts Z_{2}\sup_{y\in\bbR^{\datdim}}\Verts{\nabla^{2}\tilde{g}_{r}(y)}_{\mathrm{op}}.
\end{align*}
Multiplying by $w(x)^{-1}$ completes the evaluation of (a2). 

Similarly, we can evaluate the term (b) in (\ref{eq:gv-grad-diff}).
We again apply the mean value theorem to obtain 
\begin{align*}
 & \EE_{Z}\left\Vert \frac{\nabla\log\weight(t-\rho Z)}{\weight(x-\rho Z)}-\frac{\nabla\log\weight(x)}{w(x)}\right\Vert _{2}\\
 & \leq\rho\EE_{Z}\left[\Verts Z_{2}\right]\cdot\sup_{y\in\bbR^{\datdim}}\left(\left\Vert \frac{\nabla^{2}\log\weight(y)}{\weight(y)}\right\Vert _{\mathrm{op}}+\frac{1}{\weight(y)}\left\Vert \nabla\log\weight(y)\right\Vert _{2}^{2}\right).
\end{align*}
Then, 
\begin{align*}
\mathrm{(b)} & =\norm{\Ex_{Z}\left[\frac{\nabla\log\weight(t-\rho Z)}{\weight(x-\rho Z)}\otimes\nabla\tilde{g}_{r}(x-\rho Z)-\frac{\nabla\log\weight(x)}{\weight(x)}\otimes\nabla\tilde{g}_{r}(x)\right]}_{\mathrm{op}}\\
 & \leq\frac{1}{\weight(x)}\norm{\Ex_{Z}\left[\frac{\weight(x)\nabla\log\weight(t-\rho Z)}{\weight(x-\rho Z)}\otimes\bigl(\nabla\tilde{g}_{r}(x-\rho Z)-\nabla\tilde{g}_{r}(x)\bigr)\right]}_{\mathrm{op}}\\
 & \hphantom{\leq}\quad+\norm{\Ex_{Z}\left[\left(\frac{\nabla\log\weight(t-\rho Z)}{\weight(x-\rho Z)}-\frac{\nabla\log\weight(x)}{\weight(x)}\right)\otimes\nabla\tilde{g}_{r}(x)\right]}_{\mathrm{op}}\\
 & \leq\frac{1}{\weight(x)}\Ex_{Z}\norm{\frac{w(x)\bigl(\nabla\log\weight(t-\rho Z)\bigr)}{\weight(x-\rho Z)}}_{2}\Verts{\nabla\tilde{g}_{r}(x-\rho Z)-\nabla\tilde{g}_{r}(x)}_{\mathrm{op}}\\
 & \hphantom{\leq}\quad+\Verts{\nabla\tilde{g}_{r}(x)}_{\mathrm{op}}\EE_{Z}\norm{\frac{\nabla\log\weight(t-\rho Z)}{\weight(x-\rho Z)}-\frac{\nabla\log\weight(x)}{w(x)}}_{2}\\
 & \leq\frac{\rho M_{\weight}}{\weight(x)}\Ex_{Z}\left[\Verts Z_{2}B_{\weight}(\rho Z)\right]\sup_{y\in\bbR^{\datdim}}\Verts{\nabla^{2}\tilde{g}_{r}(y)}_{\mathrm{op}}\\
 & \hphantom{\leq}\quad+\frac{\rho}{\weight(x)}\left(N_{\weight}+O_{\weight}\right)\EE_{Z}\left[\Verts Z_{2}\right]\Verts{\nabla\tilde{g}_{r}(x)}_{2}\\
 & \leq\frac{\rho}{\weight(x)}\left\{ M_{\weight}\Ex_{Z}\left[\Verts Z_{2}B_{\weight}(\rho Z)\right]\sup_{y\in\bbR^{\datdim}}\Verts{\nabla^{2}\tilde{g}_{r}(y)}_{\mathrm{op}}+\left(N_{\weight}+O_{\weight}\right)\EE_{Z}\left[\Verts Z_{2}\right]\Verts{\nabla\tilde{g}_{r}(x)}_{\mathrm{op}}\right\} .
\end{align*}

Since the operator norm of $\nabla^{2}\tilde{g}_{r}$ is bounded by
\begin{align}
\begin{aligned} & \sup_{x\in\bbR^{\datdim}}\Verts{\nabla^{2}\tilde{g}_{r}(x)}_{\mathrm{op}}\\
 & \leq\sup_{x\in\bbR^{\datdim}}\left(\iota_{r}(x)\Verts{\nabla^{2}g(x)}_{\mathrm{op}}+2\Verts{\nabla\iota_{r}(x)}_{2}\Verts{\nabla g(x)}_{\mathrm{op}}+\Verts{\nabla^{2}\iota_{r}(x)}_{\mathrm{op}}\Verts{g(x)}_{2}\right).
\end{aligned}
\label{eq:gkdelta-grad2-bound}
\end{align}
Combining the above evaluations, we obtain 
\begin{align*}
 & \Verts{\nabla T_{\varphi,\rho}\tilde{g}_{r}^{\weight}(x)-\nabla\tilde{g}_{r}^{\weight}(x)}_{\mathrm{op}}\\
 & \leq\frac{\rho}{\weight(x)}\left[\left\{ M_{\weight}\Ex_{Z}\left[\norm Z_{2}B_{\weight}(\rho Z)\right]+\left(N_{\weight}+O_{\weight}\right)\EE_{Z}\left[\Verts Z_{2}\right]\right\} \cdot A_{r,1} \right.\\
 & \hphantom{\leq\frac{\rho}{\weight(x)}}\quad\left.+\left(\Ex_{Z}\Verts Z_{2}+M_{\weight}\Ex_{Z}\left[\Verts Z_{2}B_{\weight}(\rho Z)\right]\right)A_{r,1}\right],
\end{align*}
where 
\[
A_{r,1} = \left(\sup_{\Verts y_{2}\leq r}\norm{\nabla g(y)}_{\mathrm{op}}+C_{r,1}\sup_{\Verts y_{2}\leq r}\norm{g(y)}_{\mathrm{2}}\right),
\]
and 
\[
A_{r,2} = \sup_{\Verts y_{2}\leq r}\left(\Verts{\nabla^{2}g(y)}_{\mathrm{op}}+2C_{r,1}\Verts{\nabla g(y)}_{2}+C_{r,2}\Verts{g(y)}_{2}\right)
\]
with $C_{r,1}=\sup_{x\in\bbR^{\datdim}}\Verts{\nabla\iota_{r}(x)}_{2}$ and 
$C_{r,2}=\sup_{x\in\bbR^{\datdim}}\Verts{\nabla^{2}\iota_{r}(x)}_{\mathrm{op}}$. 
\end{proof}
\begin{cor}[Derivative approximation error with respect to solutions of Stein equations]
\label{cor:convolution-grad-diff-bound-stein}Define symbols as in
Lemma \ref{lem:convolution-diff-bound}. Let $v:\bbR^{\datdim}\to\bbR^{\datdim}$
be a function satisfying the growth conditions in Lemma \ref{lem:finite-stein-factors}.
For each fixed $\rho>0,$ we have 
\[
\Verts{\nabla T_{\varphi,\rho}\tilde{g}_{r}^{\weight}(x)-\nabla\tilde{g}_{r}^{\weight}(x)}_{\mathrm{op}}\leq\frac{\rho(1+r^{\polyorder-1})}{\weight(x)}\tilde{u}_{P,\datdim,\weight,\rho}^{(2)},
\]
where 
\begin{align*}
\tilde{u}_{P,\datdim,\weight,\rho}^{(2)} & =\left[M_{\weight}\left\{ \mu_{\varphi,\rho,\weight}+\left(N_{\weight}+O_{\weight}\right)\mu_{\varphi}\right\} \left(2^{-1}\zeta_{2}+C_{r,1}\zeta_{1}\sqrt{\datdim}\right)\right.\\
 & \hphantom{=}\quad+\left.\left(\mu_{\varphi}+M_{\weight}\mu_{\varphi,\rho,\weight}\right)\left(2^{-1}\zeta_{3}+C_{r,1}\zeta_{2}+C_{r,2}\sqrt{\datdim}\zeta_{1}\right)\right].
\end{align*}
\end{cor}

\begin{lem}[Growth of convolution functions]
\label{lem:convolution-decay-rate}
Let $\Verts{\cdot}$ be an arbitrary
norm on $\bbR^{\datdim}.$ Let ${\cal S}$ be a separable Banach space
with norm $\Verts{\cdot}_{{\cal S}}.$ Let $f:\bbR^{\datdim}\to\mathcal{S}$
be a function supported in $\{\Verts x\leq r_{f}\}$ for some $r_{f}>0.$
Let $g:\bbR^{\datdim}\to\bbR$ be a real-valued function satisfying
\begin{align*}
\verts{g(x)} & \leq C_{g}\Delta_{g}\bigl(\Verts x\bigr),
\end{align*}
for $\Verts x\geq r_{g},$ some positive constants $C_{g},$ $r_{g},$
and a non-increasing function $\Delta_{g}.$ Then, the convolution
\[
h(x)=\int f(y)g(x-y)\dd y
\]
satisfies 
\begin{align*}
\Verts{h(x)}_{{\cal S}} & \leq C_{g}\left(\int\Verts{f(y)}_{{\cal S}}\dd y\right)\cdot\Delta_{g}(2^{-1}\Verts x)
\end{align*}
for $\Verts x\geq2(r_{f}\lor r_{g}),$ where $h$ is defined as the
Bochner integral. In particular, if $\Verts f_{\infty}\coloneqq\sup_{x\in\bbR^{\datdim}}\Verts{f(x)}_{{\cal S}}<\infty,$
\[
\Verts{h(x)}_{{\cal S}}\leq C_{g}\Verts f_{\infty}\lambda(\{\Verts x\leq r_{f}\})\Delta_{g}(2^{-1}\Verts x),
\]
where $\lambda$ is the Lebesgue measure. 
\end{lem}

\begin{proof}
Let $r=(r_{f}\lor r_{g}).$ For $\Verts y\leq r$ and $\Verts x>2r\geq2\Verts y,$
we have 
\begin{align*}
\Verts{x-y} & \geq\Verts x-\Verts y\\
 & \geq\frac{\Verts x}{2}>r.
\end{align*}
Thus, under these conditions, we obtain 
\begin{align*}
\Verts{h(x)}_{{\cal S}} & \leq\int\verts{g(x-y)}\Verts{f(y)}_{{\cal S}}\dd y  =\int_{\Verts y\leq r}\verts{g(x-y)}\Verts{f(y)}_{{\cal S}}\dd y\\
 & \leq C_{g}\int_{\Verts y\leq r}\Delta_{g}(\Verts{x-y})\Verts{f(y)}_{{\cal S}}\dd y\\
 & \leq C_{g}\Delta_{g}(2^{-1}\Verts x)\int_{\Verts y\leq r}\Verts{f(y)}_{{\cal S}}\dd y.
\end{align*}
\end{proof}
\section{Known results \label{sec:Known-results}}
\begin{thm}[{Finite Stein factors from Wasserstein decay, \citealt[Theorem 3.2]{Erdogdu2018}}]
\label{thm:finite-stein-factors-reference}Suppose Assumptions \ref{assu:poly-grow-coef},
\ref{assu:dissipativity-diffusion}, and \ref{assu:diff-wasserstein-rate}
hold. Assume $f$ is pseudo-Lipschitz of order $\polyorder$ with
at most degree-$\polyorder$ polynomial growth of its $i$-th derivatives
for $i=2,3,4$. Then, the solution $u_{f}$ to the equation (\ref{eq:SteinEq-generator})
is pseudo-Lipschitz of order $\polyorder$ with constant $\zeta_{1},$
and has $i$th order derivative with degree-$\polyorder$ polynomial
growth for $i=2,3,4:$ 
\[
\Verts{\nabla^{i}u_{f}(x)}_{\mathrm{op}}\leq\zeta_{i}(1+\Verts x_{2}^{\polyorder})\ \text{for }i\in\{2,3,4\},\ \text{and }x\in\bbR^{\datdim}.
\]
The constants $\zeta_{i}$ (called Stein factors) are given as follows:
\begin{align*}
\zeta_{i} & =\tau_{i}+\xi_{i}\int_{0}^{\infty}\rho_{1}(t)\omega_{\polyorder_{m}+1}(t+i-2)\dd t\ \text{for }i=1,2,3,4,
\end{align*}
where 
\[
\omega_{\polyorder_{m}+1}(t)=1+4\rho_{1}(t)^{1-1/(\polyorder_{m}+1)}\rho_{1}(0)^{\frac{1}{2}}\left[1+\frac{1}{\tilde{\alpha}_{\polyorder_{m}+1}^{\polyorder}}\{(1\vee\tilde{\rho}_{\polyorder_{m}+1}(t))8\lambda_{m}\polyorder+3(\polyorder_{m}+1)\beta\}^{\polyorder}\right],
\]
with $\tilde{\alpha}_{1}=\alpha,$ $\tilde{\alpha}_{2}=\inf_{t\geq0}[\alpha-4\polyorder\lambda_{m}(1\vee\tilde{\rho}_{2}(t)]_{+},$
\begin{align*}
\tau_{1}=0,\ \tau_{i} & =\plipconst f{\polyorder}\degnpolycoef f{2:i}{\polyorder}\tilde{\nu}_{1:\polyorder}(\sigma)\kappa_{\polyorder_{m}}(6\polyorder)\ \text{for }i=2,3,4,\\
\xi_{1}=\plipconst f{\polyorder}, & \xi_{i}=\plipconst f{\polyorder}\tilde{\nu}_{1:i}(b)\tilde{\nu}_{0:i-2}(\sigma^{-1})\rho_{1}(0)\omega_{\polyorder_{m}+1}(1)\kappa_{\polyorder_{m}+1}(6\polyorder)^{i-1}\ \text{for }i=2,3,4,
\end{align*}
where $\degnpolycoef fi{\polyorder}=\sup_{x\in\bbR^{\datdim}}\Verts{\nabla^{i}f(x)}_{\mathrm{op}}/(1+\Verts {x}_2^{\polyorder}),$
$\degnpolycoef f{a:b}{\polyorder}\coloneqq\max_{i=a,\dots,b}\degnpolycoef fi{\polyorder},$
$\tilde{\nu}_{a:b}(g)$ is a constant whose precise form is given
in the proof of \citet[Theorem 3.2]{Erdogdu2018}, and 
\[
\kappa_{\polyorder_{m}+1}(\polyorder)=2+\frac{2\beta}{\alpha}+\frac{\polyorder\lambda_{m}}{\alpha}+\frac{\tilde{\alpha}_{\polyorder_{m}+1}}{\alpha}\left(\frac{4\polyorder\lambda_{m}+6(\polyorder_{m}+1)\beta}{2r\tilde{\alpha}_{\polyorder_{m}+1}}\right).
\]
\end{thm}

\begin{thm}[{\citealt[Theorem 4.15]{EliasM.Stein1971}}]
\label{thm:Bochner-Riesz-mean-func} If $\Phi:\bbR^{\datdim}\to[0,1]$
is the function defined by 
\[
\Phi(t)=\begin{cases}
\bigl(1-\Verts t_{2}^{2}\bigr)^{\delta} & \Verts t_{2}\leq1\\
0 & \Verts t_{2}>1,
\end{cases}
\]
 where $\delta>0,$ then 
\[
\hat{\Phi}(x) %
= 2^{\delta} \Gamma(\delta+1)\Verts{x}_2^{-\datdim/2 - \delta} J_{\datdim/2+\delta}\bigl( \Verts{x}_2\bigr)
\]
with $J_{\datdim/2+\delta}$ the Bessel function of the first kind
with order $\datdim/2+\delta.$
\end{thm}

\begin{rem*}
The above result differs from the cited result \citep[Theorem 4.15]{EliasM.Stein1971}
in the definition of the Fourier transform. We use 
\[
\hat{f}(\omega)=\frac{1}{(2\pi)^{\datdim/2}}\int f(x)e^{-i\la x,\omega\ra}\dd x
\]
whereas \citet{EliasM.Stein1971} uses 
\[
\tilde{f}(\omega)=\int f(x)e^{-i2\pi\la x,\omega\ra}\dd x.
\]
We have $\tilde{f}(\omega)=(2\pi)^{\datdim/2}\hat{f}(2\pi\omega).$ The same
applies to the inverse transform. 
\end{rem*}

\section{Miscellaneous results \label{sec:Miscellaneous-results}}
\begin{lem}[Constants for multiquadratic functions]
\label{lem:multiquadratic-helper}
Let $\weight(x)=\Bigl(a+b\Verts{x-\mu}_{2}^{2}\Bigr)^{\polyorder}$
with $a>0,$ $b>0,$ $\polyorder\in \bbR$, and $\mu\in\bbR^{\datdim}.$
Then,
\begin{align*}
B_{\weight}(z) & \coloneqq\sup_{x\in\bbR^{\datdim},u\in[0,1]}\frac{\weight(x)}{\weight(x-uz)}\\
& \leq\left\{1+2\left(1+\frac{b\Verts z_{2}^{2}}{a}\right)\right\}^{\verts\polyorder},
\end{align*}
\[
M_{\weight}\coloneqq\sup_{x\in\bbR^{\datdim}}\Verts{\nabla\log\weight(x)}_{2}=\sup_{x\in\bbR^{\datdim}}\frac{2b\verts\polyorder\Verts{x-\mu}_{2}}{a+b\Verts{x-\mu}_{2}^{2}}\leq\verts\polyorder\sqrt{\frac{b}{a}},
\]
 and 
\[
N_{\weight}\coloneqq\sup_{x\in\bbR^{\datdim}}\left\lVert\frac{\nabla^{2}\log\weight(x)}{\weight(x)}\right\rVert_{2}\leq6\verts\polyorder\frac{b}{a}.
\]
If $\polyorder \geq -1$, then 
\begin{align*}
O_{\weight} \coloneqq \sup_{x\in\bbR^{\datdim}} \frac{\left\lVert\nabla \log \weight(x)\right\rVert^2_2}{\weight(x)} 
&\leq (2b\polyorder)^2 \sup_{x\in\bbR^{\datdim}}  \frac{\Verts{x}_2^2}{\bigl(a+b\Verts{x}_2^2\bigr)^{2+\polyorder}} \\
&\leq \frac{4b\polyorder^2}{a^{\polyorder+1} } \left\{\frac{(1+\polyorder)^{1+\polyorder}}{(2+\polyorder)^{2+\polyorder}}\lor 1\{\polyorder=-1\}\right\}.
\end{align*}
\end{lem}

\begin{proof}
    Suppose $\polyorder\geq0$. With $\bar{x}=x-\mu,$ we have
\begin{align*}
\frac{\weight(x+uz)}{\weight(x)} & =\left(\frac{a+b\Verts{\bar{x}+uz}_{2}^{2}}{a+b\Verts{\bar{x}}_{2}^{2}}\right)^{\polyorder}\leq\left(1+\frac{\Verts{\bar{x}+uz}_{2}^{2}}{(a/b)+\Verts{\bar{x}}_{2}^{2}}\right)^{\polyorder}\\
 & \leq\left(1+2\frac{\Verts{\bar{x}}_{2}^{2}+\Verts z_{2}^{2}}{(a/b)+\Verts{\bar{x}}_{2}^{\polyorder}}\right)^{\polyorder}\\
 & \leq\left\{1+2\left(1+\frac{b\Verts z_{2}^{2}}{a}\right)\right\}^{\polyorder}.
\end{align*}
The first claim follows from the observation 
\[
\sup_{x\in\bbR^{\datdim},u\in[0,1]}\frac{\weight(x)}{\weight(x-uz)}\leq\sup_{u\in[0,1]}\sup_{x\in\bbR^{\datdim}}\frac{\weight(x+uz)}{\weight(x)}.
\]
The derivation above applies to $\polyorder<0$ (replace $\polyorder$ with $\verts{\polyorder}$). 
The second and final claims can be checked easily.  
The third claim follows from 
\begin{align*}
N_{\weight} & \coloneqq\sup_{x\in\bbR^{\datdim}}\left\Vert \frac{\nabla^{2}\log\weight(x)}{\weight(x)}\right\Vert _{\mathrm{op}}=\sup_{x\in\bbR^{\datdim}}\frac{1}{\weight(x)}\left\Vert \left(\frac{\nabla^{2}\weight(x)}{\weight(x)}-\frac{\nabla w(x)\otimes\nabla\weight(x)}{\weight(x)^{2}}\right)\right\Vert _{\mathrm{op}}\\
 & \leq\sup_{x\in\bbR^{\datdim}}8b^{2}\verts\polyorder\left\Vert \frac{\bar{x}\otimes\bar{x}}{\bigl(a+b\Verts{\bar{x}}_{2}^{2}\bigr)^{2}}\right\Vert _{\mathrm{op}}+2\verts\polyorder b\left\Vert \frac{I}{a+b\Verts{\bar{x}}_{2}^{2}}\right\Vert _{\mathrm{op}}\\
 & =\sup_{x\in\bbR^{\datdim}}8\verts\polyorder\frac{b\Verts{\bar{x}}_{2}^{2}}{\bigl(a+b\Verts{\bar{x}}_{2}^{2}\bigr)^{2}}+2\verts\polyorder b\frac{1}{a+b\Verts{\bar{x}}_{2}^{2}}\leq6\verts\polyorder\frac{b}{a}
\end{align*}
\end{proof}
\begin{lem}[Derivative constants of a multi-quadratic function]
\label{lem:multi-quad-derivatives}
Let $\polyorder \in \bbR$. 
The multiquadratic function $\weight(x) = (1+\Verts{x}^2_2)^\polyorder$ satisfies the following: 
\begin{align*}
    &\Verts{\nabla \weight(x)}_{\mathrm{op}} \leq \verts{\polyorder}\weight(x), 
    \Verts{\nabla^2 \weight(x)}_{\mathrm{op}} \leq \verts{\polyorder(\polyorder+1)} \weight(x),\ \text{and}\\
    &\Verts{\nabla^3 \weight(x)}_{\mathrm{op}} \leq \verts{\polyorder(\polyorder-1)(\polyorder+4)}\weight(x)
\end{align*}
\end{lem}
\begin{proof}
    The results follow from the following expressions: 
    \begin{align*}
        \nabla \weight(x) &= 2\polyorder (1+\Verts{x}_2^2)^{\polyorder-1} x,\\
        \nabla^2 \weight(x) &= 4\polyorder(\polyorder-1) (1+\Verts{x}_2^2)^{\polyorder-2} x\otimes x + 2\polyorder(1+\Verts{x}_2^2)^{\polyorder-1} \idmat, \\
        \nabla^3 \weight(x) &= 8\polyorder(\polyorder-1)(\polyorder-2) (1+\Verts{x}_2^2)^{\polyorder-3} x\otimes x\otimes x \\
        &\quad + 4\polyorder(\polyorder-1) (1+\Verts{x}_2^2)^{\polyorder-2} \sum_{i=1}^\datdim \left(e_i \otimes e_i \otimes x + e_i \otimes x \otimes e_i + x\otimes e_i \otimes e_i\right). 
    \end{align*}
\end{proof}
\begin{lem}[$\calC^2$-cutoff function]
\label{lem:cutoff-bound-C2}Let $f(t)=t^{3}1_{(0,\infty)}(t).$ Let
\[
\text{\ensuremath{h(t)}}=\frac{f\bigl(r+\delta-t\bigr)}{f\bigl(r+\delta-t\bigr)+f\bigl(t-r\bigr)}.
\]
 Then, $h$ is twice continuously differentiable; we have $h(t)=0$
for $t\ge r+\delta,$ $h(t)=1$ for $t\leq r,$ and $0<h(t)<1$ otherwise.
Moreover, we have 
\begin{align*}
 & \left|\dv{h}{t}\right|(t)\leq3\delta^{-1},\\
 & \left|\dv[2]{h}{t}\right|(t)\leq\left(72+\frac{35+13\sqrt{13}}{12}\right)\delta^{-1}+2^{3+2/3}\delta^{-2}.
\end{align*}
\end{lem}

\begin{proof}
The first claim follows from the twice continous differentiablity
of $f,$ and the second claim holds by the definition of $f.$ To
evaluate the derivatives, we consider the interval $(r,r+\delta),$
as $h$ is constant outside this interval. By reparametrizing $t=\theta(1+\theta)^{-1}\delta+r$
with $\theta\in(0,\infty),$ 
\[
h(\theta)=\frac{1}{1+\theta^{3}}.
\]
Thus, 
\begin{align*}
\left|\dv{h}{t}\right|=\left|\dv{h}{\theta}\dv{\theta(t)}{t}\right| & \leq3\delta^{-1}\frac{\theta^{2}(1+\theta)^{2}}{(1+\theta^{3})^{2}}\\
 & \leq3\delta^{-1},
\end{align*}
and
\begin{align*}
\left|\dv[2]{h}{t}\right| & \le\left|\dv[2]{h}{\theta}\dv{\theta(t)}{t}\right|+\left|\dv{h}{\theta}\dv[2]{\theta(t)}{t}\right|\\
 & \leq \delta^{-1}\left(\frac{18\theta^{4}}{(1+\theta^{3})^{3}}+\frac{6\theta}{(1+\theta^{3})^{2}}\right)(1+\theta)^{2}+ \frac{3\theta^{2}}{(1+\theta^{3})^{2}}\cdot\frac{2(1+\theta)^{3}}{\delta^{2}} \\
 & \leq\left(18\cdot4+6\cdot\frac{35+13\sqrt{13}}{72}\right)\delta^{-1}+2^{3+2/3}\delta^{-2}
\end{align*}
 
\end{proof}
\begin{cor}[$\calC^2$-cutoff function on $\bbR^{\datdim}$]
\label{lem:cutoff-ball-bound-C2} $f(t)=t^{3}1_{(0,\infty)}(t).$
Let $1_{r,\delta}(x)=h\bigl(\Verts x_{2}\bigr)$ be a $\calC^{2}$-function
defined by 
\[
\text{\ensuremath{h(t)}}=\frac{f\bigl(r+\delta-t\bigr)}{f\bigl(r+\delta-t\bigr)+f\bigl(t-r\bigr)}.
\]
Then, we have $1_{r,\delta}(x)=1$ for $\Verts x_{2}\leq r,$ $0<1_{r,\delta}(x)<1$
for $r<\Verts x_{2}<r+\delta,$ and $1_{r,\delta}(x)=0$ for $\Verts x_{2}\geq r+\delta;$
the gradient $\nabla1_{r,\delta}(x)$ vanishes outside $\{x\in\bbR^{\datdim}:r<\Verts x_{2}<r+\delta\}.$
Furthermore, we can uniformly bound the derivative norms $\Verts{\nabla_{x}1_{r,\delta}(x)}_{2}$
and $\Verts{\nabla_{x}1_{r,\delta}(x)}_{2}$ by constants depending
only on $\delta.$ Specifically, 
\begin{align*}
\Verts{\nabla_{x}1_{r,\delta}(\Verts x_{2})}_{2}=\left\Vert \frac{x}{\Verts x_{2}}\left.\dv{h}{t}\right|_{t=\Verts x_{2}}\right\Vert _{2}\leq3\delta^{-1},
\end{align*}
and 
\begin{align*}
\Verts{\nabla_{x}^{2}1_{r,\delta}(x)}_{\mathrm{op}} & <\left|\dv[2]{h}{t}\bigl(\Verts x_{2}\bigr)\right|+\max_{r\leq t\leq r+\delta}\frac{2}{t}\left|\dv{h}{t}\bigl(\Verts x_{2}\bigr)\right|\\
 & \leq\left(72+\frac{35+13\sqrt{13}}{12}+\frac{6}{r}\right)\delta^{-1}+2^{3+2/3}\delta^{-2}
\end{align*}
\end{cor}

\begin{proof}
We check the statement about the second derivatives. For $r<\Verts x_{2}<r+\delta,$
\begin{align*}
\Verts{\nabla^{2}1_{r,\delta}(x)}_{\mathrm{op}} & =\norm{\nabla\left(h'(\Verts x)\frac{x}{\Verts x_{2}}\right)}_{\mathrm{op}}\\
 & =\norm{\nabla h'(\Verts x_{2})\otimes\frac{x}{\Verts x_{2}}}_{\mathrm{op}}+\frac{\verts{h'(\Verts x)}}{\Verts x_{2}}\norm{\idmat}_{\mathrm{op}}+\frac{\verts{h'(\Verts x)}}{\Verts x^{3}}\Verts{x\otimes x}_{\mathrm{op}}\\
 & \leq\verts{h''(\Verts x_{2})}+\frac{2}{\Verts x_{2}}\verts{h'(\Verts x)}.
\end{align*}
The rest of the proof follows from Lemma \ref{lem:cutoff-bound-C2}.
\end{proof}
\begin{lem}[Student's $t$-distributions are dissipative]
\label{lem:dissipative-student}For $\nu>2,$ let the density of the standard multivariate
$t$-distribution be 
\[
p(x)=\frac{\Gamma\left(\frac{\nu+\datdim}{2}\right)}{\Gamma\left(\frac{\nu}{2}\right)\nu^{\frac{\datdim}{2}}\pi^{\frac{\datdim}{2}}}\left(1+\frac{\Verts x_{2}^{2}}{\nu}\right)^{-\frac{\nu+\datdim}{2}}.
\]
Let $m(x)=(1+\nu^{-1}\Verts x_{2}^{2})\idmat.$ Then, Assumption \ref{assu:poly-grow-coef}
holds with $\lambda_{b}=(\nu+\datdim-2)/(2\nu)$ and $\lambda_{m}=\nu^{-1},$
and Assumption \ref{assu:dissipativity} with $\alpha=1-2\nu^{-1},$
$\beta_{1}=0,$ and $\beta_{0}=\datdim.$ Moreover, the diffusion
(\ref{eq:ito_diffusion}) satisfies the dissipativity condition in
Proposition \ref{prop:wasserstein-decay-uniform-diss} with $\polyorder=1$
and $\sigma(x)=\sqrt{1+\nu^{-1}\Verts x_{2}^{2}}\idmat.$ 
\end{lem}

\begin{proof}
We have
\begin{align*}
\nabla\log p(x)= & -\frac{\nu+\datdim}{\nu}\frac{x}{1+\nu^{-1}\Verts x_{2}^{2}},\ \text{and}\\
2b(x) & =m(x)\nabla\log p(x)+\la\nabla,m(x)\ra=-\frac{\nu+\datdim-2}{\nu}x.
\end{align*}
Thus, 
\begin{align*}
\Verts{b(x)}_{2} & =\frac{\nu+\datdim-2}{2\nu}\Verts x\ \text{and}\ \\
\Verts{m(x)}_{\mathrm{op}} & =1+\nu^{-1}\Verts x^{2}=\nu^{-1}(\nu+\Verts x_{2}^{2}).
\end{align*}
The dissipativity holds since 
\begin{align*}
2\la b(x),x\ra+\Verts{\sigma(x)}_{\mathrm{F}}^{2} & =-\frac{\nu+\datdim-2}{\nu}\Verts x^{2}+\datdim(1+\nu^{-1}\Verts x_{2}^{2})\\
 & =-(1-2\nu^{-1})\Verts x^{2}+\datdim.
\end{align*}
Moreover, we have 
\begin{align*}
 & 2\la b(x)-b(y),x-y\ra+\Verts{\sigma(x)-\sigma(y)}_{\mathrm{F}}^{2}-\Verts{\sigma(x)-\sigma(y)}_{\mathrm{op}}^{2}\\
 & =-\frac{\nu+\datdim-2}{\nu}\Verts{x-y}_{2}^{2}+(\datdim-1)\left(\sqrt{1+\nu^{-1}\Verts x_{2}^{2}}-\sqrt{1+\nu^{-1}\Verts y_{2}^{2}}\right)^{2}\\
 & \leq-\frac{\nu+\datdim-2}{\nu}\Verts{x-y}_{2}^{2}+\frac{\datdim-1}{\nu}\Verts{x-y}_{2}^{2}\\
 & =-\left(1-\frac{1}{\nu}\right)\Verts{x-y}_{2}^{2},
\end{align*}
showing the uniform dissipativity. 
\end{proof}

\begin{lem}[{Derivative norm estimates of Gaussian smoothing; an extension of \citealp[Lemma 2.2]{Mackey_2016}}\label{lem:plip-Gauss-conv-derivative-bounds}]
Let $G$ be a $\datdim$-dimensional standard normal vector, and $s>0$. 
Let $f:\bbR^{\datdim}\to\bbR$ be a pseudo-Lipschitz function of order $s$ with $\plipconst{f}{s} \leq 1$. 
Then, the Gaussian convolution $f_{t}(x)=\EE[f(x+tG)]$ satisfies, for each $t>0$,  
\begin{align*}
    M_{1, s}(f_{t})\leq 1+ C_{t,s},\ 
    M_{2, s}(f_{t})  \leq \frac{1}{t}\left(\sqrt{\frac{2}{\pi}} + C_{t,s}\right),\  
    M_{3, s}(f_{t}) \leq \frac{\sqrt{2}}{t^2} ( 1 + C_{t,s})
\end{align*}
where $M_{i, s}(f_{t})=\sup_{x\in\bbR^{\datdim}}\Verts{\nabla^{i}f_{t}(x)}_{\mathrm{op}} / (1+\Verts{x}_2^s)$ for $i\in\{1,2,3\}$, 
and 
\begin{equation*}
    C_{t, } =  1\{s>0\} \cdot (1\lor2^{s-1})\Bigl(1 + t^{s}\sqrt{\Ex_G[\Verts{G}_2^{2s}]}\Bigr). 
\end{equation*}
\end{lem}
\begin{proof}
    We address $s>0$, as the case $s=0$ is treated in \citet[Lemma 2.2]{Mackey_2016}. The following proof closely follow theirs. 
    
    Since a pseudo-Lipschitz function is locally Lipschitz, by Rademacher's theorem, the gradient of $f$ exists almost everywhere and satisfies 
    \begin{equation}
        \Verts{\nabla f(x)}_2 \leq 1 + \Verts{x}_2^{s}. \label{eq:plip-grad-upperbound}
    \end{equation}
    Moreover, from the proof of \citet[Lemma 2.2]{Mackey_2016}, we have 
    \begin{align*}
        \nabla f_t(x) &= \int \nabla f(x+tz) \phi(z) \dd z = \frac{1}{t} \int z f(x+tz) \phi(z) \dd z\\
        \nabla^2 f_t(x) &= \frac{1}{t} \int  \bigl(\nabla f(x+tz) \otimes z\bigr) \phi(z) \dd z 
            = \frac{1}{t^2} \int f(x+tz)  (z\otimes z - \idmat)  \phi(z) \dd z\\
        \nabla^3 f_t(x) &= \frac{1}{t^2} \int \nabla f(x+tz) \otimes (z\otimes z - \idmat) \phi(z) \dd z,
    \end{align*}
    where $\phi(z)$ is the density of the standard normal distribution. 
    
    From the first equality and the estimate \eqref{eq:plip-grad-upperbound}, we obtain 
    \begin{align*}
        \Verts{\nabla f_t(x)}_{\mathrm{op}} 
        & \leq  \int \bigl(1+c_{s}\Verts{x}_2^{s} + c_{s}t^{s}\Verts{z}_2^{s}\bigr) \phi(z) \dd z, 
    \end{align*}
    where $c_s = 1 \lor 2^{s-1}$. 
    From the second expression, 
    for any $v_1, v_2 \in \bbR^{\datdim}$ with $\Verts{v_1} =\Verts{v_2}=1$, 
    \begin{align*}
         \verts{ \la \nabla^2 f_t(x), v_1 \otimes v_2\ra}
         &= \frac{1}{t} \left\lvert \int  \la \nabla f(x+tz), v_1\ra \la z, v_2\ra \phi(z) \dd z \right \rvert\\
         &\leq \frac{1}{t} \int  \Verts{\nabla f(x+tz)}_{2} \verts{\la z, v_2\ra} \phi(z) \dd z \\
         &\leq \frac{1}{t} \left\{ \sqrt{\frac{2}{\pi}}\bigl (1 + c_{s}\Verts{x}_2^{s} \bigr) +  c_{s}t^{s}\sqrt{\int \Verts{z}_2^{2s} \phi(z) \dd z}\right\}, 
    \end{align*}
    where we have used the triangle inequality and the Cauchy-Schwarz inequality in the first inequality; 
    the second inequality follows from the estimate \eqref{eq:plip-grad-upperbound} $\la z,v_2\ra$ being a standard normal variable, and again the Cauchy-Schwarz inequality (for the integral). 

    Similarly, from the third equation, additionally taking any $v_3\in \bbR^{\datdim}$ with $\Verts{v_3}_2=1$, 
    \begin{align*}
         &\verts{\la \nabla^3 f_t(x), v_1\otimes v_2 \otimes v_3\ra} \\
         &= \frac{1}{t^2} \left \lvert \int \la \nabla f(x+tz), v_1\ra  \bigl(\la z_, v_2\ra \la z,v_3\ra - \la v_2, v_3 \ra \bigr) \phi(z) \dd z \right \rvert\\
         &\leq \frac{\sqrt{2}}{t^2}\left\{ \int (1+c_{s}\Verts{x}_2^{s} + c_{s}t^{s}\Verts{z}_2^{s} )\bigl\lvert\la z_, v_2\ra \la z,v_3\ra - \la v_2, v_3 \ra \bigr\rvert \phi(z) \dd z\right\}\\
         &\leq \frac{1}{t^2}\left( 1+c_{s}\Verts{x}_2^{s} + c_{s}\sqrt{\int \Verts{z}_2^{2s}\phi(z) \dd z } \right) \cdot \sqrt{ \int  \lvert\la z_, v_2\ra \la z,v_3\ra - \la v_2, v_3 \ra \bigr\rvert^2 \phi(z) \dd z } \\
         &= \frac{1}{t^2}\left( 1+c_{s}\Verts{x}_2^{s} + c_{s}\sqrt{\int \Verts{z}_2^{2s}\phi(z) \dd z } \right) \cdot \sqrt{ \la v_2, v_3\ra^2 + \Verts{v_2}_2\Verts{v_3}_2  } \\
         &\leq \frac{\sqrt{2}}{t^2}\left\{ (1+c_{s}\Verts{x}_2^{s}) + c_{s} t^{s}\sqrt{\int \Verts{z}_2^{2s}\phi(z)\dd z} \right\},
    \end{align*}
    where the second equality follows from Isserlis’ theorem. 
\end{proof}